\documentclass{article}

\usepackage[preprint]{neurips_2021}

\usepackage[linesnumbered,ruled,vlined]{algorithm2e}

\SetKwComment{Comment}{$\triangleright$\ }{}

\SetKwInput{KwInput}{Input}                %
\SetKwInput{KwOutput}{Output}              %

\SetKw{Andd}{and} %
\SetKw{Break}{break}
\SetKw{Continue}{continue}
\SetKw{Return}{return}

\newcommand{\cf}{cf.~}

\newcommand{\eg}{e.g.~}

\newcommand{\ie}{i.e.~}

\usepackage{tikz}
\usetikzlibrary{positioning, shapes.geometric, backgrounds}

\usetikzlibrary{arrows.meta}

\usepackage{subcaption}

\usepackage{grffile}

\usepackage{graphicx} %
\DeclareGraphicsExtensions{.pdf,.jpeg,.png,.jpg}

\usepackage{rotating}

\usepackage{wrapfig}

\usepackage[hidelinks, bookmarksnumbered, backref=page]{hyperref}
\renewcommand*{\backref}[1]{}
\renewcommand*{\backrefalt}[4]{
  \ifcase #1 (Broken backref)
  \or        (page~#2)
  \else      (pages~#2)
  \fi
}

\hypersetup{
  colorlinks   = true,
  urlcolor     = blue,
  linkcolor    = blue,
  citecolor   = blue
}

\usepackage{url}            %

\usepackage{scalerel}
\usepackage{amsmath}
\usepackage{amsthm}
\usepackage{amssymb}
\usepackage{amsfonts}       %
\usepackage{mathtools}
\usepackage{cancel}
\usepackage{bbm} %

\usepackage{mathtools}

\newcommand{\E}[2]{\mathbb{E}_{#1} \left[ #2 \right]} %

\newcommand{\interior}[1]{\mathrm{interior}(#1)}

\newcommand{\norm}[1]{\|#1\|} %

\newcommand{\prob}[1]{\mathrm{Pr}\{#1\}}

\newcommand{\real}[1]{\mathbb{R}^{#1}}

\newcommand{\setname}[1]{\mathcal{#1}}
\newcommand{\setsize}[1]{|\mathcal{#1}|}

\newcommand{\eqdef}{\coloneqq}
\newcommand{\eqdefr}{\eqqcolon}

\DeclareMathOperator*{\argmax}{\arg\!\max}

\newcommand{\vecb}[1]{\boldsymbol{#1}}
\newcommand{\mat}[1]{\boldsymbol{#1}}

\let\originalleft\left
\let\originalright\right
\renewcommand{\left}{\mathopen{}\mathclose\bgroup\originalleft}
\renewcommand{\right}{\aftergroup\egroup\originalright}

\newtheorem{theorem}{Theorem}[section]

\newtheorem{definition}{Definition}[section]

\newtheorem{assumption}{Assumption}[section]

\newcommand{\strans}{s_{\mathrm{tr}}}
\newcommand{\srecur}{s_{\mathrm{re}}}

\newcommand{\ppimat}{\mat{P}_{\!\!\pi}} %

\newcommand{\dpimat}{\mat{H}_{\!\!\pi}} %

\newcommand{\rpivec}{\vecb{r}_{\!\!\pi}}
\newcommand{\rpi}{r_{\!\pi}}

\newcommand{\piset}[1]{\Pi_{\mathrm{#1}}}

\newcommand{\settrans}{\mathcal{S}_{\mathrm{tr}}}

\newcommand{\setrecur}{\mathcal{S}_{\mathrm{re}}}

\newcommand{\isd}{\mathring{p}} %

\newcommand{\trajd}{p_{\!\xi}} %

\newcommand{\tmax}{t_{\mathrm{max}}}

\newcommand{\txepmax}{\hat{t}_{\mathrm{max}}^{\mathrm{xep}}} %

\newcommand{\tmix}{t_{ \mathrm{mix}} } %

\newcommand{\tabsrv}{T_{\mathrm{abs}}}

\newcommand{\tabsmaxc}{ t_{\mathrm{abs:max}} }
\newcommand{\tabsmin}{ t_{\mathrm{abs:min}} }
\newcommand{\tabsminpiso}{ t_{\mathrm{abs:min}}^{\pi,s_0} }
\newcommand{\tabsminhat}{ \hat{t}_{\mathrm{abs:min}} }
\newcommand{\tabsminhatb}{ \hat{t}_{\mathrm{abs:min}} }

\newcommand{\nxep}{n_{\mathrm{xep}}}

\newcommand{\gammabw}{\gamma_{\mathrm{Bw}}} %
\newcommand{\pistarbw}{\pi^*_{\mathrm{Bw}}}

\newcommand{\fsatheta}{ \vecb{f}_{\!\!\!\vecb{\theta}} } %

\newcommand{\famat}{\mat{\Phi}_{\!a}} %
\newcommand{\fsmat}{\mat{\Phi}_{\!s}} %
\newcommand{\fbmat}{\mat{\Phi}_{\!b}} %
\newcommand{\fgamat}{\mat{\Phi}_{\!ga}} %
\newcommand{\fgamathat}{\hat{\mat{\Phi}}_{\!ga}}
\newcommand{\fgamatbar}{\bar{\mat{\Phi}}_{\!ga}}
\newcommand{\fgsamat}{\mat{\Phi}_{\!gsa}} %
\newcommand{\fbamat}{\mat{\Phi}_{\!ba}} %
\newcommand{\fbamathat}{\hat{\mat{\Phi}}_{\!ba}}
\newcommand{\fbamatbar}{\bar{\mat{\Phi}}_{\!ba}}
\newcommand{\fbamatsam}{\tilde{\mat{\Phi}}_{\!ba}} %
\newcommand{\fbsamatsam}{\tilde{\mat{\Phi}}_{\!bsa}} %
\newcommand{\ftimat}{\mat{\Phi}_{\!\xi(\infty)}} %
\newcommand{\ftlmat}{\mat{\Phi}_{\!\xi(\tau)}} %
\newcommand{\fdamat}{\mat{\Phi}_{\!\gamma a}} %

\newcommand{\bigO}[1]{\mathcal{O}\left(#1\right)}

\newcommand{\algref}[1]{Algo~\ref{#1}}
\newcommand{\algrefand}[2]{Algo~\ref{#1} and~\ref{#2}}

\newcommand{\assref}[1]{Assumption~\ref{#1}}

\newcommand{\defref}[1]{Def~\ref{#1}}

\newcommand{\eqrefand}[2]{(\ref{#1},~\ref{#2})}

\newcommand{\figref}[1]{Fig~\ref{#1}}

\newcommand{\figrefand}[2]{Figs~\ref{#1} and~\ref{#2}}

\newcommand{\lineref}[1]{Line~\ref{#1}}

\newcommand{\secref}[1]{Sec~\ref{#1}}

\newcommand{\thmref}[1]{Thm~\ref{#1}}
\newcommand{\tblref}[1]{Table~\ref{#1}}

\newcommand{\alg}[1]{Algo~#1}

\newcommand{\ch}[1]{Ch~#1}

\newcommand{\cor}[1]{Corollary~#1} %

\newcommand{\deff}[1]{Def~#1}

\newcommand{\fig}[1]{Fig~#1}

\newcommand{\page}[1]{p#1}

\newcommand{\secc}[1]{Sec~#1}

\newcommand{\thm}[1]{Thm~#1}

\usepackage{booktabs}       %
\usepackage{multirow}
\usepackage{longtable}

\usepackage{colortbl}

\usepackage{enumitem}
\setlist{nosep}

\usepackage{xcolor}
\colorlet{darkgreen}{green!50!black}

\definecolor{stateblue}{cmyk}{0.96,0,0,0}
\definecolor{lightgray}{gray}{0.85}
\definecolor{halfgreen}{rgb}{0,0.5,0}

\usepackage{pdflscape}
\usepackage{afterpage}

\usepackage{nicefrac}       %
\usepackage{microtype}      %
\usepackage[yyyymmdd,hhmmss]{datetime}
\usepackage{lipsum}

\usepackage[utf8]{inputenc} %
\usepackage[T1]{fontenc}    %

\usepackage{import}

\usepackage{epigraph}

\usepackage{acronym}

\usepackage{multicol}

\newcommand{\manus}{paper} %

\title{A nearly Blackwell-optimal policy gradient method}

\author{%
  Vektor Dewanto, Marcus Gallagher \\
  School of Information Technology and Electrical Engineering \\
  University of Queensland, Australia \\
  \texttt{v.dewanto@uqconnect.edu.au, marcusg@uq.edu.au} \\
}

\begin{document}
\maketitle

\begin{abstract}
For continuing environments, reinforcement learning (RL) methods commonly maximize
the discounted reward criterion with discount factor close to 1 in order to
approximate the average reward (the gain).
However, such a criterion only considers the long-run steady-state performance,
ignoring the transient behaviour in transient states.
In this work, we develop a policy gradient method that optimizes the gain,
then the bias (which indicates the transient performance and
is important to capably select from policies with equal gain).
We derive expressions that enable sampling for the gradient of the bias
and its preconditioning Fisher matrix.
We further devise an algorithm that solves the gain-then-bias (bi-level) optimization.
Its key ingredient is an RL-specific logarithmic barrier function.
Experimental results provide insights into the fundamental mechanisms of our proposal.
\end{abstract}

\section{Introduction} \label{sec:intro}

\subsection{A motivating example}
Consider an environment in \figref{fig:symbolic_diagram_a1},
for which any stationary policy $\pi \in \piset{S}$ has equal gain values
(the expected long-run average reward in steady-state), namely $v_g(\pi) = -1$.
Therefore, all stationary policies are gain optimal, satisfying
$v_g(\pi) \ge v_g(\pi'),\ \forall \pi, \pi' \in \piset{S}$.
In this case, the gain optimality (which is equivalent to $(n=-1)$-discount optimality)
is underselective because there is no way to discriminate policies based
on their gain values.
Nonetheless, there exists a more selective optimality criterion that can be obtained
by increasing the value of $n$ from $-1$ to $0$ (or greater, if needed)
in the family of $n$-discount optimality \citep[\ch{10}]{puterman_1994_mdp}.
\begin{wrapfigure}[15]{r}{0.35\textwidth} %
\centering

\resizebox{0.35\textwidth}{!}{ %
\begin{tikzpicture}[ %
node distance = 3cm and 3cm, on grid,
-latex, %
semithick, %
state/.style={circle, top color = white,  bottom color = stateblue!20,
draw, stateblue, text=black, minimum width = 1cm},
]
\node[state](A)[] {$s^0$};
\node[state](B)[right=of A] {$s^1$};
\path (A) edge [out=45, in=135, red, loop] node[above] {$(5, 0.5)$} (A);
\path (A) edge [out=45, in=135, red] node[above] {$(5, 0.5)$} (B);
\path (A) edge [blue, bend right] node[below] {\small{$(r=10, p=1)$}} (B);
\path (B) edge [loop above] node[above] {$(-1, 1)$} (B);
\end{tikzpicture}
} %

\caption{A two-state environment with two policies (red and blue actions)
whose gains are all equally $-1$ \citep[\page{500}]{puterman_1994_mdp}.
Each edge label indicates a tuple of reward~$r$ and transition probability~$p$.
}

\label{fig:symbolic_diagram_a1}

\end{wrapfigure}
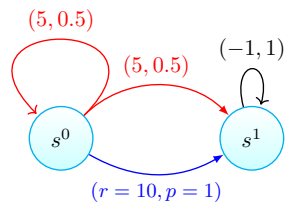 %

For the environment in \figref{fig:symbolic_diagram_a1},
$(n=0)$-discount optimality is the most selective because
its optimal policy, \ie selecting the red action,
is also optimal with respect to all other $n$-discount optimality criteria for all $n > 0$.
This selectiveness is achieved by considering not only steady-state (gain)
but also transient rewards.
This selectiveness can also be achieved by $\gamma$-discounted optimality
\emph{only if} the discount factor $\gamma > 10/11 \approx 0.909$.
Such a critical lower threshold for $\gamma$ is difficult to estimate for
environments with unknown dynamics.
Moreover, finite-precision computation induces an upper threshold
such that the discounted value $v_\gamma(\pi)$ is numerically
not too close to the gain $v_g(\pi)$.
Otherwise, $(\gamma \approx 1)$-discounted optimality becomes underselective
as well since $\lim_{\gamma \to 1} (1 - \gamma) v_\gamma(\pi) = v_g(\pi)$
\citep[\cor{8.2.5}]{puterman_1994_mdp}.

\subsection{An overview of our proposed method}
We believe that one always desires the most selective criterion.
This can be obtained by solving ($n=\setsize{S} - 2$, at most)-discount optimality
for some type of environments with a finite state set~$\setname{S}$,
as shown by \citet[\thm{10.3.6}]{puterman_1994_mdp}.
As a first step in leveraging this fact in reinforcement learning (RL),
we focus on $(n=0)$-discount optimality,
for which we have the notion of nearly-Blackwell optimality as
in \defref{def:bias_optim} below.
\begin{definition} \label{def:bias_optim}
A stationary policy $\pi_b^* \in \piset{S}$ is nearly-Blackwell optimal
(nBw-optimal, also termed bias optimal) if it is gain optimal, and in addition
\begin{align}
& v_b(\pi_b^*, s) \ge v_b(\pi_g^*, s),
    \quad \forall s \in \setname{S}, \forall \pi_g^* \in \Pi_g^*,
    \quad \text{with}\ \Pi_g^* \eqdef \argmax_{\pi \in \piset{S}} v_g(\pi),\
    \text{and the bias value} \notag \\
& v_b(\pi, s)
    \eqdef \lim_{t_{\mathrm{max}} \to \infty}
        \E{A_t \sim \pi(\cdot|s_t), S_{t+1} \sim p(\cdot|s_t, a_t)}{
            \sum_{t=0}^{t_{\mathrm{max}} - 1}
            \big( r(S_t, A_t) - v_g(\pi) \big) \Big| S_0 = s, \pi},
    \label{equ:bias_def}
\end{align}
where $v_g$ is the gain value and $r$ is a reward function.
The above limit is assumed to exist.
See \citetext{\citealp[Sec 4]{blackwell_1962_ddp}; \citealp[\deff{3.1}]{feinberg_2002_hmdp}}.
\end{definition}

To our knowledge, \citet{mahadevan_1996_sensitivedo} proposed
the first (and the only) tabular Q-learning that can obtain nBw-optimal policies
through optimizing the family of $n$-discount optimality.
It relies on stochastic approximation (SA) to estimate the optimal values of
gain $v_g^* (=v_{-1}^*)$, bias $v_b^* (=v_{0}^*)$, and bias-offset $v_1^*$.
Since those values are in two \emph{nested} equations,
standard SA is not guaranteed to converge, as noted by the author.
For other related works in a broader scope, refer to \secref{sec:rwork_nbwpg}.

The main contribution of this \manus~is the development of
a policy gradient method capable of learning nBw-optimal policies in unichain MDPs.
A policy gradient method
works by maximizing the expected value of a policy with respect to
some initial state distribution, \ie $\E{\isd}{v(\pi, S_0)}$.
This generally requires \emph{stochastic} optimization on a \emph{non-concave} landscape.
The gradient ascent update for iteratively solving the optimization is given by
\begin{equation}
\vecb{\theta} \gets \vecb{\theta} + \alpha\ \mat{C}^{-1}\ \nabla v(\vecb{\theta}, s_0),
\quad \text{
    for a parameterized policy $\pi(\vecb{\theta}),
    \forall \vecb{\theta} \in \Theta = \real{\dim(\vecb{\theta})}$,}
\label{equ:polparam_update}
\end{equation}
where $\alpha$ is a positive step length,
$\mat{C} \in \real{\dim(\vecb{\theta}) \times \dim(\vecb{\theta})}$
is a positive-definite preconditioning matrix,
$s_0$~is a~deterministic initial state (without loss of generality),
$\nabla \eqdef \partial/\partial \vecb{\theta}$, and
$v(\vecb{\theta}, s_0) \eqdef v(\pi(\vecb{\theta}), s_0)$.
Such optimization needs $\pi(\vecb{\theta})$ that is
twice continuously differentiable with bounded first and second derivatives,
for which $\pi(\vecb{\theta})$ is necessarily a randomized (stochastic) policy.

In essense, our nBw-optimal policy gradient method switches the optimization objective
from gain~$v_g$ to bias~$v_b$ once the gain optimization is deemed to have converged.
This necessitates three novel ingredients, to which we contribute.
\textbf{First} is a sampling-based approximator for
the (standard, vanilla) gradient of bias (\secref{sec:grad}).
\textbf{Second} is a suitable preconditioning matrix~$\mat{C}$,
for which we define the Fisher (information) matrix of bias,
leading to natural gradients of bias (\secref{sec:natgrad}).
\textbf{Third} is an algorithm that utilizes the gradient and Fisher matrix estimates
to carry out gain then bias optimization (\secref{sec:bwoptim_algo}).
For solving the induced bi-level optimization, we formulate
a logarithmic barrier function that ensures the optimization iterates stay
in the gain-optimal region (as prescribed in \defref{def:bias_optim}).
The experiment results in \secref{sec:nbwpg_xprmtresult} provide insights into
the fundamental mechanisms of our proposal, as well as
comparison with the discounted reward method.

\section{Preliminaries}

The interaction between
an RL agent and its environment is modelled as a Markov decision process (MDP).
An MDP induces Markov chains (MCs), one for each stationary policy~$\pi$.
Each induced MC has a stochastic one-step transition matrix under~$\pi$,
denoted as $\ppimat$ and of size $\setsize{S}$-by-$\setsize{S}$.
Raising it to the $t$-th power gives $\ppimat^t$, whose $[s_0, s_t]$-th component
indicates the $t$-step transition probability
$p_\pi^t(s_t | s_0) \eqdef \prob{S_t = s_t | s_0; \pi}$.
The limiting matrix is defined as
$\ppimat^\star \eqdef \lim_{t \to \infty} \ppimat^t$,
which exists whenever the MC is aperiodic.
Since $\ppimat^\star \ppimat = \ppimat^\star$, its component $p_\pi^\star(s|s_0)$
can be interpreted as the stationary (time-invariant) probability of
visiting state $s$ when starting in an initial state $s_0$.
The time $t$ at which such stationarity
(equivalently, the limiting state distribution $p^\star$)
is achieved, at least approximately, is of special interest.
For that, the notion of mixing time is defined in \defref{def:tmix} below.

\begin{definition} \label{def:tmix}
The mixing time, denoted as $\tmix^\pi$, is the time required by
an MC induced by a policy $\pi$ such that the distance to the stationarity is
small \citep[\ch{4.5}]{levin_2009_mvmt}.
That is,
\begin{equation*}
\tmix^\pi(\varepsilon) \eqdef \min \Big\{
  t: \max_{s_0 \in \setname{S}}
      \Big[ \frac{1}{2} \sum_{s \in \setname{S}}
      \big| \ppimat^t[s_0, s] - \ppimat^\star[s_0, s] \big| \Big]
  \le \varepsilon
\Big\},
\end{equation*}
for an arbitrarily small $\varepsilon < 1/2$.
\end{definition}

We consider an agent-environment interaction whose MDP model
satisfies \assref{ass:unichain_aperiodic} below.
We remark that the mixing time (\defref{def:tmix}) is commonly specified
with a positive $\varepsilon$ in order to quantify some approximation.
To suppress such approximation (at least notationally), we idealize
the concept of mixing by having $\tmix^\pi \eqdef \tmix^\pi(\varepsilon = 0)$,
\ie completely mixing.
Note that in some MCs, the stationarity is indeed exactly achieved in finite time,
hence $\tmix^\pi < \infty$.
\figref{fig:timestep_line} depicts this $\tmix$ milestone on the timestep line
of an infinite-horizon MDP.

\begin{assumption} \label{ass:unichain_aperiodic}
The modelling MDP is unichain and aperiodic.
Every induced MC has a finite mixing time (\defref{def:tmix}),
that is $\tmix^\pi(\varepsilon) < \infty$.
\end{assumption}

\begin{figure}[t]
\centering

\resizebox{\textwidth}{!}{
\begin{tikzpicture}
\draw (0,0) -- (11,0);
\foreach \i in {0, 1,...,11}
  \draw (\i,0.1) -- + (0,-0.2);

\node[rectangle,
draw = green,
fill = green,
minimum width = 9cm,
minimum height = 0.4cm] (e)
at (4.5,-0.35) {};
\node at (4.5,-1.25) [align=center, fill=green]
    {\small pre-mixing \\ (transient phase)};

\node[rectangle,
draw = blue!65,
fill = blue!65,
minimum width = 5cm,
minimum height = 0.4cm] (e)
at (2.5,0.35) {};
\node at (2.5,0.9) [align=center, fill=blue!65]
    {\small transient states};

\node[rectangle,
draw = yellow,
fill = yellow,
minimum width = 6cm,
minimum height = 0.4cm] (e)
at (8,0.35) {};
\node at (8,0.9) [align=center, fill=yellow]
    {\small recurrent states};

\node[rectangle,
draw = red!85,
fill = red!85,
minimum width = 2cm,
minimum height = 0.4cm] (e)
at (10,-0.35) {};
\node at (10,-1.25) [align=center, fill=red!85]
    {\small post-mixing \\ (steady-state phase)};

\node at (0,-0.35){$t\!\!=\!0$};
\node at (1,-0.35){$\ldots$};
\node at (2,-0.35){$\ldots$};
\node at (3,-0.35){$\tabsmin^{\pi,s_0}$};
\node at (4,-0.35){$\ldots$};
\node at (5,-0.35){$\tabsrv^{\pi,s_0}$};
\node at (6,-0.35){$\ldots$};
\node at (7,-0.35){$\tabsmaxc^{\pi}$};
\node at (8,-0.35){$\ldots$};
\node at (9,-0.35){$\tmix^\pi$};
\node at (10,-0.35){$\ldots$};
\node at (11,-0.35){$\infty$};

\node at (0,0.35){$s_0$};
\node at (1,0.35){$\ldots$};
\node at (2,0.35){$\ldots$};
\node at (3,0.35){$\ldots$};
\node at (4,0.35){$\ldots$};
\node at (5,0.45){$s_{\tabsrv}^{\pi,s_0}$};
\node at (6,0.35){$\ldots$};
\node at (7,0.35){$\ldots$};
\node at (8,0.35){$\ldots$};
\node at (9,0.35){$s_{\tmix^\pi}$};
\node at (10,0.35){$\ldots$};
\node at (11,0.35){$s_{\infty}$};

\end{tikzpicture}
} %

\caption{A diagram for pre-mixing (green) and post-mixing (red) phases,
as well as transient (blue) and recurrent (yellow) state visitations
on the timestep line of an infinite-horizon MDP.
The absorption time $\tabsrv^{\pi,s_0}$ is a random variable, whose superscript
indicates the dependency on a policy $\pi$ and an initial state $s_0$.
Transient states are visited only before absorption, whereas recurrent states
are visited at and after absorption.
The other important milestones are
the mixing time $\tmix^\pi$ (\defref{def:tmix}),
the minimum absorption time $\tabsmin^{\pi,s_0}$ (\defref{def:tabsmin}),
and the maximum absorption time $\tabsmaxc^{\pi}$.
}
\label{fig:timestep_line}
\end{figure}
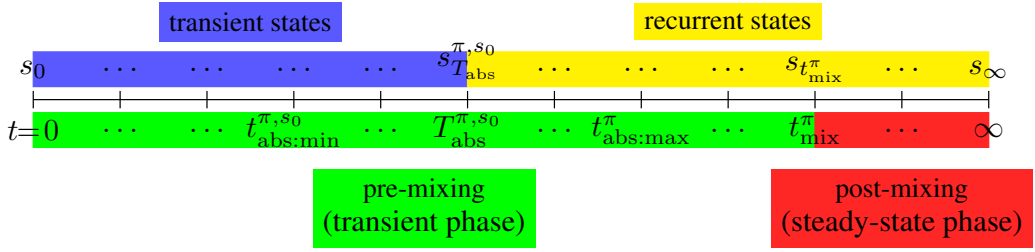

\section{Gradients of the bias} \label{sec:grad}

In matrix forms, the bias is equal to
$\vecb{v}_b(\pi) = (\mat{I} - \ppimat + \ppimat^\star)^{-1}
(\mat{I} - \ppimat^\star) \vecb{r}_\pi
= \mat{K}_{\!\pi} (\mat{I} - \ppimat^\star) \vecb{r}_\pi$
for  a non-stochastic matrix
$\mat{K}_{\!\pi} \eqdef (\mat{I} - \ppimat + \ppimat^\star)^{-1}$.
This can be written in state-wise form
(from every initial state $s_0 \in \setname{S}$) as follows.
\begin{align*}
v_b(\pi, s_0)
&= \sum_{s \in \setname{S}} k_\pi(s|s_0) \Big(
    \sum_{a \in \setname{A}} \pi(a|s) r(s, a)- v_g(\pi) \Big)
    \tag{Equivalent to \eqref{equ:bias_def}} \\
& = \sum_{s \in \setname{S}} \sum_{a \in \setname{A}} k_\pi(s|s_0) \pi(a|s) r(s, a)
    - \sum_{s \in \setname{S}} k_\pi(s|s_0) v_g(\pi),
\end{align*}
where $k_\pi(s|s_0)$ is the $[s_0, s]$-component of $\mat{K}_{\!\pi}$.
The gradient of the bias therefore is given by
\begin{align}
\nabla v_b(\pi, s_0)
& = \sum_{s \in \setname{S}} \sum_{a \in \setname{A}} k_\pi(s|s_0) \pi(a|s) r(s, a)
    \big\{ \nabla \log k_\pi(s|s_0) + \nabla \log \pi(a|s) \big\} \notag \\
& \quad - \sum_{s \in \setname{S}} k_\pi(s|s_0)
    \big\{ v_g(\pi) \nabla \log k_\pi(s|s_0) + \nabla v_g(\pi) \big\}.
\label{equ:gradbias_naive}
\end{align}
Because $k_\pi$ is \emph{not} a probability distribution, the above expression
does not enable sampling-based approximation.\footnote{
    This is in contrast to the gradient of the gain $\nabla v_g(\pi)$,
    which involves the stationary state distribution $p_\pi^\star$.
    That is, $\nabla v_g(\pi)
    = \sum_{s \in \setname{S}} \sum_{a \in \setname{A}}
    p_{\pi}^\star(s)\ \pi(a|s)\ r(s, a)\
    \{ \nabla \log p_{\pi}^\star(s) + \nabla \log \pi(a|s) \}$,
    equivalent to \eqref{equ:gaingrad_simple}.
}
Therefore, we derive the following \thmref{thm:gradbias} that gives us
a score-function gradient estimator for the bias.
Its proof is provided in \secref{sec:grad_thm_proof}.

\begin{theorem} \label{thm:gradbias}
The gradient of the bias of a randomized policy $\pi(\vecb{\theta}) \in \piset{SR}$
is given by
\begin{align}
\nabla v_b(\vecb{\theta}, s_0)
& = \textcolor{green}{\Big\{} \underbrace{
    \sum_{t=0}^{\tmix^\pi - 1} \Big\{
    \mathbb{E}_{S_t \sim p_\pi^t(\cdot | s_0), A_t \sim \pi(\cdot|s_t; \vecb{\theta})}
    \Big[ q_b^\pi(S_t, A_t) \nabla \log \pi(A_t|S_t; \vecb{\theta}) \Big]
    - \nabla v_g(\vecb{\theta}) \Big\}
}_\text{The pre-mixing part} \textcolor{green}{\Big\}} \notag \\
& \quad + \textcolor{red}{\Big\{} \underbrace{
    \mathbb{E}_{S_{\tmix^\pi} \sim p_\pi^\star,
        A_{\tmix^\pi} \sim \pi(\cdot|s_{\tmix^\pi}; \vecb{\theta})}
    \Big[ \nabla q_b^\pi(S_{\tmix^\pi}, A_{\tmix^\pi}) \Big]
    +  \nabla v_g(\vecb{\theta})
}_\text{The post-mixing part} \textcolor{red}{\Big\}},
\label{equ:grad_bias}
\end{align}
for all $s_0 \in \setname{S}$, and for all $\vecb{\theta} \in \Theta$,
where $\nabla \eqdef \partial/\partial \vecb{\theta}$ and
$\tmix^\pi \eqdef \tmix^\pi(\varepsilon=0)$.
Here, $q_b^\pi$ denotes the bias state-action value of a policy $\pi$,
which is defined in a similar fashion as the bias state value $v_b^\pi$
in \eqref{equ:bias_def}.
That is,
\begin{equation}
q_b^\pi(s, a)
\eqdef \lim_{\tmax \to \infty} \E{S_t, A_t}{\sum_{t=0}^{\tmax - 1} \left( r(S_t, A_t)
    - v_g^\pi \right) \Big| S_0 = s, A_0 = a, \pi},
\quad \forall (s, a) \in \setname{S} \times \setname{A}.
\label{equ:bias_q}
\end{equation}
For pre- and post-mixing parts in \eqref{equ:grad_bias}, see the timestep line diagram
in \figref{fig:timestep_line}.
\end{theorem}

The bias gradient expression in \eqref{equ:grad_bias} has pre- and post-mixing parts.
Both involve the gain gradient $\nabla v_g(\vecb{\theta})$,
which is $\vecb{0}$ everywhere in the gain-optimal region
(in which the bias optimization should be carried out as per \defref{def:bias_optim}).
We explain each part of \eqref{equ:grad_bias} further,
beginning with the post-mixing part (\secref{sec:grad_postmix}) as
it shares similarities with the existing gain gradient expression
\citep[\thm{1}]{sutton_2000_pgfnapprox}.
That is, both require state samples from the stationary state distribution $p_\pi^\star$
and the same policy evaluation quantity, \ie $q_b^\pi$, evaluated from recurrent states.
Afterward, we explain the pre-mixing part of \eqref{equ:grad_bias}
in \secref{sec:grad_premix}.

\subsection{The post-mixing part} \label{sec:grad_postmix}
A single i.i.d state sample from $p_\pi^\star$ can be obtained at or after
the mixing time $\tmix^\pi$.\footnote{
    Note that multiple state samples are not independent (they are Markovian)
    and after mixing, they are identically distributed by $p_\pi^\star$.
    These non-i.i.d samples lead to \emph{biased} sample-mean gradient estimates.
}
Assuming that the length of an experiment-episode (trial) is equal or larger than $\tmix^\pi + 1$,
such an $S_{\tmix^\pi} \sim p_\pi^\star$ is sampled at the last timestep of
an experiment-episode.\footnote{
    The term ``experiment-episode~($\mathrm{xep}$)'' is used because
    the environments of interest do not have any inherent notion of episodes.
    They are non-episodic (continuing), hence modelled as MDPs with
    infinite horizons $\tmax = \infty$.
    In experimental practice however, such infinite horizons are implemented as
    finite length experiment-episodes with maximum timesteps
    $\txepmax < \infty$.
}
Thus, an unbiased sampling-based estimate of $\nabla v_b(\vecb{\theta})$ is available
after every experiment-episode.
This leads to experiment-episode-wise policy parameter update \eqref{equ:polparam_update},
as shown in \algref{alg:gainbiasbarrier_optim}.

The post-mixing part in \eqref{equ:grad_bias} involves the derivative of
the action value $q_b^\pi$ with respect to the policy parameter $\vecb{\theta}$.
It therefore requires that the action-value estimator
(termed the critic) depends on~$\vecb{\theta}$.
This can be achieved for example, by using
a policy-compatible state-action feature for the critic, \ie
$\fsatheta(s, a) \eqdef \nabla \log \pi(a|s; \vecb{\theta}),
\forall (s, a) \in \setname{S} \times \setname{A}$.
There also exists a class of critics that takes as input the policy parameter
\citep{faccio_2021_pvf, harb_2020_pen}.

Alternatively, our \thmref{thm:postmixing_identity} presents
an equivalent form of the post-mixing part that is free from
$\nabla q_b^\pi(S_{\tmix^\pi}, A_{\tmix^\pi})$.
It is beneficial whenever the critic is independent of $\vecb{\theta}$.
This is the case for instance, when policy (actor) and critic neural-networks
do not share any parameters.

\begin{theorem}
The post-mixing part of \eqref{equ:grad_bias},
denoted as $\nabla v_{b, post}(\vecb{\theta}, s_0)$, has the following identity,
\begin{equation*}
\nabla v_{b, post}(\vecb{\theta}, s_0)
= \mathbb{E}_{S_{\tmix^\pi} \sim p_\pi^\star,
        A_{\tmix^\pi} \sim \pi(\cdot|s_{\tmix^\pi}; \vecb{\theta})}
    \Big[ q_1^\pi(S_{\tmix^\pi}, A_{\tmix^\pi})
    \nabla \log \pi(A_{\tmix^\pi}|S_{\tmix^\pi}; \vecb{\theta}) \Big],
\end{equation*}
where $q_1^\pi$ denotes the state-action value of a policy $\pi$ in terms of
the ($n=1$)-discount optimality.
\label{thm:postmixing_identity}
\end{theorem}

\begin{proof}

One key ingredient is the relationship between state values $v_1^\pi$ and
state-action values $q_1^\pi$ in $(n=1)$-discount optimality.
That is,
\begin{equation}
q_1^\pi(s, a) = - v_b^\pi(s) + \sum_{s' \in \setname{S}} p(s'|s, a)
    \underbrace{\sum_{a' \in \setname{A}} \pi(a'|s') q_1^\pi(s', a')}_{v_1^\pi(s')},
\quad \quad \forall (s, a) \in \setname{S} \times \setname{A},
\label{equ:q1_v1}
\end{equation}
which is derived from a similar identity in terms of state values, \ie
\begin{equation*}
v_1^\pi(s) = - v_b^\pi(s) + \sum_{s' \in \setname{S}} p_\pi(s'|s) v_1^\pi(s'),
\quad \forall s \in \setname{S}.
\tag{\citet[\thm{8.2.8}]{puterman_1994_mdp}}
\end{equation*}
After having \eqref{equ:q1_v1}, we take $\nabla v_1^\pi(s)$,
substitute $\E{A \sim \pi}{q_1^\pi(s, A)}$ for $v_1^\pi(s)$, and finally
sum over the stationary state distribution $p_\pi^\star$.
We essentially follow the technique used by \cite{sutton_2000_pgfnapprox} for
deriving the randomized policy gradient theorem for gain gradients $\nabla v_g^\pi$.

For all $s \in \setname{S}$ and $\pi(\vecb{\theta}) \in \piset{SR}$, we have
\begin{align*}
\nabla v_1^\pi(s)
& = \nabla \Big\{ \sum_{a \in \setname{A}} \pi(a|s) q_1^\pi(s, a) \Big\}
= \sum_{a \in \setname{A}} \Big\{ q_1^\pi(s, a) \nabla \pi(a|s) + \pi(a|s) \nabla q_1^\pi(s, a) \Big\} \\
& = \sum_{a \in \setname{A}} \Big\{
    q_1^\pi(s, a) \nabla \pi(a|s)
    + \pi(a|s) \nabla \{ - v_b^\pi(s) + \sum_{s' \in \setname{S}} p(s'|s, a) v_1^\pi(s') \}
\Big\} \tag{Plug-in \eqref{equ:q1_v1}} \\
& = \sum_{a \in \setname{A}} \Big\{
    q_1^\pi(s, a) \nabla \pi(a|s)
    - \pi(a|s) \nabla  v_b^\pi(s)
    + \pi(a|s) \sum_{s' \in \setname{S}} p(s'|s, a) \nabla v_1^\pi(s') \}
\Big\} \\
& = \sum_{a \in \setname{A}} \Big\{
    q_1^\pi(s, a) \nabla \pi(a|s)
    + \pi(a|s) \sum_{s' \in \setname{S}} p(s'|s, a) \nabla v_1^\pi(s') \Big\}
    - \nabla  v_b^\pi(s) \cancelto{1}{\sum_{a \in \setname{A}} \pi(a|s)}.
\end{align*}
Consequently,
\begin{equation}
\nabla  v_b^\pi(s)
= \sum_{a \in \setname{A}} \Big\{
    q_1^\pi(s, a) \nabla \pi(a|s)
    + \pi(a|s) \sum_{s' \in \setname{S}} p(s'|s, a) \nabla v_1^\pi(s') \Big\}
    - \nabla  v_1^\pi(s).
\label{equ:grad_vb_equal_q1_v1}
\end{equation}

For each side in \eqref{equ:grad_vb_equal_q1_v1} above, we can sum across
all states $s \in \setname{S}$,
and weight each state by its stationary probability $p_\pi^\star(s)$,
while maintaining the equality (between both sides).
That is,
\begin{align*}
& \sum_{s \in \setname{S}} p_\pi^\star(s) \nabla v_b^\pi(s) \\
& = \sum_{s \in \setname{S}} p_\pi^\star(s)
    \sum_{a \in \setname{A}} q_1^\pi(s, a) \nabla \pi(a|s)
    + \sum_{s \in \setname{S}} p_\pi^\star(s)
    \sum_{s' \in \setname{S}} \underbrace{
        \sum_{a \in \setname{A}} \pi(a|s)  p(s'|s, a)
    }_{p_\pi(s'|s)}
    \nabla v_1^\pi(s')
    - \sum_{s \in \setname{S}} p_\pi^\star(s) \nabla  v_1^\pi(s) \\
& = \sum_{s \in \setname{S}} p_\pi^\star(s)
    \sum_{a \in \setname{A}} q_1^\pi(s, a) \nabla \pi(a|s)
    + \sum_{s' \in \setname{S}}
    \underbrace{\sum_{s \in \setname{S}} p_\pi^\star(s) p_\pi(s'|s)}_{p_\pi^\star(s')}
    \nabla v_1^\pi(s')
    - \sum_{s \in \setname{S}} p_\pi^\star(s) \nabla  v_1^\pi(s)
    \tag{Here, $p_\pi^\star(s') = \sum_{s \in \setname{S}} p_\pi^\star(s) p_\pi(s'|s)$
        due to the stationarity of $\ppimat^\star$, \ie
        $\ppimat^\star \ppimat = \ppimat^\star$} \\
& = \sum_{s \in \setname{S}} p_\pi^\star(s)
    \sum_{a \in \setname{A}} q_1^\pi(s, a) \nabla \pi(a|s) + \vecb{0}.
\end{align*}

Because the policy is randomized (and never degenerates to a deterministic policy),
the score function or the likelihood ratio,
\ie $\nabla \log \pi(a|s) = \nabla \pi(a|s) / \pi(a|s)$, is always well-defined.
Hence, we have
\begin{equation}
\sum_{s \in \setname{S}} p_\pi^\star(s) \nabla v_b^\pi(s)
= \sum_{s \in \setname{S}} \sum_{a \in \setname{A}}
    p_\pi^\star(s) \pi(a|s) \Big[ q_1^\pi(s, a) \nabla \log \pi(a|s) \Big].
\label{equ:postmix_identity_match}
\end{equation}

The LHS of \eqref{equ:postmix_identity_match} above is exactly
the last term in the decomposition of $\nabla v_b^\pi$ in \eqref{equ:grad_bias_expansion},
which is equivalent to the post-mixing term in \thmref{thm:gradbias}.
This concludes the proof.

\end{proof}

\subsection{The pre-mixing part} \label{sec:grad_premix}

The state samples for the pre-mixing part of \eqref{equ:grad_bias} are
all but the last state of an experiment-episode whenever $\txepmax = \tmix^\pi$.
That is, $S_0 \sim p_\pi^0$, $S_1 \sim p_\pi^1$, $\ldots$,
$S_{\txepmax - 1} \sim p_\pi^{\txepmax - 1}$.
This is because the state distribution from which states are sampled at every timestep,
\ie $p_\pi^t(\cdot|s_0)$, is the correct distribution of its corresponding term
in the pre-mixing part.

\algref{alg:gradfisher_sampling} shows a way to obtain pre-mixing estimates,
which are unbiased whenever $\txepmax \ge \tmix^\pi$.
If an experiment-episode runs longer than $\tmix^\pi$,
the sum of pre-mixing terms from $t = \tmix^\pi$ to $\txepmax - 1$ is
equal to $\vecb{0}$ in exact cases
(in unbiased estimation, it is equal to $\vecb{0}$ in expectation).
That is,
\begin{align*}
& \Big\{
    \underbrace{
        \sum_{s}
        \underbrace{ p_\pi^{\tmix^\pi} (s|s)}_{p_\pi^\star(s)}
        \sum_{a} q_b^\pi(s, a) \nabla \pi(a|s)
    }_\text{$\nabla v_g^\pi$, as shown by \citet[\thm{1}]{sutton_2000_pgfnapprox}}
    - \nabla v_g^\pi \Big\}
+ \Big\{
    \underbrace{
        \sum_{s \in \setname{S}} \underbrace{p_\pi^{\tmix^\pi + 1} (s|s)}_{p_\pi^\star(s)}
        \sum_{a \in \setname{A}} q_b^\pi(s, a) \nabla \pi(a|s)
    }_{\nabla v_g^\pi}
    - \nabla v_g^\pi \Big\} \\
& + \ldots + \Big\{
    \underbrace{
        \sum_{s \in \setname{S}} \underbrace{p_\pi^{\txepmax - 1} (s|s)}_{p_\pi^\star(s)}
        \sum_{a \in \setname{A}} q_b^\pi(s, a) \nabla \pi(a|s)
    }_{\nabla v_g^\pi}
    - \nabla v_g^\pi \Big\}
= \vecb{0}.
\end{align*}
This implies that the estimates of pre-mixing terms from $t = \tmix^\pi$ to $\txepmax - 1$
contribute relatively minimally to the estimate of $\nabla v_b^\pi$ in \eqref{equ:grad_bias}.

To obtain pre-mixing estimates, \algref{alg:gradfisher_sampling} postpones
the substraction of $\nabla v_g(\vecb{\theta})$ in each pre-mixing term until
the estimate of $\nabla v_g(\vecb{\theta})$ is available at
the last step of an experiment-episode.

\subsection{The proof of \thmref{thm:gradbias}} \label{sec:grad_thm_proof}

\begin{proof}
For simplicity, we notationally suppress the dependency of a policy $\pi$ to
its parameter~$\vecb{\theta}$.
Also note that $\nabla \eqdef \partial/\partial \vecb{\theta}$, and
all summations are either over all states $s \in \setname{S}$
or all actions $a \in \setname{A}$.

The proof of \thmref{thm:gradbias} relies on three facts, which
are similar to those used in proving the randomized policy gradient theorem
for the discounted reward optimality by \cite{sutton_2000_pgfnapprox}.
\textbf{First} is the relationship between the state value $v_b^\pi$ and
the (state-)action value $q_b^\pi$.
This includes
\begin{align}
v_b^\pi(s)
    & = \sum_{a \in \setname{A}} \pi(a|s)\ q_b^\pi(s, a)
    \quad \forall s \in \setname{S}, \quad \text{and} \label{equ:vb_qb}\\
q_b^\pi(s, a)
    & = r(s, a) - v_g^\pi + \sum_{s' \in \setname{S}} p(s'|s, a) v_b^\pi(s'),
    \quad \forall (s, a) \in \setname{S} \times \setname{A}.
    \label{equ:qb_vbnext}
\end{align}
\textbf{Second}, the policy $\pi$ is randomized so that its bias gradient involves
$\sum_{a \in \setname{A}} \pi(a|s; \vecb{\theta}) \nabla p(s' | s, a) = \vecb{0}$, and
$\sum_{a \in \setname{A}} \pi(a|s; \vecb{\theta}) \nabla r(s, a) = \vecb{0}$.
Note that if $\pi$ was deterministic, the bias gradient would involve
$\nabla p(s' | s, a = \pi(s; \vecb{\theta}))$ and
$\nabla r(s, a = \pi(s; \vecb{\theta}))$, which are generally not zero.
A similar discussion can be found in \citep[\page{28}]{deisenroth_2013_polsearchrob}.
\textbf{Third}, summation and differentiation can be exchanged, \ie
$\nabla \sum_{x \in \setname{X}} f(x; \vecb{\theta})
=  \sum_{x \in \setname{X}} \nabla f(x; \vecb{\theta})$,
under some regularity conditions \citep[\page{32}]{liero_2011_tsi}, including
that the randomized policy $\pi$ is a smooth function of~$\vecb{\theta}$,
the action set $\setname{A}$ is independent of $\vecb{\theta}$, and
the policy parameter set $\Theta$ is an open interval.

For every state $s \in \setname{S}$ and
any randomized stationary policy $\pi(\vecb{\theta}) \in \piset{SR}$,
we proceed as follows.
\begin{align}
& \nabla v_b^\pi(s)
= \nabla \Big\{ \sum_{a \in \setname{A}} \pi(a|s) q_b^\pi(s, a) \Big\}
    \tag{Expand $v_b^\pi(s)$ based on \eqref{equ:vb_qb}}\\
& = \sum_{a \in \setname{A}} \Big\{ q_b^\pi(s, a) \nabla \pi(a|s)
    + \Big( \pi(a|s) \nabla q_b^\pi(s, a) \Big) \Big\}
    \tag{\small Exchange $\nabla$ and $\sum_a$, then take the derivative} \\
& = \sum_{a \in \setname{A}} \Big\{ q_b^\pi(s, a) \nabla \pi(a|s)
    + \Big( \pi(a|s) \nabla \Big[ r(s, a) - v_g^\pi
        + \sum_{s' \in \setname{S}} p(s'|s, a) v_b^\pi(s') \Big] \Big)
\Big\} \tag{Plug-in \eqref{equ:qb_vbnext}} \\
& = \sum_{a \in \setname{A}} \Big\{ q_b^\pi(s, a) \nabla \pi(a|s) - \pi(a|s) \nabla v_g^\pi
    + \Big\{ \pi(a|s) \sum_{s' \in \setname{S}} p(s'|s, a) \nabla v_b^\pi(s') \Big\} \Big\},
    \label{equ:grad_vb_1}
\end{align}

In a similar fashion as above, we essentially expand $v_b^\pi(s')$ in \eqref{equ:grad_vb_1}
based on \eqref{equ:vb_qb}, take the derivative, then plug-in \eqref{equ:qb_vbnext}.
That is,
\begin{align}
& \sum_{a} \Big\{ q_b^\pi(s, a) \nabla \pi(a|s) - \pi(a|s) \nabla v_g^\pi
+ \Big\{
    \pi(a|s) \sum_{s'} p(s'|s, a) \nabla \Big[ \sum_{a'} \pi(a'|s') q_b^\pi(s', a') \Big]
    \Big\} \Big\} \notag \\
& = \sum_{a} \Big\{ q_b^\pi(s, a) \nabla \pi(a|s) - \pi(a|s) \nabla v_g^\pi \notag \\
& \quad + \Big\{
    \pi(a|s) \sum_{s'} p(s'|s, a) \sum_{a'} \Big[
    q_b^\pi(s', a') \nabla \pi(a'|s') + \pi(a'|s') \nabla q_b^\pi(s', a') \Big]
    \Big\} \Big\} \notag \\
& = \sum_{a} \Big\{ q_b^\pi(s, a) \nabla \pi(a|s) - \pi(a|s) \nabla v_g^\pi \notag \\
& \quad + \Big\{
\pi(a|s) \sum_{s'} p(s'|s, a) \sum_{a'} \Big[ q_b^\pi(s', a') \nabla \pi(a'|s')
    - \pi(a'|s') \nabla v_g^\pi + \pi(a'|s') \sum_{s''} p(s''|s', a') \nabla v_b^\pi(s'') \Big]
    \Big\} \Big\} \notag \\
& = \sum_{a} \Big\{ q_b^\pi(s, a) \nabla \pi(a|s) - \pi(a|s) \nabla v_g^\pi \notag \\
& \quad +  \Big\{
    \pi(a|s) \sum_{s'} p(s'|s, a) \sum_{a'} q_b^\pi(s', a') \nabla \pi(a'|s')
    - \pi(a|s) \sum_{s'} p(s'|s, a) \sum_{a'} \pi(a'|s') \nabla v_g^\pi \notag \\
& \quad + \pi(a|s) \sum_{s'} p(s'|s, a) \sum_{a'} \pi(a'|s') \sum_{s''} p(s''|s', a') \nabla v_b^\pi(s'')
    \Big\} \Big\}, \label{equ:grad_vb_2}
\end{align}
which can be re-arranged to obtain
\begin{align}
& \Big\{ \sum_{a} q_b^\pi(s, a) \nabla \pi(a|s) - \nabla v_g^\pi \Big\}
    + \Big\{ \sum_{s'} \underbrace{\sum_{a} \pi(a|s)  p(s'|s, a)}_{p_\pi^1(s'|s)}
        \sum_{a'} q_b^\pi(s', a') \nabla \pi(a'|s') - \nabla v_g^\pi \Big\}
    \notag \\
& + \Big\{ \sum_{s''} \underbrace{
        \sum_{a} \pi(a|s) \sum_{s'} p(s'|s, a) \sum_{a'} \pi(a'|s') p(s''|s', a')
    }_{p_\pi^2(s''|s)}
    \nabla v_b^\pi(s'') \Big\} \notag \\
& = \Big\{ \sum_{\mathring{s}} p_\pi^0(\mathring{s}|s) \sum_{a} q_b^\pi(s, a) \nabla \pi(a|s)
        - \nabla v_g^\pi \Big\}
    + \Big\{ \sum_{s'} p_\pi^1(s'|s) \sum_{a'} q_b^\pi(s', a') \nabla \pi(a'|s')
        - \nabla v_g^\pi \Big\} \notag \\
    & \quad + \Big\{ \sum_{s''} p_\pi^2(s''|s) \nabla v_b^\pi(s'') \Big\},
    \label{equ:grad_vb_3}
\end{align}
where $p_\pi^0(\mathring{s}|s) = 1$ if $\mathring{s} = s$ and 0 otherwise.
Here, $\sum_{s}$ and $\sum_{a}$ (and the like) denote the summation over
all states $s \in \setname{S}$ and over all actions $a \in \setname{A}$, respectively.

If we keep expanding the factor of $\nabla v_b^\pi$
(starting from $\nabla v_b^\pi(s'')$ in \eqref{equ:grad_vb_2}) and
applying the same procedure as for obtaining \eqref{equ:grad_vb_3}
several times until the mixing time $\tmix^\pi$ (\defref{def:tmix}), we obtain
\begin{align}
& \nabla v_b^\pi(s)
= \Big\{ \sum_{\tilde{s} \in \setname{S}} p_\pi^0(\tilde{s}|s)
    \sum_{\tilde{a} \in \setname{A}} q_b^\pi(\tilde{s}, \tilde{a})
    \nabla \pi(\tilde{a}|\tilde{s}) - \nabla v_g^\pi \Big\}
+ \Big\{ \sum_{\tilde{s} \in \setname{S}} p_\pi^1(\tilde{s}|s)
    \sum_{\tilde{a} \in \setname{A}} q_b^\pi(\tilde{s}, \tilde{a})
    \nabla \pi(\tilde{a}|\tilde{s}) - \nabla v_g^\pi \Big\}
    \notag \\
& + \Big\{ \sum_{\tilde{s} \in \setname{S}} p_\pi^2(\tilde{s}|s)
    \sum_{\tilde{a} \in \setname{A}} q_b^\pi(\tilde{s}, \tilde{a})
    \nabla \pi(\tilde{a}|\tilde{s}) - \nabla v_g^\pi \Big\}
    + \ldots + \Big\{
    \sum_{\tilde{s} \in \setname{S}}
        \underbrace{p_\pi^{\tmix^\pi}(\tilde{s})}_{p_\pi^\star(\tilde{s})}
        \nabla v_b^\pi(\tilde{s}) \Big\}.
\label{equ:grad_bias_expansion}
\end{align}

Since $\nabla \log \pi(a|s) = \nabla \pi(a|s) / \pi(a|s)$,
the Equation \eqref{equ:grad_bias_expansion} above can be expressed as
\begin{align*}
\nabla v_b^\pi(s)
& = \textcolor{green}{\Big\{} \sum_{t=0}^{\tmix^\pi - 1} \Big\{
    \sum_{s \in \setname{S}} \sum_{a \in \setname{A}}
    p_\pi^t(s|s_0) \pi(a|s) \Big[ q_b(s, a) \nabla \log \pi(a|s) \Big] - \nabla v_g^\pi
    \Big\} \textcolor{green}{\Big\}} \\
& \quad + \textcolor{red}{\Big\{}
    \sum_{s \in \setname{S}} p_\pi^\star(s)
    \sum_{a \in \setname{A}} \Big[
    q_b^\pi(s, a) \nabla \pi(a|s) + \pi(a|s) \nabla q_b^\pi(s, a) \Big]
    \textcolor{red}{\Big\}} \\
& = \textcolor{green}{\Big\{} \sum_{t=0}^{\tmix^\pi - 1} \Big\{
    \sum_{s \in \setname{S}} \sum_{a \in \setname{A}}
    p_\pi^t(s|s_0) \pi(a|s)
        \Big[ q_b(s, a) \nabla \log \pi(a|s) \Big] - \nabla v_g^\pi
    \Big\} \textcolor{green}{\Big\}} \\
& \quad + \textcolor{red}{\Big\{}
    \sum_{s \in \setname{S}} \sum_{a \in \setname{A}}
        p_\pi^\star(s) \pi(a|s) \nabla q_b^\pi(s, a)
    + \underbrace{\sum_{s \in \setname{S}} \sum_{a \in \setname{A}}
        p_\pi^\star(s) \pi(a|s) \Big[ q_b^\pi(s, a) \nabla \log \pi(a|s) \Big]
        }_\text{$\nabla v_g^\pi$, as shown by \citet[\thm{1}]{sutton_2000_pgfnapprox}}
    \textcolor{red}{\Big\}}.
\end{align*}
Here, the green and red curly braces indicate the pre- and post-mixing terms,
respectively, which are related to the transient and steady-state phases
(see \figref{fig:timestep_line}).
This finishes the proof.

\end{proof}

\section{Natural gradients of the bias} \label{sec:natgrad}

The standard gradient ascent of the form \eqref{equ:polparam_update} without preconditioner
is \emph{not} invariant under policy parameterization.
One of possible ways to overcome this is
by utilizing natural (covariant) gradients of the bias,
\ie $\fbmat^{-1} \nabla v_b(\vecb{\theta})$,
where $\fbmat$ denotes a bias Fisher matrix derived in this Section.
The use of natural gradients is also anticipated to increase the optimization convergence rate.

\cite{kakade_2002_npg} firstly proposed a gain Fisher $\fgamat$
of a parameterized policy $\pi(\vecb{\theta})$.
That is,
\begin{equation}
\fgamat(\vecb{\theta})
\eqdef \sum_{s \in \setname{S}} p_\pi^\star(s) \famat(\vecb{\theta}, s),
\quad \text{with}\
\famat(\vecb{\theta}, s)
\eqdef \underbrace{
    \sum_{a \in \setname{A}} \pi(a|s; \vecb{\theta})
    \nabla \log \pi(a | s; \vecb{\theta}) \nabla^\intercal \log \pi(a | s; \vecb{\theta}),
}_\text{The action Fisher, which is independent of optimality criteria}
\label{equ:fisher_gain}
\end{equation}
where $\famat(\vecb{\theta}, s) \in \real{\dim(\vecb{\theta}) \times \dim(\vecb{\theta})}$
is a positive semidefinite Fisher information matrix that defines
a semi-Riemannian metric\footnote{
    A Riemannian metric on a manifold is an assignment of an inner product on
    the tangent space of that manifold.
    Note that the term ``metric'' here is \emph{not} in its typical sense as a distance function.
    Nevertheless, a Riemannian metric induces a natural distance function on
    the corresponding manifold.
}
on the action probability manifold (``surface'') at a state~$s \in \setname{S}$.
This $\famat(\vecb{\theta}, s)$ is related to
the Kullback-Leibler divergence $\mathcal{K}_{\mathrm{div}}$ between two policies
(at state $s$) parameterized by
$\vecb{\theta}$ and $(\vecb{\theta} + \delta_{\vecb{\theta}})$
for an infinitesimal $\delta_{\vecb{\theta}}$.
That is,
$\mathcal{K}_{\mathrm{div}}(\vecb{\theta}, \vecb{\theta} + \delta_{\vecb{\theta}})
= \frac{1}{2}\delta_{\vecb{\theta}}^\intercal \famat(\vecb{\theta}, s) \delta_{\vecb{\theta}}
+ \bigO{\delta_{\vecb{\theta}}^3}$
via the second-order Taylor expansion, where
$\famat(\vecb{\theta}, s)
= \nabla^2 \mathcal{K}_{\mathrm{div}}(\vecb{\theta}, \vecb{\theta} + \delta_{\vecb{\theta}})$.

The gain Fisher \eqref{equ:fisher_gain} is defined as the weighted sum of
the action Fisher $\famat$
whose weights are the components of $\ppimat^\star$, by which
the gain is formulated as $v_g(\pi) \cdot \vecb{1}= \ppimat^\star \rpivec$,
where $\rpivec \in \real{\setsize{S}}$ is a reward vector with elements
$\rpi(s) = \E{A \sim \pi(\cdot|s)}{r(s, A)}, \forall s \in \setname{S}$,
and $\ppimat^\star$ has identical rows in unichain MDPs.
Now, from the definition of bias $v_b$ in \eqref{equ:bias_def},
we can express $v_b$ as
\begin{equation}
v_b(\pi, s_0) = \sum_{s \in \setname{S}} \Big\{
    \underbrace{
        \lim_{\tmax \to \infty} \sum_{t=0}^{\tmax - 1}
            \Big( p_\pi^t(s|s_0) - p_\pi^\star(s) \Big)
    }_\text{is defined as $h_\pi(s|s_0)$}
    \Big\} r_\pi(s)
    = \sum_{s \in \setname{S}} h_\pi(s|s_0) r_\pi(s).
\label{equ:vb_devmat}
\end{equation}
Equivalently in matrix forms, we have
$\vecb{v}_b(\pi) = \dpimat \rpivec \in \real{\setsize{S}}$, where
$\dpimat$ is an $\setsize{S}$-by-$\setsize{S}$ deviation matrix whose
$[s_0, s]$-component is $h_\pi(s|s_0)$.
Whenever the induced MC is aperiodic (\assref{ass:unichain_aperiodic}),
$\lim_{\tmax \to \infty} p_\pi^{\tmax}(s|s_0) =  p_\pi^\star$.
Hence, we conjecture that the limit in \eqref{equ:vb_devmat} exists, and
moreover that the infinite series
$\sum_{t=0}^\infty ( p_\pi^t(s|s_0) - p_\pi^\star(s) ) \eqdefr h_\pi(s|s_0)$
converges absolutely.

In analogy with gain $v_g$ and gain Fisher $\fgamat$ formulation therefore,
we anticipate that $h_\pi(s|s_0)$ can be used to in such a way to
weight $\famat$ to yield a bias Fisher;
serving the same role as $p_\pi^\star$ in \eqref{equ:fisher_gain}.
The main issue is that $h_\pi$ is not a probability distribution,
hence $\dpimat$ is a non-stochastic matrix (whose components may be negative).
As a remedy, we utilize only the deviation magnitude,
\ie the absolute value of $h_\pi(s|s_0)$,
in order to maintain the positive semidefiniteness of $\famat$.
That is,
\begin{equation}
| h_\pi(s|s_0) |
= \Big| \sum_{t=0}^\infty \big( p_\pi^t(s|s_0) - p_\pi^\star(s) \big) \Big|
\le \sum_{t=0}^\infty | p_\pi^t(s|s_0) - p_\pi^\star(s)|,
\qquad \forall s, s_0 \in \setname{S}.
\label{equ:deviation_abs}
\end{equation}
Taking the absolute value per timestep in \eqref{equ:deviation_abs} is beneficial
for trajectory-sampling-based approximation in model-free RL,
as will be explained in the remainder of this Section.
Note that the infinite series in the RHS of the inequality in \eqref{equ:deviation_abs}
converges because its LHS counterpart is absolutely convergent,
as conjectured in the previous paragraph.

We are now ready to define the bias Fisher $\fbamat$ as follows.
For all $\pi(\vecb{\theta}) \in \piset{SR}$ and for all $s_0 \in \setname{S}$,
\begin{align}
& \fbamat(\vecb{\theta}, s_0)
\eqdef \sum_{s \in \setname{S}} \Big\{
    \underbrace{
        \sum_{t=0}^\infty |p_\pi^t(s|s_0) - p_\pi^\star(s)|
    }_\text{$\ge | h_\pi(s|s_0) |$ in \eqref{equ:deviation_abs}}
    \Big\} \famat(\vecb{\theta}, s)
= \sum_{t=0}^\infty \Big\{ \sum_{s \in \setname{S}}
    \underbrace{|p_\pi^t(s|s_0) - p_\pi^\star(s)|}_\text{a weight (not a probability)}
     \famat(\vecb{\theta}, s) \Big\}
    \notag \\
& = \sum_{t=0}^{\tabsminpiso - 1} \Big\{
    \sum_{\strans^\pi \in \settrans^\pi} |p_\pi^t(\strans^\pi | s_0)
    - \cancelto{0}{p_\pi^\star(\strans^\pi)}|\quad \famat(\vecb{\theta}, \strans^\pi)
    + \sum_{\srecur^\pi \in \setrecur^\pi}
        | \cancelto{0}{p_\pi^t(\srecur^\pi | s_0)}
        -\ p_\pi^\star(\srecur^\pi)|\
        \famat(\vecb{\theta}, \srecur^\pi) \Big\}
    \notag \\
& + \underbrace{
    \quad  \sum_{t=\tabsminpiso}^{\tmix^\pi - 1} \Big\{
    \sum_{\strans^\pi \in \settrans^\pi}
        |p_\pi^t(\strans^\pi | s_0) - \cancelto{0}{p_\pi^\star(\strans^\pi)}|
        \quad \famat(\vecb{\theta}, \strans^\pi)
    + \sum_{\srecur^\pi \in \setrecur^\pi}
        |p_\pi^t(\srecur^\pi | s_0) - p_\pi^\star(\srecur^\pi)|\
        \famat(\vecb{\theta}, \srecur^\pi)
    \Big\}
    }_\text{These terms are ignored in the sampling-enabler expression \eqref{equ:fisher_b_sampling}.}
    \notag \\
& \quad + \sum_{t=\tmix^\pi}^{\infty} \Big\{
    \sum_{\strans^\pi \in \settrans^\pi}
    |\cancelto{0}{p_\pi^t(\strans^\pi | s_0)} - \cancelto{0}{p_\pi^\star(\strans^\pi)}|
        \quad \famat(\vecb{\theta}, \strans^\pi)
    + \sum_{\srecur^\pi \in \setrecur^\pi}
    \cancelto{0}{|p_\pi^t(\srecur^\pi | s_0) - p_\pi^\star(\srecur^\pi)|}\
        \famat(\vecb{\theta}, \srecur^\pi)
    \Big\},
\label{equ:fisher_b}
\end{align}
where we decompose the time summation into pre-absorption, absorption-to-mixing, and post-mixing,
as well as the state summation into disjoint transient and recurrent state subsets
under $\pi(\vecb{\theta})$, namely $\setname{S} = \settrans^\pi \cup \setrecur^\pi$.
Here, the stationary probability of any transient state is zero,
\ie $p_\pi^\star(\strans^\pi) = 0$.
The same goes to the probability of visiting recurrent states before
the minimum absorption time (\defref{def:tabsmin}),
\ie $p_\pi^t(\srecur^\pi | s_0) = 0$ for $t < \tabsminpiso$.
Note that bringing $\famat(s)$ inside the absolute operator yields
$\sum_{t=0}^\infty \sum_{s \in \setname{S}}
| p_\pi^t(s|s_0) \famat(\vecb{\theta}, s) - p_\pi^\star(s) \famat(\vecb{\theta}, s)|$,
which may not equal to $\fbamat(\vecb{\theta}, s_0)$
since $\famat(s)$ may have negative off-diagonal entries.

\begin{definition} \label{def:tabsmin}
The minimum absorption time from an initial state $s_0$ and under a policy $\pi$,
denoted by $\tabsminpiso$, is the time required by the induced MC
such that the stepwise state distribution $p_\pi^{\tabsminpiso}$ contains
at least one recurrent state in its support.
That is,
\begin{equation*}
\tabsminpiso \eqdef \min \{ t : p_\pi^t(\srecur|s_0) > 0,\
    \text{for any recurrent state $\srecur \in \setname{S}$} \}.
\end{equation*}
Thus, $\tabsminpiso$ indicates the minimum number of timesteps
before absorption for any given $\pi$ and $s_0$.
\end{definition}

Based on time and state decompositions of the bias Fisher $\fbamat$
in \eqref{equ:fisher_b}, we propose a simplification of it that enables
approximation through sampling.
That is,
\begin{align}
\fbamatsam^{\tabsminpiso}(\vecb{\theta}, s_0)
& \eqdef \Bigg\{ \underbrace{
        \sum_{t=0}^{\tabsminpiso - 1}
        \sum_{s \in \setname{S}} p_\pi^t(s | s_0) \famat(\vecb{\theta}, s)
        }_\text{$\mat{\Phi}_{\!\xi(\tabsminpiso)}(\vecb{\theta}, s_0)$
            in \eqref{equ:finite_traj_fisher}}
    \Bigg\}
    + \Bigg\{ \tabsminpiso \cdot \underbrace{
            \sum_{s \in \setname{S}} p_\pi^\star(s) \famat(\vecb{\theta}, s)
        }_\text{$\ftimat(\vecb{\theta})$ in \eqref{equ:infinite_traj_fisher}}
    \Bigg\} \label{equ:fisher_b_sampling} \\
& \approx \fbamat(\vecb{\theta}, s_0)
    \tag{That is, $\fbamatsam$ approximates $\fbamat$ in \eqref{equ:fisher_b}}.
\end{align}
Recall that before the minimum absorption time, state samples are all transient
since $p_\pi^t(\srecur) = 0$ for $t < \tabsminpiso$, whereas
after the mixing time, state samples are all recurrent since
$p_\pi^t(\strans) = 0$ for $t \ge \tmix^\pi$.
Such samples are obtained by an agent during interaction with its environment.

The simplication from $\fbamat$ to $\fbamatsam^{\tabsminpiso}$ implies that
all terms from $t=\tabsminpiso$ to $\tmix^\pi - 1$ in \eqref{equ:fisher_b} are
not taken into account in the sampling-enabler expression \eqref{equ:fisher_b_sampling}.
There are at least two justifications for this.
\textbf{First}, the terms involving $p_\pi^t(\strans^\pi | s_0)$ and
$|p_\pi^t(\srecur^\pi | s_0) - p_\pi^\star(\srecur^\pi)|$ decrease to $0$ as
$t$ approaches $\tmix^\pi$, specifically $p_\pi^t(\strans^\pi | s_0)$ goes to
$0$ more quickly, generally far before mixing.
\textbf{Second}, the absolute deviation terms,
\ie $|p_\pi^t(\srecur^\pi | s_0) - p_\pi^\star(\srecur^\pi)|$, are not a state probability
(inheriting the non-stochastic characteristic of the deviation matrix $\dpimat$).
Therefore, states cannot be sampled from such terms during
agent-environment interaction in model-free RL.

\subsection{Alternative interpretations}

\paragraph{\textbf{Interpretation 1:}}
The sampling-enabler expression for the bias Fisher can also be interpreted
as the sum of Fisher matrices of finite and infinite trajectories,
as shown in \eqref{equ:fisher_b_sampling}.
Such matrices were introduced by \cite{bagnell_2003_covps, peters_2003_rlhum}
who identified the gain Fisher \eqref{equ:fisher_gain} as
the Fisher matrix on the manifold of the infinite trajectory distribution.
That is, $\fgamat(\vecb{\theta})$ is equal to
\begin{equation}
\ftimat(\vecb{\theta})
\eqdef \lim_{\tau \to \infty} \frac{1}{\tau}
    \underbrace{
        \sum_{t=0}^{\tau  - 1} \sum_{s \in \setname{S}} p_\pi^t(s | s_0; \vecb{\theta})
        \famat(\vecb{\theta}, s)
    }_\text{$\ftlmat(\vecb{\theta} | s_0)$ in \eqref{equ:finite_traj_fisher} below}
= \sum_{s \in \setname{S}} \Big\{
    \underbrace{
        \lim_{\tau \to \infty} \frac{1}{\tau} \sum_{t=0}^{\tau - 1}  p_\pi^t(s | s_0; \vecb{\theta})
    }_{p_\pi^\star(s)}
    \Big\} \famat(\vecb{\theta}, s),
\label{equ:infinite_traj_fisher}
\end{equation}
which involves the trajectory Fisher $\ftlmat$ defined as
\begin{align}
\ftlmat(\vecb{\theta}, s_0)
& \eqdef \E{\Xi \sim \trajd}
    {\big\{ \nabla \log \trajd(\Xi; \vecb{\theta})
        - \E{}{\nabla \log \trajd(\Xi; \vecb{\theta})} \big\}^2}
    \tag{Variance of the score} \notag \\
& = \E{\Xi \sim \trajd}{
        \sum_{t=0}^{\tau - 1} \nabla \log \pi(a_t | s_t; \vecb{\theta})
        \nabla^\intercal \log \pi(a_t | s_t; \vecb{\theta})
    } \tag{See below text} \\
& = \sum_{t=0}^{\tau - 1} \sum_{s_t \in \setname{S}} \sum_{a_t \in \setname{A}}
    \pi(a_t|s_t; \vecb{\theta}) p_\pi^t(s_t | s_0)
    \Big[ \nabla \log \pi(a_t | s_t; \vecb{\theta})
        \nabla^\intercal \log \pi(a_t | s_t; \vecb{\theta}) \Big].
\label{equ:finite_traj_fisher}
\end{align}
Here, $\xi(\tau) \eqdef (s_0, a_0, s_1, a_1, \ldots, s_{\tau - 1}, a_{\tau - 1})$
denotes a trajectory $\xi$ of length $\tau$, where the last action $a_{\tau - 1}$
(at the last state $s_{\tau - 1}$) is included.
There exists a trajectory distribution $\trajd$, from which a trajectory $\xi$
(whose random variable is $\Xi$) is sampled, namely $\Xi \sim \trajd$.
Given a policy $\pi(\vecb{\theta})$, the state and action sequence of $\xi$
is generated following $\pi(\vecb{\theta})$, for which
$\trajd$ is denoted as $\trajd(\cdot; \vecb{\theta})$.
The penultimate equality in \eqref{equ:finite_traj_fisher} is because
the mean of the score is $\E{}{\nabla \log \trajd(\Xi; \vecb{\theta})} = \vecb{0}$, and
\begin{align*}
\nabla \log \trajd(\xi(\tau); \vecb{\theta})
& = \nabla \log \Big[ \isd(s_0) \Big\{
        \prod_{t=0}^{\tau - 2} p(s_{t+1} | s_t, a_t)\ \pi(a_t| s_t; \vecb{\theta})
    \Big\} \pi(a_{\tau - 1}| s_{\tau - 1}; \vecb{\theta}) \Big] \\
& = \sum_{t=0}^{\tau - 1} \nabla \log \pi(a_t | s_t; \vecb{\theta}),
\end{align*}
since $\nabla \isd(s_0) = \vecb{0}$ obviously, and
$\nabla p(s_{t+1}|s_t, a_t) = \vecb{0}$ for randomized policies
$\pi(\vecb{\theta}) \in \piset{SR}$.

\paragraph{\textbf{Interpretation 2:}}
The bias Fisher \eqref{equ:fisher_b_sampling} uses the Fisher of
a policy $\pi(\vecb{\theta})$, that is $\famat$ in \eqref{equ:fisher_gain}.
However, there is another statistical model, \ie the state distribution $p_\pi^t$,
that also changes due to the changes of the policy parameter $\vecb{\theta}$.
The bias Fisher therefore, can also be defined using
the Fisher of the joint state-action distribution
$p_\pi^t(s, a |s_0) = p_\pi^t(s|s_0) \pi(a|s, \cancel{s_0}; \vecb{\theta})$.
This follows \cite{morimura_2008_nnpg} who specified a state-action gain Fisher as
$\fgsamat(\vecb{\theta})
\eqdef \fsmat^\star(\vecb{\theta}) + \fgamat(\vecb{\theta})$
with the stationary state Fisher defined as
$\fsmat^\star(\vecb{\theta}) \eqdef \sum_{s \in \setname{S}} p_\pi^\star(s)
\nabla \log p_\pi^\star(s) \nabla^\intercal \log p_\pi^\star(s)$.

In a similar fashion as the (action) bias Fisher \eqref{equ:fisher_b_sampling},
we define the state-action bias Fisher as follows,
\begin{align*}
& \fbsamatsam^{\tabsminpiso}(\vecb{\theta}, s_0)
\eqdef \sum_{t=0}^{\tabsminpiso - 1}
    \sum_{s \in \setname{S}} \sum_{a \in \setname{A}}
    p_\pi^t(s, a | s_0)
    \nabla \log p_\pi^t(s, a | s_0) \nabla^\intercal \log p_\pi^t(s, a | s_0)
    + \tabsminpiso \cdot \fgsamat(\vecb{\theta}) \\
& = \underbrace{\sum_{t=0}^{\tabsminpiso - 1} \fsmat^t(\vecb{\theta})
        + \tabsminpiso \cdot \fsmat^\star(\vecb{\theta})
    }_\text{These terms are neglected (set to $\vecb{0}$) in \eqref{equ:fisher_b_sampling}.}
    + \underbrace{
        \sum_{t=0}^{\tabsminpiso - 1} \sum_{s \in \setname{S}}
        p_\pi^t(s | s_0) \famat(s, \vecb{\theta})
        + \tabsminpiso \cdot \fgamat(\vecb{\theta})
    }_\text{$\fbamatsam^{\tabsminpiso}(\vecb{\theta} | s_0)$
        in \eqref{equ:fisher_b_sampling}},
\end{align*}
where
$\fsmat^t(\vecb{\theta}) \eqdef
\sum_{s \in \setname{S}} p_\pi^t(s | s_0)
\nabla \log p_\pi^t(s | s_0) \nabla^\intercal \log p_\pi^t(s | s_0)$.
As can be seen, the bias Fisher~\eqref{equ:fisher_b_sampling} does not
take into account the changes of the state distribution under a policy $\pi(\vecb{\theta})$
due to perturbations in its parameter $\vecb{\theta}$,
\ie setting $\nabla p_\pi^t(s|s_0) \gets \vecb{0}$ and $\nabla p_\pi^\star(s) \gets \vecb{0}$.

\subsection{Practical considerations} \label{sec:natgrad_practical}

In model-free RL, we estimate the \emph{minimum} absorption time
$\tabsminpiso \approx \tabsminhat$,
which is held constant for all policies $\pi$ and all initial states $s_0$ for simplicity.
Larger $\tabsminhat$ does not necessarily yield lower approximation error
in \eqref{equ:fisher_b_sampling} because the agent becomes more likely to already
visit recurrent states (\ie being absorbed in the recurrent class).
Although larger $\tabsminhat$ may compensate for the ignored terms of \eqref{equ:fisher_b},
it yields larger variance in sampling-based approximation.
In contrast, smaller $\tabsminhat$ lowers the sample variance with the cost of
missing a number (\ie $\tabsminpiso - \tabsminhat$) of expectation terms
$\E{S \sim p_\pi^t}{\famat(\vecb{\theta}, S)}$ and
$\E{S \sim p_\pi^\star}{\famat(\vecb{\theta}, S)}$ in \eqref{equ:fisher_b_sampling}.

For an agent that begins at a transient state $s_0 \in \settrans^\pi$,
we set $\tabsminhat \gets 2$.
This corresponds to the bias Fisher \eqref{equ:fisher_b_sampling} that consists of
a 1-step trajectory Fisher and twice the infinite trajectory Fisher.
Empirical results (\secref{sec:biasopt}) suggest that this choice is reasonable
for the purpose of conditioning the bias gradient in \eqref{equ:polparam_update}.
Note that setting $\tabsminhat \gets 1$ (the lowest because $s_0$ is transient)
diminishes the involvement of the transition probability as no finite-step trajectory
is taken into account in \eqref{equ:fisher_b_sampling}.

\section{A nearly Blackwell optimal policy gradient algorithm}
\label{sec:bwoptim_algo}

In this Section, we propose an nBw-optimal algorithm through formulating
the objective (\defref{def:bias_optim}) as a bi-level optimization problem.
That is,
\begin{equation}
\pi_b^* \in \argmax_{\pi_g^*} b(\pi_g^*, s_0),
\quad \text{subject to}\
\pi_g^* \in \Pi_g^* \eqdef \argmax_{\pi \in \piset{S}} g(\pi),
\label{equ:bilevel_opt}
\end{equation}
where $b(\pi) \eqdef v_b(\pi)$, $g(\pi) \eqdef v_g(\pi)$, and
the policy parameterization $\pi(\vecb{\theta})$ is notationally hidden.
Here, the lower-level gain optimization yields the feasible set $\Pi_g^*$
for the upper-level bias optimization.
\citet[\thm{10.1.5}]{puterman_1994_mdp} showed that
there exist a stationary deterministic nBw-optimal policy~$\pi_b^*$ for finite MDPs
(as long as the policy parameterization contains such a policy).

Numerically in practice, the lower-level gain optimization in \eqref{equ:bilevel_opt}
is deemed converged once the norm of the gain gradient falls below a small positive~$\epsilon$,
without considering the negative definiteness condition,
\ie $\nabla^2 g(\vecb{\theta}) \prec 0$, for a maximum.
In addition, such gain optimization is carried out using a local optimizer
so there is no guarantee of finding a global maximum.
Practically therefore, we approximate \eqref{equ:bilevel_opt} by imposing
a set constraint of \emph{approximately} gain optimal policies,
namely $\Pi_g^{*_\epsilon}
\eqdef \{ \pi(\vecb{\theta}) : \norm{\nabla g(\vecb{\theta})} \le \epsilon\}$.
Note that one can perform multiple (independent/parallelizable) runs of
local optimization with randomized initial points for
approximate global optimization in non-concave objectives
\citep[\ch{2.1}]{zhigljavsky_1991_tgrs}.

We approximately solve \eqref{equ:bilevel_opt} using a barrier method,
assuming that $\Pi_g^{*_\epsilon}$ has an interior and any boundary point
can be attained by approaching it from the interior
(for comparison with penalty methods, see \secref{sec:barrier_vs_penalty}).
As an interior-point technique, the barrier method must be initiated at
a strictly feasible point~$\vecb{\theta}$ satisfying
$\norm{\nabla g(\vecb{\theta})} < \epsilon$.
This is fulfilled by performing gain optimization.
Let $\pi_g^{*_\epsilon}$ be the resulting gain optimal policy,
whose gain is denoted by $g^*_\epsilon$.
Then, we construct another set
$\Pi_{g+}^{*_\epsilon} = \{\pi\ |\ (g^*_\epsilon - \zeta) - g(\pi) \le 0 \}$,
which can serve the same role as $\Pi_g^{*_\epsilon}$ (as a constraint set)
and is beneficial because it does not involve the gradients of gain.
Here, $\zeta$ is a positive scalar specifying slackness in $g^*_\epsilon$
such that $\pi_g^{*_\epsilon}$ becomes an interior point of $\Pi_{g+}^{*_\epsilon}$.
It functions similarly as~$\epsilon$
(but recall that $\epsilon$ is used while searching for~$\pi_g^{*_\epsilon}$,
whereas $\zeta$ is used after $\pi_g^{*_\epsilon}$ is found).

Given a constraint set $\Pi_{g+}^{*_\epsilon}$ and its interior point $\pi_g^{*_\epsilon}$,
we have the following approximation to \eqref{equ:bilevel_opt},
\begin{equation}
\argmax_{\pi} \big[
\underbrace{\mathbbm{b}(\pi, s_0, g^*_\epsilon, \beta, \zeta)}_\text{Bias-barrier}
\eqdef \underbrace{b(\pi, s_0)}_\text{Bias}
+ \underbrace{\beta \log (g(\pi) - g^*_\epsilon + \zeta)}_\text{Barrier}
\big],
\quad \underbrace{
    \text{subject to}\ \pi \in \interior{\Pi_{g+}^{*_\epsilon}}
    }_\text{No need to consider this in practice},
\label{equ:biasbarrier_opt}
\end{equation}
where $\beta$ is a non-negative barrier parameter, and
$\interior{\Pi_{g+}^{*_\epsilon}}$ denotes the interior of $\Pi_{g+}^{*_\epsilon}$.
The first and second derivatives of the bias-barrier $\mathbbm{b}$
in \eqref{equ:biasbarrier_opt} are given by
\begin{align}
& \nabla \mathbbm{b}(\vecb{\theta}, s_0)
= \nabla b(\vecb{\theta}, s_0) + \tilde{\beta} \nabla g(\vecb{\theta}),
    \quad \text{and} \label{equ:biasbarrier_der_1st} \\
& \nabla^2 \mathbbm{b}(\vecb{\theta}, s_0)
= \nabla^2 b(\vecb{\theta}, s_0) + \tilde{\beta} \nabla^2 g(\vecb{\theta})
    - (\tilde{\beta}^2/\beta) \nabla g(\vecb{\theta}) \nabla^\intercal g(\vecb{\theta}),
    \label{equ:biasbarrier_der}
\end{align}
where $\tilde{\beta} = \beta/(g(\pi) -  g^*_\epsilon + \zeta)$.

We argue setting the slackness parameter
$\zeta \gets 1$ in \eqref{equ:biasbarrier_opt} as follows.
For environments where all policies are gain optimal,
we want to suppress the contribution of the barrier to~$\mathbbm{b}$
so that \eqref{equ:biasbarrier_opt} becomes bias-only optimization.
This can be achieved by $\log (g(\pi) - g^*_\epsilon + (\zeta = 1)) = 0$
since $g^*_\epsilon = g(\pi)$ for any stationary policy $\pi$ in such environments.
For other environments, a distance of one unit gain
(between the interior $\pi_g^{*_\epsilon}$ and the boundary)
can be justified by observing the behaviour of a logarithmic function,
where $\log (g(\pi) - g^*_\epsilon + (\zeta = 1))$ drops (without bound) quickly
as $(g(\pi) - g^*_\epsilon)$ approaches $-1$ from its initial value of~$0$
(since initially $g(\pi) = g^*_\epsilon$).
This prevents the optimization iterate from going to a policy parameter
that induces an undesirable lower gain $g(\pi) < g^*_\epsilon$.

The barrier method proceeds by solving a sequence of \eqref{equ:biasbarrier_opt}
for $k=0, 1, \ldots$ with monotonically decreasing barrier parameter
$\beta \to 0$ as $k \to \infty$, where the solution of the previous $k$-th iteration
is used as the initial point for the next $(k+1)$-th iteration.
It is anticipated that as $\beta$ decreases, the solution of \eqref{equ:biasbarrier_opt}
becomes close to that of \eqref{equ:bilevel_opt}, see \citet[\ch{11.3}]{boyd_2004_cvxopt}.
However, the Hessian matrix  $\nabla^2 \mathbbm{b}(\vecb{\theta}, s_0)$ in
\eqref{equ:biasbarrier_der} tends to be increasingly ill-conditioned as $\beta$ shrinks.
This demands a proper preconditioning matrix
(for gradient-ascent optimization via \eqref{equ:polparam_update}),
which can be obtained by following a 2-step procedure below.
\begin{itemize}
\item Multiply the last term of RHS in \eqref{equ:biasbarrier_der} by $-1$
    so that such a last term contains a positive semidefinite matrix.
    That is, for any $\vecb{x} \in \real{\dim(\vecb{\theta})}$ and $\beta \ge 0$,
    we have
    \begin{equation*}
    \vecb{x}^\intercal \big[ (-1) (- \tilde{\beta}^2/\beta)
     \nabla g(\vecb{\theta}) \nabla^\intercal g(\vecb{\theta}) \big] \vecb{x}
    = (\tilde{\beta}^2/\beta) (\vecb{x}^\intercal \nabla g(\vecb{\theta}))^2
    \ge 0.
    \end{equation*}
\item Substitute the gain and bias Hessian matrices in \eqref{equ:biasbarrier_der}
    with their corresponding Fisher matrices, \ie
    $\fgamat(\vecb{\theta})$ from \eqref{equ:fisher_gain} and
    $\fbamatsam^{\tabsminpiso}(\vecb{\theta}, s_0)$ from \eqref{equ:fisher_b_sampling},
    respectively.
\end{itemize}
Thus, based on \eqref{equ:biasbarrier_der}, we obtain a preconditioning matrix below
for $\nabla \mathbbm{b}(\vecb{\theta}, s_0)$ used in \eqref{equ:polparam_update},
\begin{equation}
\mat{C} \eqdef
\fbamatsam^{\tabsminpiso}(\vecb{\theta}, s_0) + \tilde{\beta} \fgamat(\vecb{\theta})
    + (\tilde{\beta}^2/\beta) \nabla g(\vecb{\theta}) \nabla^\intercal g(\vecb{\theta}).
    \label{equ:biasbarrier_precond}
\end{equation}

\subsection{Pseudocode}

\algref{alg:gainbiasbarrier_optim} implements the barrier method to approximately
obtain nBw-optimal policies (as described in the beginning of \secref{sec:bwoptim_algo}).
It employs \algref{alg:gradfisher_sampling} to estimate
the gradients and Fisher matrices of gain and bias.
These algorithms altogether can be regarded as approximation to
the $n$-discount optimality policy iteration \citep[\ch{10.3}]{puterman_1994_mdp}.

\begin{algorithm}[t]
\caption{An (approximately) nBw-optimal policy gradient method (\secref{sec:bwoptim_algo})}
\label{alg:gainbiasbarrier_optim}
\DontPrintSemicolon

\KwInput{
\begin{itemize}
\item A parameterized policy $\pi(\vecb{\theta})$.
\item A small positive scalar $\epsilon$ for specifying gain-optimization convergence.
\item An initial value of the barrier parameter, $\beta_0$.
\item A barrier divisor $\beta_{\mathrm{div}}$ for shrinking the barrier parameter,
    \eg $\beta_{\mathrm{div}} = 10$.
\item A maximum numbers of inner ($j$-index) optimization iteration, $n_j$.
\item A maximum numbers of outer ($k$-index) optimization iteration, $n_k$.
\end{itemize}
}
\KwOutput{An (approximately) nBw-optimal parameterized policy.}

Initialize the barrier parameter $\beta \gets \beta_0$ and
    a variable \texttt{IsGainOptimizationConverged} to False.  \\
\For{Outer optimization iteration, $k = 0, 1, \ldots, n_k - 1$}{
    \For{Inner optimization iteration, $j = 0, 1, \ldots, n_j - 1$}{
        Evaluate the current policy $\pi(\vecb{\theta})$ for
            its gain $g(\vecb{\theta})$ and bias state-action value $q_b(\vecb{\theta})$. \\
        Approximate the gradients and Fisher matrices of gain and bias, \newline
        namely
        $\hat{\nabla} g(\vecb{\theta})$, $\fgamathat(\vecb{\theta})$,
        $\hat{\nabla} b(\vecb{\theta})$, and $\fbamathat(\vecb{\theta})$
        using \algref{alg:gradfisher_sampling} that takes $q_b(\vecb{\theta})$ as input. \\
        \If{\emph{\texttt{IsGainOptimizationConverged}} is \emph{True}}{
            $\tilde{\beta} \gets \beta/(g(\vecb{\theta}) -  g^*_\epsilon + 1)$. \\
            $\hat{\nabla} v(\vecb{\theta}) \gets \hat{\nabla} b(\vecb{\theta})
                + \tilde{\beta} \hat{\nabla} g(\vecb{\theta})$.
                \Comment*[r]{For bias-barrier gradients \eqref{equ:biasbarrier_der_1st}}
            $\hat{\mat{C}}(\vecb{\theta}) \gets \fbamathat(\vecb{\theta})
            + \tilde{\beta} \fgamathat(\vecb{\theta})
            + (\tilde{\beta}^2/\beta)
                \hat{\nabla} g(\vecb{\theta}) \hat{\nabla}^\intercal g(\vecb{\theta})$.
                \Comment*[r]{Based on \eqref{equ:biasbarrier_precond}}
                \label{line:bias_precond}
        }
        \Else{
            $\hat{\nabla} v(\vecb{\theta}) \gets \hat{\nabla} g(\vecb{\theta})$,
            and $\hat{\mat{C}}(\vecb{\theta}) \gets \fgamathat(\vecb{\theta})$.
            \Comment*[r]{For gain optimization}
        }
        \If{$\norm{\hat{\nabla} v(\vecb{\theta})} < \epsilon$ %
            \Andd \emph{\texttt{IsGainOptimizationConverged}} is \emph{False}} {
            $g^*_\epsilon \gets g(\vecb{\theta})$. \\
            Set \texttt{IsGainOptimizationConverged} to True, then \Continue.
        }
        Calculate the step direction,
        $\delta_{\vecb{\theta}} \gets \hat{\mat{C}}(\vecb{\theta})^{-1} \hat{\nabla} v(\vecb{\theta})$,
            and the step length, $\alpha$. \\
        Update the policy parameter,
            $\vecb{\theta} \gets \vecb{\theta} + \alpha \delta_{\vecb{\theta}}$.
            \Comment*[r]{See \eqref{equ:polparam_update}}
    }
    \If{\emph{\texttt{IsGainOptimizationConverged}} is \emph{True}}{
        Shrink the barrier parameter, $\beta \gets \beta / \beta_{\mathrm{div}}$.
    }
}

\Return $\pi(\vecb{\theta})$.

\end{algorithm}

\begin{algorithm}[t]
\caption{Sampling-based approximation for the gradients and Fisher matrices of gain and bias}
\label{alg:gradfisher_sampling}
\DontPrintSemicolon

\KwInput{
\begin{itemize}
\item A parameterized policy $\pi(\vecb{\theta})$ and
    its bias state-action value $q_b(\pi(\vecb{\theta}))$.
\item A sample size, specified as the number of experiment-episodes (trials) $\nxep$.
\item An experiment-episode length, specified as the maximum timestep $\txepmax$.
\end{itemize}
}
\KwOutput{
    Estimates of the gradients and Fisher matrices of gain and bias.
    The resulting estimates are unbiased whenever $\txepmax \ge \tmix^{\pi(\vecb{\theta})}$,
    see \defref{def:tmix}.
}

Initialize empty lists of the gradients and Fisher matrices of gain and bias,
\ie $\setname{G}_g$, $\setname{G}_b$, $\setname{F}_g$, and $\setname{F}_b$. \\
Set $\tabsminhat \gets 2$.
    \Comment*[r]{See justification in \secref{sec:natgrad_practical}}

\For{Each experiment-episode $i = 0, 1, \ldots, \nxep - 1$}{
    Reset the environment and obtain an initial state $s$. \\
    Reset
        $\hat{\nabla} g(\vecb{\theta}) \gets \vecb{0}$
        and $\fgamathat(\vecb{\theta}) \gets \vecb{0}$.
        \Comment*[r]{\small Estimates for gain gradient and Fisher}
    Reset
        $\hat{\nabla} b(\vecb{\theta}) \gets \vecb{0}$
        and $\fbamathat(\vecb{\theta}) \gets \vecb{0}$.
        \Comment*[r]{\small Estimates for bias gradient and Fisher}

    \For{Each timestep $t = 0, 1, \ldots, \txepmax$}{
        Choose to then execute an action $a$ based on $\pi(\cdot|s; \vecb{\theta})$. \\
        Observe the next state $s'$ (and the reward). \\
        $\vecb{z}(s, a, \vecb{\theta}) \gets q_b^\pi(s, a) \nabla \log \pi(a|s; \vecb{\theta})$.
            \Comment*[r]{A temporary vector $\vecb{z}$}
        $\mat{X}(s, a, \vecb{\theta}) \gets
            \nabla \log \pi(a|s; \vecb{\theta}) \nabla^\intercal \log \pi(a|s; \vecb{\theta})$.
            \Comment*[r]{A temporary matrix $\mat{X}$}
        \If{$t = \txepmax$ (\ie the last timestep in an experiment-episode)}{
            $\hat{\nabla} g(\vecb{\theta}) \gets \vecb{z}(s, a, \vecb{\theta})$.
                \Comment*[r]{Based on \eqref{equ:gaingrad_simple}}
            $\hat{\nabla} b(\vecb{\theta}) \gets
                \hat{\nabla} b(\vecb{\theta}) + \nabla q_b^\pi(s, a)
                - (\txepmax - 1) \hat{\nabla} g(\vecb{\theta})$.
                \Comment*[r]{Based on \eqref{equ:grad_bias}} \label{line:gradbias}
            $\fgamathat(\vecb{\theta}) \gets
                \mat{X}(s, a, \vecb{\theta})$.
                \Comment*[r]{Based on \eqref{equ:fisher_gain}}
            $\fbamathat(\vecb{\theta}) \gets
                \fbamathat(\vecb{\theta})
                + \tabsminhat \mat{X}(s, a, \vecb{\theta})$.
                \Comment*[r]{Based on \eqref{equ:fisher_b_sampling}}
            Append $\hat{\nabla} g$, $\hat{\nabla} b$, $\fgamathat$, and $\fbamathat$
            to their corresponding lists
            $\setname{G}_g$, $\setname{G}_b$, $\setname{F}_g$, and $\setname{F}_b$.
        }
        \Else{
            $\hat{\nabla} b(\vecb{\theta}) \gets
                \hat{\nabla} b(\vecb{\theta}) + \vecb{z}(s, a, \vecb{\theta})$.
                    \Comment*[r]{\small Subtraction of $\hat{\nabla} g$
                        in \eqref{equ:grad_bias} is in \lineref{line:gradbias}}
            \If{$t < \tabsminhat$}{
                $\fbamathat(\vecb{\theta}) \gets
                    \fbamathat(\vecb{\theta}) + \mat{X}(s, a, \vecb{\theta})$.
                    \Comment*[r]{Based on \eqref{equ:fisher_b_sampling}}
            }
        }
        Update the state $s \gets s'$.
    }
}

Compute the sample means of gain and bias gradients
    based on $\setname{G}_g$ and $\setname{G}_b$, \newline
    namely $\bar{\nabla} g(\vecb{\theta}) \gets (\sum_i \setname{G}_g[i])/\nxep$,
    and $\bar{\nabla} b(\vecb{\theta}) \gets (\sum_i \setname{G}_b[i])/\nxep$. \\
Compute the sample means of gain and bias Fisher matrices
    based on $\setname{F}_g$ and $\setname{F}_b$, \newline
    namely $\fgamatbar(\vecb{\theta}) \gets (\sum_i \setname{F}_g[i])/\nxep$,
    and $\fbamatbar(\vecb{\theta}) \gets (\sum_i \setname{F}_b[i])/\nxep$. \\
\vspace{1.5mm}
\Return The sample means
    $\bar{\nabla} g(\vecb{\theta})$, $\bar{\nabla} b(\vecb{\theta})$,
    $\fgamatbar(\vecb{\theta})$, and $\fbamatbar(\vecb{\theta})$.
\end{algorithm}

\subsection{Barrier versus penalty methods} \label{sec:barrier_vs_penalty}

For obtaining nBw-optimal policies, the utilization of a barrier method
(in \algref{alg:gainbiasbarrier_optim}) brings several advantages over
its penalty counterpart (including the augmented Lagrangian variant)
in the form of
\begin{equation}
\pi_b^* \in \argmax_{\pi \in \piset{SR}}\ \Big[
    \mathfrak{b}(\pi) \eqdef b(\pi) - \frac{1}{2}\ \phi \norm{\nabla g(\pi)}_2^2
\Big],
\quad \text{with a positive penalty parameter $\phi$}.
\label{equ:optim_penalty}
\end{equation}
Here, a penalty term is used for transforming
the constrained (bi-level) optimization \eqref{equ:bilevel_opt} into
an unconstrained (single-level) optimization \eqref{equ:optim_penalty}.
Note that this penalty term is different from entropy regularization that is
typically utilized for encouraging exploration.

We identify three advantages of the barrier method for solving \eqref{equ:bilevel_opt}.
\textbf{First}, as an interior-point approach, it begins with a preliminary phase
searching for a feasible point of \eqref{equ:biasbarrier_opt},
which represents a gain-optimal policy.
This means that we can benefit from the state-of-the-art gain optimization
(or alternatively, the discounted-reward optimization with $\gamma$ sufficiently close to~1).
For instance, the work of \citet{yang_2020_sosp, liu_2020_npg, zhang_2020_polgrad}
about global optimality and global convergence of
gain (or discounted-reward) policy gradient methods.
In non-interior point approaches, including penalty methods,
there is no such preliminary phase.

\begin{wrapfigure}[18]{r}{0.5\textwidth}
\centering
\vspace{0mm}
\includegraphics[width=0.5\textwidth]{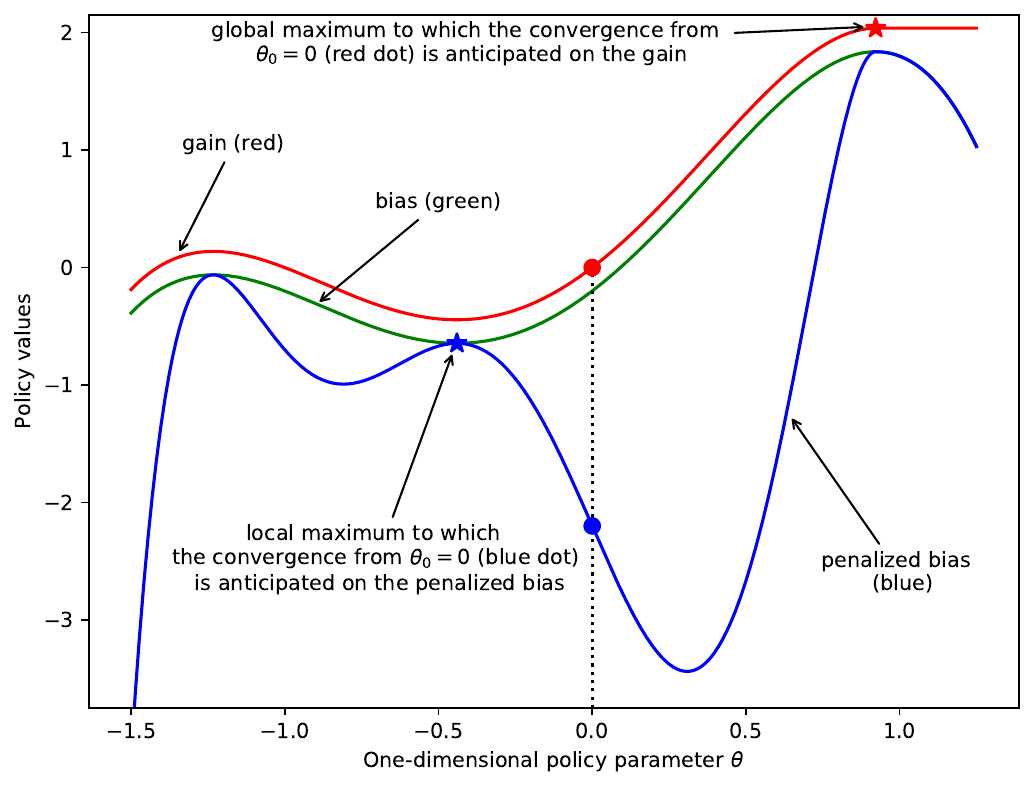}
\caption{An illustration of a deep valley on the penalized bias landscape
along with comparisons to gain and bias landscapes.
The initial $\theta_0$ is at 0.}
\label{fig:penalized_bias_illustration}
\end{wrapfigure}

\textbf{Second}, the barrier function \eqref{equ:biasbarrier_opt} contains only the gain value.
In contrast, a penalty function likely involves the gain gradients,
such as $\norm{\nabla g(\vecb{\theta})}$ in \eqref{equ:optim_penalty}.
This implies that first-order optimization on the penalized objective necessitates
the second derivative
(as in $\nabla \mathfrak{b}(\vecb{\theta}) = \nabla b(\vecb{\theta})
    - \phi \nabla^\intercal g(\vecb{\theta}) \nabla^2 g(\vecb{\theta})$),
whereas any second-order optimization necessitates the third derivative.
Since a penalty method does not necessarily begin at a feasible point,
its penalized objective, \eg \eqref{equ:optim_penalty}, cannot be turned into
a function of solely gain terms similar to~\eqref{equ:biasbarrier_opt}.

\textbf{Third}, we observe that the gain landscape
(on which the barrier method's preliminary operates) is more favourable
than the penalized bias landscape in achieving higher local maxima.
The penalized objective $\mathfrak{b}$ in \eqref{equ:optim_penalty}
may create a deep valley in between its two local maxima,
as illustrated in \figref{fig:penalized_bias_illustration}.
Thus, optimization from some initial point $\vecb{\theta}_0$ in the valley is
attracted to the nearest (instead of highest) local maximum of~$\mathfrak{b}$.
This phenomenon does not occur in the gain landscape when
the gain optimization (part of the preliminary phase of the barrier method) is
started at the same initial point $\vecb{\theta}_0$.
Note that such a valley of $\mathfrak{b}$ is formed partly because
$\norm{\nabla g(\vecb{\theta})}_2^2$ is always zero at the stationary points of bias,
but is likely to be non-zero otherwise.

\section{Experimental setup} \label{sec:xprmt_setup_nbwpg}

In this section, we discuss the implementation and experimental details for
the proposed nBw-optimal policy gradient algorithm (\secref{sec:bwoptim_algo}).
We begin with environment descriptions (\secref{sec:env_nbpwg}), followed by
policy parameterization (\secref{sec:polparam}) and the parameter values of
\algrefand{alg:gainbiasbarrier_optim}{alg:gradfisher_sampling} (\secref{sec:algoparam}).
In \secref{sec:xprmtsetup_disrew}, we provide the settings for
the discounted-reward policy gradient method, which is used for comparison.
Throughout the experiments, $\log$ refers to the natural logarithm, \ie $\log_e$.
The research code is publicly available at \url{https://github.com/tttor/nbwpg}.

\subsection{Environments} \label{sec:env_nbpwg}

There are two types of environments that we consider in this work:
a) environments for which all stationary policies are gain optimal, and
b) those for which some stationary policies are gain optimal and some are not.
Each environment type has three instances.
The first type has Env-A1, Env-A2, and Env-A3, whereas
the second type has Env-B1, Env-B2, and Env-B3.
Both types share three common properties:
i)~all stationary policies induce unichain Markov chains,
ii)~the nBw optimality is the most selective, and
iii)~there is no terminal state,
hence all environments are continuing (non-episodic).

In every environment, all states have the same number of available actions.
That is, two actions per state.
At some states, those two actions are duplicates
(having the same transition probability and reward).
This is simply to ease the implementation of policy parameterization,
which accomodates generalization across 2-action states.
Additionally, every environment's initial state distribution $\isd$ has
a single state support $s^0 \in \setname{S}$, that is $\isd(s^0) = 1$.

The symbolic diagrams and descriptions for Env-A1, A2, B1, B2, and B3 are provided
in \figrefand{fig:symbolic_diagram_a12_b12}{fig:symbolic_diagram_b3}.
Env-A3 is complex and unintuitive to draw therefore we describe it in words below.

\paragraph{\textbf{Env-A3:}}
This environment has 4 states with 2 actions per state,
yielding $16 (=2^4)$ stationary deterministic policies.
It resembles $2$-by-$2$ grid-world navigation, whose
start state $s^0$ is at the bottom-left, whereas goal state $s^3$ at the top-right.
Each state has its own unique action set, namely:
$s^0$ has East and North actions,
$s^1$ (bottom-right) has West and North actions,
$s^2$ (top-left) has East and South actions, and
$s^3$ has 2 self-loop (self-transition) actions.
In all but the goal state, the navigating agent goes to the intended direction
at $90\%$ of the times.
Otherwise, it goes to the other direction.
Every action incurs a cost of~$-1$, except those at the goal state (having a zero cost).

All stationary deterministic policies for this environment have an equal gain of~$0$,
hence all are gain optimal.
There are 6 distinct bias values, ranging from $-17.000$ to $0.778$.
There exist 4 nBw-optimal policies (out of 16, hence $4/16=0.25$) that choose
any action at the start state $s^0$, East action at the top-left state $s^2$,
North action at the bottom-right state $s^1$, and
any action at the goal state $s^3$.

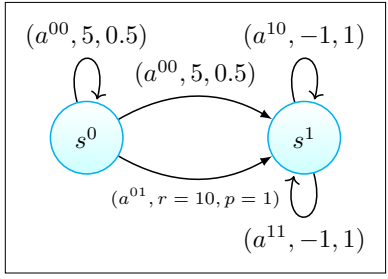
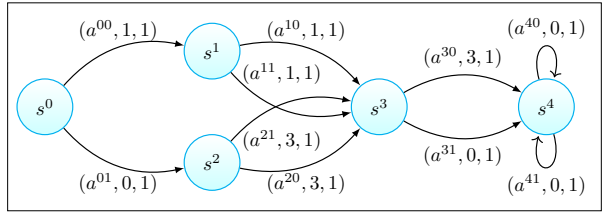
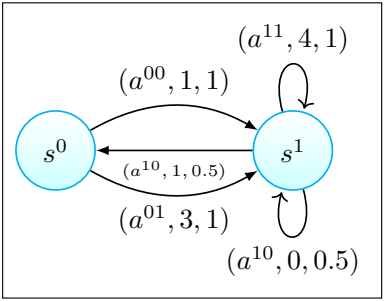
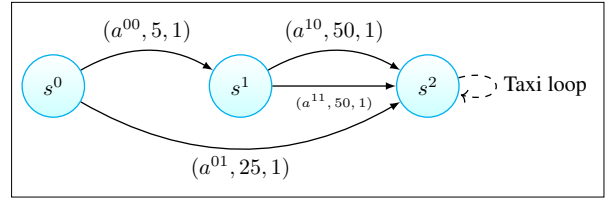
\begin{figure}[t]
\centering

\begin{subfigure}{0.375\textwidth}
\resizebox{\textwidth}{!}{
\begin{tikzpicture}[ %
show background rectangle,
node distance = 3cm and 3cm, on grid,
-latex, %
semithick, %
state/.style={circle, top color = white,  bottom color = stateblue!20,
draw, stateblue, text=black, minimum width = 1cm},
]
\node[state](A)[] {$s^0$};
\node[state](B)[right=of A] {$s^1$};
\path (A) edge [loop above] node[above] {$(a^{00}, 5, 0.5)$} (A);
\path (A) edge [bend left] node[above] {$(a^{00}, 5, 0.5)$} (B);
\path (A) edge [bend right] node[below] {\scriptsize{$(a^{01}, r=10, p=1)$}} (B);
\path (B) edge [loop above] node[above] {$(a^{10}, -1, 1)$} (B);
\path (B) edge [loop below] node[below] {$(a^{11}, -1, 1)$} (B);
\end{tikzpicture}
} %
\subcaption{\textbf{Env-A1}.
There exist four stationary deterministic policies,
whose gains are all $-1$.
Two policies have the bias of $11$, whereas the other two have the bias of $12$.
There are two nBw-optimal policies (out of four, hence $2/4=0.5$) that
choose $a^{00}$ at~$s^0$, and any action at~$s^1$.
}
\end{subfigure}%
\hspace{0.025\textwidth}
\begin{subfigure}{0.575\textwidth}
\resizebox{\textwidth}{!}{
\begin{tikzpicture}[
show background rectangle,
node distance = 1cm and 3cm, on grid,
-latex, %
semithick, %
state/.style={circle, top color = white,  bottom color = stateblue!20,
draw, stateblue, text=black, minimum width = 1cm},
]
\node[state](A)[] {$s^0$};
\node[state](B)[above right=of A] {$s^1$};
\node[state](C)[below right=of A] {$s^2$};
\node[state](D)[below right=of B] {$s^3$};
\node[state](E)[right=of D] {$s^4$};
\path (A) edge [bend left] node[above] {$(a^{00}, 1, 1)$} (B);
\path (A) edge [bend right] node[below] {$(a^{01}, 0, 1)$} (C);
\path (B) edge [bend left] node[above] {$(a^{10}, 1, 1)$} (D);
\path (B) edge [bend right] node[above] {} (D);
\path (C) edge [bend right] node[below] {$(a^{20}, 3, 1)$} (D);
\path (C) edge [bend left] node[below] {} (D);
\path (D) edge [bend left] node[above] {$(a^{30}, 3, 1)$} (E);
\path (D) edge [bend right] node[below] {$(a^{31}, 0, 1)$} (E);
\path (E) edge [loop above] node[above] {$(a^{40}, 0, 1)$} (E);
\path (E) edge [loop below] node[below] {$(a^{41}, 0, 1)$} (E);
\node at (4.25, 0.6){$(a^{11}, 1, 1)$};
\node at (4.25, -0.6){$(a^{21}, 3, 1)$};
\end{tikzpicture}
} %
\subcaption{\textbf{Env-A2}.
There are 5 states (where $s^4$ is absorbing) and 2 actions per state,
yielding $32 (=2^5)$ stationary deterministic policies.
All of those policies have the same gain of~0, hence they are all gain optimal.
There are 4 distinct bias values, namely $2$, $3$, $5$ and $6$.
There exist 8 nBw-optimal policies (out of 32, hence $8/32=0.250$) that choose
$a^{01}$ at $s^0$, $a^{30}$ at $s^3$, and any action at the other states.
}
\end{subfigure}%
\hspace{0.025\textwidth}
\begin{subfigure}{0.375\textwidth}
\resizebox{\textwidth}{!}{
\begin{tikzpicture}[
show background rectangle,
node distance = 3cm and 3cm, on grid,
-latex, %
semithick, %
state/.style={circle, top color = white,  bottom color = stateblue!20,
draw, stateblue, text=black, minimum width = 1cm},
]
\node[state](A)[] {$s^0$};
\node[state](B)[right=of A] {$s^1$};
\path (A) edge [bend left] node[above] {$(a^{00}, 1, 1)$} (B);
\path (A) edge [bend right] node[below] {$(a^{01}, 3, 1)$} (B);
\path (B) edge [left] node[below] {\tiny{$(a^{10}, 1, 0.5)$}} (A);
\path (B) edge [loop above] node[above] {$(a^{11}, 4, 1)$} (B);
\path (B) edge [loop below] node[below] {$(a^{10}, 0, 0.5)$} (B);
\end{tikzpicture}
} %
\subcaption{\textbf{Env-B1}.
There are four stationary deterministic policies, from which
there are 3 distinct gain values (\ie $0.67, 1.33, 4$), and
4 distinct bias values (\ie $-3, -1, 0.22, 1.11$).
There exists a single nBw-optimal policy (out of four, hence $1/4=0.250$)
that chooses $a^{01}$ at $s^0$, and $a^{11}$ at $s^1$.
Its bias is of $-1$.
}
\end{subfigure}
\hspace{0.025\textwidth}
\begin{subfigure}{0.575\textwidth}
\resizebox{\textwidth}{!}{
\begin{tikzpicture}[
show background rectangle,
node distance = 3cm and 3cm, on grid,
-latex, %
semithick, %
state/.style={circle, top color = white,  bottom color = stateblue!20,
draw, stateblue, text=black, minimum width = 1cm},
]
\node[state](A)[] {$s^0$};
\node[state](B)[right=of A] {$s^1$};
\node[state](C)[right=of B] {$s^2$};
\path (A) edge [bend left] node[above] {$(a^{00}, 5, 1)$} (B);
\path (A) edge [bend right] node[below] {$(a^{01}, 25, 1)$} (C);
\path (B) edge [bend left] node[above] {$(a^{10}, 50, 1)$} (C);
\path (B) edge [right] node[below] {\tiny{$(a^{11}, 50, 1)$}} (C);
\path (C) edge [dashed, loop right] node[right] {Taxi loop} (C);
\end{tikzpicture}
} %
\subcaption{\textbf{Env-B2}.
The Taxi loop at $s^2$ refers to the 3-state Taxicab environment of
\citet[\secc{3.5}]{hordijk_1985_disc},
where $s^2$ denotes the state number~1 in the original environment description.
The Taxi loop (including $s^2$) forms a recurrent class.
Thus, this environment has 5 states in total, and 2 actions per state,
yielding $32 (=2^5)$ stationary deterministic policies.
There are 8 distinct gain values, ranging from $8.621$ to $13.345$,
and 16 distinct bias values, ranging from $-1.683$ to $35.907$.
There are 4 policies with the maximum gain, whose bias are of $-0.289$ and $16.367$.
There exist 2 nBw-optimal policies (out of 32, hence $2/32=0.0625$)
that choose $a^{00}$ at $s^0$, any action at $s^1$, and
action number 2 (of the original description) at all states in the Taxi loop.
}
\end{subfigure}

\caption{Symbolic diagrams of Env-A1, A2, B1, and B2.
Each node indicates a state, which is labeled with a superscript state-index,
\ie $\setname{S} = \{ s^0, s^1, \ldots, s^{\setsize{S} - 1} \}$.
Each 3-tuple label of solid edges indicates the corresponding action~$a$,
immediate reward $r(s, a, s')$, and transition probability $p(s'|s, a)$.
The action is labeled with 2-digit superscript action-index:
the first digit denotes the state index in which such an action is available, whereas
the second the action number (either $0$ or $1$) in that state.
Here, all above-mentioned bias values are measured from
a deterministic initial state $s_0 = s^0$.
For Env-A3 and B3, refer to the text in \secref{sec:env_nbpwg} and
\figref{fig:symbolic_diagram_b3}, respectively.
}
\label{fig:symbolic_diagram_a12_b12}

\end{figure}

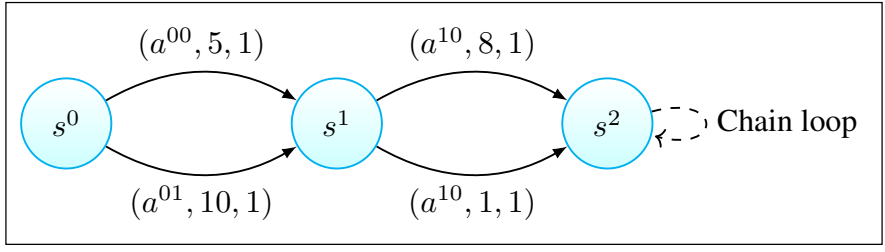
\begin{figure}[t]
\centering

\resizebox{0.85\textwidth}{!}{
\begin{tikzpicture}[
show background rectangle,
node distance = 3cm and 3cm, on grid,
-latex, %
semithick, %
state/.style={circle, top color = white,  bottom color = stateblue!20,
draw, stateblue, text=black, minimum width = 1cm},
]
\node[state](A)[] {$s^0$};
\node[state](B)[right=of A] {$s^1$};
\node[state](C)[right=of B] {$s^2$};
\path (A) edge [bend left] node[above] {$(a^{00}, 5, 1)$} (B);
\path (A) edge [bend right] node[below] {$(a^{01}, 10, 1)$} (B);
\path (B) edge [bend left] node[above] {$(a^{10}, 8, 1)$} (C);
\path (B) edge [bend right] node[below] {$(a^{10}, 1, 1)$} (C);
\path (C) edge [dashed, loop right] node[right] {Chain loop} (C);
\end{tikzpicture}
} %

\caption{\textbf{Env-B3}.
The Chain loop at $s^2$ refers to the 5-state Chain environment
of \citet[\fig{1}]{strens_2000_bfrl},
where $s^2$ denotes the state number~1 in the original environment description.
The Chain loop (including $s^2$) forms a recurrent class.
Thus, there are 7 states in total and 2 actions per state,
yielding $128 (=2^7)$ stationary deterministic policies that induce
32 distinct gain values, ranging from $0.732$ to $3.677$,
and 128 distinct bias values, ranging from $-14.461$ to $15.871$.
There are 4 policies with the maximum gain, whose bias are of
$-14.461, -7.461, -9.461$, and $-2.461$.
There exists a single nBw-optimal policy (out of 128, hence $1/128 \approx 0.008$)
that chooses $a^{01}$ at $s^0$, $a^{10}$ at $s^1$, and
forward actions at all states in the Chain loop.
For more description, refer to the caption of \figref{fig:symbolic_diagram_a12_b12}.
}
\label{fig:symbolic_diagram_b3}

\end{figure}

\subsection{Policy parameterization} \label{sec:polparam}

In order to visualize the parameter space in 2D plane, we limit the number of
policy parameters to two.
That is, $\vecb{\theta} \in \Theta = \real{2}$,
where $\vecb{\theta} = [\theta_0, \theta_1]^\intercal$.
It turns out that this parameterization contains randomized stationary policies
that are very close to the nBw-optimal policy for our target environments
(\secref{sec:env_nbpwg}).

Thus, a randomized stationary policy $\pi \in \piset{S}$ is parameterized as
\begin{equation*}
\pi(a^{s0} | s; \vecb{\theta}) = \sigma(f(s) \theta_0 + \theta_1),
\quad \text{and} \quad
\pi(a^{s1} | s; \vecb{\theta}) = 1 - \pi(a^{s0} | s; \vecb{\theta}),
\quad \forall s \in \setname{S}, \forall \vecb{\theta} \in \Theta,
\end{equation*}
where
\begin{itemize}
\item $a^{s0}$ and $a^{s1}$ are all two actions available at state $s$,
\item $\sigma(x) = 1 / (1 + \exp(-x))$ is the sigmoid function for an input $x$, and
\item $f(s) \in \{ 0, 1, \ldots, \setsize{S} - 1 \}$ is the state feature function
    that returns the state index of $s$.
\end{itemize}
Note that such sigmoid parameterization does not contain stationary policies
that are exactly deterministic (non-randomized),
but contains those very close to being deterministic.

\subsection{Algorithm parameters} \label{sec:algoparam}

Below, we list the default parameter values and settings of
\algrefand{alg:gainbiasbarrier_optim}{alg:gradfisher_sampling}.
\begin{itemize}
\item The policy parameterization is described in \secref{sec:polparam}.
    There are $1681 (=41^2)$ combination of $\theta_0$ and $\theta_1$
    used as initial policy parameter values.
    They come from a discretized version of the parameter space $\Theta$
    in the range of $-10$ to $+10$ with $0.5$ resolution in each dimension.

\item The tolerance for gain-optimization convergence is $\epsilon \gets 10^{-4}$.
    The $\ell_2$-norm is used.

\item The initial barrier parameter is
    $\beta_0 \gets 0.1$ for Env-A, whereas $\beta_0 \gets 100$ for Env-B.
    This is based on tuning experiments described in \secref{sec:barrierparamtuning}.
    The barrier divisor is always $\beta_{\mathrm{div}} \gets 10$.

\item The maximum inner and outer iterations are $n_j \gets 100$ and $n_k \gets 1$,
    respectively.

\item The sample size is $\nxep \gets 16$, which is used in sampling-based approximation.

\item The maximum timestep of an experiment-episode (trial) is set as
    $\txepmax \gets \tmix + 5$, where
    $\tmix \eqdef \max_{\vecb{\theta} \in \tilde{\Theta}} \tmix^{\pi(\vecb{\theta})}$.
    Here, $\tilde{\Theta}$ is a discretized version of $\Theta$
    with $0.1$ resolution in the range of $-10$ to $+10$, and
    $5$ additional timesteps are to partially accomodate the mixing times of
    policies whose parameters are not contained in $\tilde{\Theta}$.
    Each $\tmix^\pi$ is set to the smallest timestep $t$ when
    the entry-wise numerical difference between the transient and stationary
    state distributions satisfies
    $|p_\pi^\star(s) - p_\pi^t(s|s_0)|
    \le \mathrm{tol}_{abs} + \mathrm{tol}_{rel} \times p_\pi^t(s|s_0)$
    for a deterministic initial state $s_0$ and all states $s \in \setname{S}$,
    where $\mathrm{tol}_{abs} = 10^{-6}$ and  $\mathrm{tol}_{rel} = 10^{-5}$.

\item All policy evaluations are exact, hence exact state values~$v$
    and exact state-action values~$q$.
    This can be thought of as having an oracle value function approximator.
    We remark that a naive estimator for $q_b^\pi(s, a)$ in \eqref{equ:bias_q}
    would be a total sum of rewards earned since the action~$a$ is
    executed at state~$s$,
    substracted by the last reward (approximating the gain)
    times the number of timesteps taken while following on a policy $\pi$
    till the end of an experiment-episode.

\item All step lengths are computed using the backtracking linesearch method
    \citep[\alg{9.2}]{boyd_2004_cvxopt}
    with an initial step length of~$1$ and a maximum of~$100$ iterations.

\item The gain optimization is carried out exactly
    since we focus on bias-only and bias-barrier optimization.
    This is to suppress the number of approximation layers, as well as
    to isolate the cause of improvement to solely the proposed expressions of
    the bias gradient \eqref{equ:grad_bias}, the bias Fisher \eqref{equ:fisher_b_sampling},
    and the bias barrier \eqref{equ:biasbarrier_opt}.
    Note that the gain policy gradient is given by
    \begin{equation}
    \nabla v_g(\vecb{\theta})
    = \sum_{s \in \setname{S}} \sum_{a \in \setname{A}}
        p_{\pi}^\star(s)\ \pi(a|s; \vecb{\theta}) q_b^\pi(s, a)
        \nabla \log \pi(a|s; \vecb{\theta}),
    \qquad \forall \vecb{\theta} \in \Theta,
    \label{equ:gaingrad_simple}
    \end{equation}
    as shown by \citet[\thm{1}]{sutton_2000_pgfnapprox}.
\end{itemize}

\subsection{Discounted-reward policy gradient methods}
\label{sec:xprmtsetup_disrew}

Here, we provide the setup for the discounted-reward policy gradient experiments,
to which we compare our proposed nBw-optimal policy gradient method
(\secref{sec:bwoptim_algo}).
Such a comparison is carried out because the nBw-optimal policies are also
attainable via discounted settings whenever the discount factor is proper,
as discussed in \secref{sec:rwork_nbwpg}.
We begin with a brief overview of the discounted method.

The expected total discounted reward with a discount factor $\gamma \in [0, 1)$
is defined as
\begin{equation}
v_\gamma(\pi, s) \eqdef
    \lim_{\tmax \to \infty}
    \E{A_t \sim \pi(\cdot| s_t), S_t \sim p(\cdot|s_t, a_t)}{
        \sum_{t = 0}^{\tmax - 1} \gamma^t r(S_t, A_t) \Big| S_0 = s, \pi},
\quad \forall s \in \setname{S}, \forall \pi \in \piset{S}.
\label{equ:disrew_v}
\end{equation}
Stacking $v_\gamma$ of all states $s \in \setname{S}$ gives us a vector
$\vecb{v}_\gamma(\pi) \in \real{\setsize{S}}$.
Thus, we can write
\begin{equation}
\vecb{v}_\gamma(\pi) = \ppimat^\gamma \rpivec,
\quad \text{where}\
\ppimat^\gamma \eqdef
\underbrace{
    \lim_{\tmax \to \infty} \sum_{t = 0}^{\tmax - 1} (\gamma \ppimat)^t
    = (\mat{I} - \gamma \ppimat)^{-1}
}_\text{See \citet[\cor{C.4}]{puterman_1994_mdp}},\
\text{and}\ \rpivec \in \real{\setsize{S}}.
\label{equ:disrew_v_vector}
\end{equation}
Each $[s_0, s]$-component of $\ppimat^\gamma \in \real{\setname{S} \times \setname{S}}$
indicates the \emph{improper} discounted state probability.
That is,
\begin{equation*}
p_\pi^\gamma(s|s_0)
\eqdef \lim_{\tmax \to \infty} \sum_{t = 0}^{\tmax - 1}
\gamma^t \underbrace{\prob{S_t = s | s_0, \pi}}_{p_\pi^t(s|s_0)},\
\text{which is improper as}\
\sum_{s \in \setname{S}} p_\pi^\gamma(s|s_0) = \frac{1}{1 - \gamma} \ne 1.
\end{equation*}

A policy $\pi_\gamma^* \in \piset{S}$ is $\gamma$-discounted optimal if it satisfies
$v_\gamma(\pi_\gamma^*, s_0) \ge v_\gamma(\pi, s_0),
\forall \pi \in \piset{S}, \forall s_0 \in \setname{S}$.
As an approximation to this condition
(involving partial ordering in all states $s_0 \in \setname{S}$),
a discounted-reward policy gradient method maximizes the discounted value of
a policy with respect to some initial state distribution $\isd$, namely
$\hat{\pi}_\gamma^* \in \argmax_{\pi \in \piset{S}} \E{\isd}{v_\gamma(\pi, S_0)}$.
Since such an objective depends on $\isd$,
the resulting optimal policy is said to be \emph{non}-uniformly optimal
\citep[\deff{2.1}]{altman_1999_cmdp}.

In a similar fashion to \eqref{equ:polparam_update}, the policy parameter update
needs a gradient of $v_\gamma$ along with its preconditioner, \eg a Fisher matrix.
\cite{sutton_2000_pgfnapprox} showed that the gradient is given by
\begin{align}
\nabla v_\gamma(\vecb{\theta}, s_0)
& = \sum_{s \in \setname{S}} p_\pi^\gamma(s|s_0) \sum_{a \in \setname{A}}
    q_\gamma^\pi(s, a) \nabla \pi(a|s; \vecb{\theta}) \notag \\
& = \sum_{t = 0}^{\infty} \sum_{s \in \setname{S}} p_\pi^t(s|s_0) \bigg[
    \gamma^t \underbrace{
        \sum_{a \in \setname{A}} q_\gamma^\pi(s, a) \nabla \pi(a|s; \vecb{\theta})
        }_\text{The so-called immediate gradient} \bigg]
\label{equ:polgrad_disrew2}\\
& = \frac{1}{1 - \gamma}\sum_{s \in \setname{S}}
    \underbrace{(1 - \gamma) p_\pi^\gamma(s|s_0)}_\text{a proper probability}
    \sum_{a \in \setname{A}} q_\gamma^\pi(s, a) \nabla \pi(a|s; \vecb{\theta}),
\qquad \forall \vecb{\theta} \in \Theta, \forall s_0 \in \setname{S},
\label{equ:polgrad_disrew}
\end{align}
where the last expresion enables approximation by sampling from
the (proper, normalized) discounted state distribution,
\ie $(1 - \gamma) p_\pi^\gamma(s|s_0)$.
Here, $q_\gamma^\pi$ denotes the discounted-reward state-action value.
Furthermore, \cite{bagnell_2003_covps, peters_2003_rlhum} introduced
the following discounted-reward Fisher matrix
(for preconditioning the gradient \eqref{equ:polgrad_disrew}),
for all $\vecb{\theta} \in \Theta$ and all $s_0 \in \setname{S}$,
\begin{equation}
\fdamat(\vecb{\theta}, s_0) \eqdef
\sum_{s \in \setname{S}}
\underbrace{(1 - \gamma) p_\pi^\gamma(s|s_0)}_\text{a proper probability}
\underbrace{
    \sum_{a \in \setname{A}} \pi(a|s; \vecb{\theta})
    \nabla \log \pi(a | s; \vecb{\theta}) \nabla^\intercal \log \pi(a | s; \vecb{\theta})
}_\text{The action Fisher $\famat(\vecb{\theta}, s)$ defined in \eqref{equ:fisher_gain}},
\label{equ:fisher_disrew}
\end{equation}
which follows the same pattern as the gain and bias Fisher matrices (\secref{sec:natgrad}).
That is, the Fisher matrix of an optimality criterion is defined as
the action Fisher \eqref{equ:fisher_gain} weighted by
the corresponding state distribution.
This can also be seen as following the pattern of policy value formulas,
\eg \eqref{equ:disrew_v_vector}.

The normalization on $p_\pi^\gamma(s|s_0)$ in \eqrefand{equ:polgrad_disrew}{equ:fisher_disrew}
implies that $\prob{S_t = s | s_0, \pi}$ is weighted by $(1 - \gamma) \gamma^t$
at each timestep $t = 0, 1, \ldots$.
This weight turns out to be equal to the probability that the Markov process
terminates (``first success'') after $t+1$ timesteps (``trials''),
yielding a sequence of $\ell = t+1$ states.
This indeed follows a geometric distribution $Geo(p = 1 - \gamma)$ that gives
\begin{equation*}
\prob{L = \ell} = (1 - p)^{\ell-1} p = (1 - (1 - \gamma))^{(t + 1) - 1} (1 - \gamma)
= \gamma^t (1 - \gamma),
\quad \text{for $\ell=1, 2, \ldots$}.
\end{equation*}
To sample a \emph{single} state from $(1 - \gamma) p_\pi^\gamma$,
one needs to sample a sequence length $\ell$ from $Geo(p = 1 - \gamma)$,
then run the process under $\pi$ from $t=0$ to $\ell - 1$.
Finally, the state $S_{\ell} \sim p(\cdot|s_{\ell -1}, a_{\ell -1})$ is
the desired state sample \citep[\alg{3}]{kumar_2020_zodpg}.
A more sample-efficient technique is proposed by \cite{thomas_2014_biasnac}.
It leverages the expression in \eqref{equ:polgrad_disrew2} so that
$\nabla v_\gamma$ is estimated using state samples
$S_t \sim p_\pi^t$ drawn from undiscounted state distributions from
$t=0$ till some finite timestep (truncating the infinite-horizon model).
Note that in \eqref{equ:polgrad_disrew2}, $\gamma$ is used to discount the immediate gradient.

We consider discounting as an approximation technique towards
maximizing the average reward (gain) objective,
as in $\lim_{\gamma \to 1} (1 - \gamma) v_\gamma(\pi) = v_g(\pi)$,
see \citet[\cor{8.2.5}]{puterman_1994_mdp}.
This is suitable for environments without \emph{inherent} notion of discounting.
Therefore, the \emph{exact} discounted reward experiments in this work
maximize the scaled discounted value, \ie $(1 - \gamma) v_\gamma(\pi)$.
The multiplication by $(1 - \gamma)$ gives expressions that enable sampling for
the gradients \eqref{equ:polgrad_disrew} and the Fisher matrix \eqref{equ:fisher_disrew}.
The factor $(1 - \gamma)$ also helps avoid numerical issues because of
a very large range of $v_\gamma(\pi)$ across all policies $\pi$
(whenever $\gamma$ is very close to~1).

For the exact discounted reward experiments in \secref{sec:nbwpg_xprmtresult},
we applied $\gamma$ values shown in \tblref{tbl:gamma_vals}.
\begin{table}[t]
\centering
\caption{The discount factors $\gamma$ used in three scenarios of
exact discounted reward experiments.
Here, $\gammabw$ denotes the Blackwell's discount factor as in \eqref{equ:bwoptim_v}.
For environment descriptions, refer to \secref{sec:env_nbpwg}.
}
\label{tbl:gamma_vals}

\begin{tabular}{@{}ccccccc@{}}
\toprule
\multirow{2}{*}{Scenarios} & \multicolumn{6}{c}{Environments} \\
\cmidrule(l){2-7} &
  \multicolumn{1}{c}{A1} &
  \multicolumn{1}{c}{A2} &
  \multicolumn{1}{c}{A3} &
  \multicolumn{1}{c}{B1} &
  \multicolumn{1}{c}{B2} &
  \multicolumn{1}{c}{B3} \\
\toprule
\multicolumn{1}{l}{$\gamma < \gammabw$} &
  \multicolumn{1}{l}{$0.50$} &
  \multicolumn{1}{l}{$0.20$} &
  \multicolumn{1}{l}{$0.00$} &
  \multicolumn{1}{l}{$0.00$} &
  \multicolumn{1}{l}{$0.30$} &
  \multicolumn{1}{l}{$0.50$} \\
\midrule
\multicolumn{1}{l}{$\gamma \approx \gammabw$} &
  \multicolumn{1}{l}{$0.95$} &
  \multicolumn{1}{l}{$0.55$} &
  \multicolumn{1}{l}{$0.05$} &
  \multicolumn{1}{l}{$0.05$} &
  \multicolumn{1}{l}{$0.80$} &
  \multicolumn{1}{l}{$0.85$} \\
\midrule
\multicolumn{1}{l}{$\gammabw < \gamma \approx 1$} &
  \multicolumn{6}{c}{$0.99999$} \\
\bottomrule
\end{tabular}
\end{table}

\section{Experimental results}
\label{sec:nbwpg_xprmtresult}

Our primary experimental aim is to validate and gain insights into
the fundamental mechanisms of our algorithms,
rather than demonstrating performance on a large RL problem.
To this end, we consider a small but sufficient setup of benchmark environments
and a policy parameterized as a single neuron with two weights,
\ie $\vecb{\theta} = [\theta_0, \theta_1]^\intercal \in \real{2}$,
as described further in \secref{sec:xprmt_setup_nbwpg}.
This allows us to visualize the algorithm performance across
a large portion of the policy space via enumeration.

\begin{figure}[t!]
\centering

\begin{subfigure}{0.1375\textwidth}
\includegraphics[width=\textwidth]{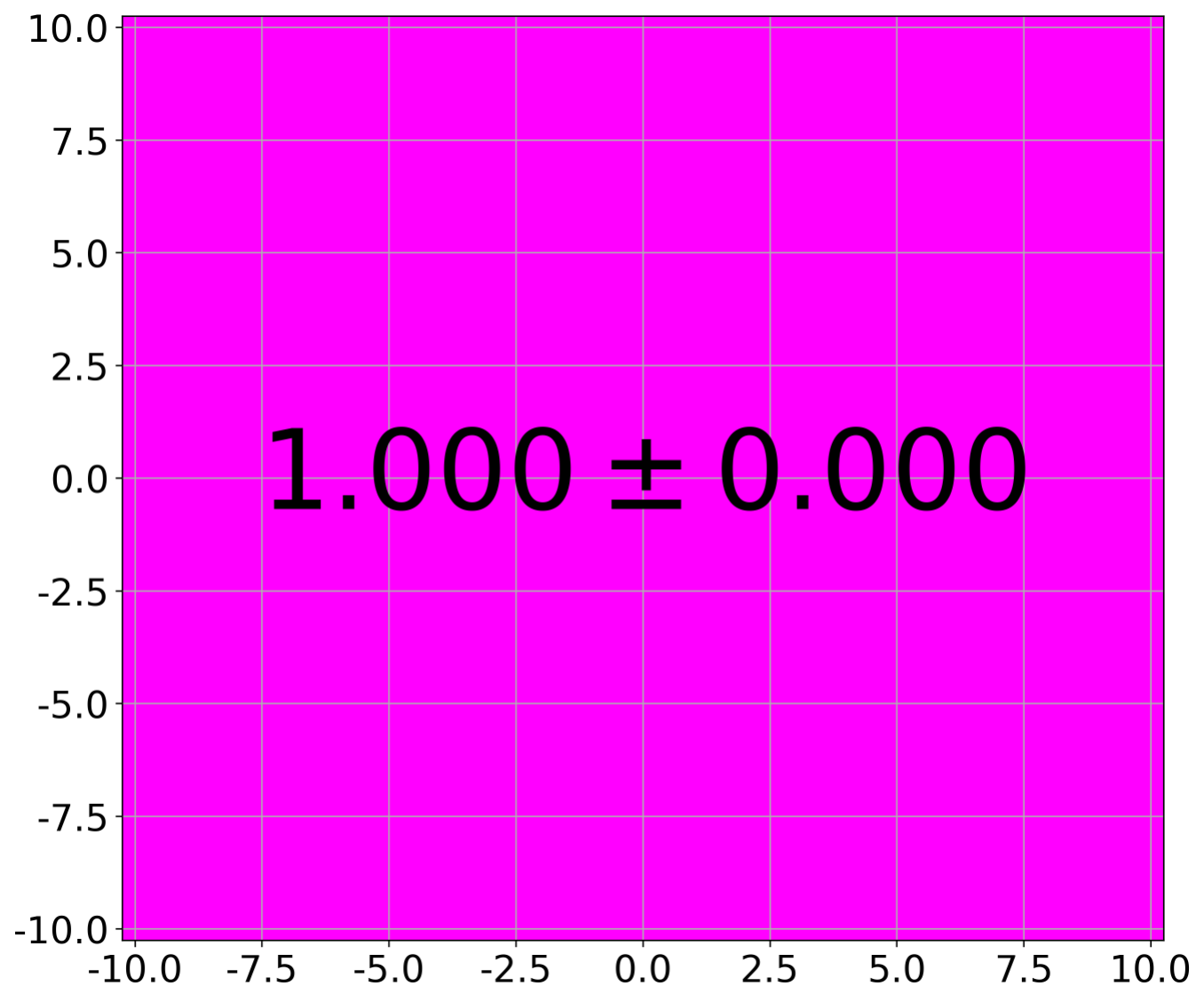}
\end{subfigure}
\begin{subfigure}{0.1375\textwidth}
\includegraphics[width=\textwidth]{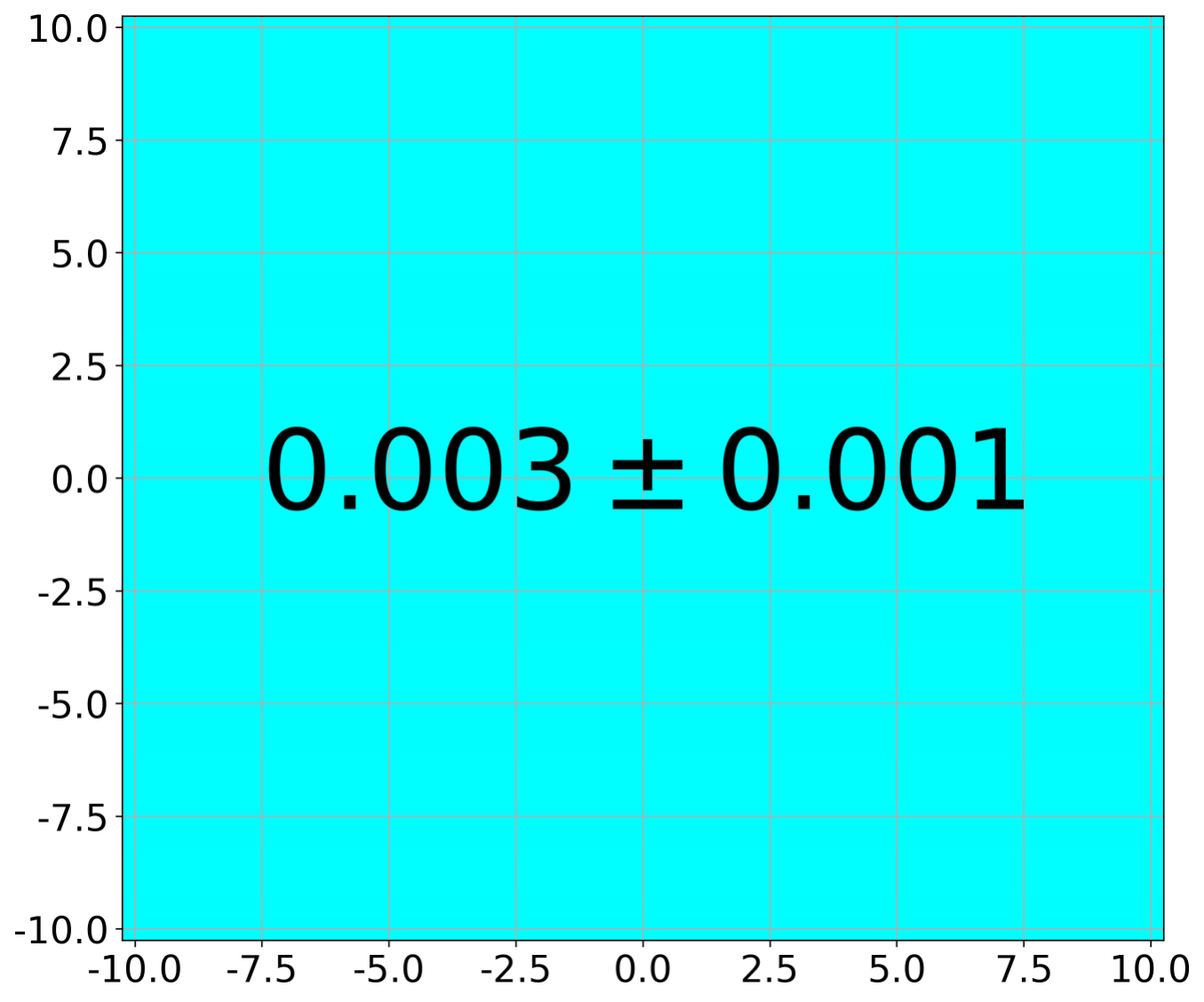}
\end{subfigure}
\begin{subfigure}{0.1375\textwidth}
\includegraphics[width=\textwidth]{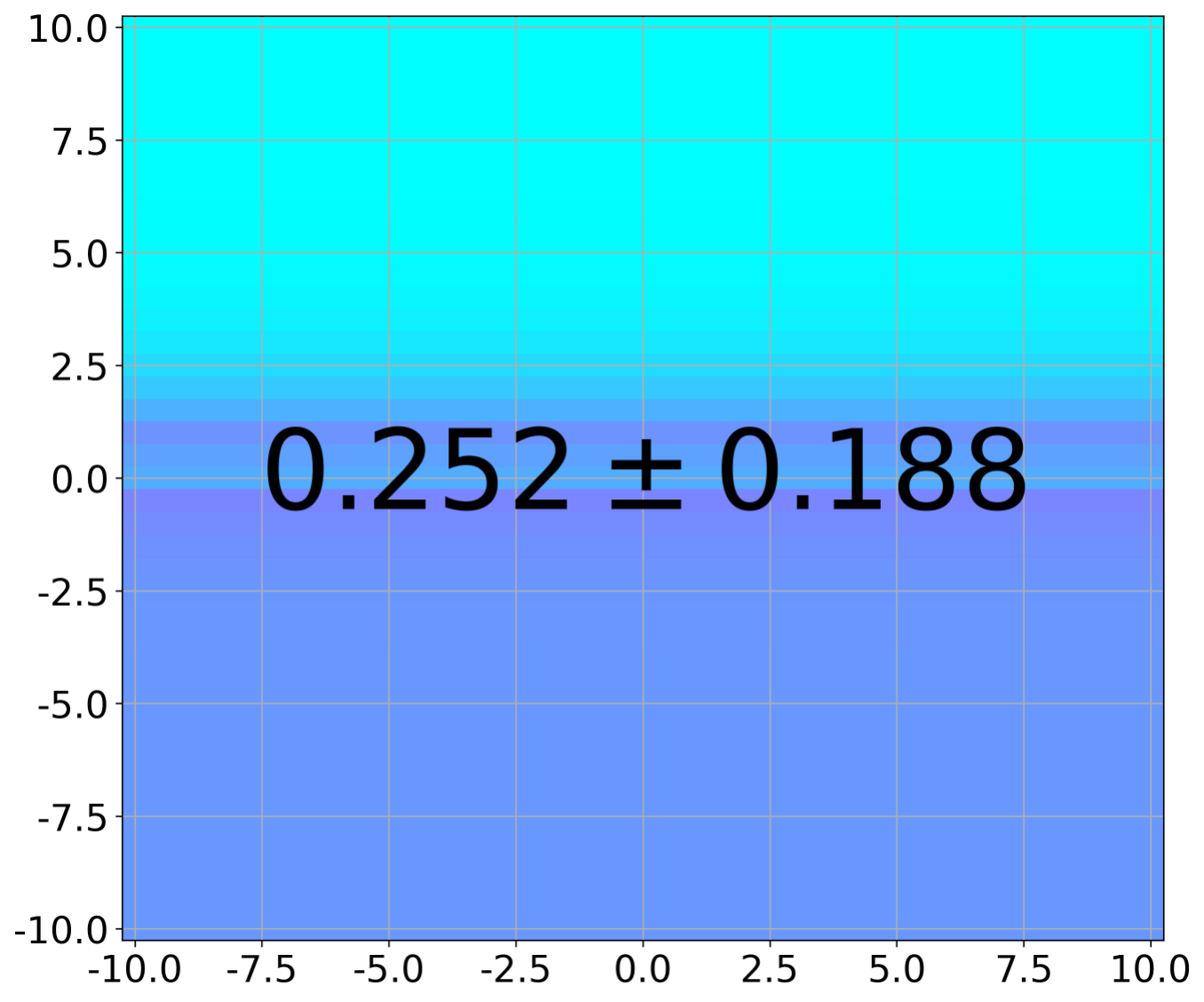}
\end{subfigure}
\begin{subfigure}{0.1375\textwidth}
\includegraphics[width=\textwidth]{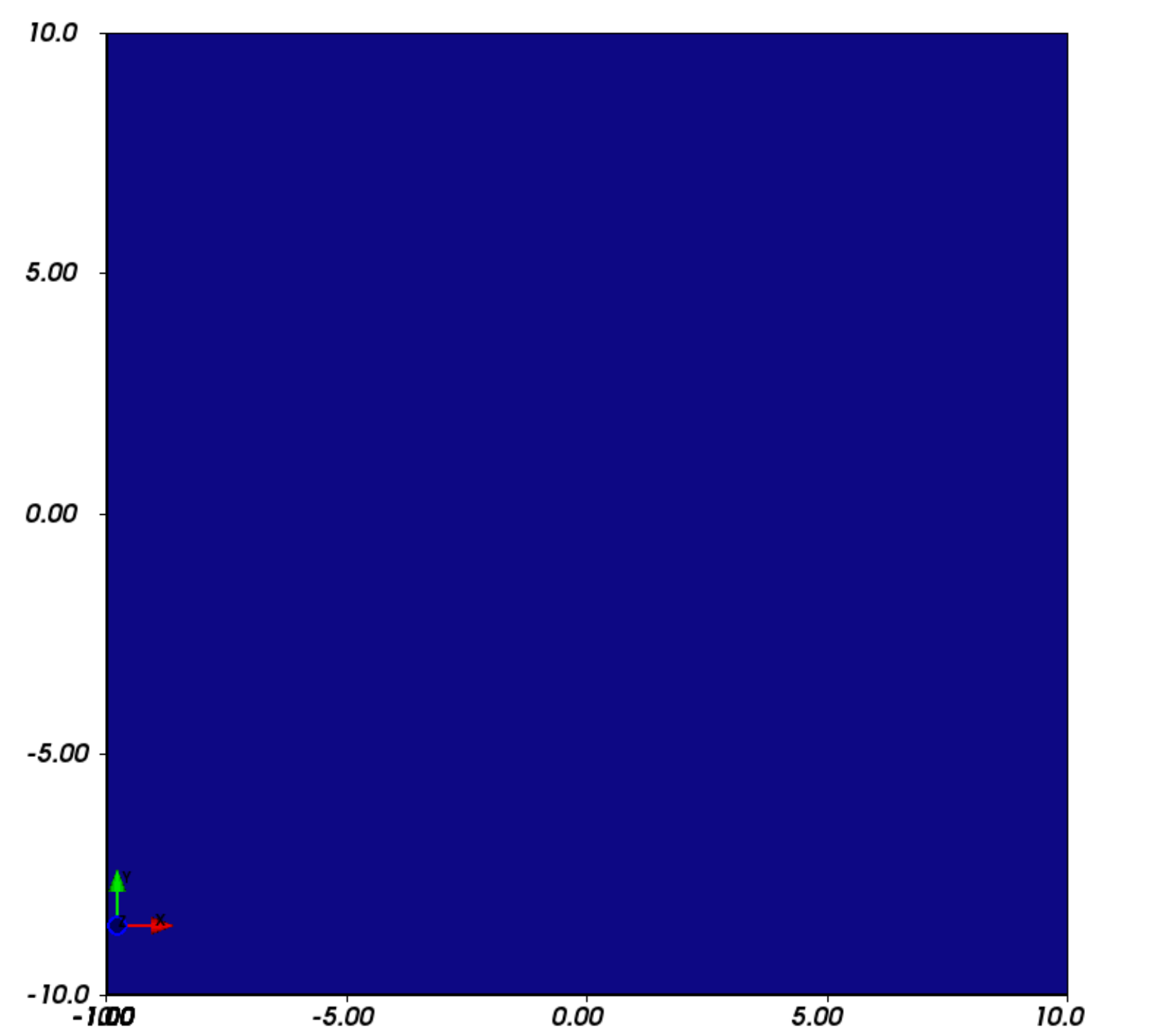}
\end{subfigure}
\begin{subfigure}{0.1375\textwidth}
\includegraphics[width=\textwidth]{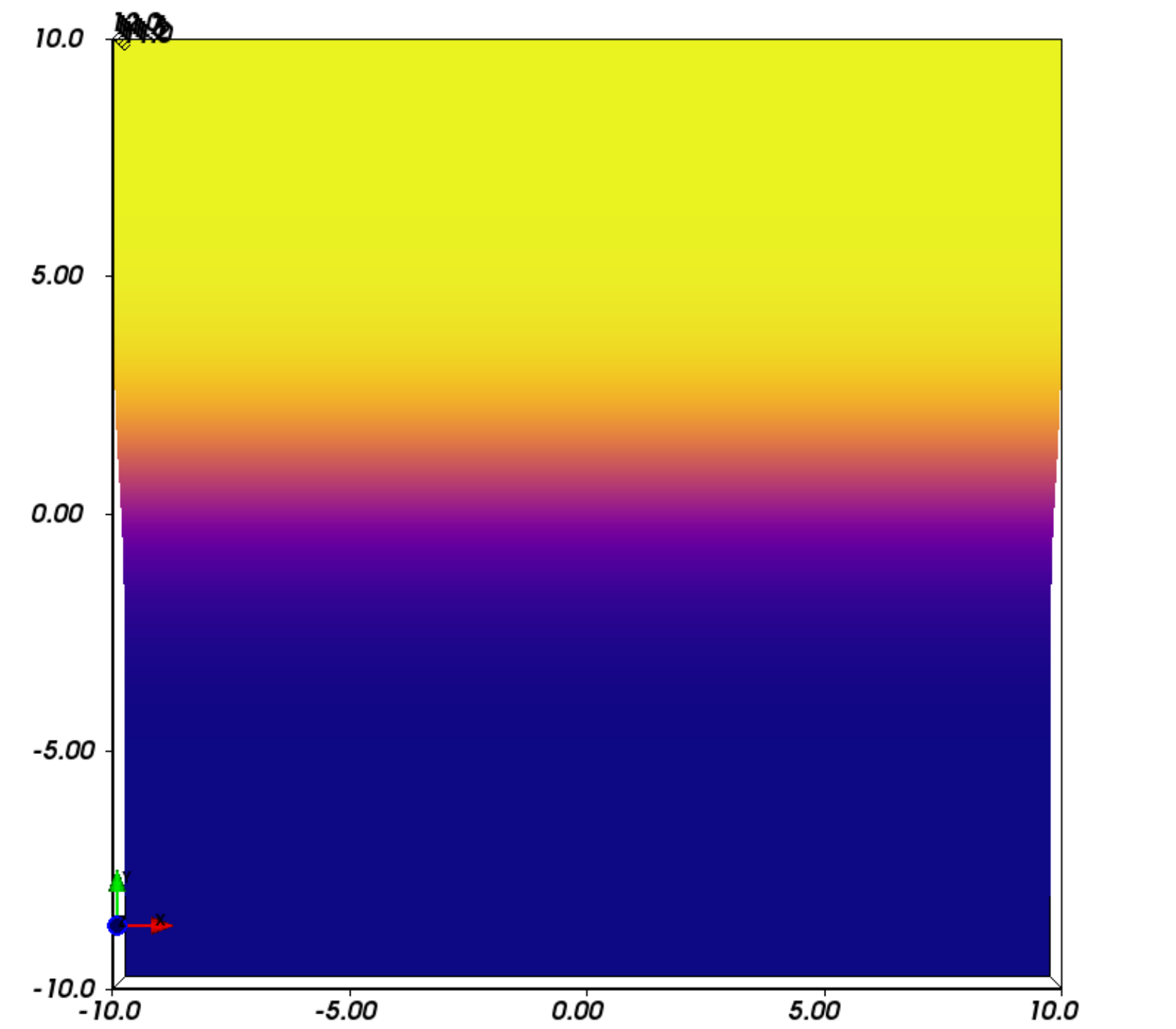}
\end{subfigure}
\begin{subfigure}{0.1375\textwidth}
\includegraphics[width=\textwidth]{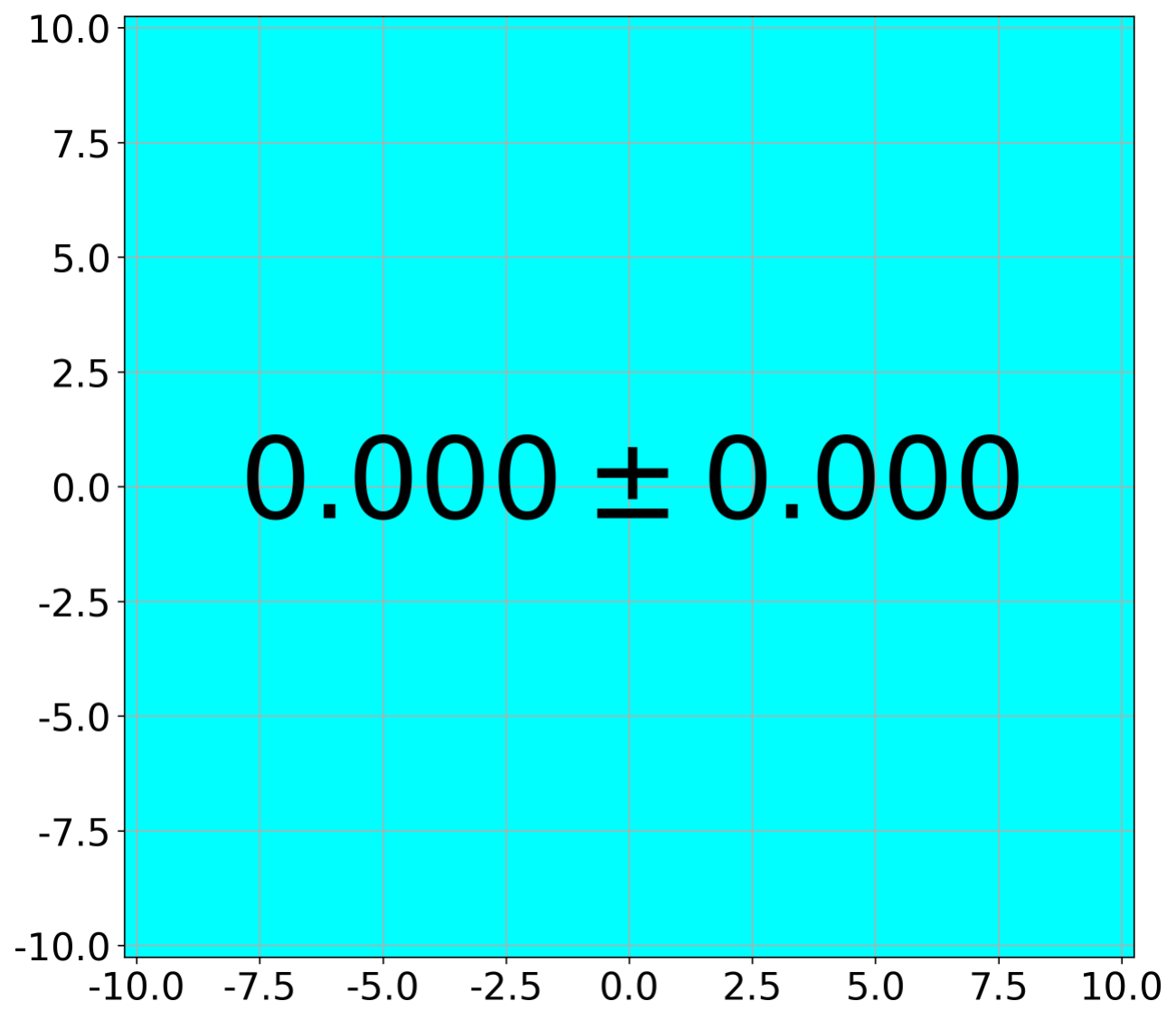}
\end{subfigure}
\begin{subfigure}{0.1375\textwidth}
\includegraphics[width=\textwidth]{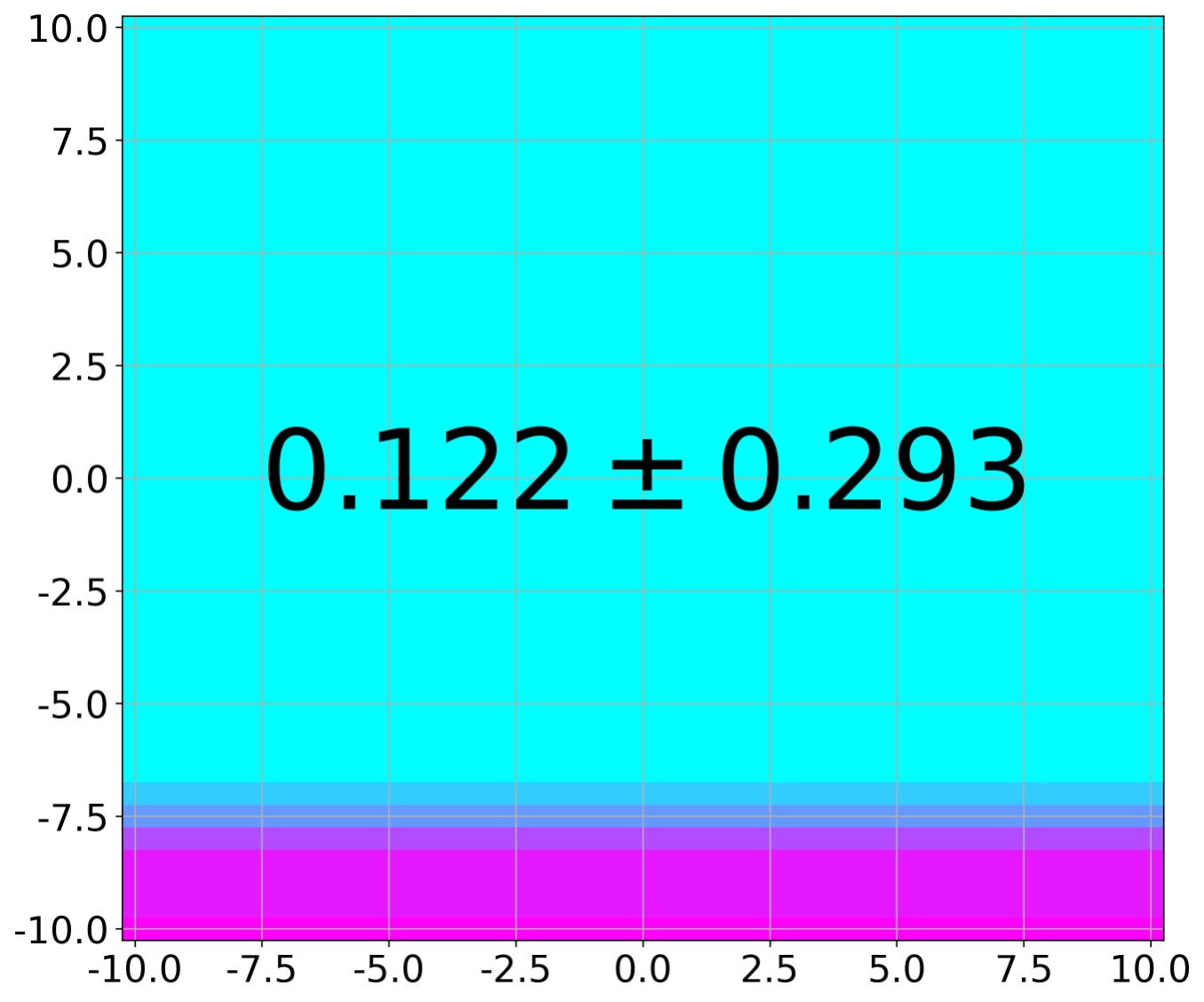}
\end{subfigure}
\begin{subfigure}{0.1375\textwidth}
\includegraphics[width=\textwidth]{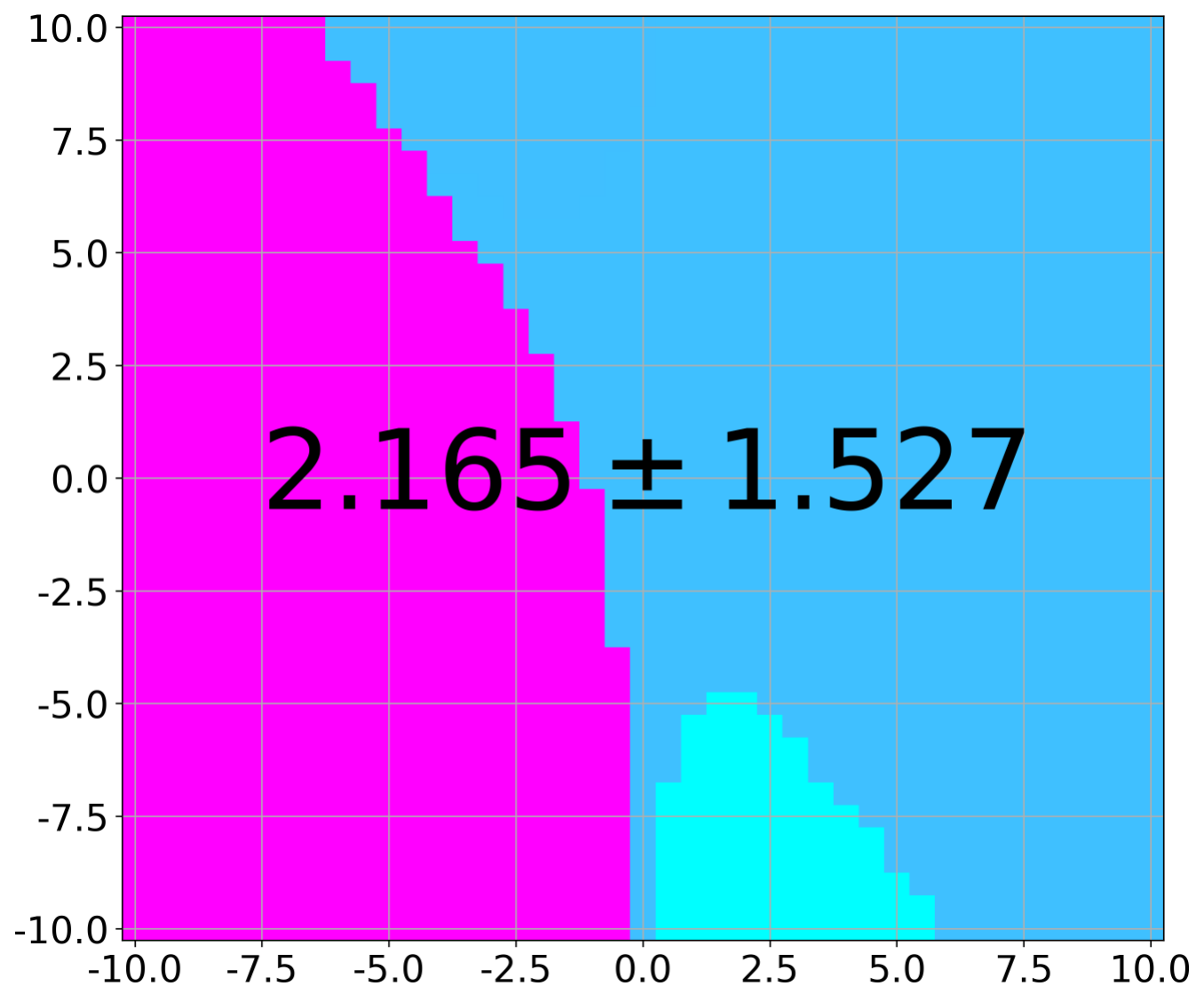}
\end{subfigure}
\begin{subfigure}{0.1375\textwidth}
\includegraphics[width=\textwidth]{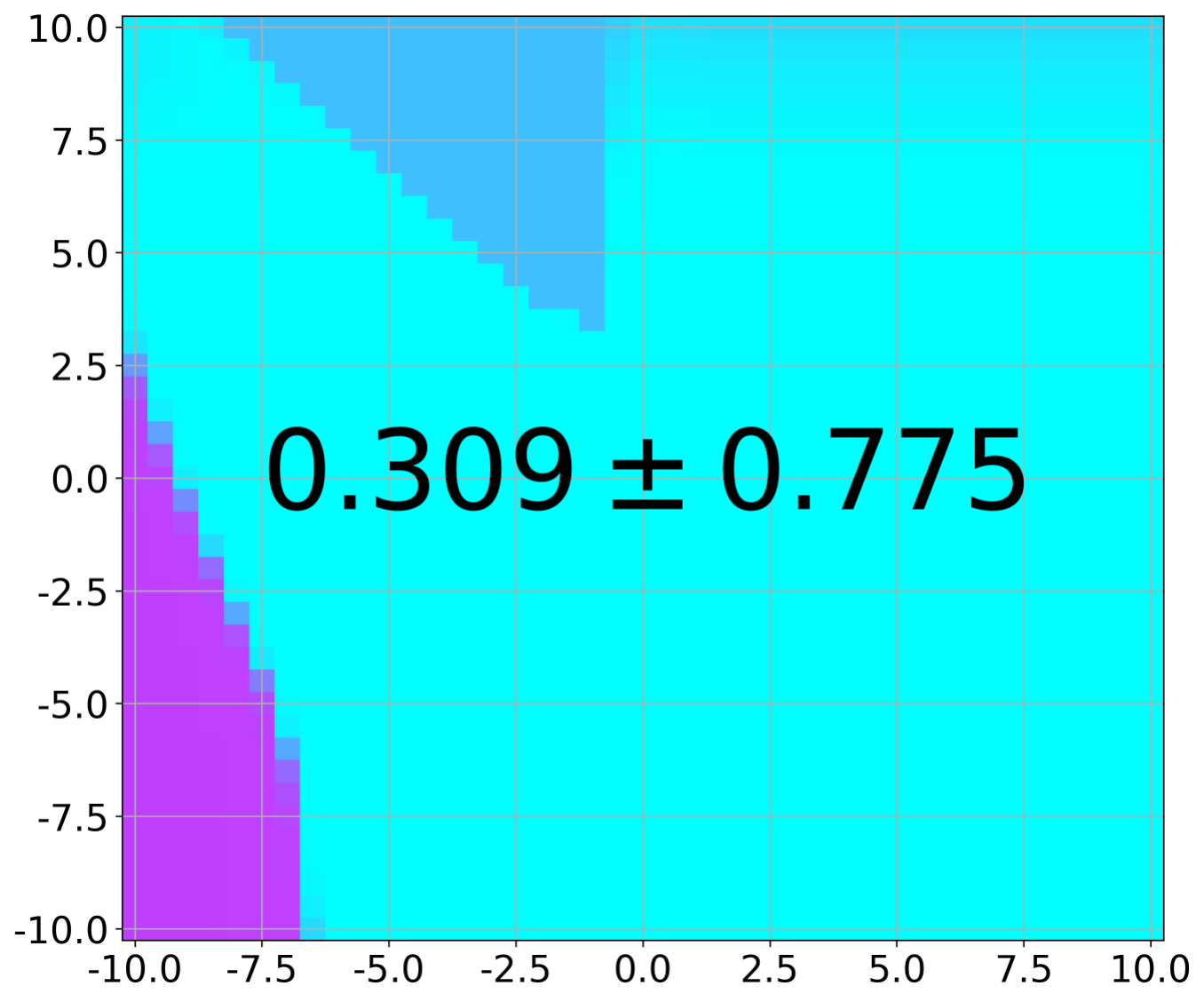}
\end{subfigure}
\begin{subfigure}{0.1375\textwidth}
\includegraphics[width=\textwidth]{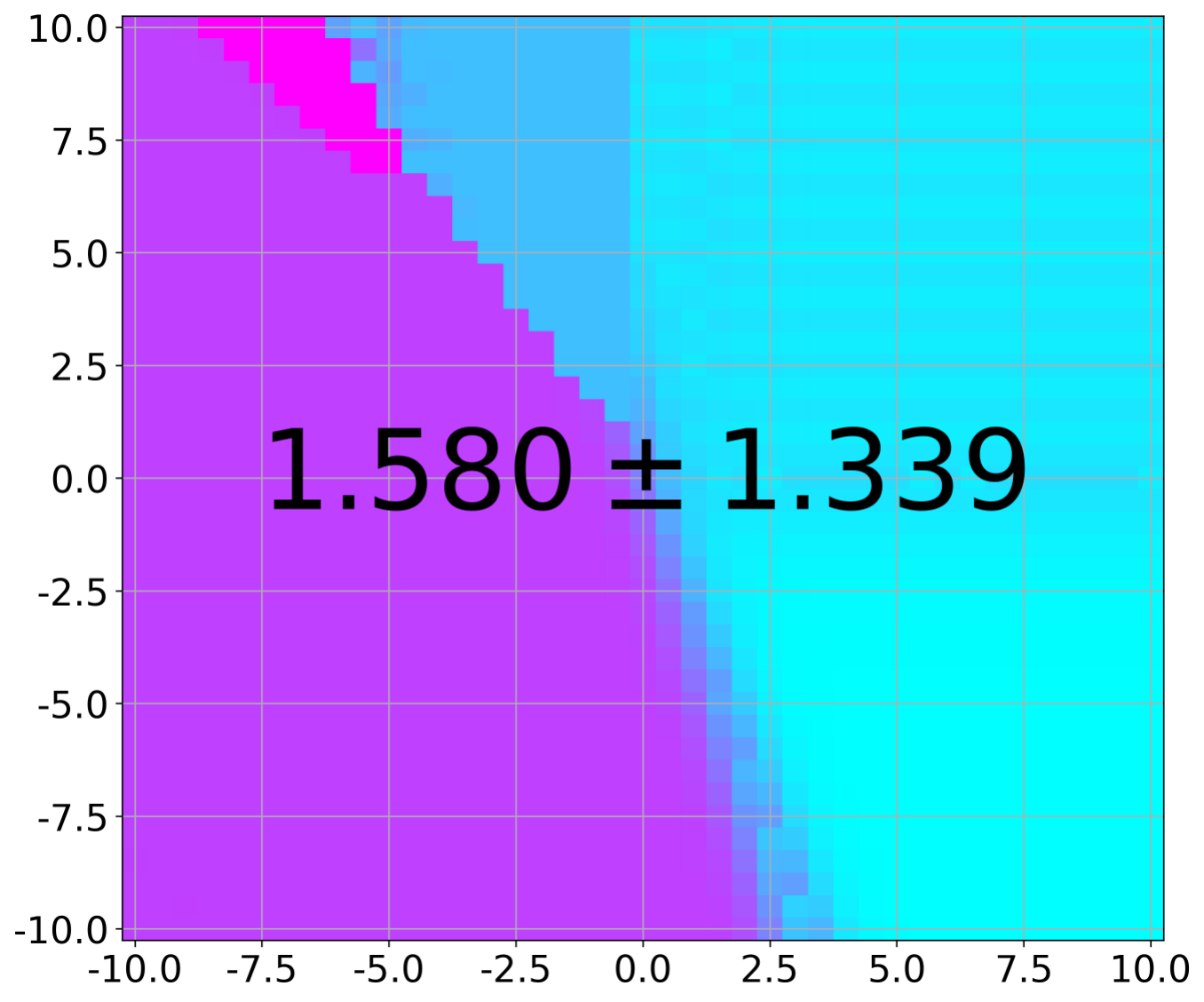}
\end{subfigure}
\begin{subfigure}{0.1375\textwidth}
\includegraphics[width=\textwidth]{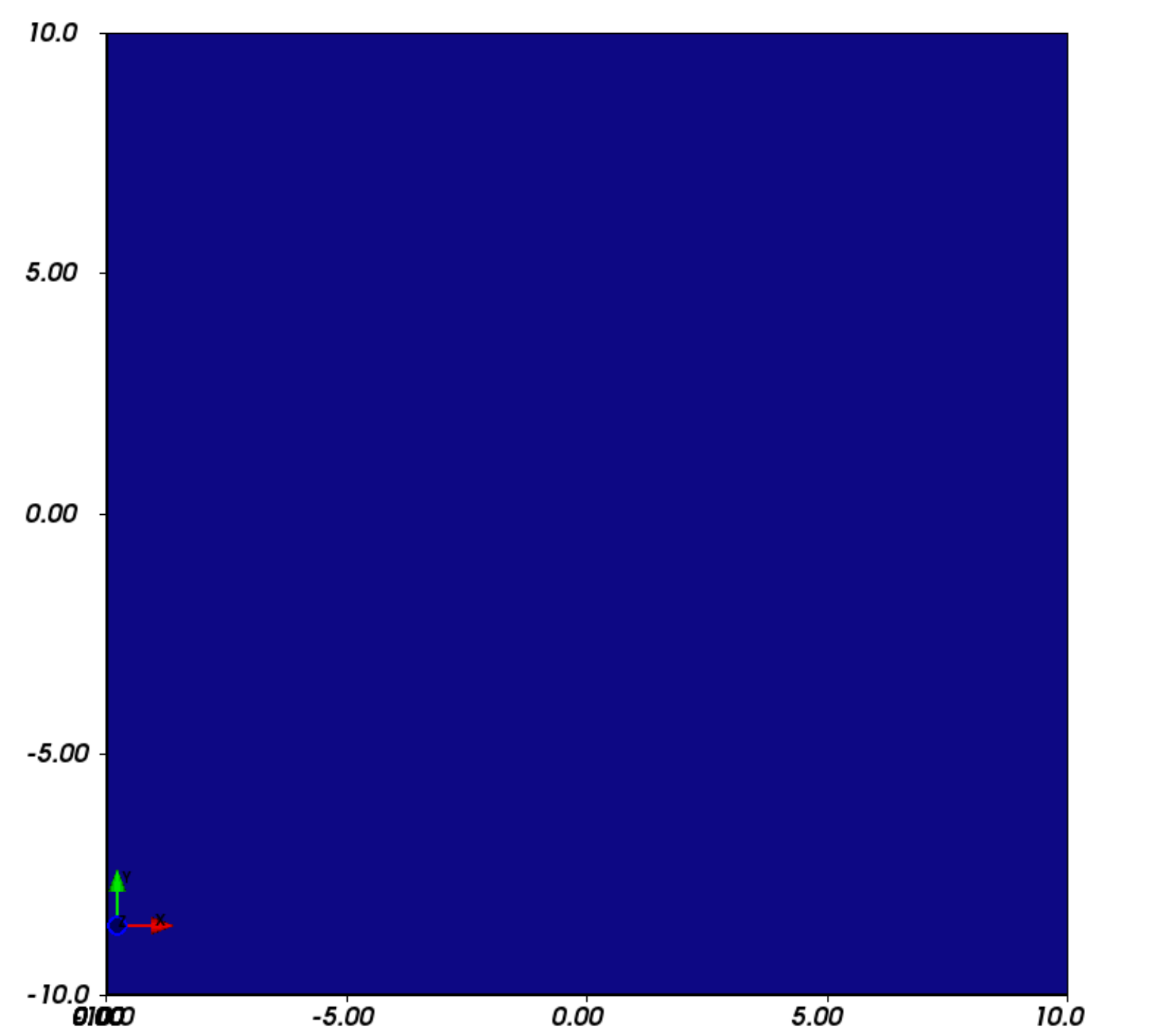}
\end{subfigure}
\begin{subfigure}{0.1375\textwidth}
\includegraphics[width=\textwidth]{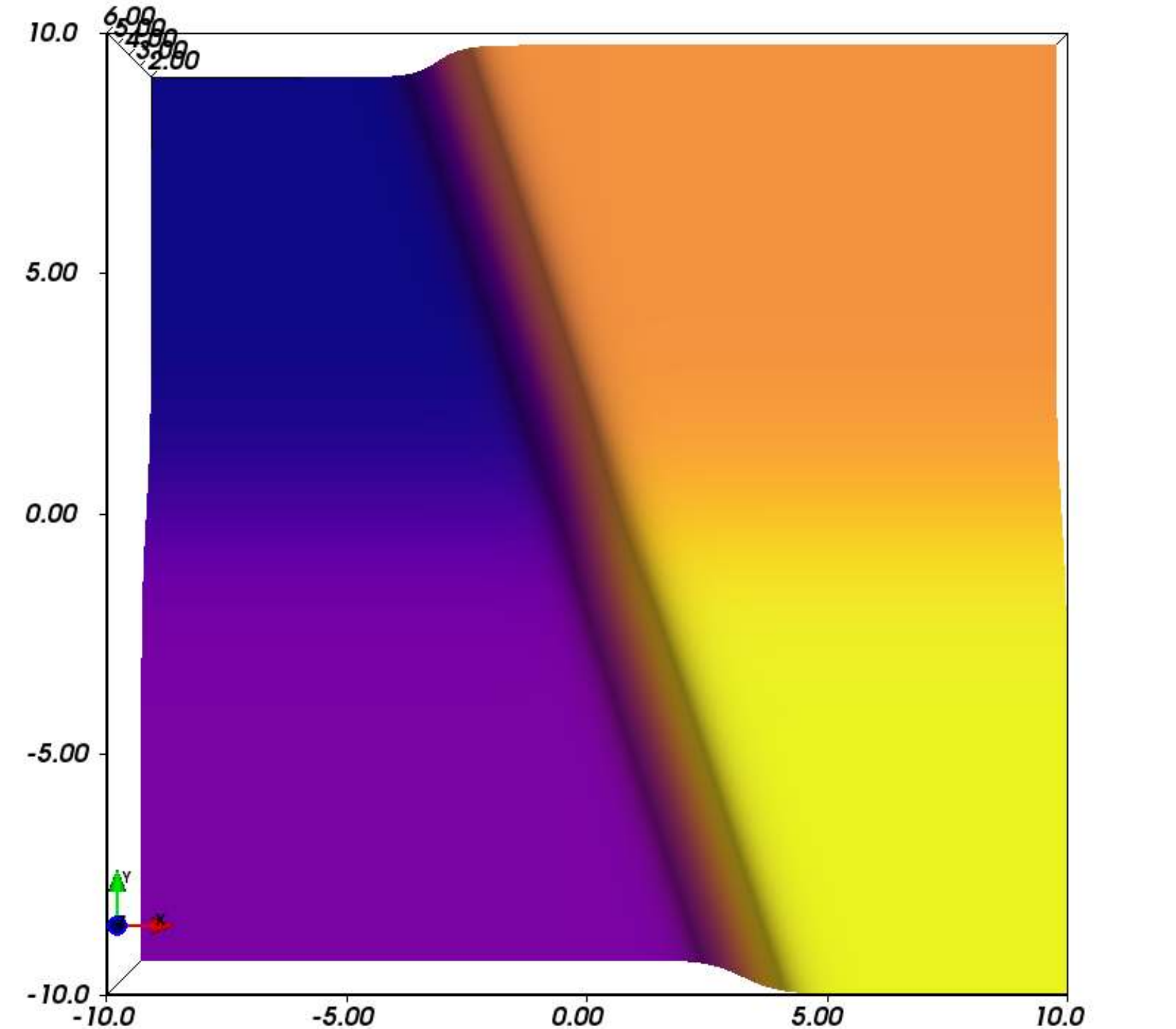}
\end{subfigure}
\begin{subfigure}{0.1375\textwidth}
\includegraphics[width=\textwidth]{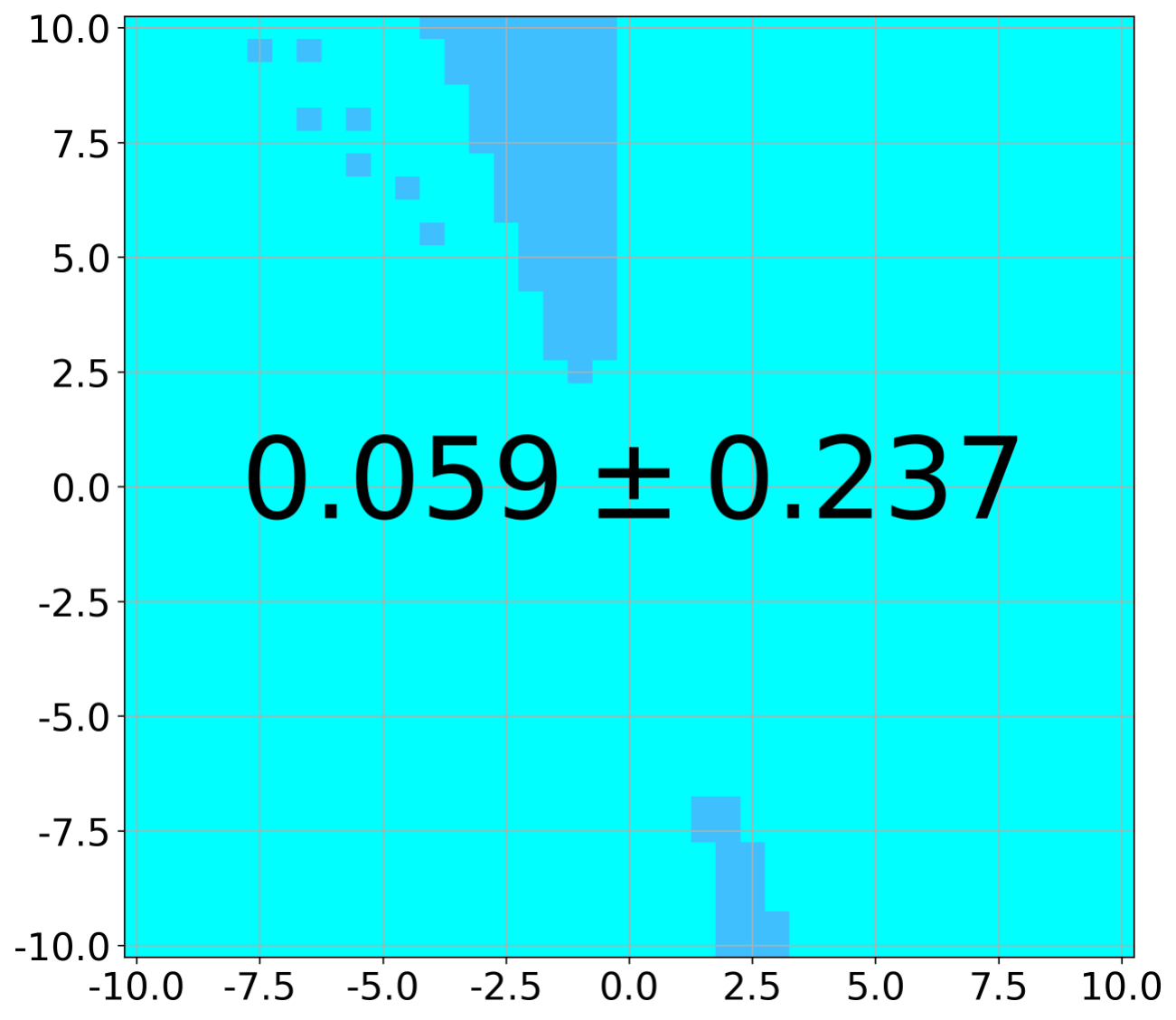}
\end{subfigure}
\begin{subfigure}{0.1375\textwidth}
\includegraphics[width=\textwidth]{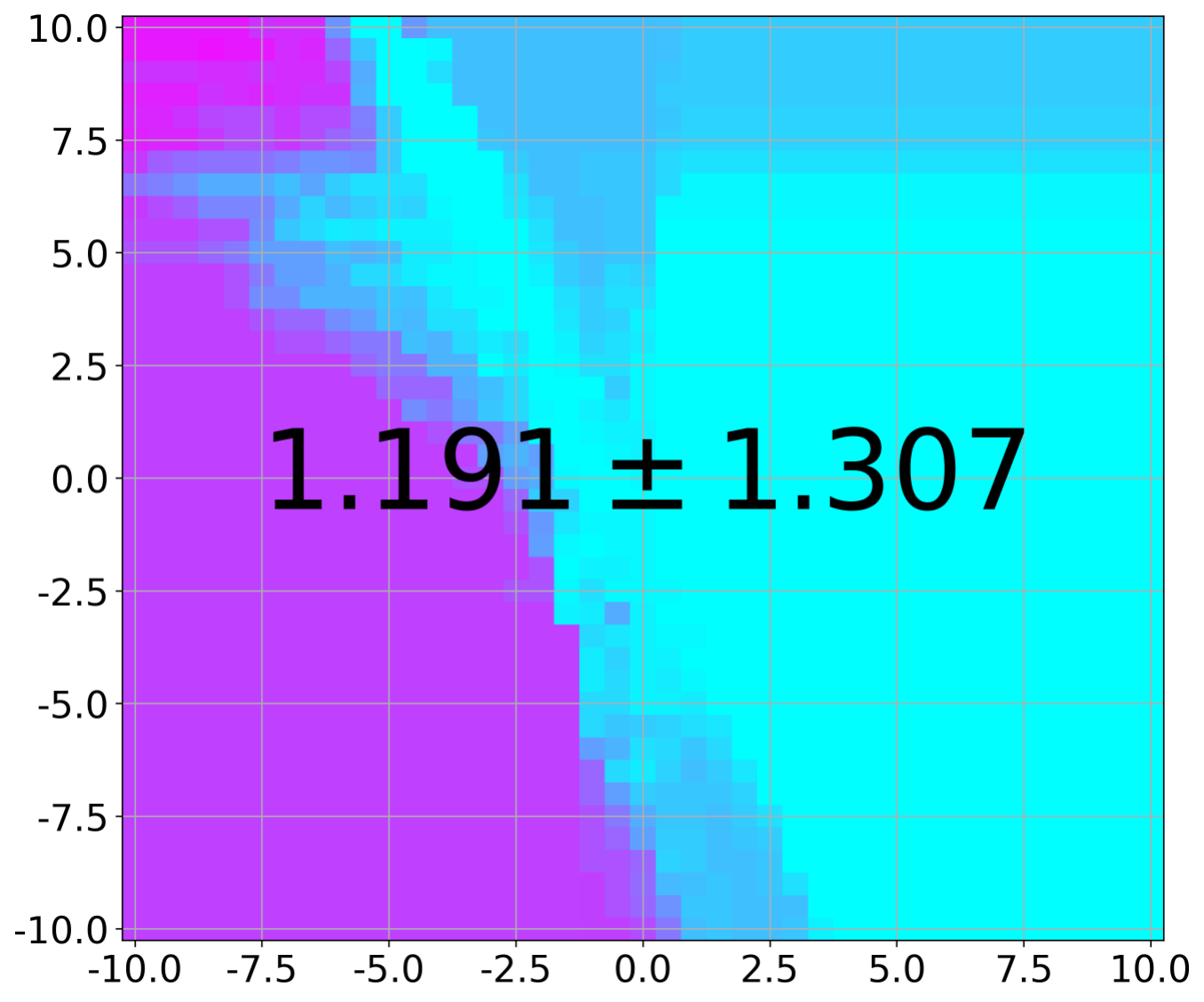}
\end{subfigure}
\begin{subfigure}{0.1375\textwidth}
\includegraphics[width=\textwidth]{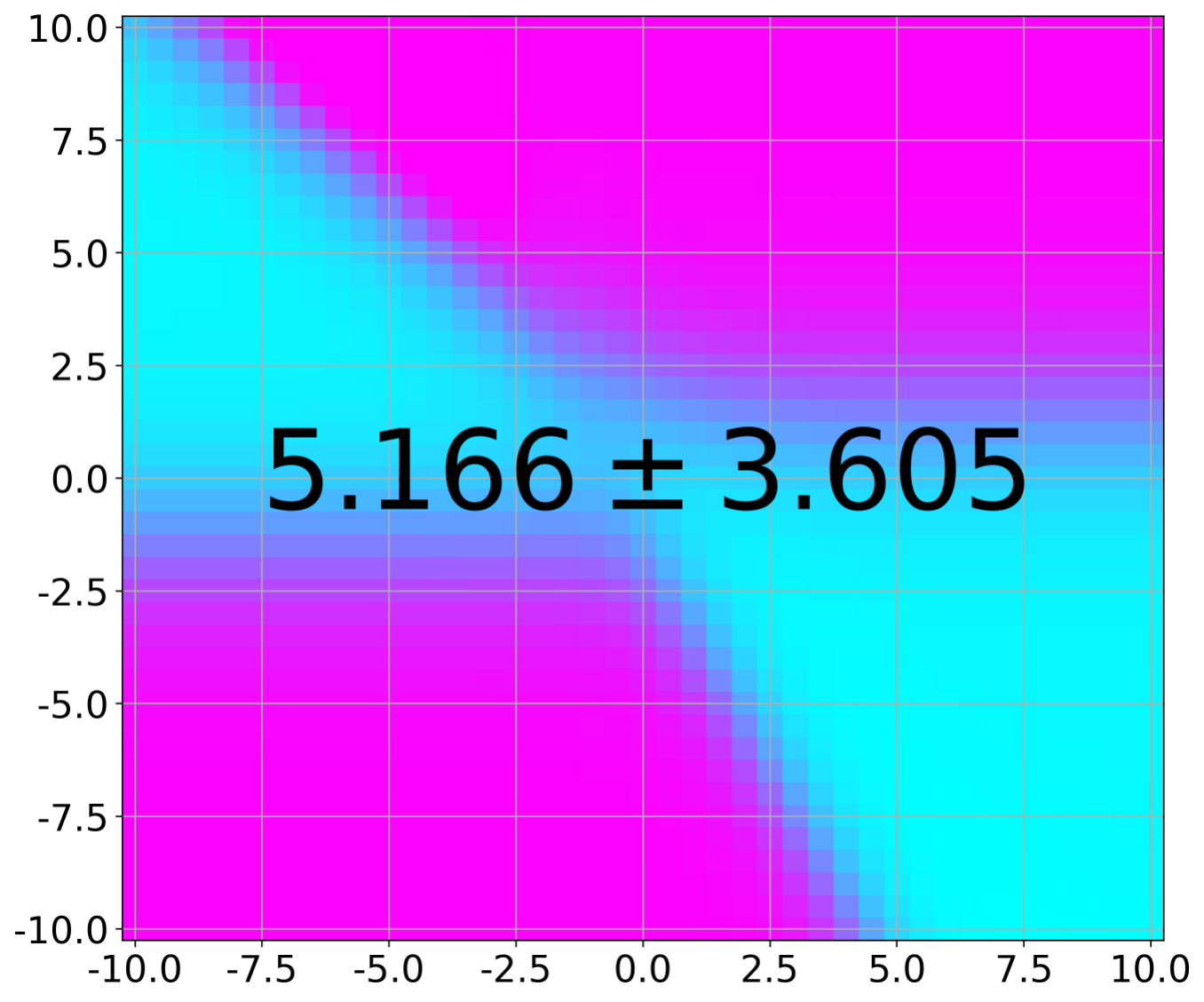}
\end{subfigure}
\begin{subfigure}{0.1375\textwidth}
\includegraphics[width=\textwidth]{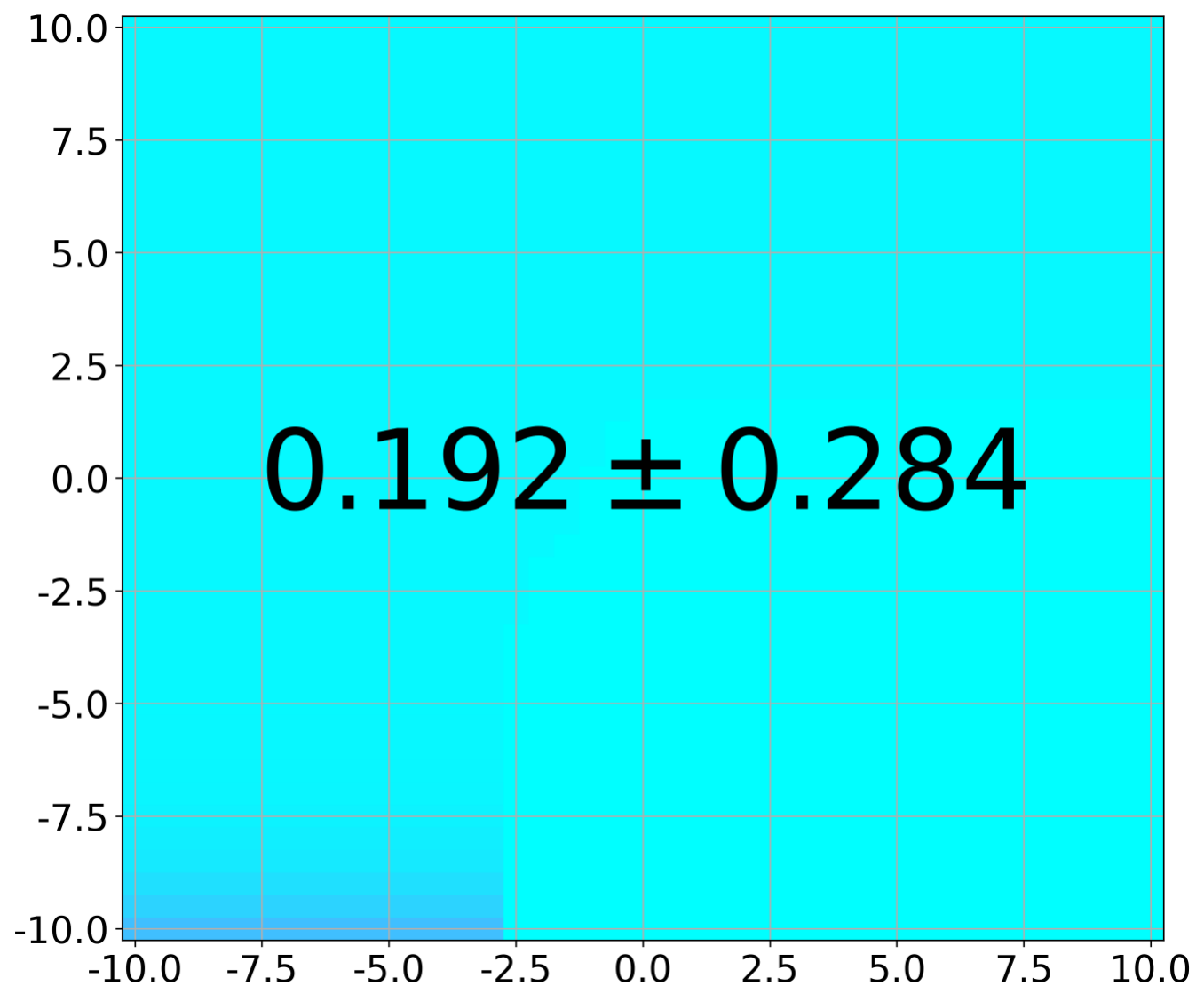}
\end{subfigure}
\begin{subfigure}{0.1375\textwidth}
\includegraphics[width=\textwidth]{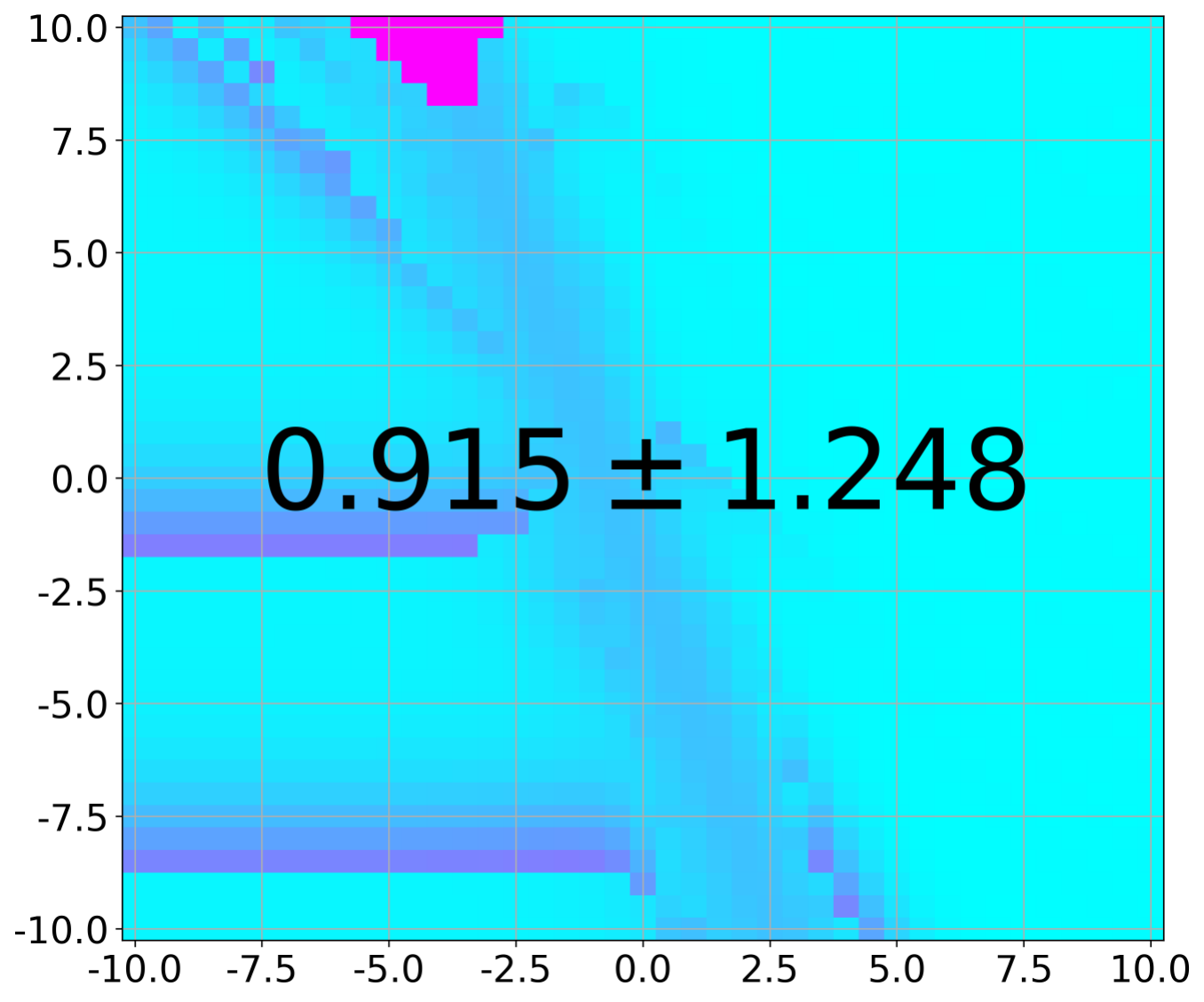}
\end{subfigure}
\begin{subfigure}{0.1375\textwidth}
\includegraphics[width=\textwidth]{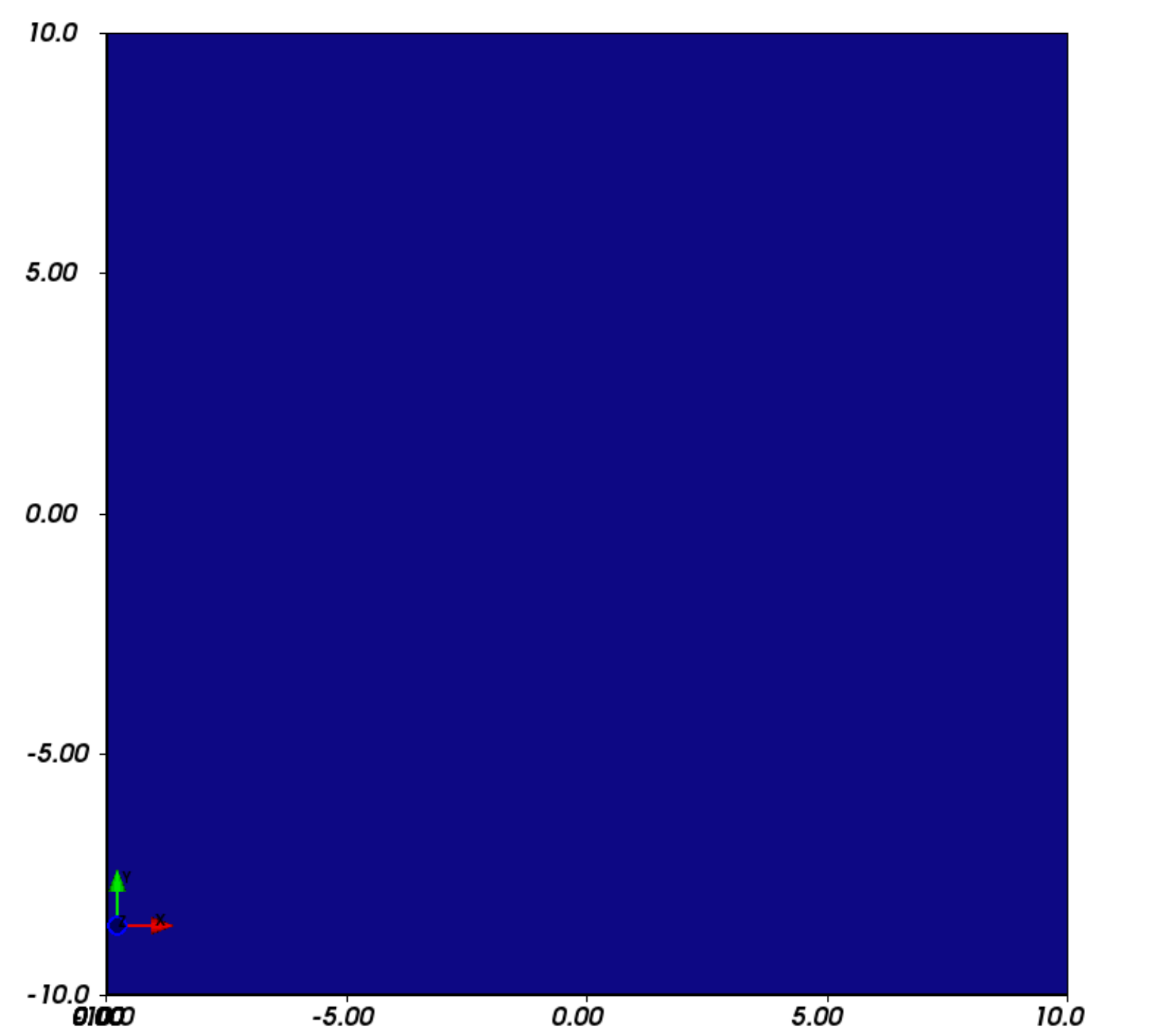}
\end{subfigure}
\begin{subfigure}{0.1375\textwidth}
\includegraphics[width=\textwidth]{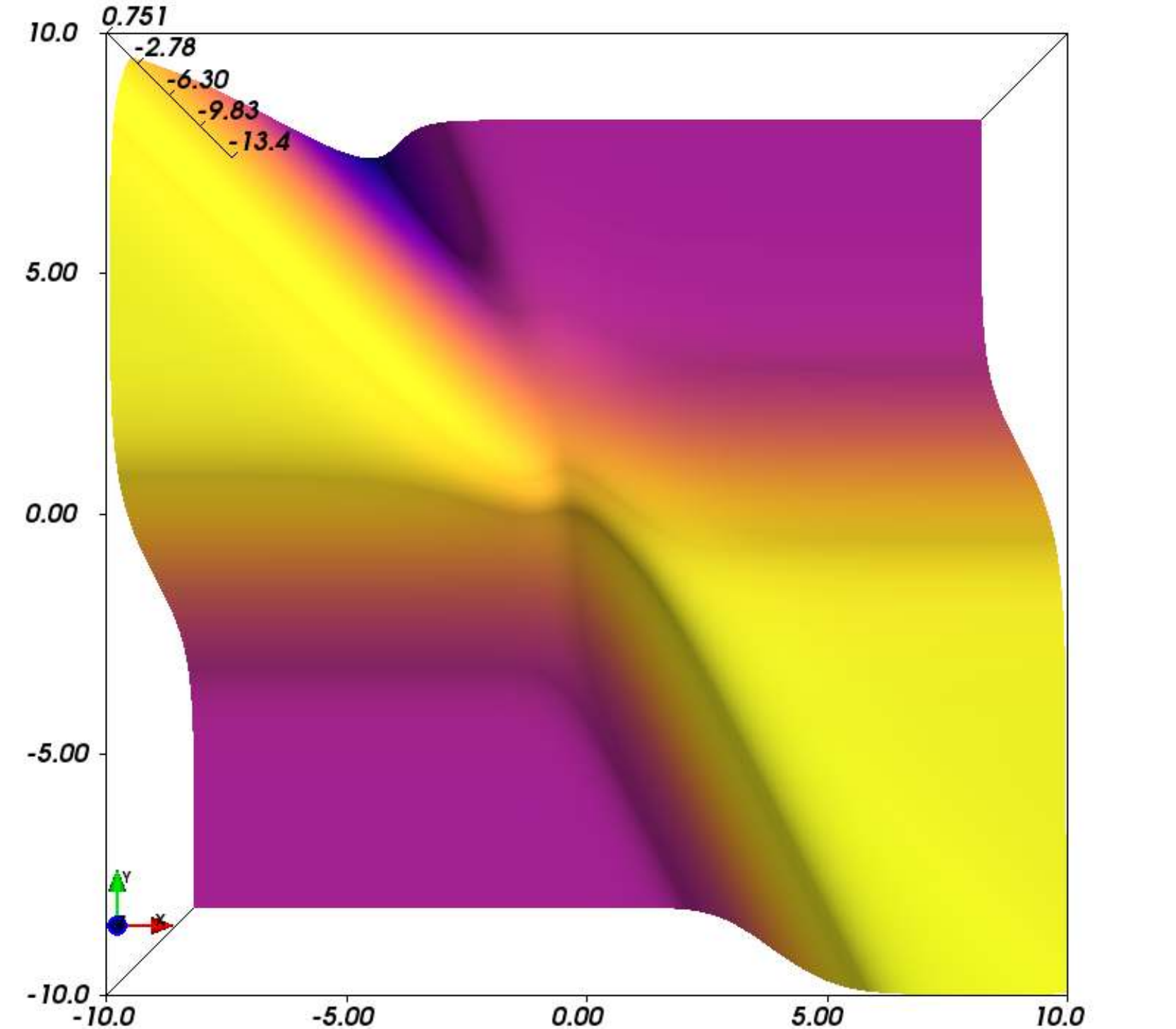}
\end{subfigure}
\begin{subfigure}{0.1375\textwidth}
\includegraphics[width=\textwidth]{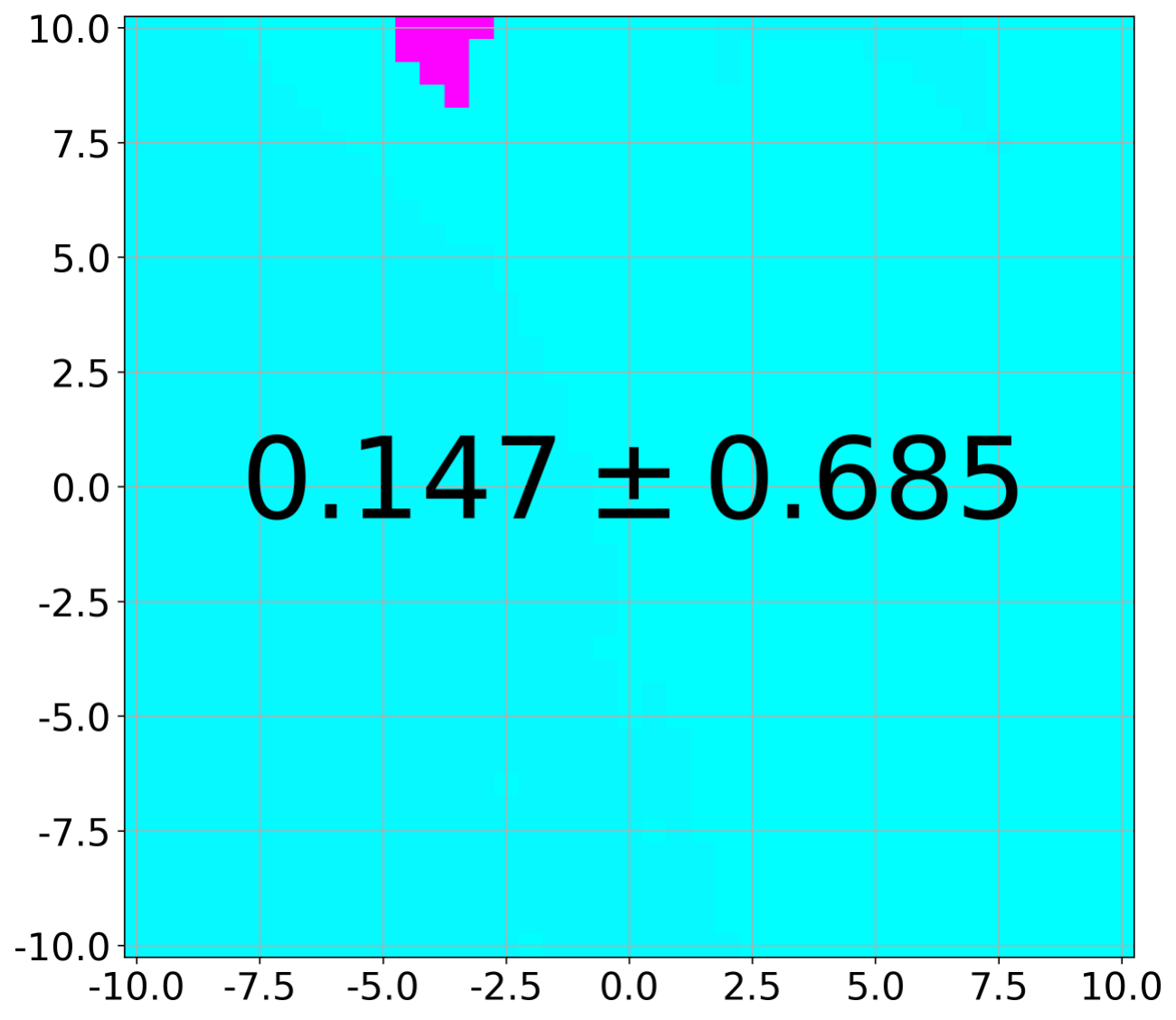}
\end{subfigure}
\begin{subfigure}{0.1375\textwidth}
\includegraphics[width=\textwidth]{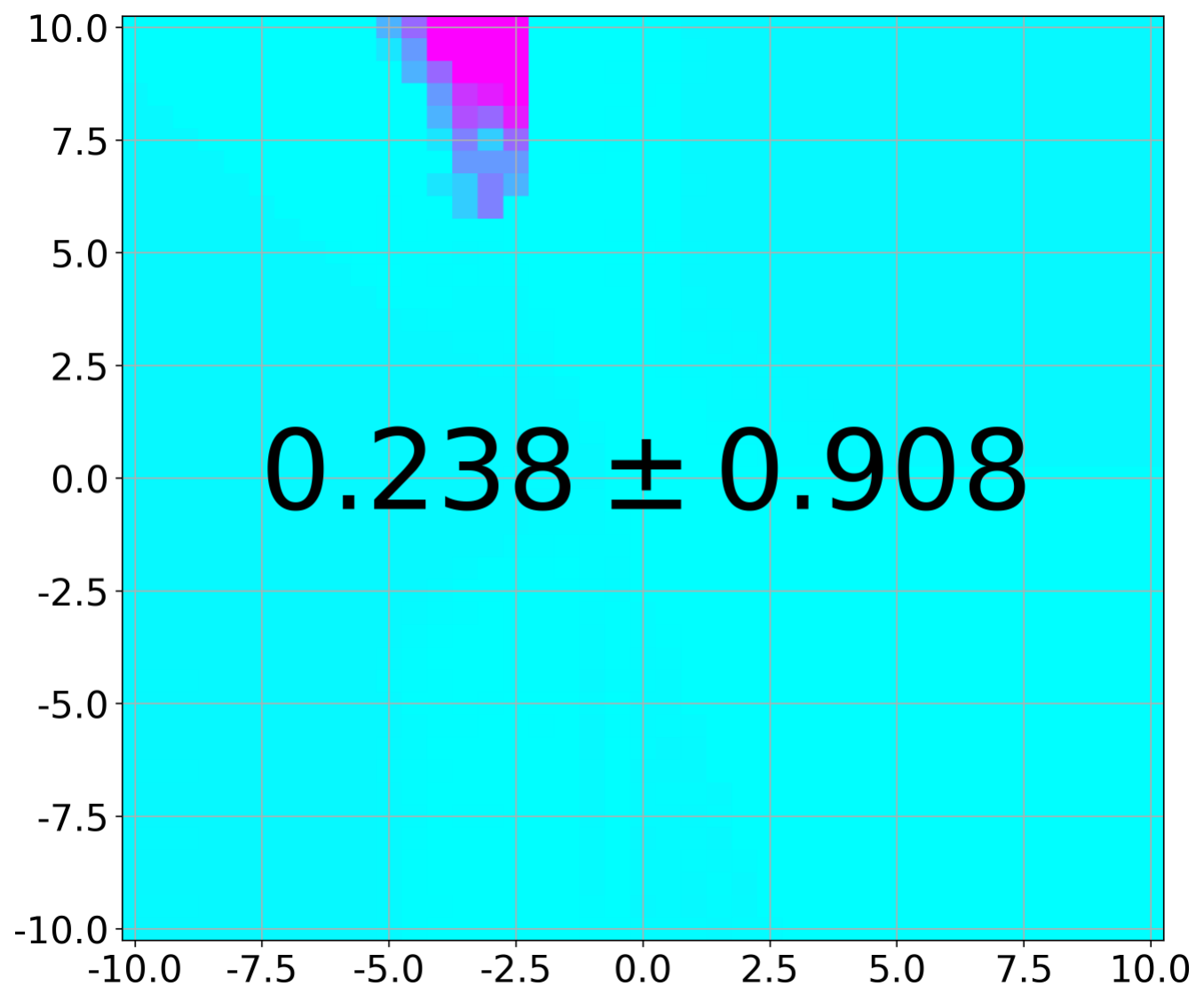}
\end{subfigure}
\begin{subfigure}{0.1375\textwidth}
\includegraphics[width=\textwidth]{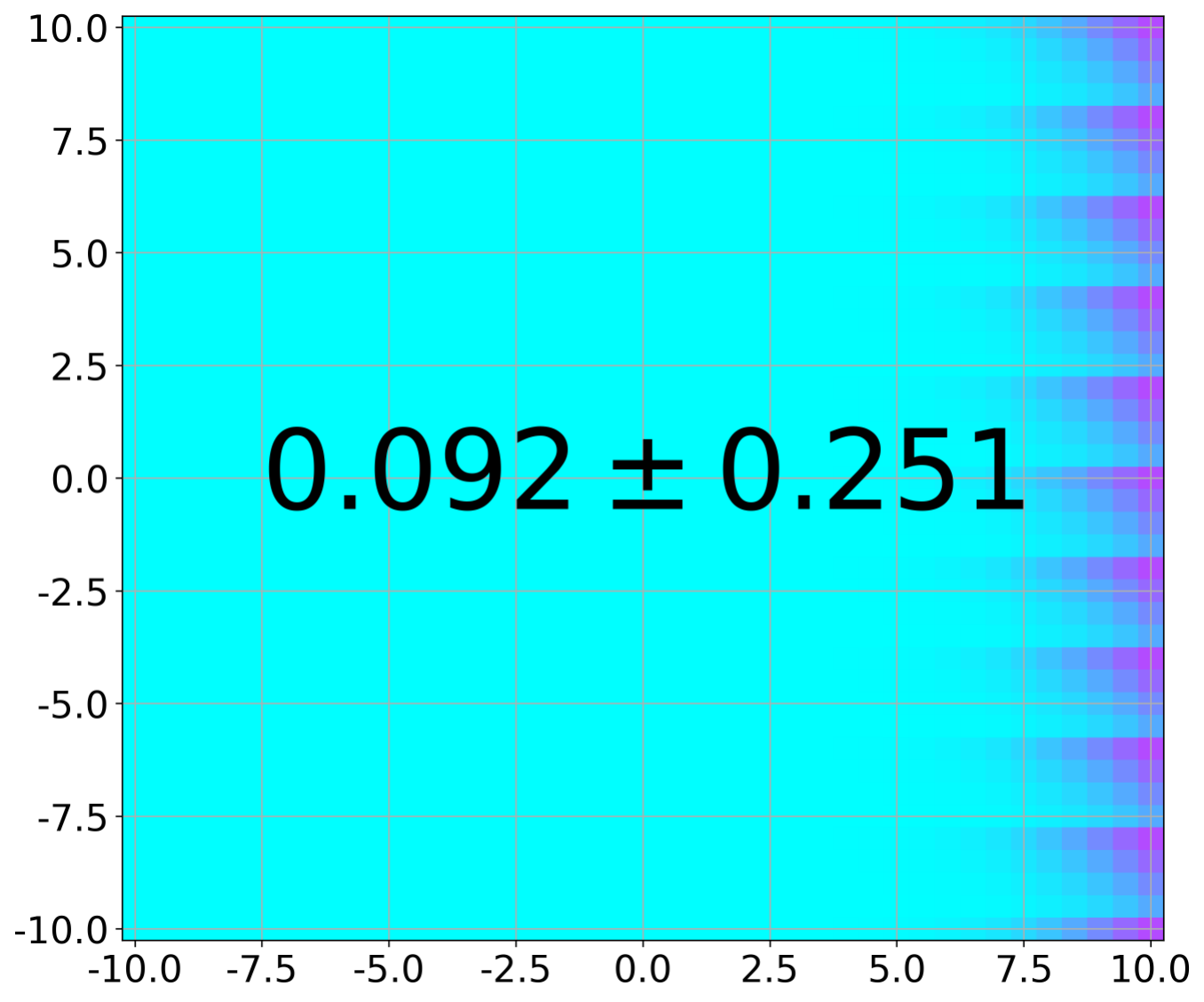}
\end{subfigure}
\begin{subfigure}{0.1375\textwidth}
\includegraphics[width=\textwidth]{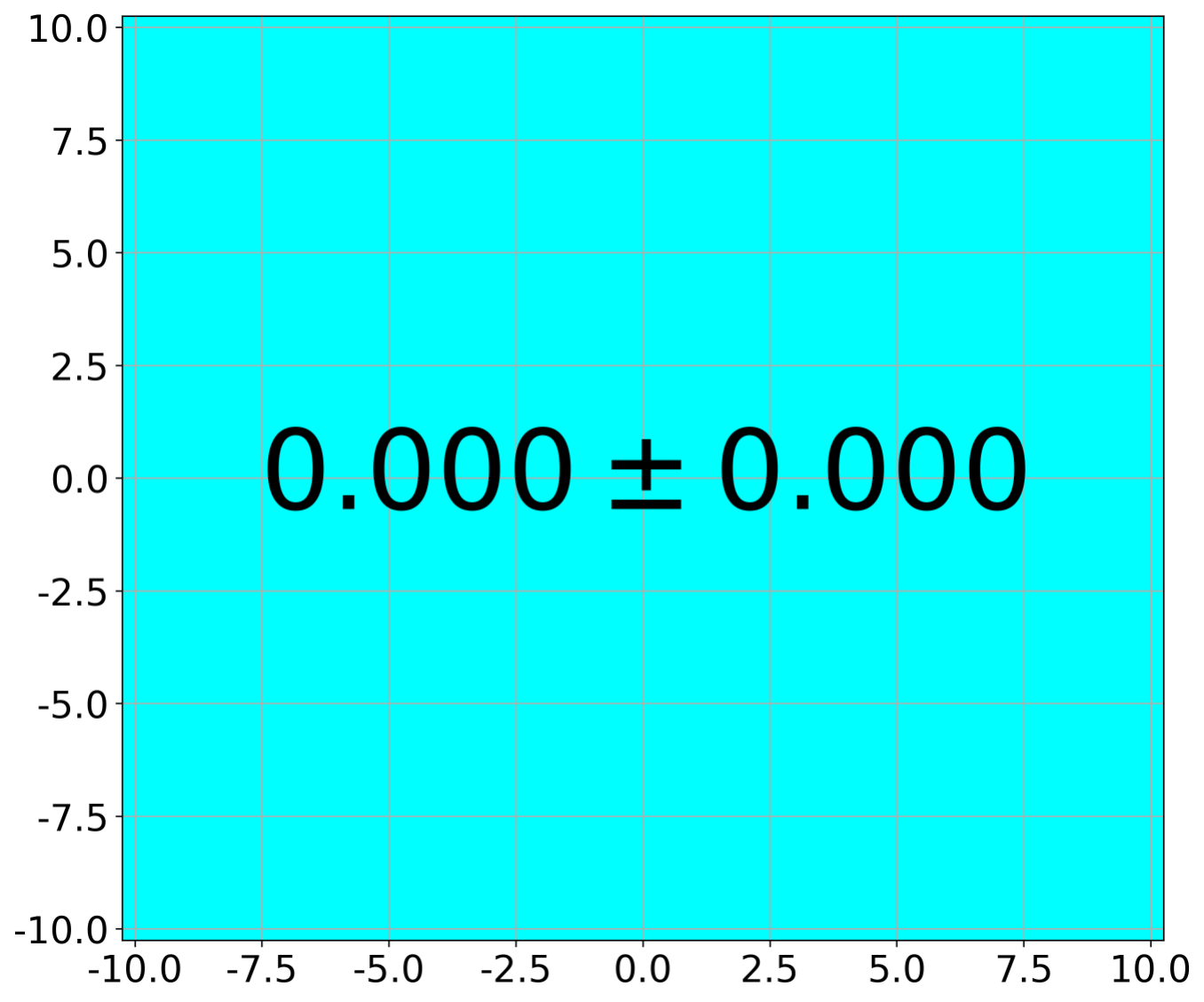}
\end{subfigure}
\begin{subfigure}{0.1375\textwidth}
\includegraphics[width=\textwidth]{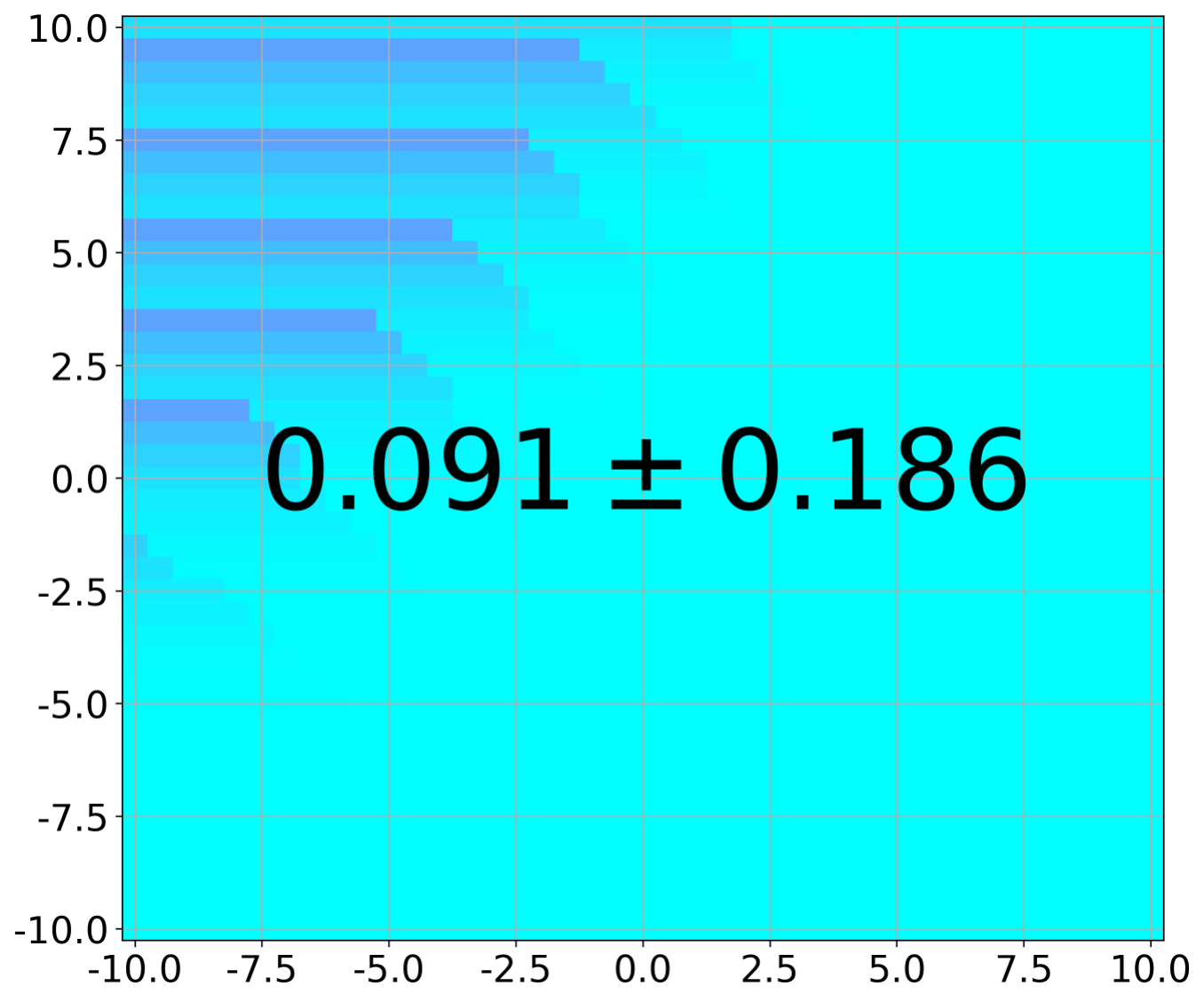}
\end{subfigure}
\begin{subfigure}{0.1375\textwidth}
\includegraphics[width=\textwidth]{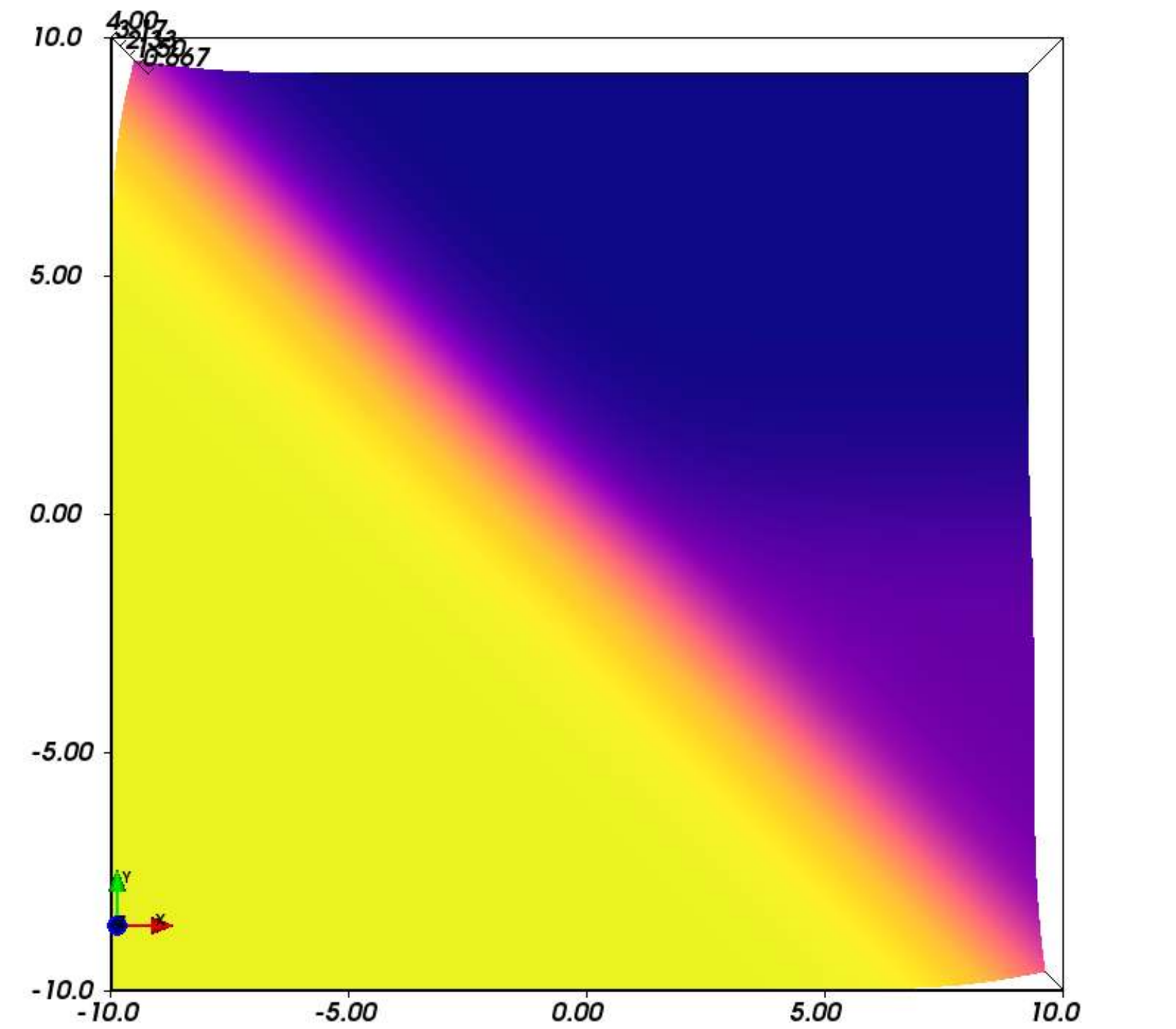}
\end{subfigure}
\begin{subfigure}{0.1375\textwidth}
\includegraphics[width=\textwidth]{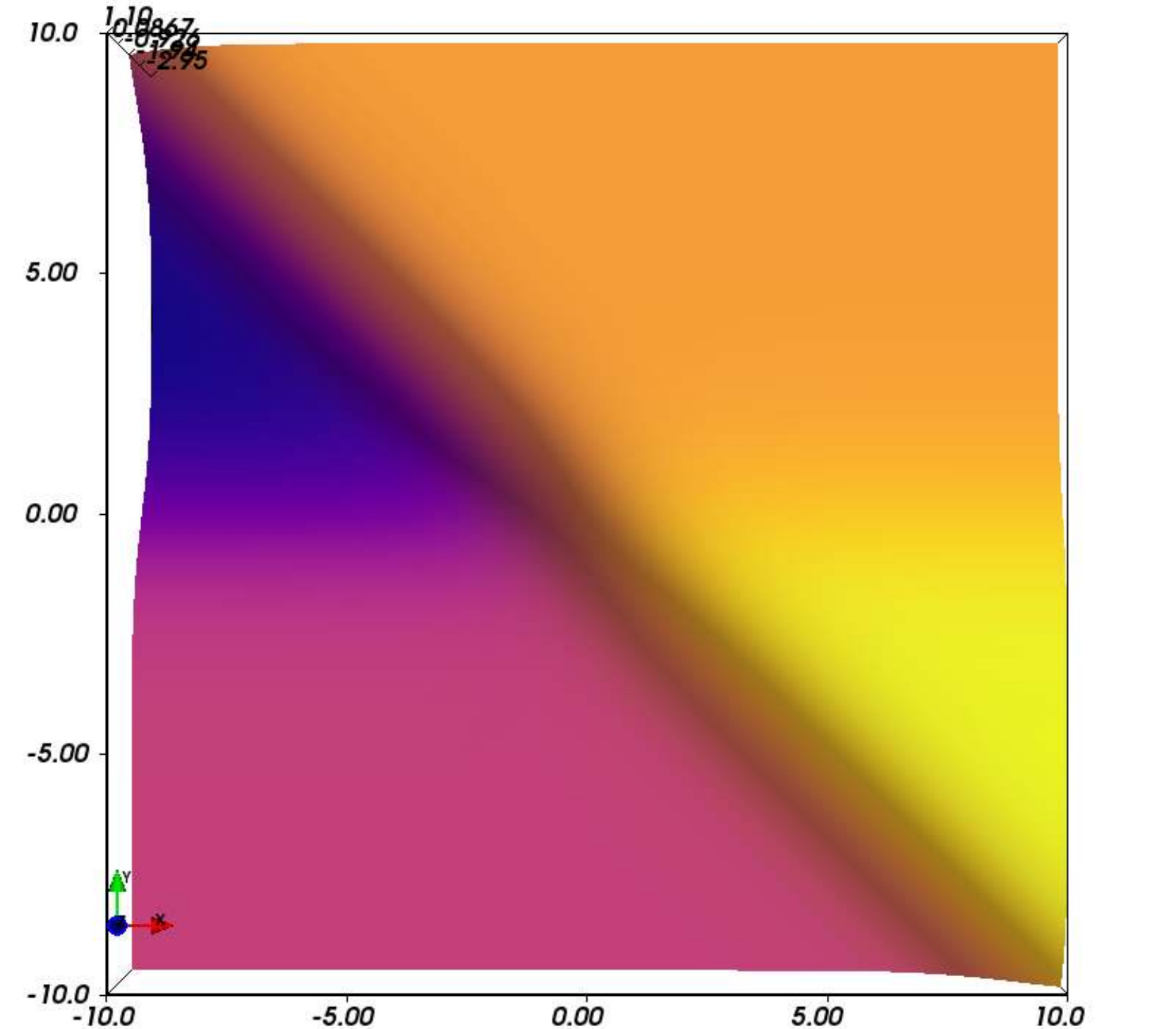}
\end{subfigure}
\begin{subfigure}{0.1375\textwidth}
\includegraphics[width=\textwidth]{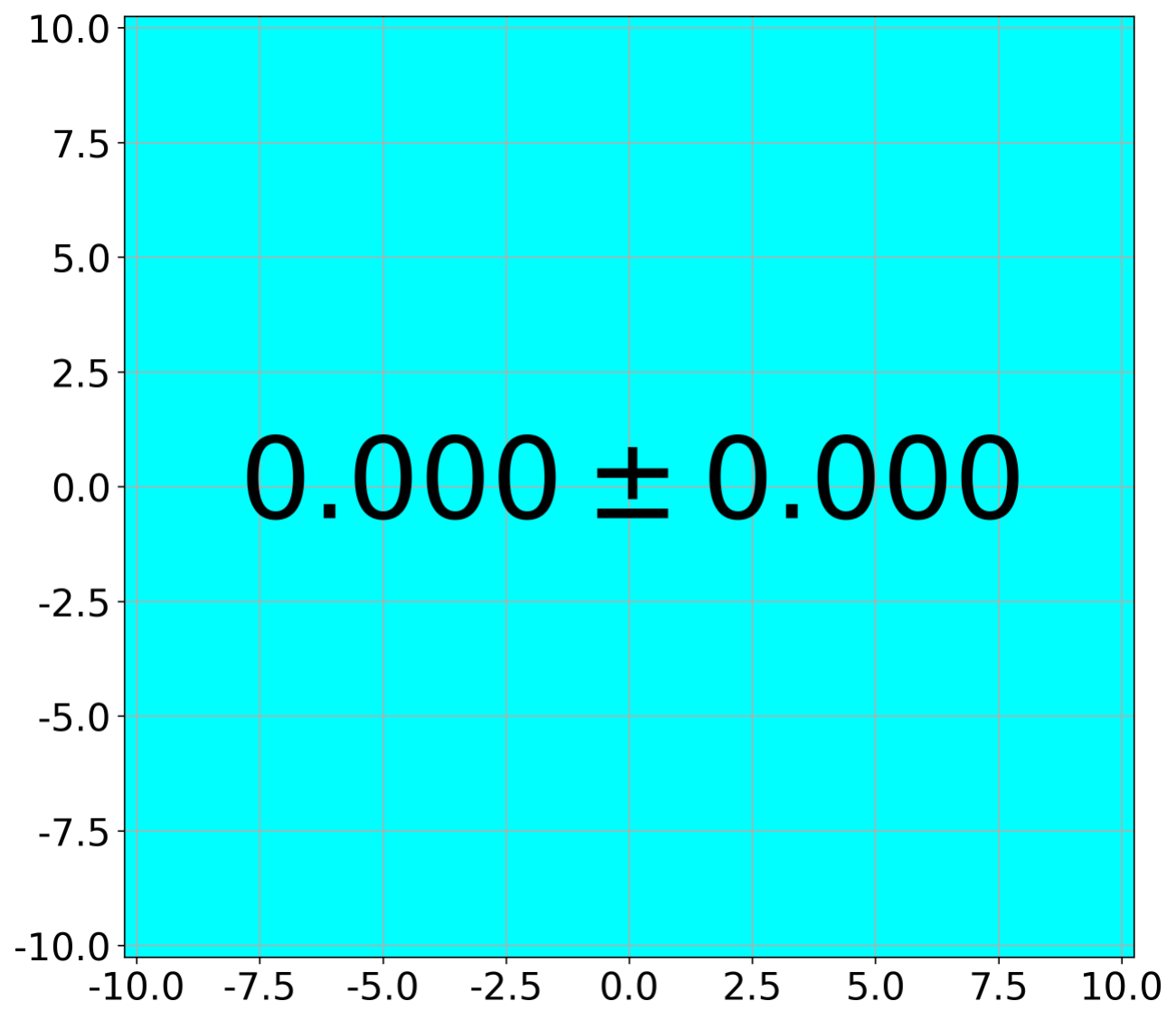}
\end{subfigure}
\begin{subfigure}{0.1375\textwidth}
\includegraphics[width=\textwidth]{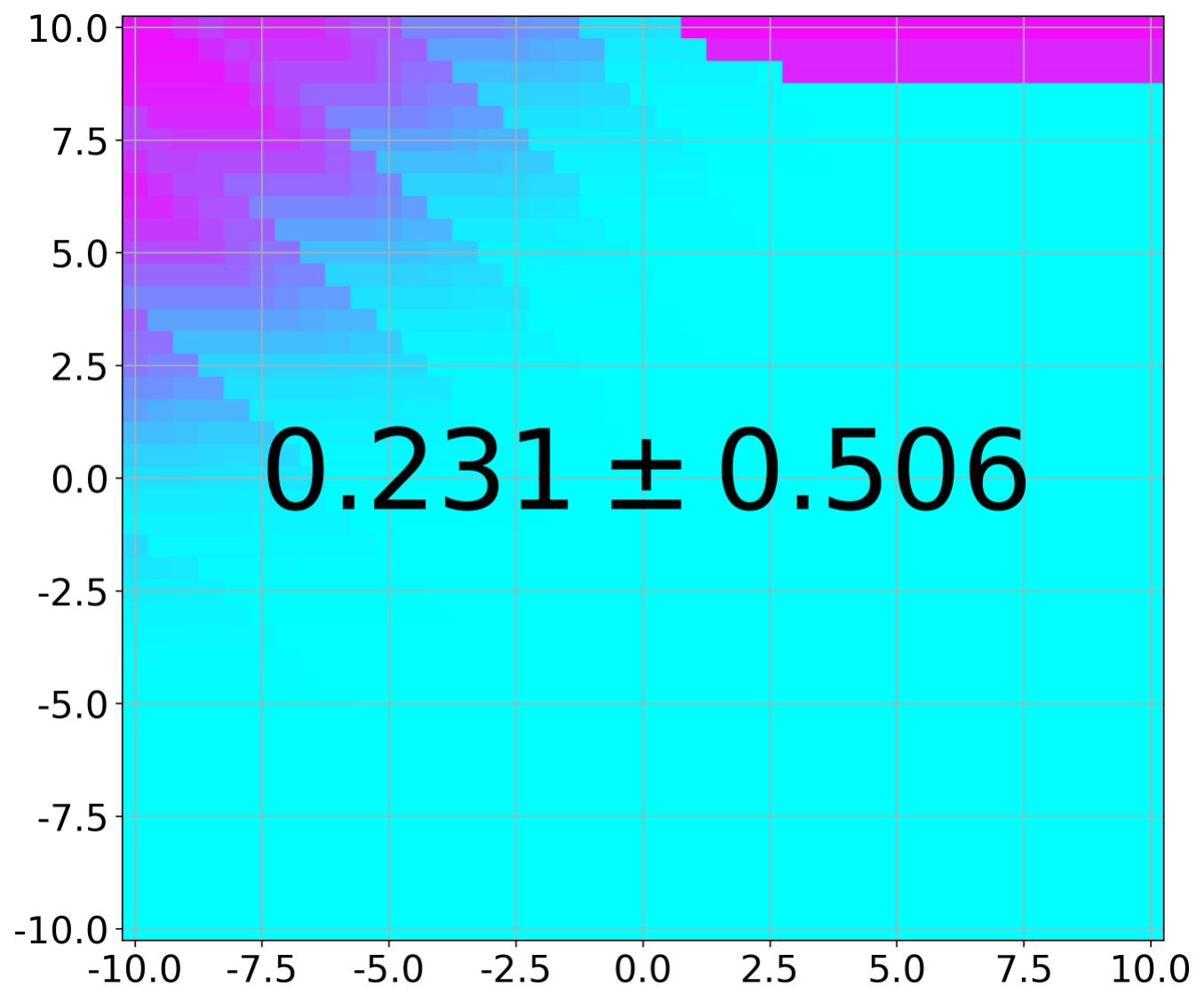}
\end{subfigure}
\begin{subfigure}{0.1375\textwidth}
\includegraphics[width=\textwidth]{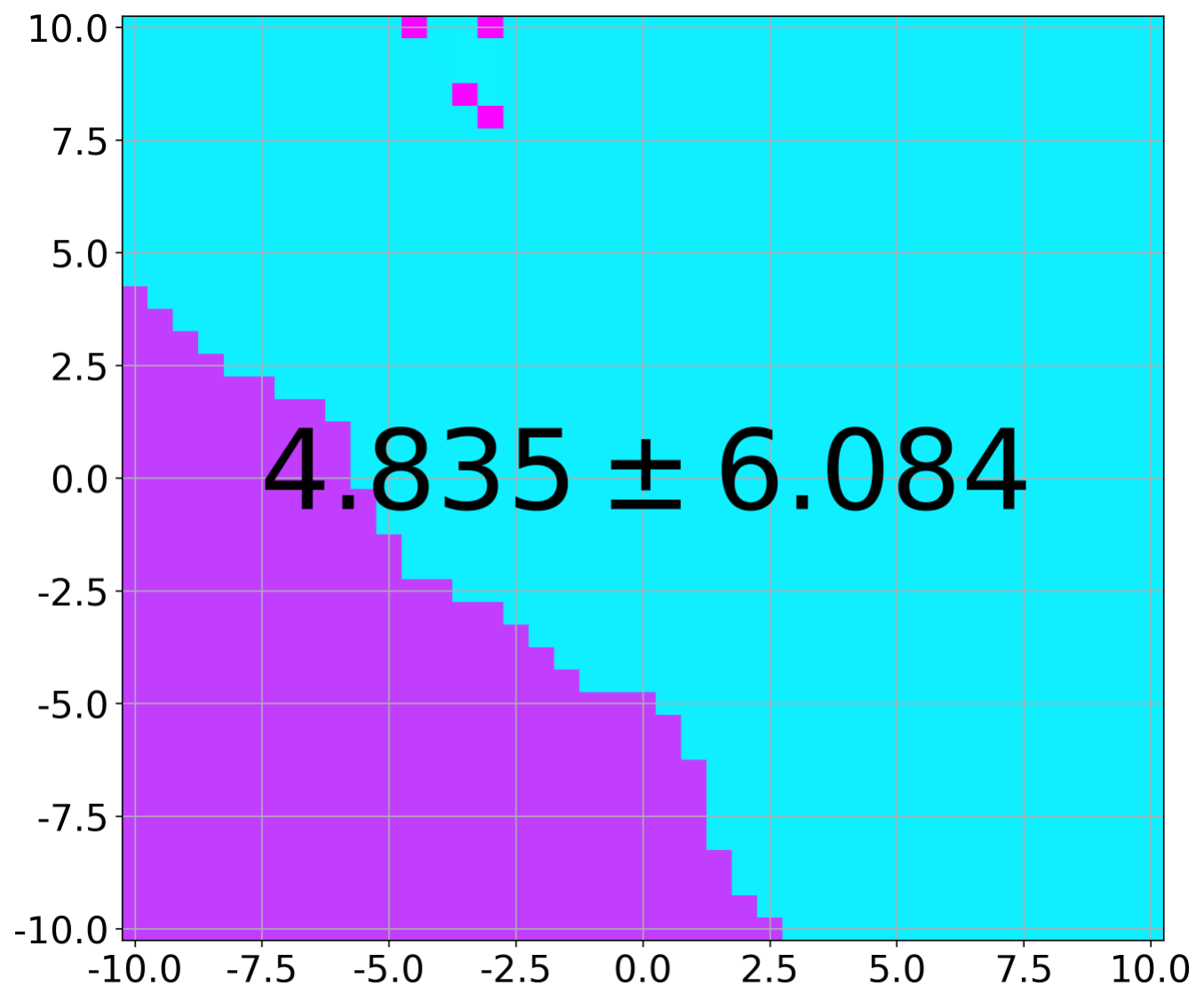}
\end{subfigure}
\begin{subfigure}{0.1375\textwidth}
\includegraphics[width=\textwidth]{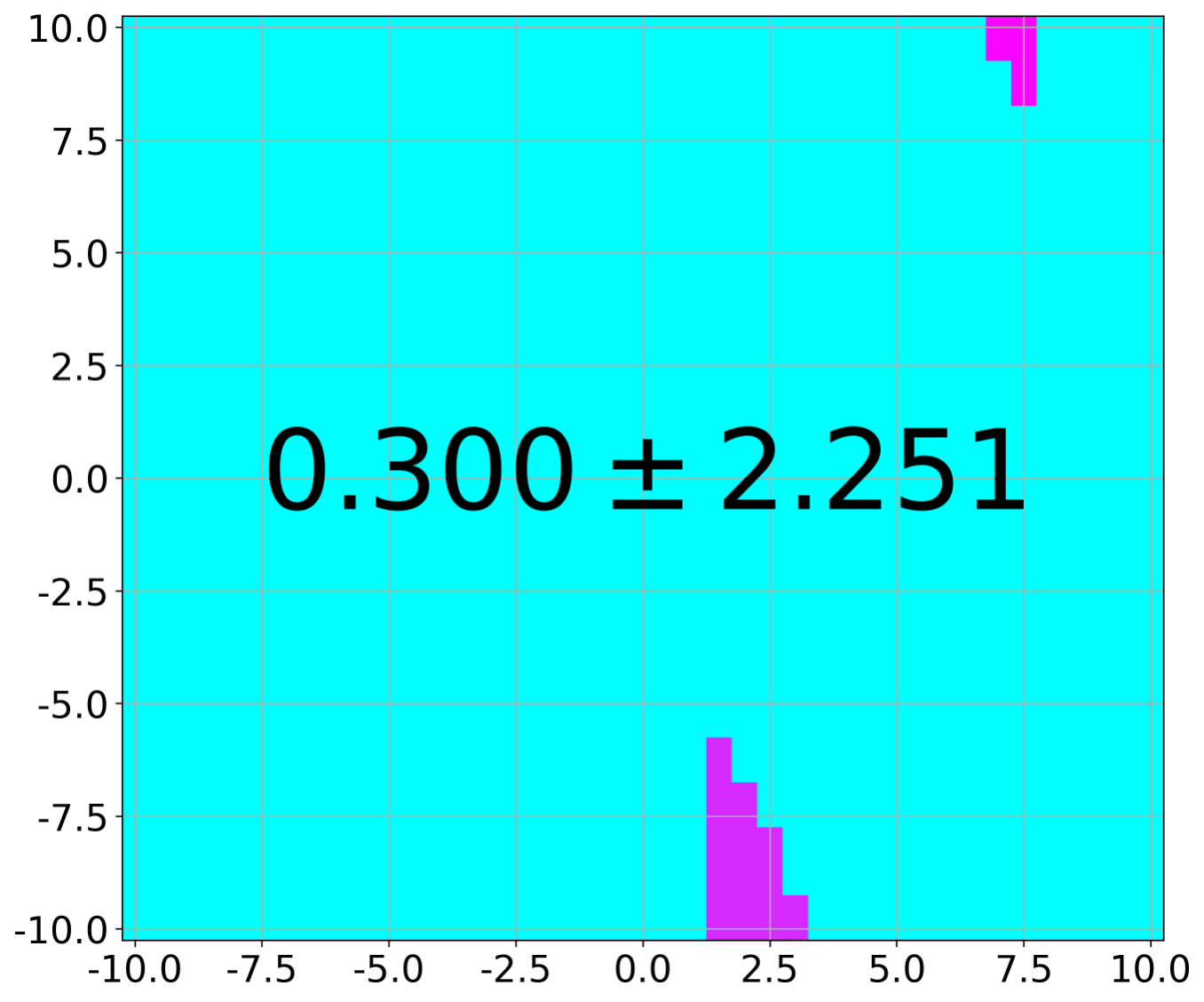}
\end{subfigure}
\begin{subfigure}{0.1375\textwidth}
\includegraphics[width=\textwidth]{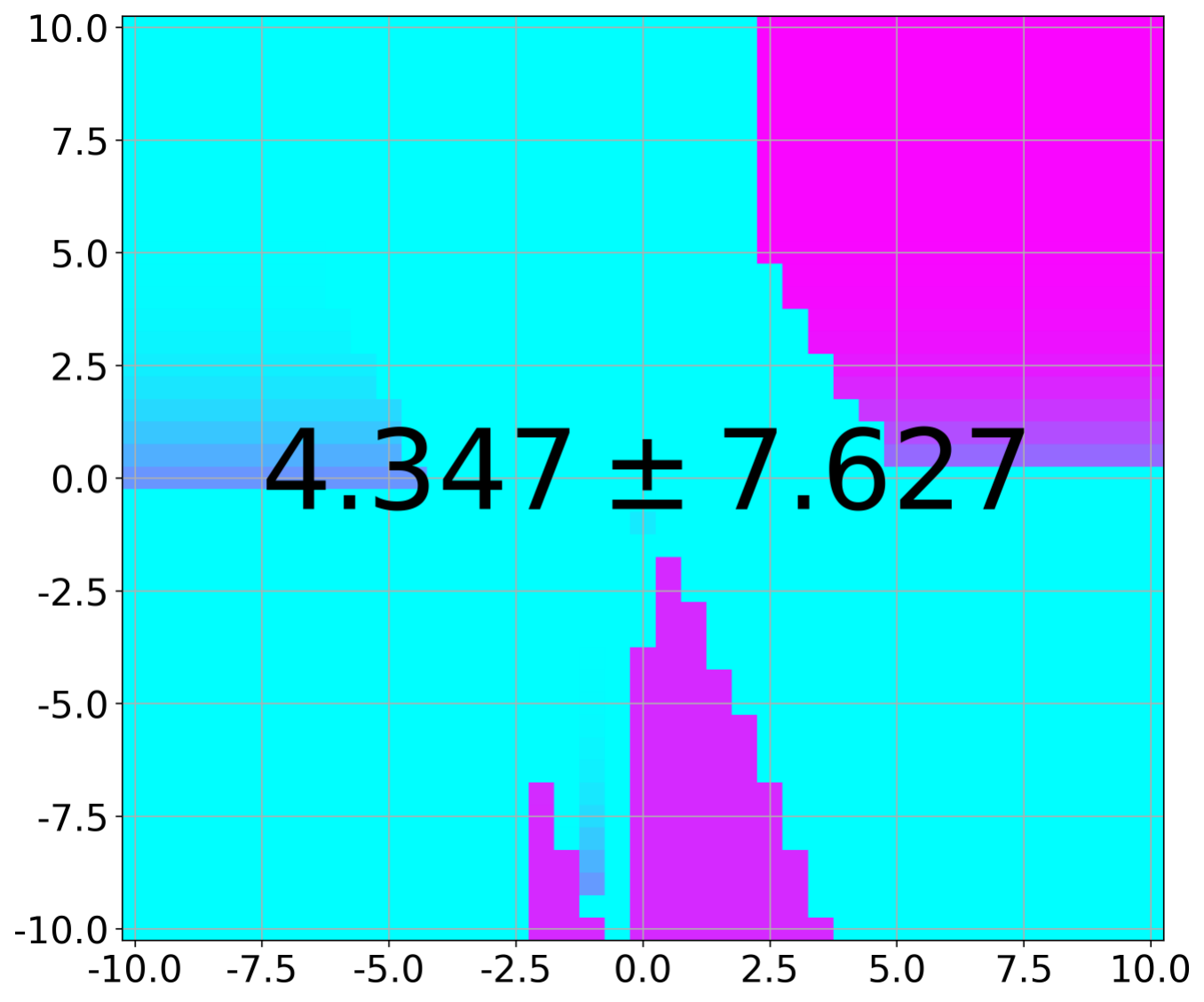}
\end{subfigure}
\begin{subfigure}{0.1375\textwidth}
\includegraphics[width=\textwidth]{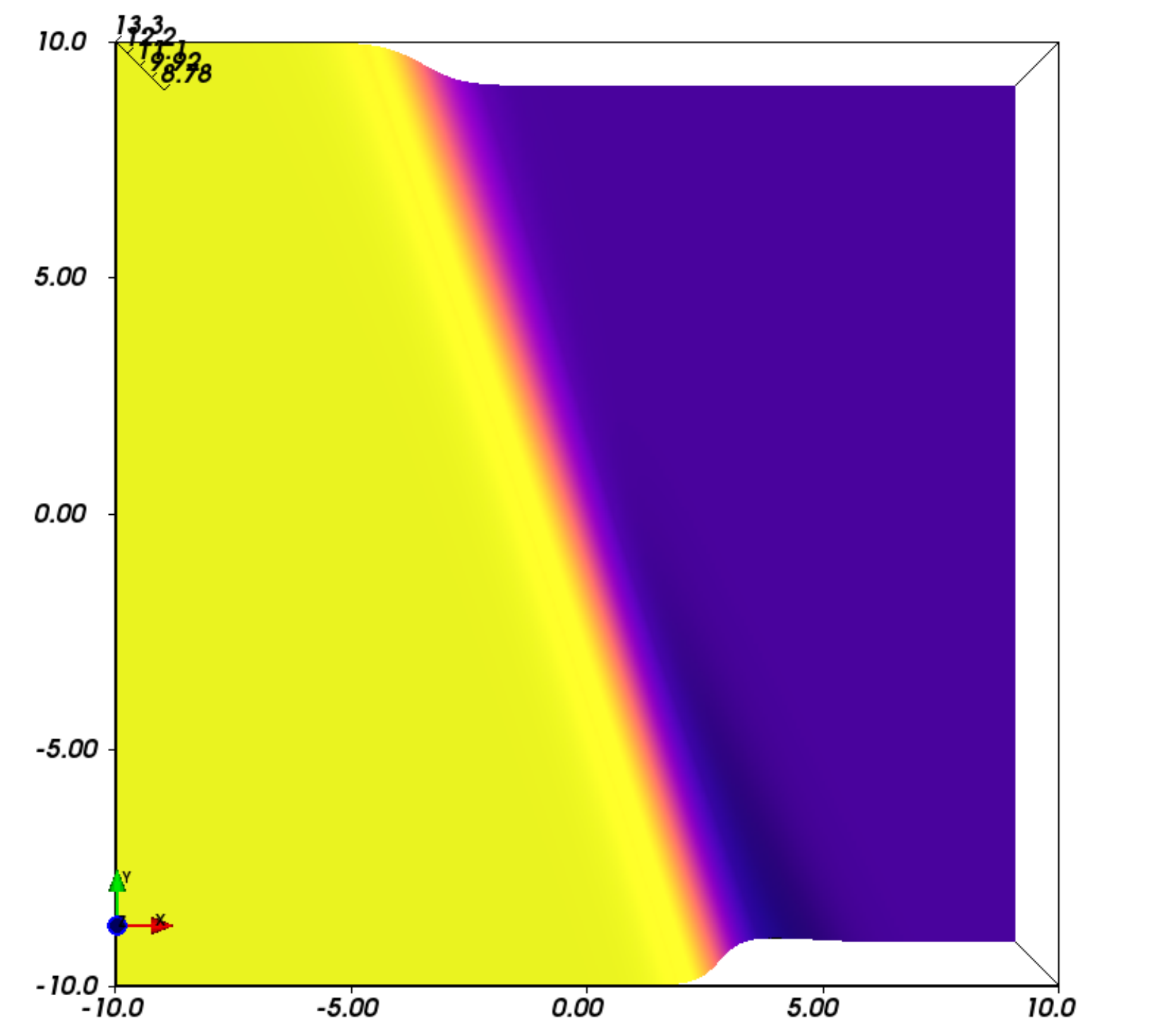}
\end{subfigure}
\begin{subfigure}{0.1375\textwidth}
\includegraphics[width=\textwidth]{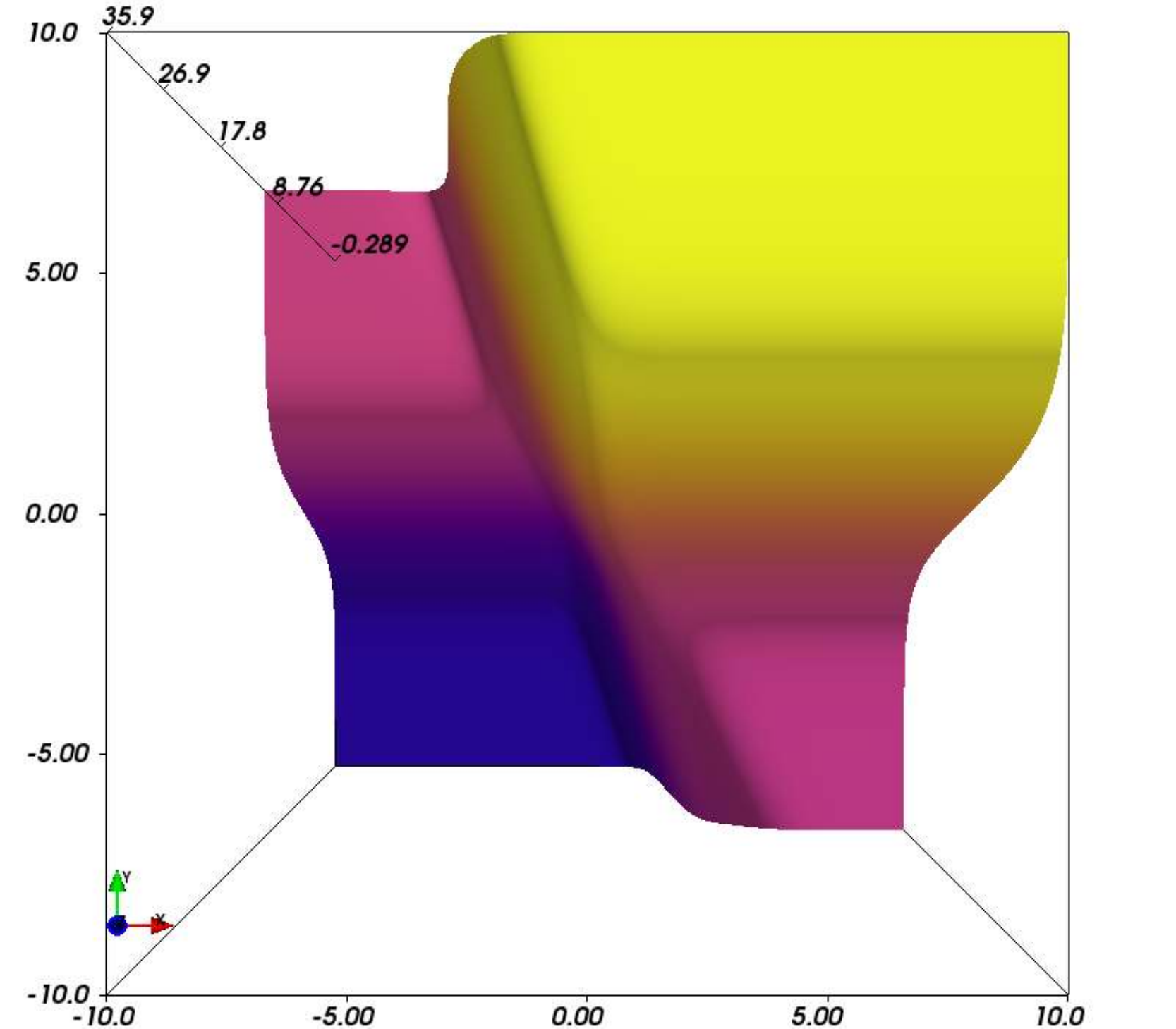}
\end{subfigure}
\begin{subfigure}{0.1375\textwidth}
\includegraphics[width=\textwidth]{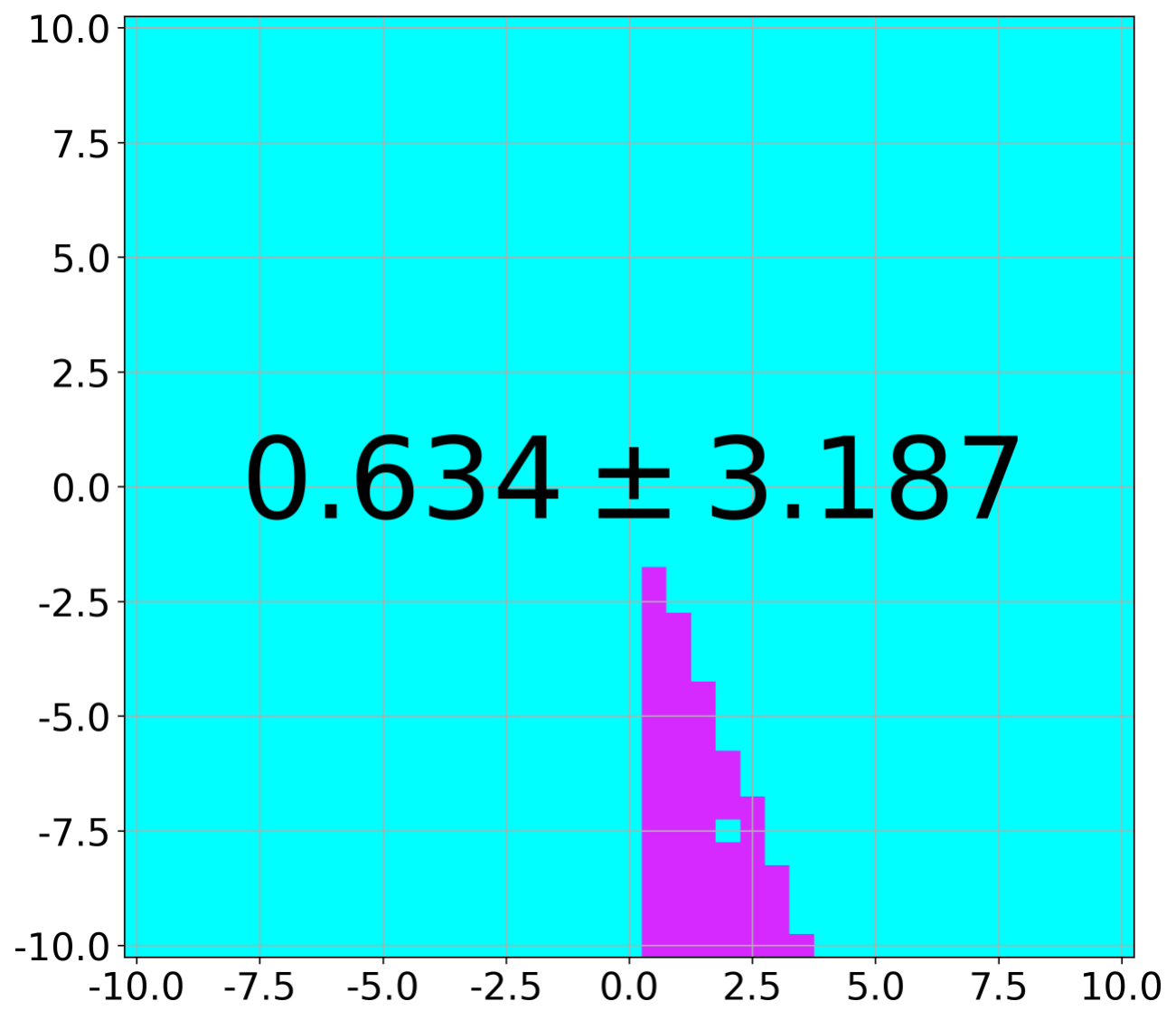}
\end{subfigure}
\begin{subfigure}{0.1375\textwidth}
\includegraphics[width=\textwidth]{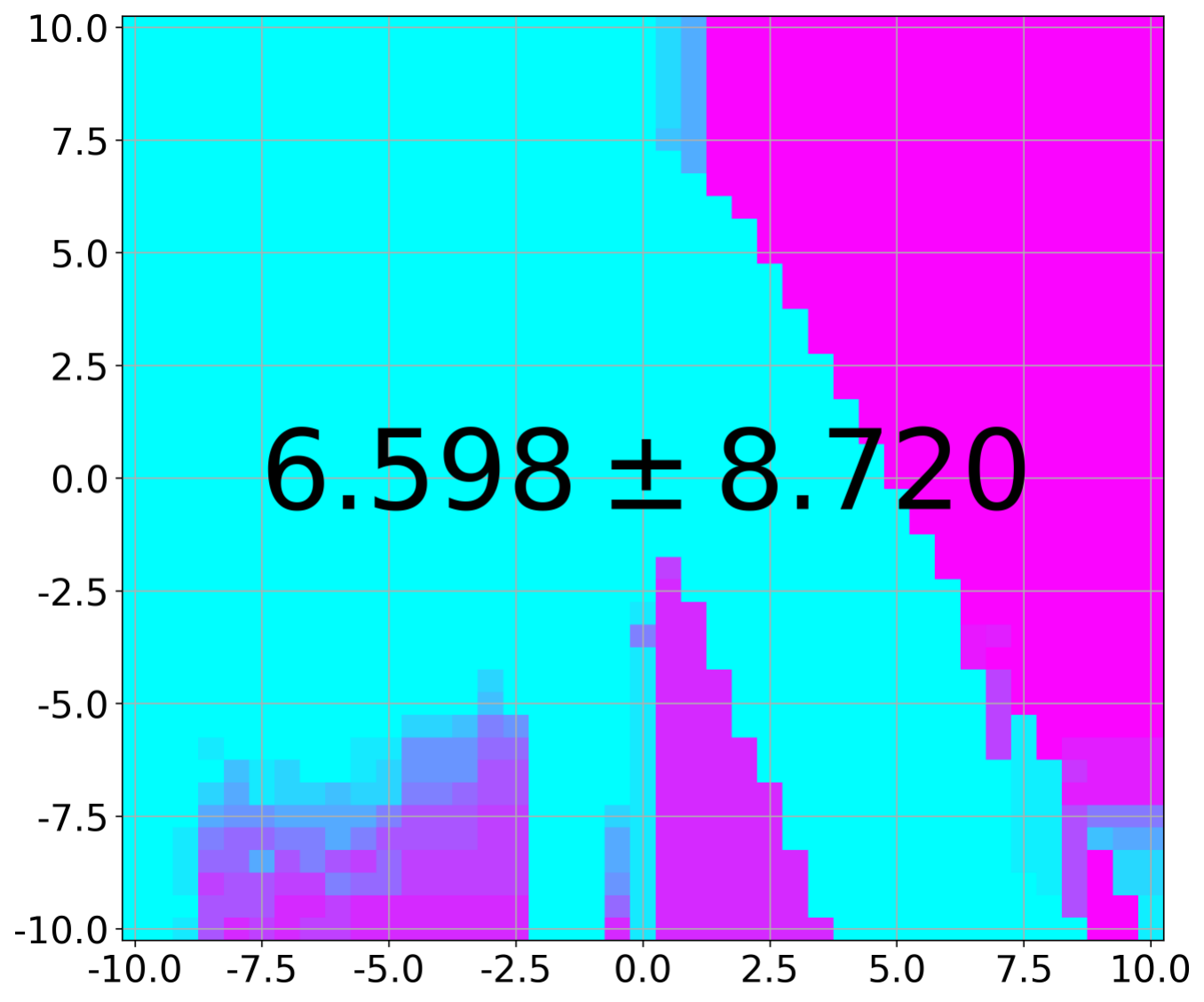}
\end{subfigure}
\begin{subfigure}{0.1375\textwidth}
\includegraphics[width=\textwidth]{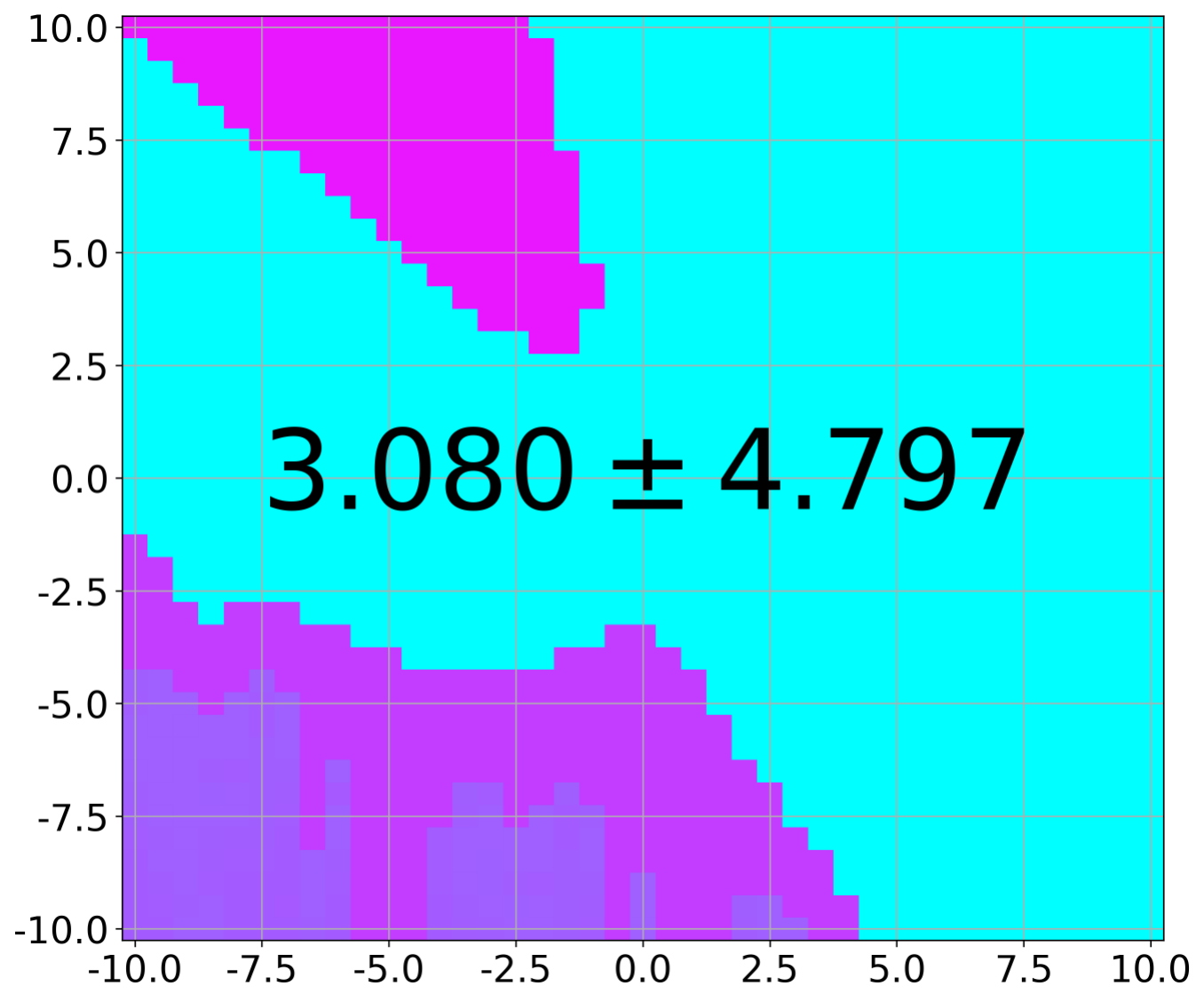}
\subcaption*{$\gamma < \gammabw$}
\end{subfigure}
\begin{subfigure}{0.1375\textwidth}
\includegraphics[width=\textwidth]{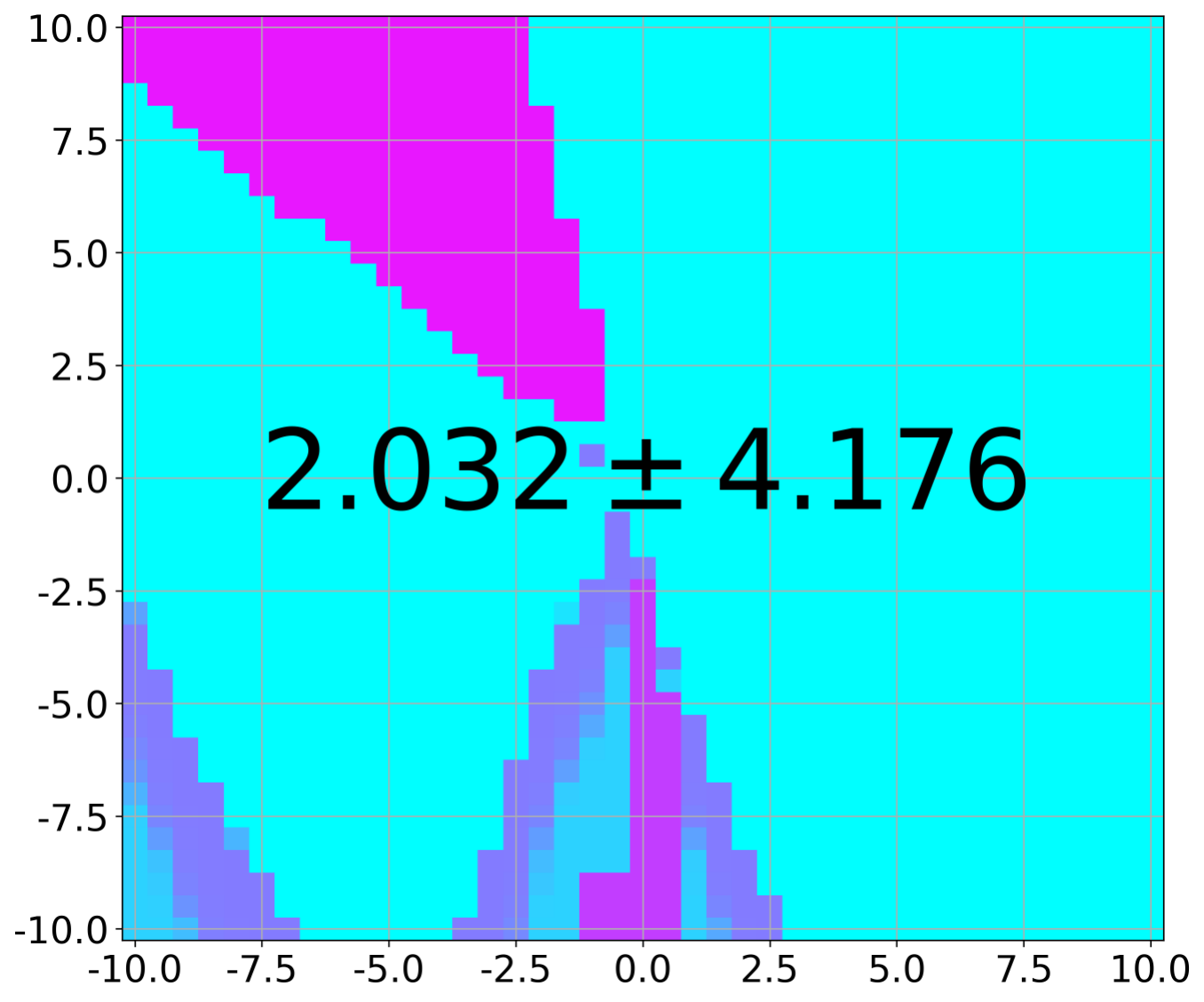}
\subcaption*{$\gamma \approx \gammabw$}
\end{subfigure}
\begin{subfigure}{0.1375\textwidth}
\includegraphics[width=\textwidth]{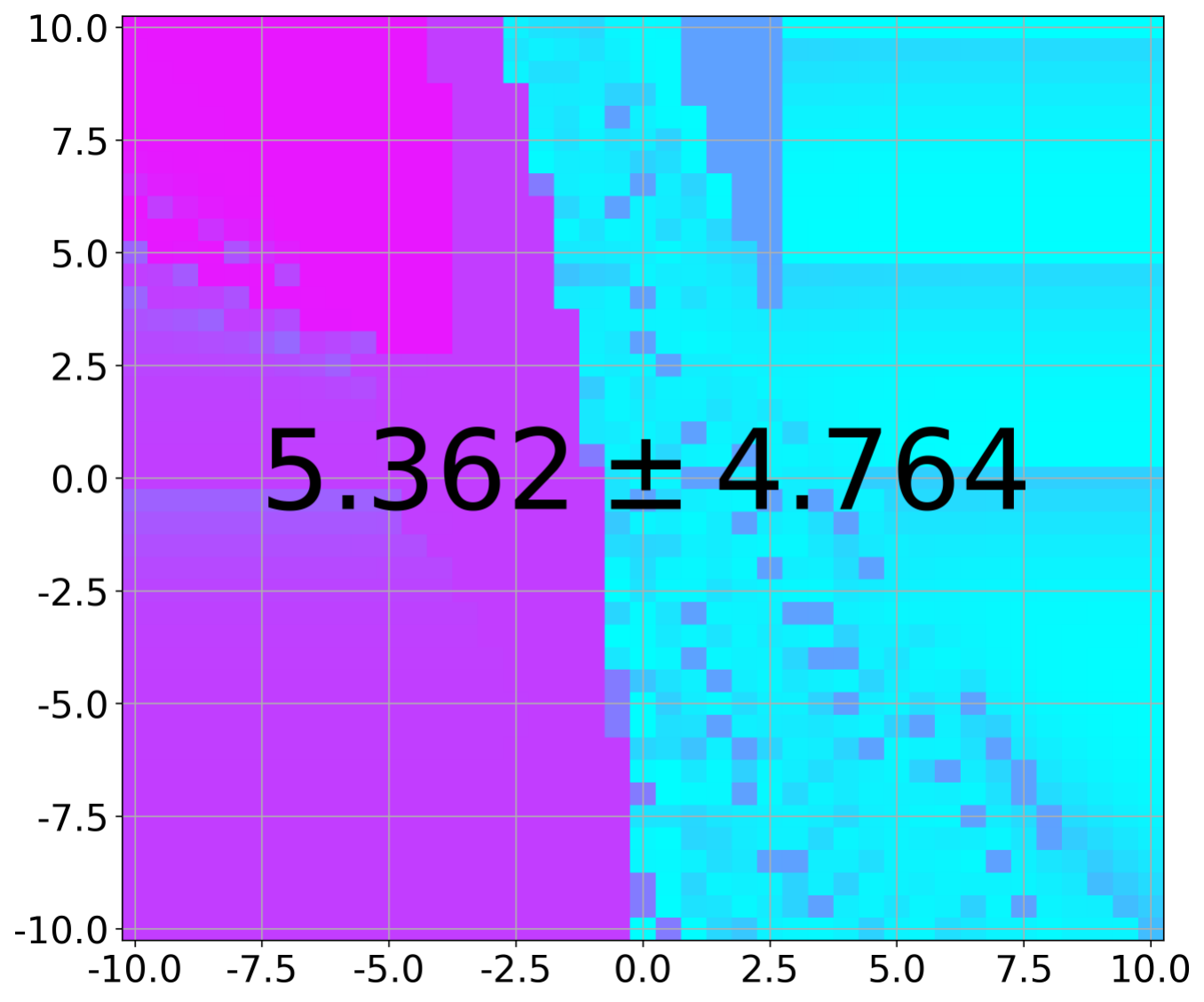}
\subcaption*{$\gammabw < \gamma \approx 1$}
\end{subfigure}
\begin{subfigure}{0.1375\textwidth}
\includegraphics[width=\textwidth]{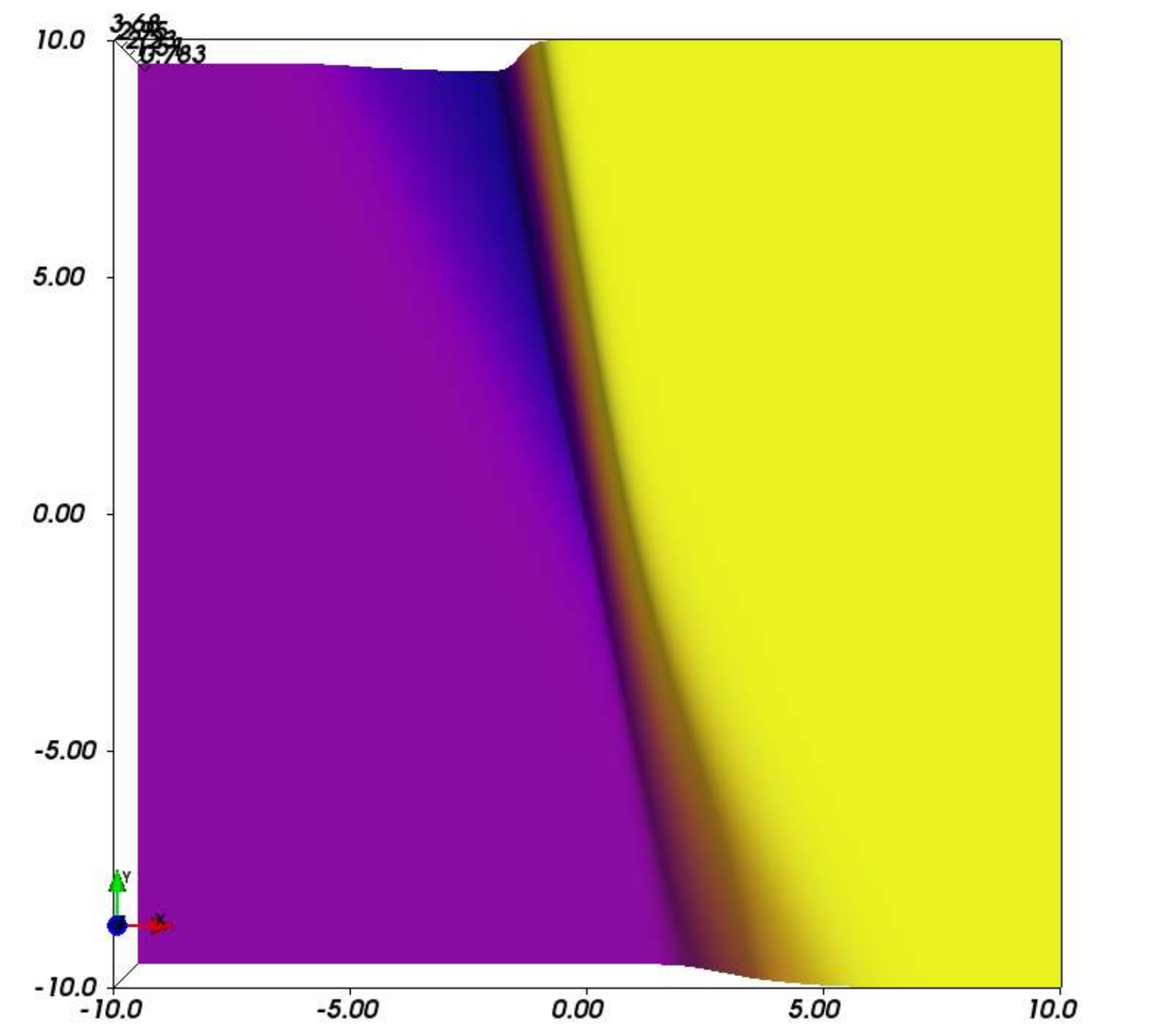}
\subcaption*{Gain landscape}
\end{subfigure}
\begin{subfigure}{0.1375\textwidth}
\includegraphics[width=\textwidth]{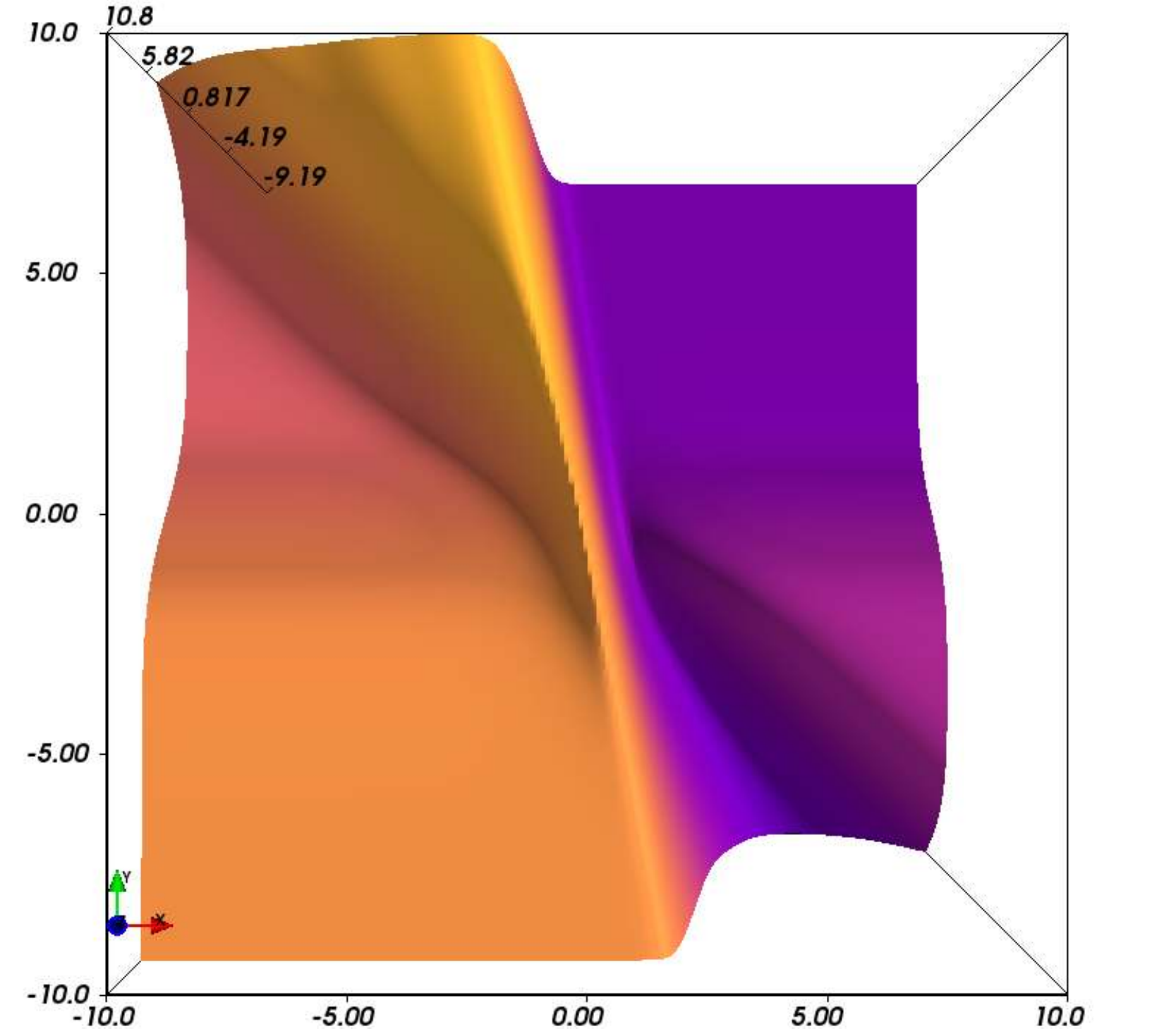}
\subcaption*{Bias landscape}
\end{subfigure}
\begin{subfigure}{0.1375\textwidth}
\includegraphics[width=\textwidth]{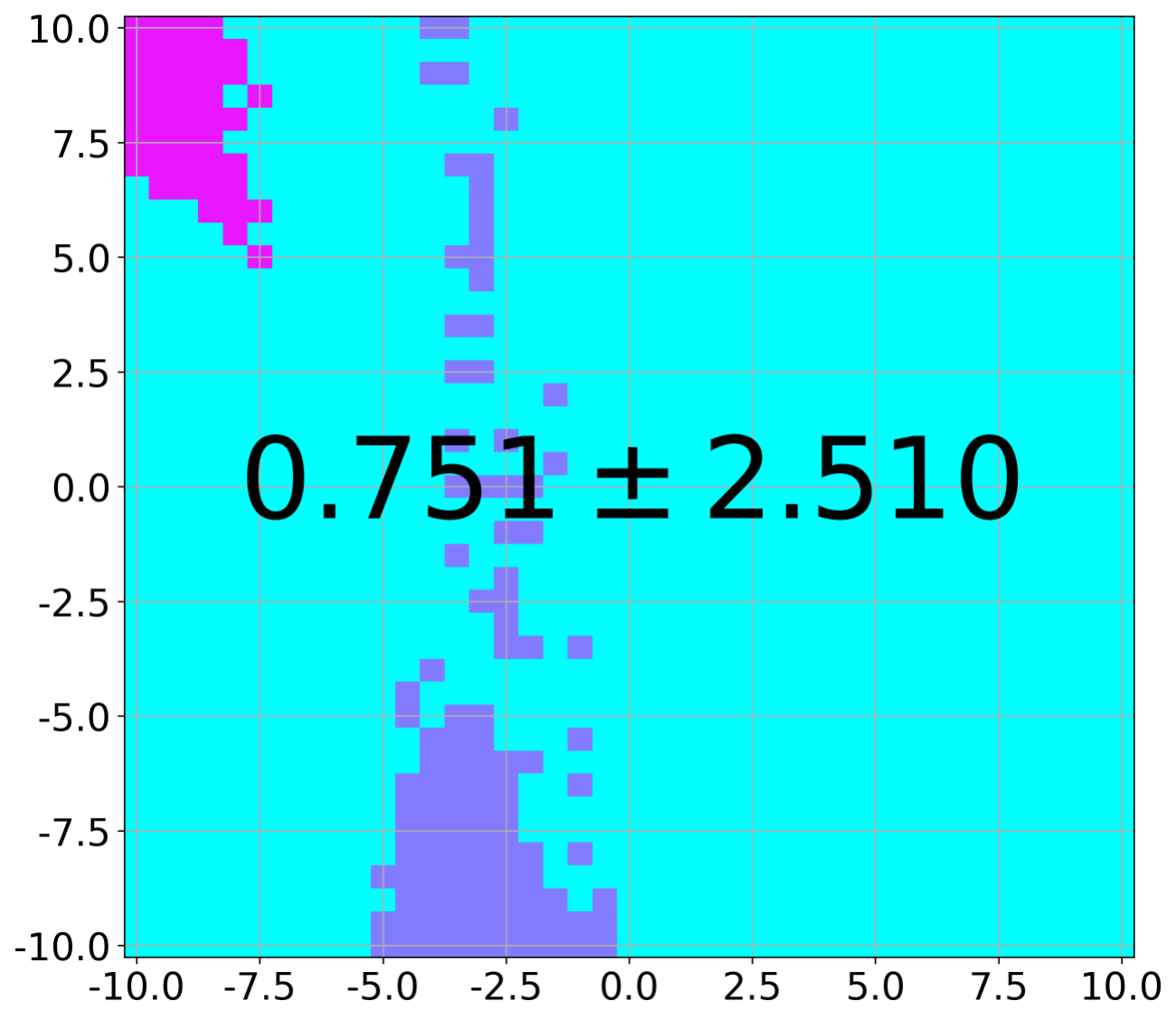}
\subcaption*{Ours}
\end{subfigure}
\begin{subfigure}{0.1375\textwidth}
\includegraphics[width=\textwidth]{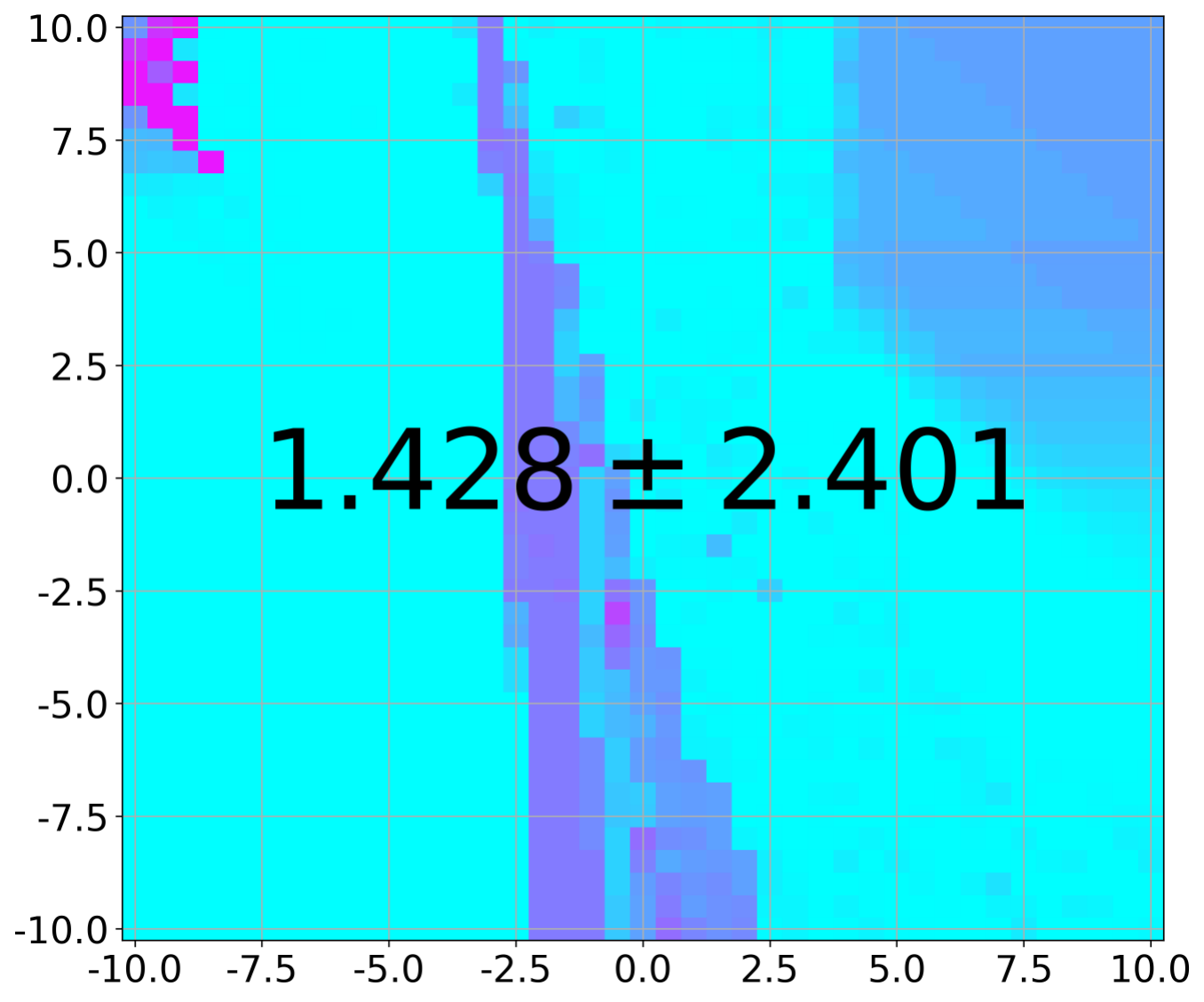}
\subcaption*{Ours \tiny{(sampling)}}
\end{subfigure}

\caption{Evaluations of our \algref{alg:gainbiasbarrier_optim}
with policy parameter $\vecb{\theta} \in \real{2}$ (horizontal and vertical axes)
on 6 environments (rows).
The landscape plots map low:dark-blue and high:yellow.
The other plots show the absolute difference or error (low:blue to high:magenta)
between the bias of the nBw-optimal deterministic policy and of the resulting randomized policy,
where the 2D-coordinate indicates the initial value of~$\vecb{\theta}$.
The numbers in the middle indicate the mean and standard deviation across $41^2$ grid values.
The lowest absolute bias difference (blue) is desirable.
For experimental setup, see \secref{sec:xprmt_setup_nbwpg}.
}
\label{fig:optim_cmp}
\end{figure}

\subsection{Main experimental results}

\figref{fig:optim_cmp} presents our main experimental results on
six environments (\secref{sec:env_nbpwg}) as rows.
The 4th and 5th columns depict optimization landscapes, where
each 2D coordinate represents a randomized stationary policy whose
gain and bias values are indicated by the coordinates' colors.
Here, the color-maps are normalized to the minimum-maximum range of
the gain and bias of deterministic policies:
the minimum is mapped to dark-blue and the maximum to yellow.
Note that the gain landscape plots of Env-A1, A2, and A3 display
only dark-blue colors because all stationary policies have the same gain value.
On the other hand, the bias landscape of Env-B3 does not display
the dark-blue and yellow colors because the displayed portion of the parameter space
does not include randomized policies close enough to deterministic policies
with minimum and maximum bias values.

The average-reward policy gradient method operates in
the gain landscape (4th column); performing gain maximization.
In the first top-3 environments (Env-A), it immediately reaches
the global maxima region because all stationary policies are gain optimal.
Therefore, it is imperative to employ our proposed method that switches to
optimizing the bias-barrier \eqref{equ:biasbarrier_opt} towards the highest bias
(the yellow region in the 5th column), yielding nBw-optimal policies.

The situation for the bottom-3 environments (Env-B) is different in that
some policies are gain optimal (the yellow region in the 4th column), some are not.
In these environments, once the gain maximization is deemed converged
(near or in the yellow region in the 4th column),
the bias-barrier optimization \eqref{equ:biasbarrier_opt} kicks in.
The barrier component prevents the optimization iterates from going outside
the gain optimal region (such iterates operate in the bias-barrier landscape, which
looks like the bias landscape (5th column) in region where the policies are gain optimal).
Thus, in the bottom-3 environments, the nBw-optimal policies are represented
by the coordinates, whose colors are yellow (the highest) in the gain landscape
and are violet (not the highest value) in the bias landscape.
Note that by optimal policies, we mean \emph{approximately} optimal policies
since all policies represented in the plots are randomized
(those approximately optimal are very close to being deterministic).

In what follows, we discuss comparisons between discounted-reward (1st, 2nd, 3rd columns)
and our proposed (6th column) policy gradient methods in exact settings
where there is no sampling-based approximation.
Then, we discuss the behaviour of our proposed method in sampling experiments
(7th right-most column), and compare it to its exact variant (6th column).
Here, the performance of a method is indicated by the absolute error (difference)
between the bias of the resulting randomized policy and
of the nBw-optimal deterministic policy.
The lowest error is mapped to blue, whereas the highest to magenta.
These become the two extreme colors in the error plots in
all but landscape columns in \figref{fig:optim_cmp}.
In those error plots, each 2D coordinate indicates the initial policy parameter,
whereas the coordinate's color indicates the bias error of a policy
parameterized by the final policy parameter.
Note that a high performance method has an error plot dominated by blue.

\textbf{First,} in exact settings, our proposed \algref{alg:gainbiasbarrier_optim}
compares favourably with the discounted reward method whose performance is
sensitive to the discount factor~$\gamma \in [0, 1)$.
Our proposed algorithm is superior to those with poor choices of $\gamma$
(either too low or too high).
It is better or at least on par with its discounted counterpart whose~$\gamma$
is close to its critical value~$\gammabw$, which is unknown in practice.

As can be observed from the 1st column, $\gamma < \gammabw$ generally
results in suboptimal policies with respect to Blackwell optimality (\secref{sec:rwork_nbwpg}).
This is indicated by larger magenta regions in the 1st than 2nd columns.
Increasing $\gamma$ too close to~1 also results in larger magenta regions
in the 3rd column, compared to the 2nd column (where $\gamma \approx \gammabw$).
This is as anticipated since the scaled discounted reward is equal to
the gain as $\gamma \to 1$.
Hence, optimizing discounted rewards with such high $\gamma$ tends to behave
like its average reward (gain) counterpart, which is undesirable for our six environments
because gain optimality is underselective for the induced unichain MDPs.

The 2nd column represents a discounted reward variant with
a proper $\gamma \approx \gammabw$, by which the nBw-optimal policy is also
attainable (\secref{sec:rwork_nbwpg}).
Therefore, its error plot is expected to be similar to that of
our proposed method (6th column).
Ours however is discounting-free so that it does not suffer from
the performance sensitivity with respect to $\gamma$ values
(as discussed in the previous paragraph).
More importantly, its hyperparameter $\beta_0$ merely serves as
the initial value for the barrier parameter~$\beta$, which is then shrunk down to~$0$
as the optimization iteration increases in \algref{alg:gainbiasbarrier_optim}.
This is in contrast to $\gamma$ that which is held fixed during optimization.
In \secref{sec:barrierparamtuning}, it is empirically shown that
the performance is less sensitive to $\beta_0$ than to $\gamma$
(whose effects to performance are depicted in the 1st to 3rd columns).

\textbf{Second,} we compare the exact variant of \algref{alg:gainbiasbarrier_optim}
(6th column) with its approximation (7th right-most column).
Here, the approximation comes from sampling-based estimates of
the gradients and Fisher matrices of the gain and bias-barrier functions,
as prescribed by the sampling procedure in \algref{alg:gradfisher_sampling}.
It can be seen that the sampling-based approximation results are
reasonably close to the exact's in most part of the policy parameter space.
They are anticipated to follow the law of large numbers.

\subsection{Decomposition of the bias gradients} \label{sec:gradbias_decompose}

We empirically investigate the contribution of each term
of the bias gradient expression \eqref{equ:grad_bias}.
Each term is computed exactly.
The pre-mixing terms are gradually added up term-by-term from
$t=0$ to $\tmix^\pi - 1$, then the single post-mixing term at $\tmix^\pi$ is
finally added up to the total of pre-mixing terms.

\figref{fig:gradbias_decompose} shows the error of the exact gradual summation
with respect to the exact full gradients of the bias;
hence, the error after all summations are carried out is exactly zero.
There are direction (angular) and magnitude (norm) errors.
Note that in practice, the magnitude error can be folded into (compensated by)
the optimization step length (\ie the learning rate).

One apparent trait is that the error drops always occur either
in the beginning (the first few terms) of the summation, at the end, or both.
Generally, there is little to no error change in between.
This suggests at least two phenomenons.
First, pre-mixing terms at later timesteps contribute less to the bias gradient,
compared to the first few.
Second, the single post-mixing term is likely to have a significant contribution
to the bias gradient.

\subsection{Bias-only optimization} \label{sec:biasopt}

In this section, we focus on bias-only optimization
(\cf gain-then-bias-barrier optimization in \secref{sec:bwoptim_algo}).
For environments where all policies are gain optimal (\ie Env-A1, A2, A3),
bias-only optimization returns nBw-optimal policies.
For others (\ie Env-B1, B2, B3), it returns policies that have
the maximum bias but \emph{not} necessarily the maximum gain
(thus, they are not necessarily nBw-optimal).

In exact settings, there are four schemes of bias-only optimization due to
different preconditioning matrices~$\mat{C}$ in \eqref{equ:polparam_update}.
They are as follows.
\begin{enumerate} [label=\roman{*}.]
\item Identity: $\mat{C} \gets \mat{I}$ for an identity matrix $\mat{I}$.
    This yields standard (vanilla) gradients.
    \label{item:identity}
\item Hessian: $\mat{C} \gets \nabla^2 v_b(\vecb{\theta})$,
    where the Hessian $\nabla^2 v_b(\vecb{\theta})$ is modified
    by adding $e \mat{I}$ for a small positive~$e$ until
    it becomes a positive definite matrix (checked by the Cholesky test).
    \label{item:hessian}
\item Analytic Fisher: $\mat{C} \gets \fbamat(\vecb{\theta})$ from \eqref{equ:fisher_b},
    yielding natural gradients.
    That the Fisher matrix is positive-semidefinite (instead of positive-definite)
    is a typical problem for optimization methods that involve the Fisher matrix inversion,
    such as \eqref{equ:polparam_update}.
    There are several resolutions to this,
    \eg using the pseudo-inverse, or adding a small positive-definite term,
    say $10^{-3} \mat{I}$.
    \label{item:natgrad}
\item Sampling-enabler Fisher:
    $\mat{C} \gets \fbamatsam(\vecb{\theta})$ from \eqref{equ:fisher_b_sampling},
    yielding approximate natural gradients (which are approximations to
    those of Scheme~\ref{item:natgrad} due some ignored terms,
    as described in \secref{sec:natgrad}).
    Here, a \emph{sampling-enabler} expression enables sampling because
    it contains only expectation terms.
    In exact settings, those expectations are computed exactly.
    \label{item:natgrad_sampling}
\end{enumerate}

In sampling-based approximation settings, there are two schemes approximating
Schemes~\ref{item:identity} and~\ref{item:natgrad_sampling}.
Here, approximation refers to estimating the gradient and Fisher of the bias by
their sample means, as described in \algref{alg:gradfisher_sampling}.
We did not perform approximation experiments corresponding to
Schemes~\ref{item:hessian} and~\ref{item:natgrad} because
the Hessian $\nabla^2 v_b$
(which is a straightforward differentiation of \eqref{equ:gradbias_naive})
and the bias Fisher $\fbamat$ \eqref{equ:fisher_b}
do not enable sampling-based approximation.

\figref{fig:biasoptcmp} shows the experiment results of bias-only optimization.
\textbf{In exact settings}, the natural gradients
(both Schemes~\ref{item:natgrad} and~\ref{item:natgrad_sampling}) mostly
yield higher bias than the standard gradients (Scheme~\ref{item:identity}),
as anticipated in \secref{sec:natgrad}.
Interestingly, as a preconditioning matrix, the proposed Fisher often outperforms
the modified Hessian (Scheme~\ref{item:hessian}), otherwise it is slightly worse.
More importantly, the sampling-enabler Fisher (computed exactly) maintains this benefit.

In exact settings, we also carried out comparisons
on seven different bias Fisher candidates for bias-only optimization.
This is to track the empirical effects of the series of approximation
(not including sampling-based approximation) that lead to
the sampling-enabler bias Fisher in \eqref{equ:fisher_b_sampling}.
As can be seen in \figref{fig:biasfishercmp}, the use of
$\sum_{t=0}^\infty |p_\pi^t(s|s_0) - p_\pi^\star(s)|$,
in lieu of $| h_\pi(s|s_0) |$ in \eqref{equ:fisher_b},
does not significantly degrade the advantage of the Fisher.
Then, dropping some terms in \eqref{equ:fisher_b} causes
acceptable degradation as anticipated.
It is interesting that increasing $\tabsminhat$
(\ie dropping fewer terms in \eqref{equ:fisher_b})
does not always result in better optimization performance as in Env-A2 and A3.
We observe that $\tabsminhat = 2$ gives us the best compromise across
at least, six benchmarking environments (\secref{sec:env_nbpwg}).

\textbf{In sampling-based approximation settings}, the stochastic natural gradients
outperform the stochastic standard gradients whenever the sample size is sufficient
in order to obtain sufficiently accurate estimates of the Fisher matrix.
Note that the behavior of sampling-based approximation is a function of sample sizes,
following the strong law of large numbers.
Recall that if $Y_1, \ldots, Y_m$ are i.i.d random variables, then
the deviation between the sample mean $\bar{Y} \eqdef \frac{1}{m} \sum_{i=1}^m Y_i$ and
its expectation $\E{}{\bar{Y}} = \E{}{Y_i}$ (due to unbiasedness of the sample mean) is
of order $\bigO{1/\sqrt{m}}$.
Assuming that $Y_i \in [y_{min}, y_{max}]$ almost surely and
for $\delta \in (0, 1)$, with probability greater than $1 - \delta$, we have
$|\bar{Y} - \E{}{Y_i}|
\le |y_{min} - y_{max}| \sqrt{ \frac{\log (2/\delta)}{2m} }$,
which is derived through the Hoeffding's inequality.

\subsection{Initial barrier parameter tuning}
\label{sec:barrierparamtuning}

The proposed \algref{alg:gainbiasbarrier_optim} requires
an initial value of the barrier parameter $\beta_0$.
For tuning $\beta_0$, we run the exact version of \algref{alg:gainbiasbarrier_optim}
several times with $\beta_0 \in \{0.01, 0.1, 1, 10, 100, 1000, 10000 \}$
representing a range of values from small to large extremes.
We note that in general, the value of $\beta_0$ goes hand in hand with
the shrinking parameter $\beta_{\mathrm{div}}$ and the number of outer iterations
in \algref{alg:gainbiasbarrier_optim}.
For a recipe for choosing those parameters' values, we refer the reader to
\citet[\page{570}]{boyd_2004_cvxopt} and \citet[\page{104, 123}]{bertsekas_1996_copt}.
Also note that the barrier parameter $\beta$ is quite analogous to the penalty parameter
of penalty methods, such as the penalty parameter $\phi$ in \eqref{equ:optim_penalty}.

\figref{fig:barrierparam_env_a_b} shows the experimental results on
Env-A1, A2, A3, B1, B2, and~B3.
\textbf{For Env-A}, a small $\beta_0$ yields best performance
(\ie lowest absolute bias error: the blue color).
This is expected because all policies of Env-A are gain optimal.
In model-free RL, if we knew that the gain landscape is flat, then
we would optimize only the bias without the barrier by setting $\beta_0 \gets 0$.
Thus for this kind of environments, values towards the small extreme are
suitable for $\beta_0$.
Note that in exact settings, $\beta_0$ matters to the gradient-ascent update
only through the multiplication of $\beta$ with the gain Fisher $\fgamat$
as in \eqref{equ:biasbarrier_precond}.
This is because the gain gradient $\nabla g(\vecb{\theta})$ is a zero vector for Env-A.

\textbf{On the other hand for Env-B}, values towards the large extreme are
suitable for $\beta_0$.
As mentioned in \secref{sec:bwoptim_algo}, a small $\beta_0$ tends to induce
ill-conditioned bias-barrier optimization landscapes, which
should be avoided in the early outer iterations of \algref{alg:gainbiasbarrier_optim}.
Recall that for Env-B the barrier function is essential since
such an environment type requires gain-then-bias-barrier bi-level optimization.

\begin{landscape}
\begin{figure}
\centering

\begin{subfigure}{0.215\textwidth}
\includegraphics[width=\textwidth]{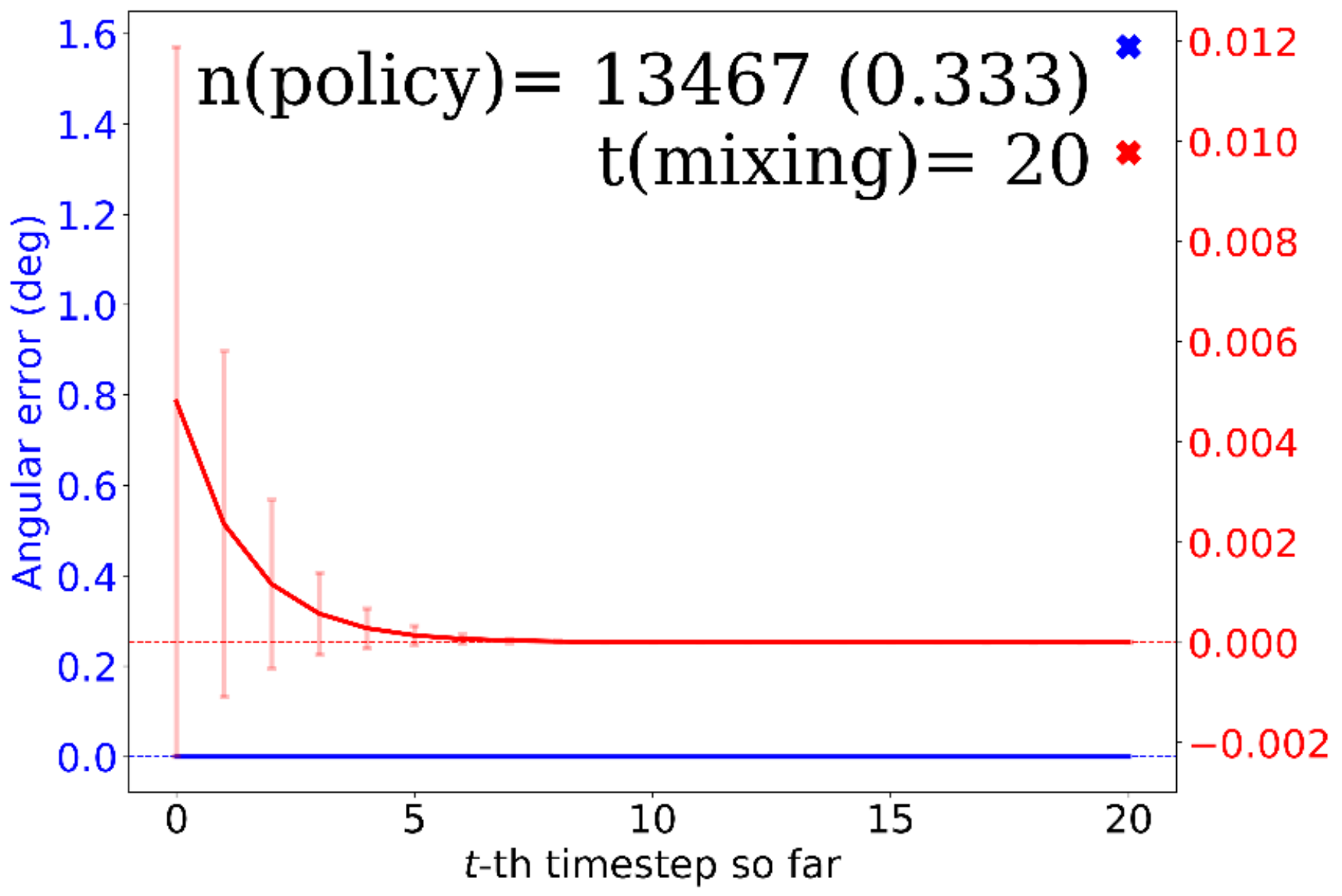}
\end{subfigure}
\begin{subfigure}{0.215\textwidth}
\includegraphics[width=\textwidth]{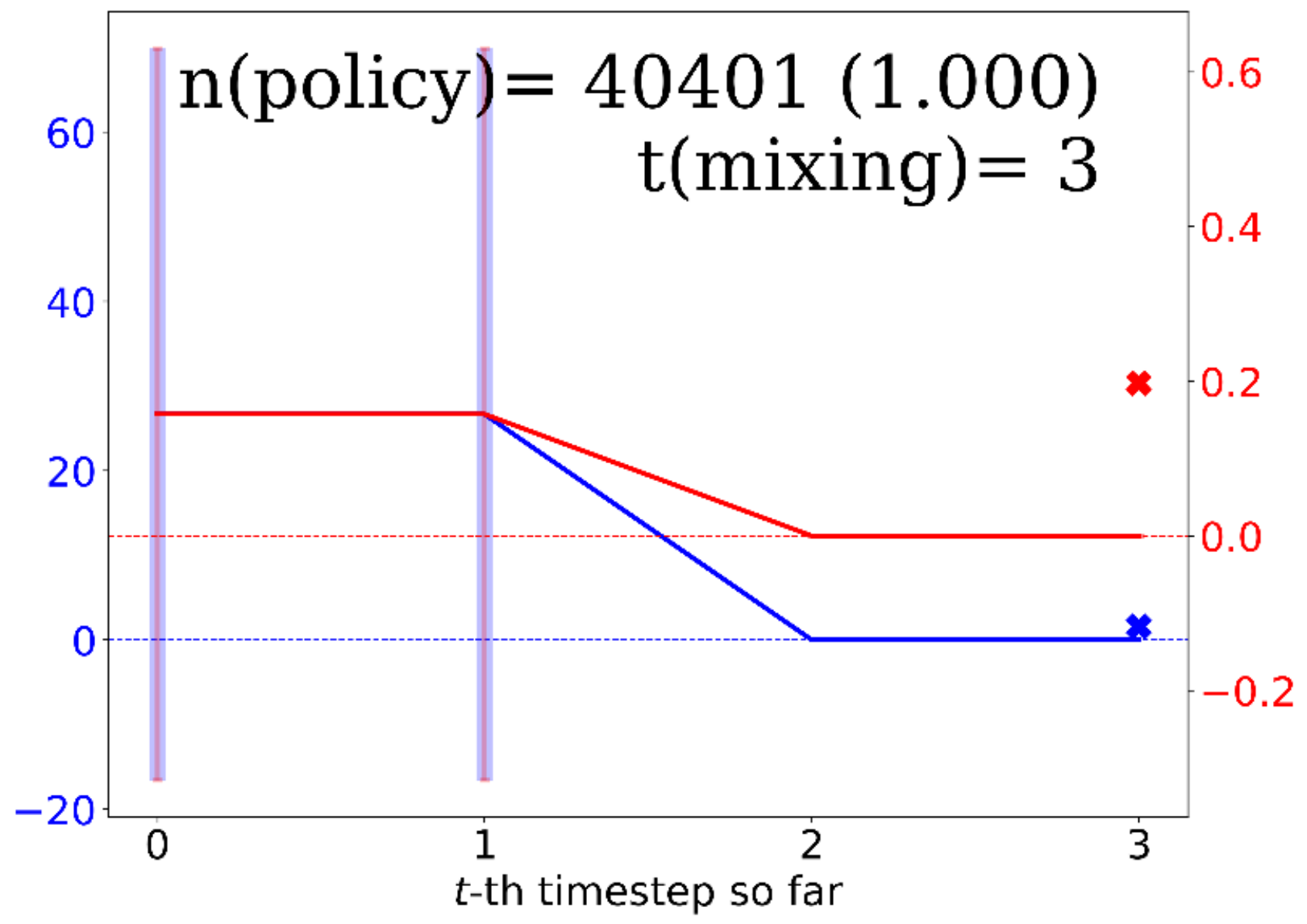}
\end{subfigure}
\begin{subfigure}{0.215\textwidth}
\includegraphics[width=\textwidth]{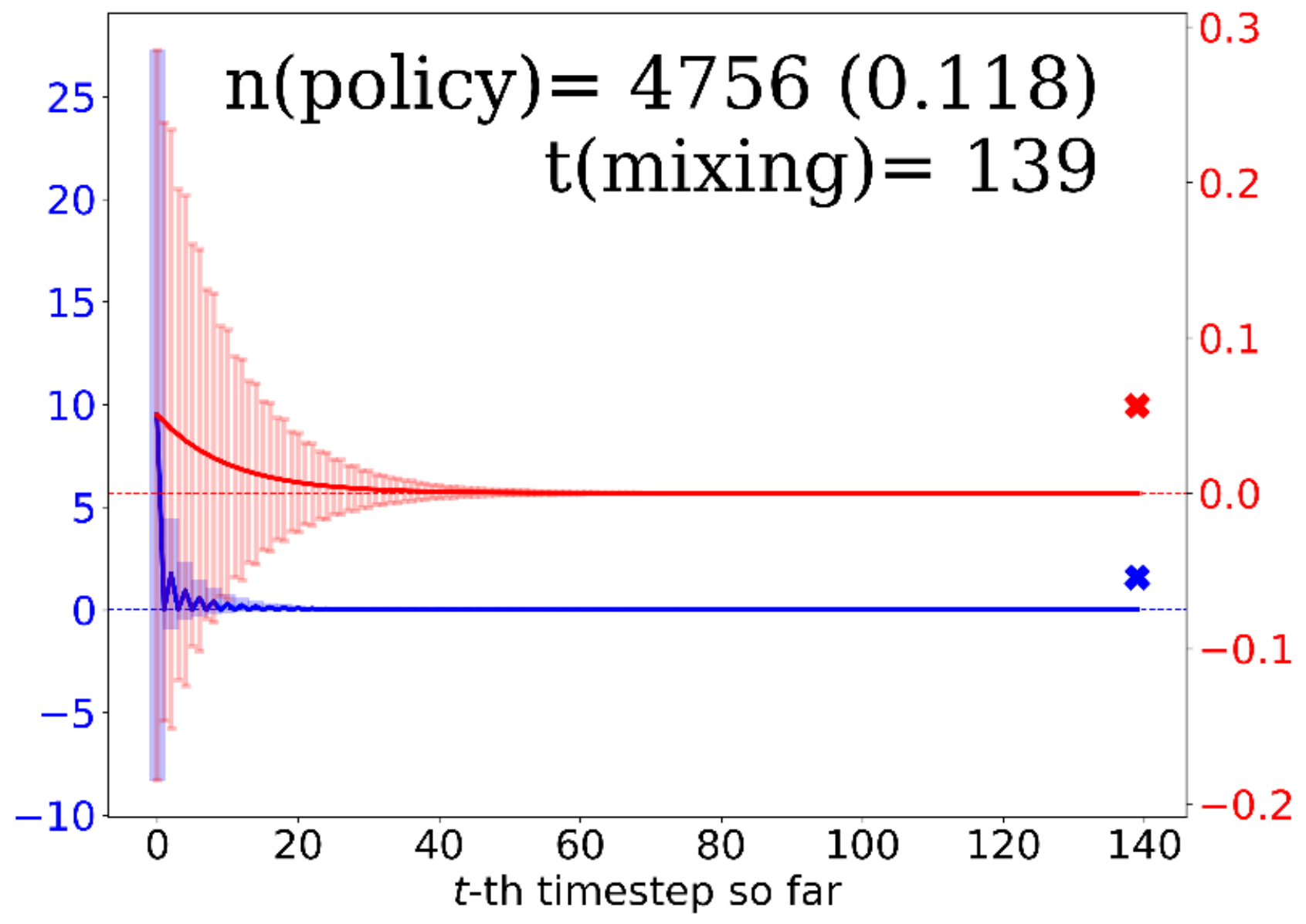}
\end{subfigure}
\begin{subfigure}{0.215\textwidth}
\includegraphics[width=\textwidth]{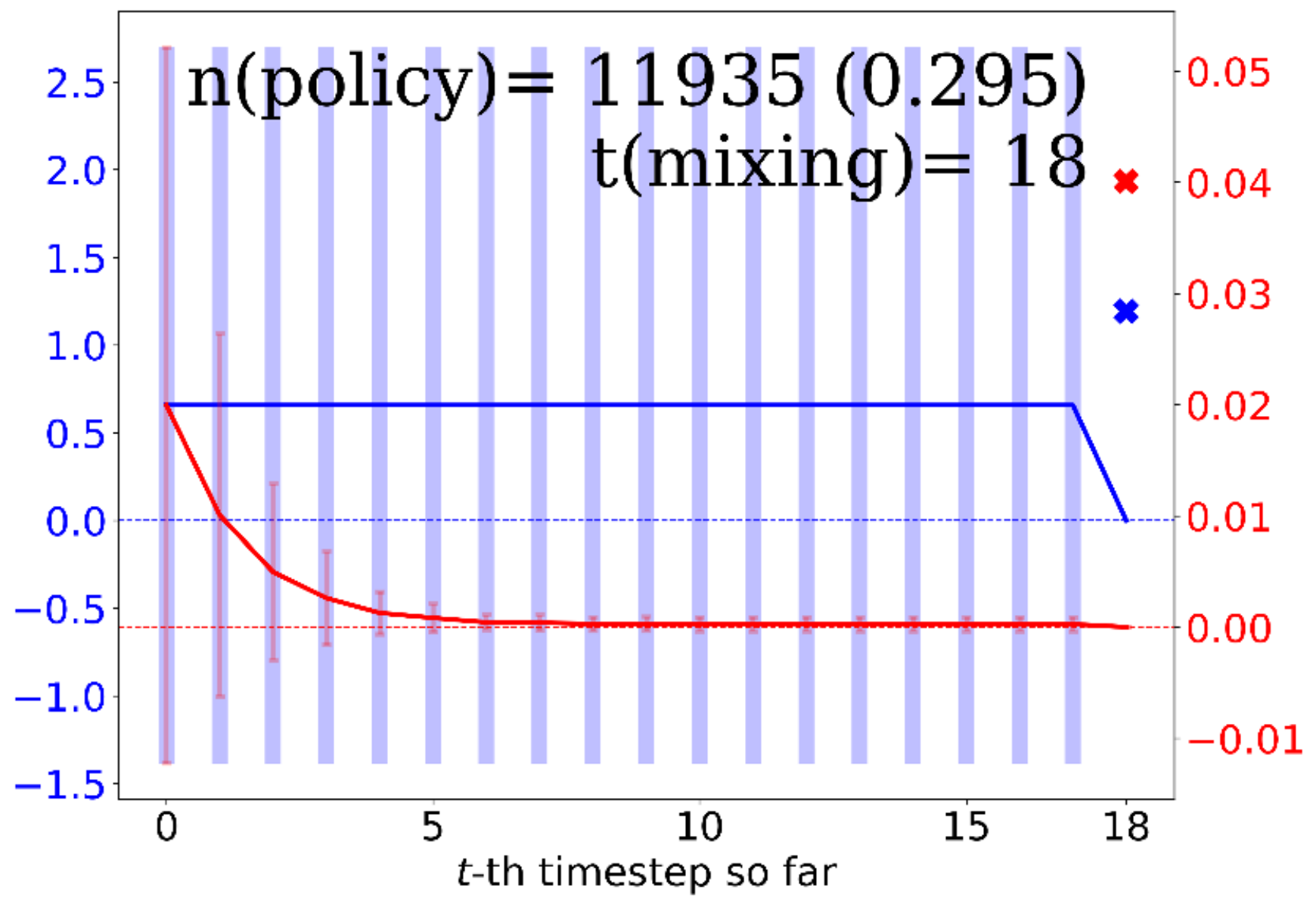}
\end{subfigure}
\begin{subfigure}{0.215\textwidth}
\includegraphics[width=\textwidth]{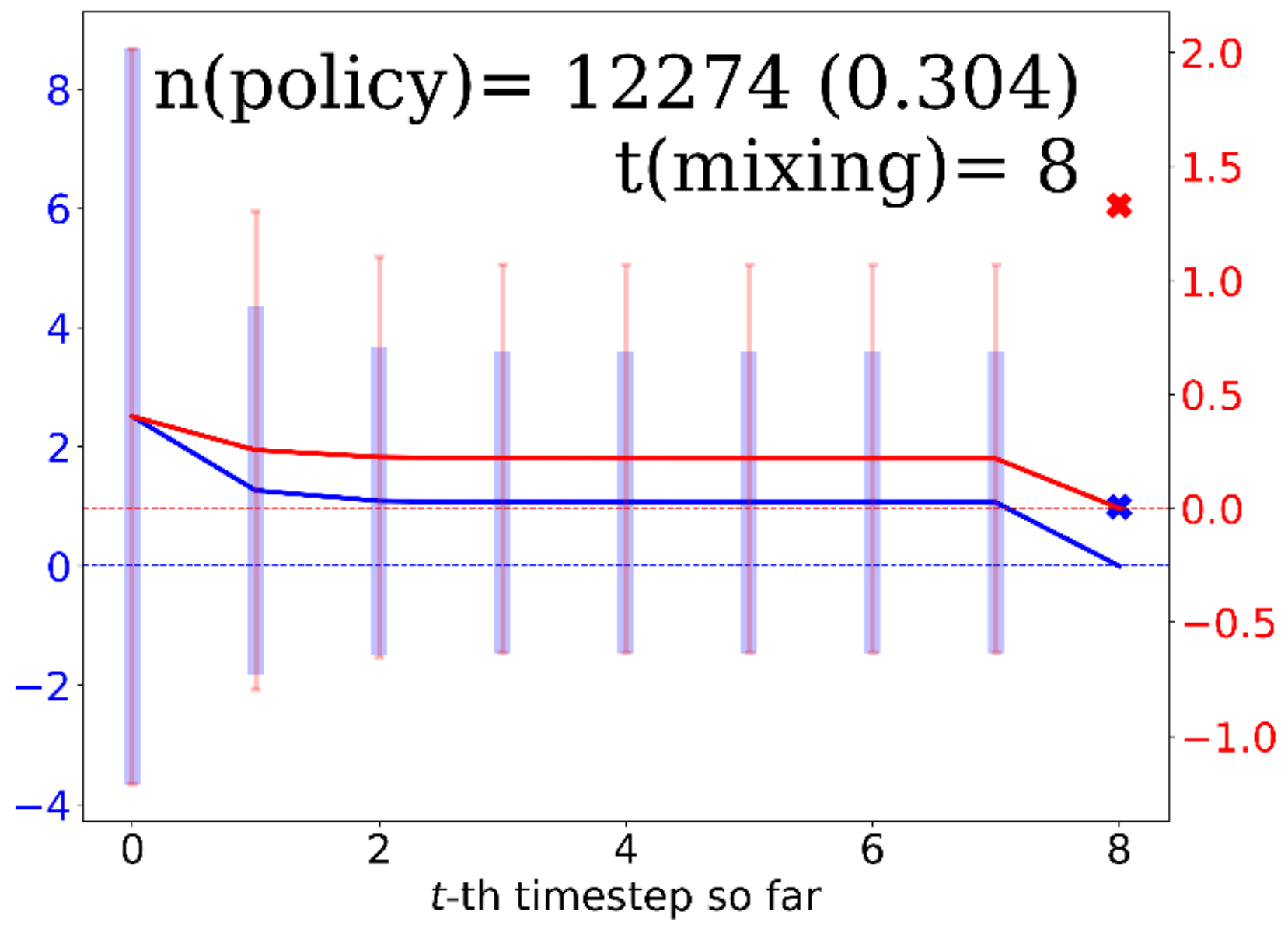}
\end{subfigure}
\begin{subfigure}{0.215\textwidth}
\includegraphics[width=\textwidth]{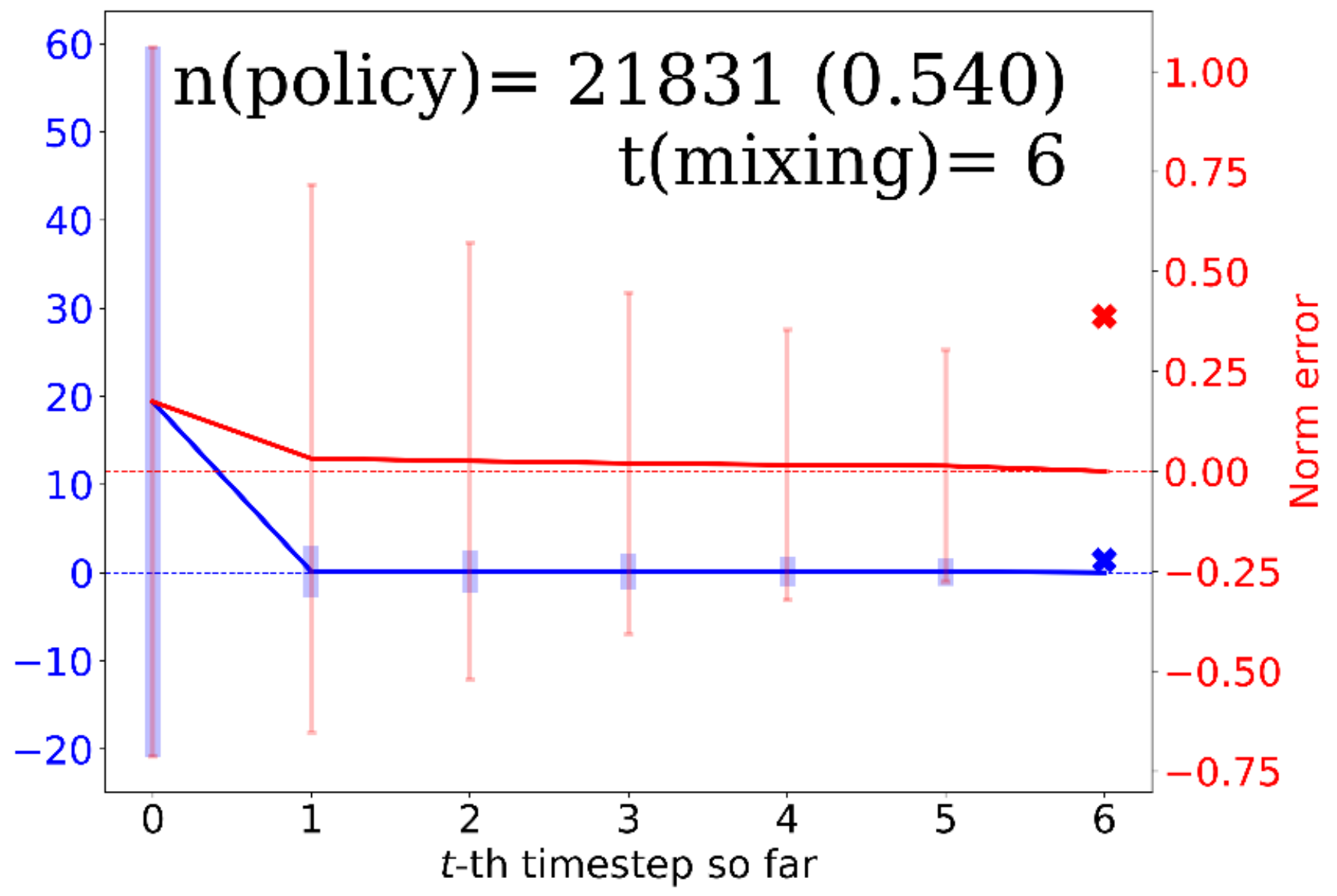}
\end{subfigure}

\begin{subfigure}{0.215\textwidth}
\includegraphics[width=\textwidth]{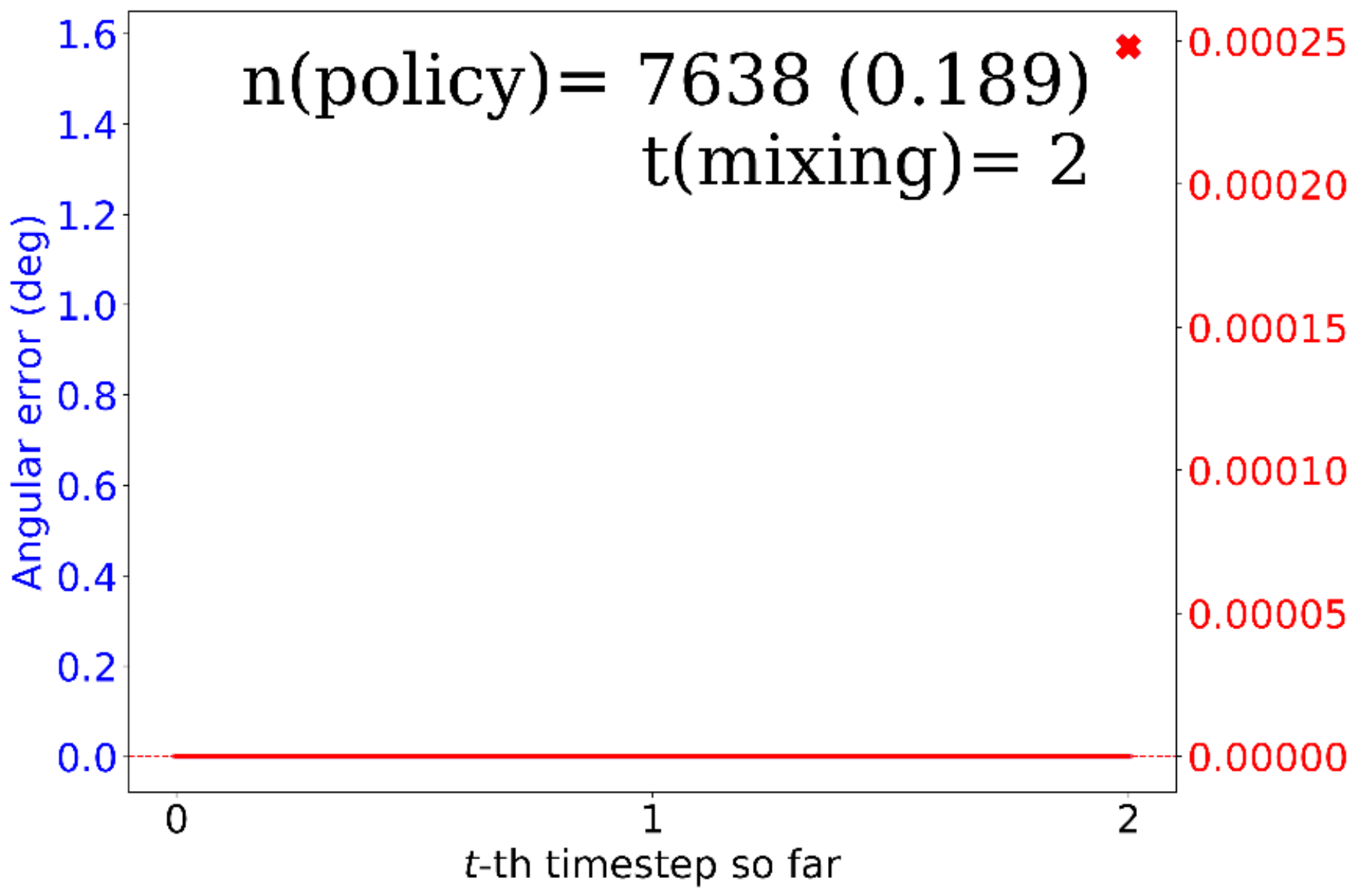}
\end{subfigure}
\begin{subfigure}{0.215\textwidth}
\includegraphics[width=\textwidth]{fig/gradbias-decompose/gradsamplingexact__1__OneDimensionStateAndTwoActionPolicyNetwork__Tor_20210307-v1}
\end{subfigure}
\begin{subfigure}{0.215\textwidth}
\includegraphics[width=\textwidth]{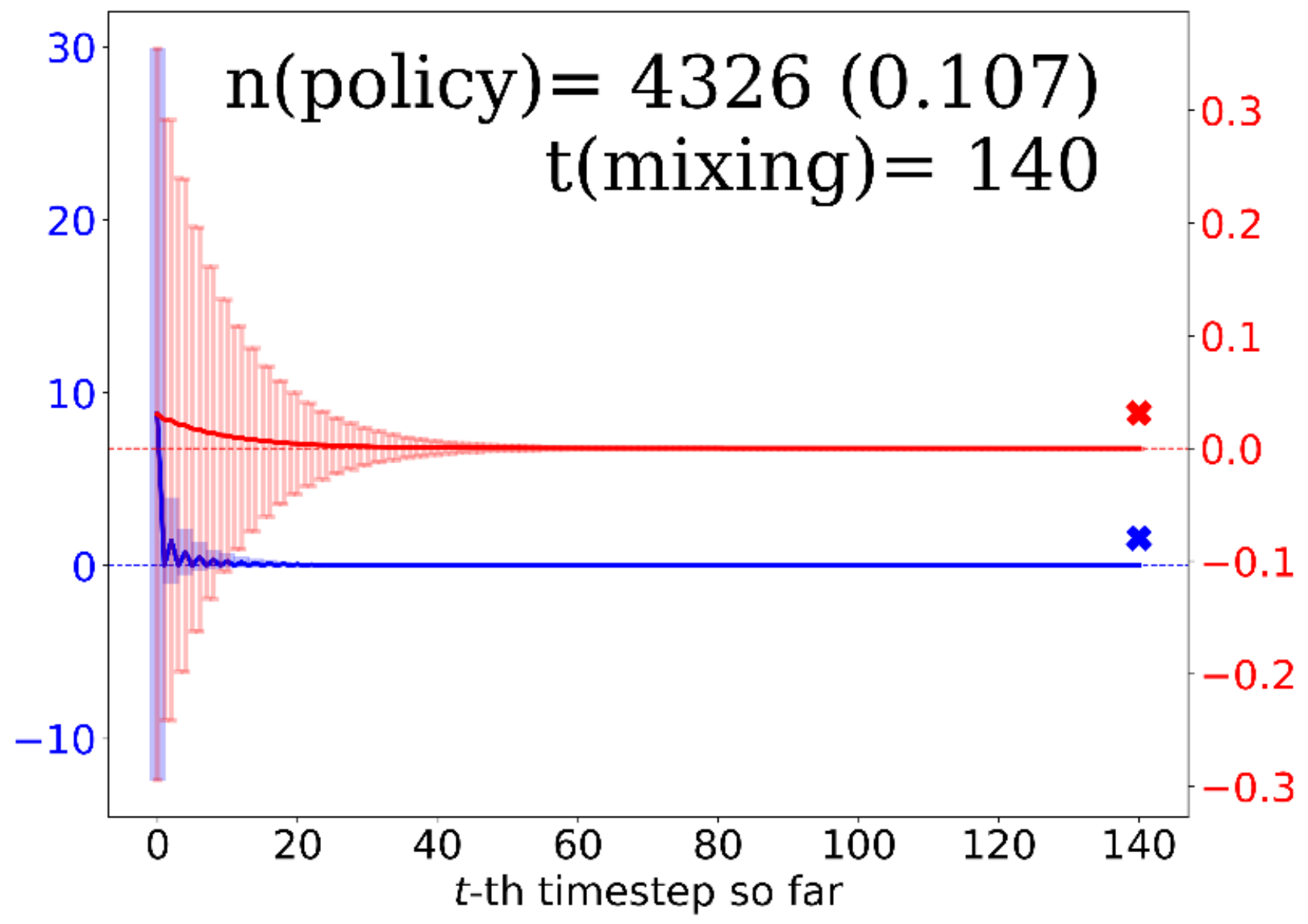}
\end{subfigure}
\begin{subfigure}{0.215\textwidth}
\includegraphics[width=\textwidth]{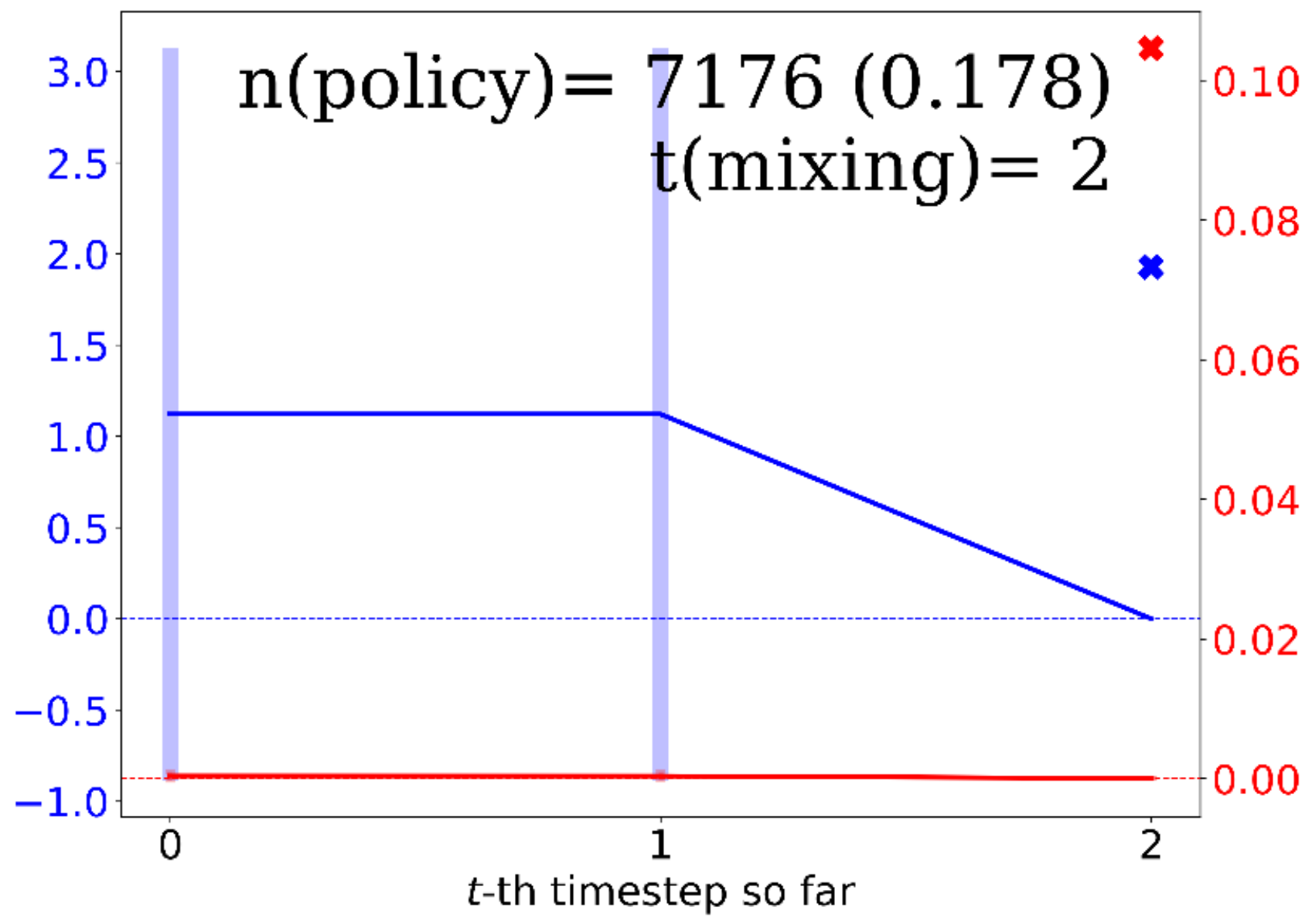}
\end{subfigure}
\begin{subfigure}{0.215\textwidth}
\includegraphics[width=\textwidth]{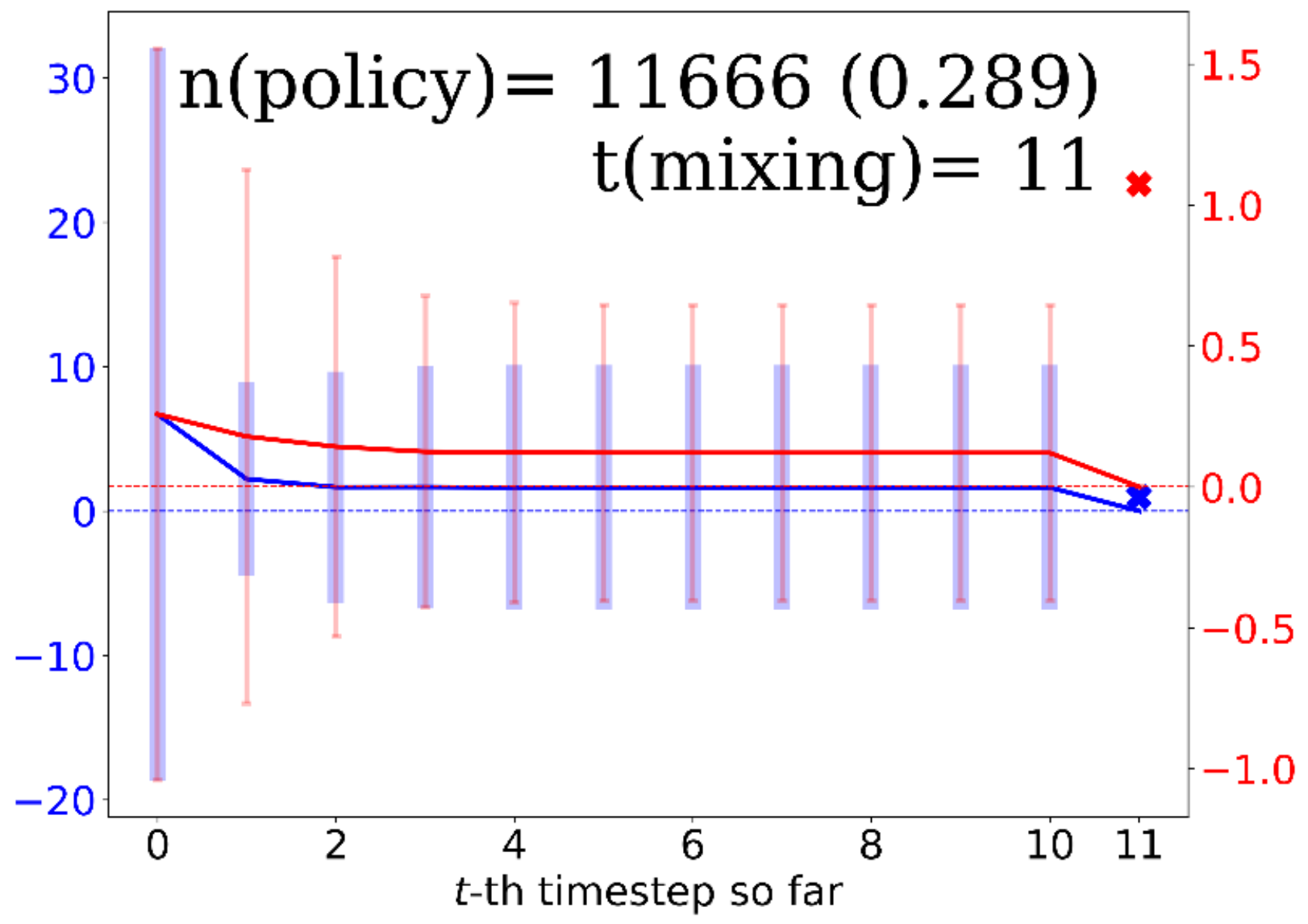}
\end{subfigure}
\begin{subfigure}{0.215\textwidth}
\includegraphics[width=\textwidth]{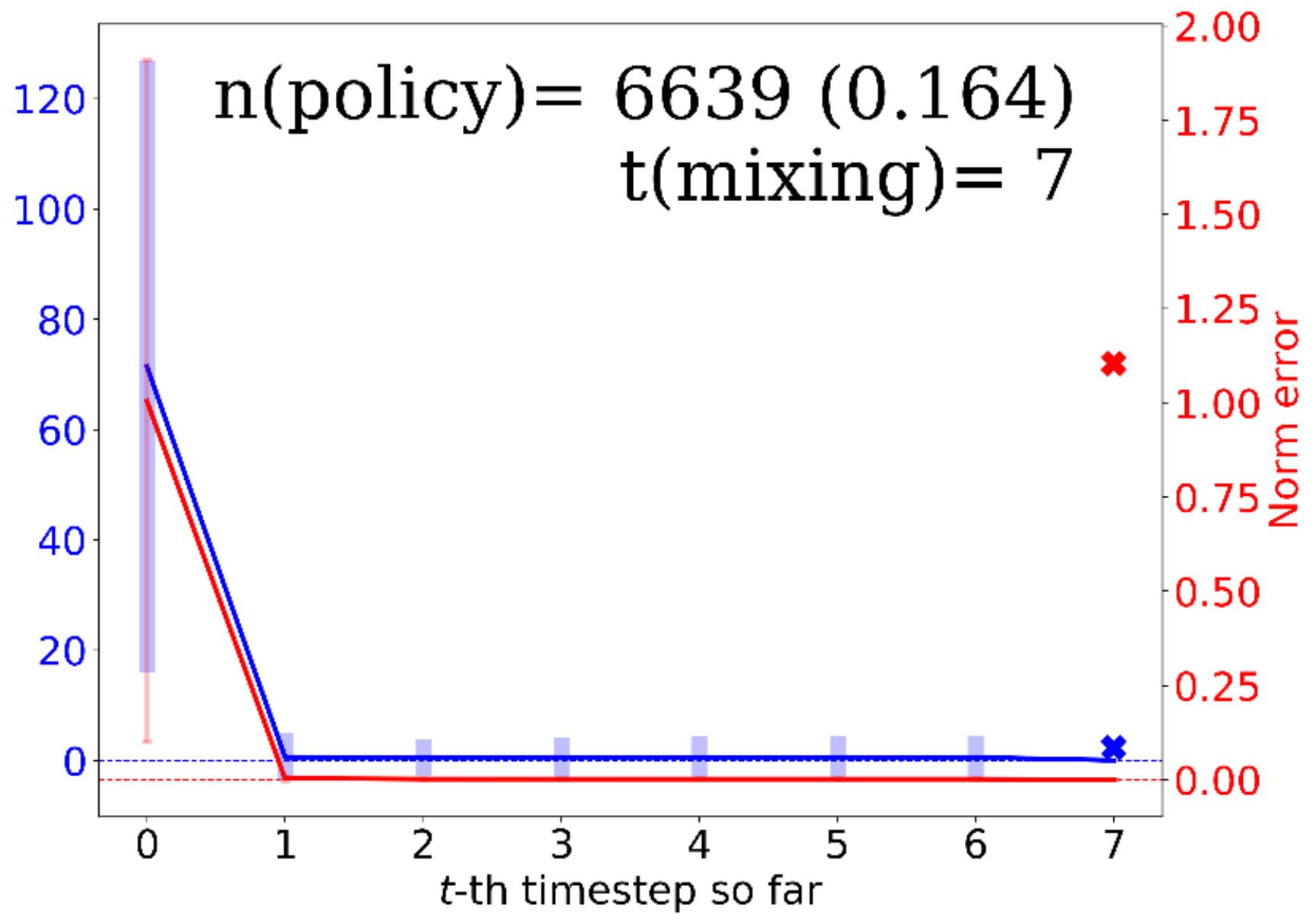}
\end{subfigure}

\begin{subfigure}{0.215\textwidth}
\includegraphics[width=\textwidth]{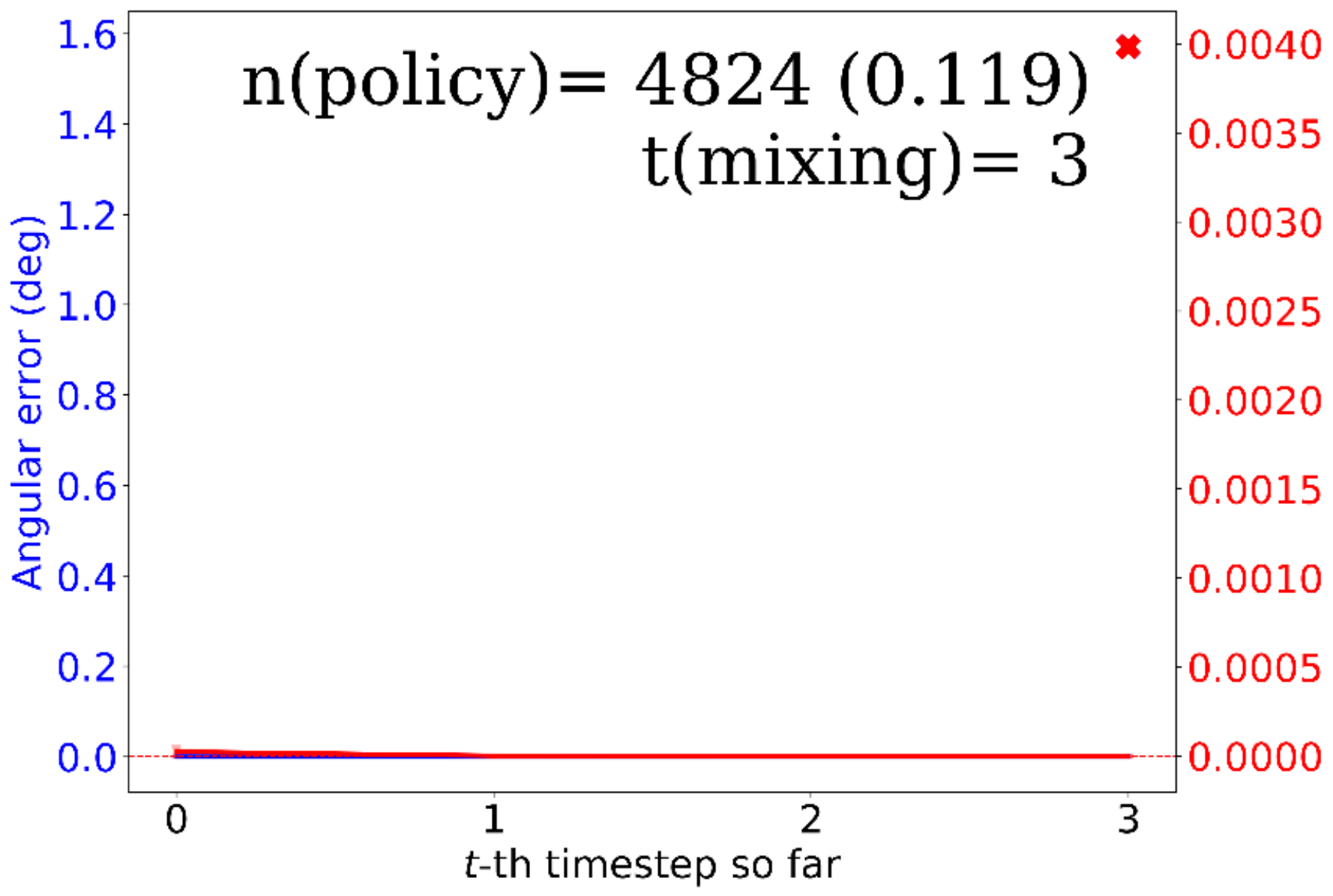}
\end{subfigure}
\begin{subfigure}{0.215\textwidth}
\includegraphics[width=\textwidth]{fig/gradbias-decompose/gradsamplingexact__1__OneDimensionStateAndTwoActionPolicyNetwork__Tor_20210307-v1}
\end{subfigure}
\begin{subfigure}{0.215\textwidth}
\includegraphics[width=\textwidth]{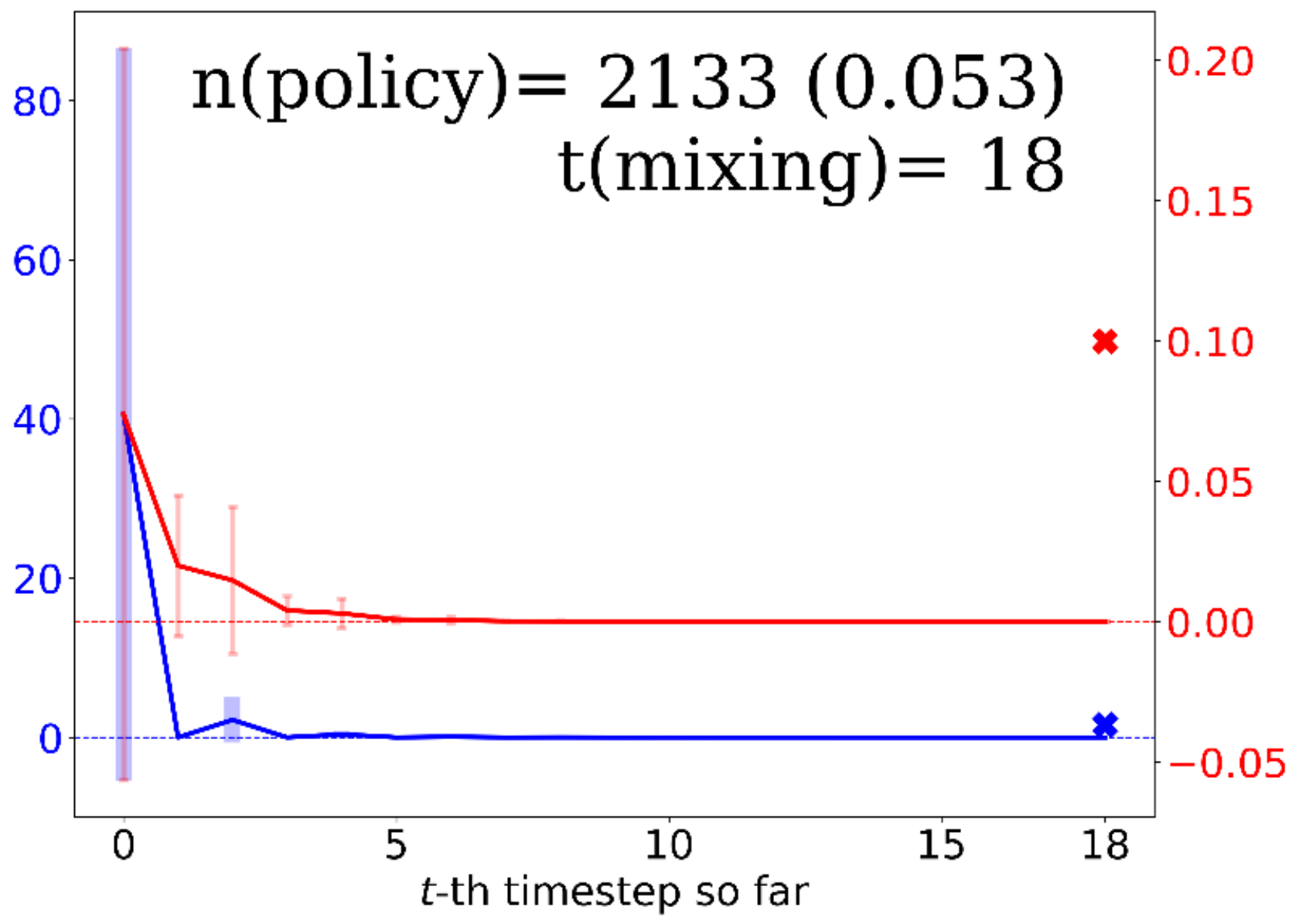}
\end{subfigure}
\begin{subfigure}{0.215\textwidth}
\includegraphics[width=\textwidth]{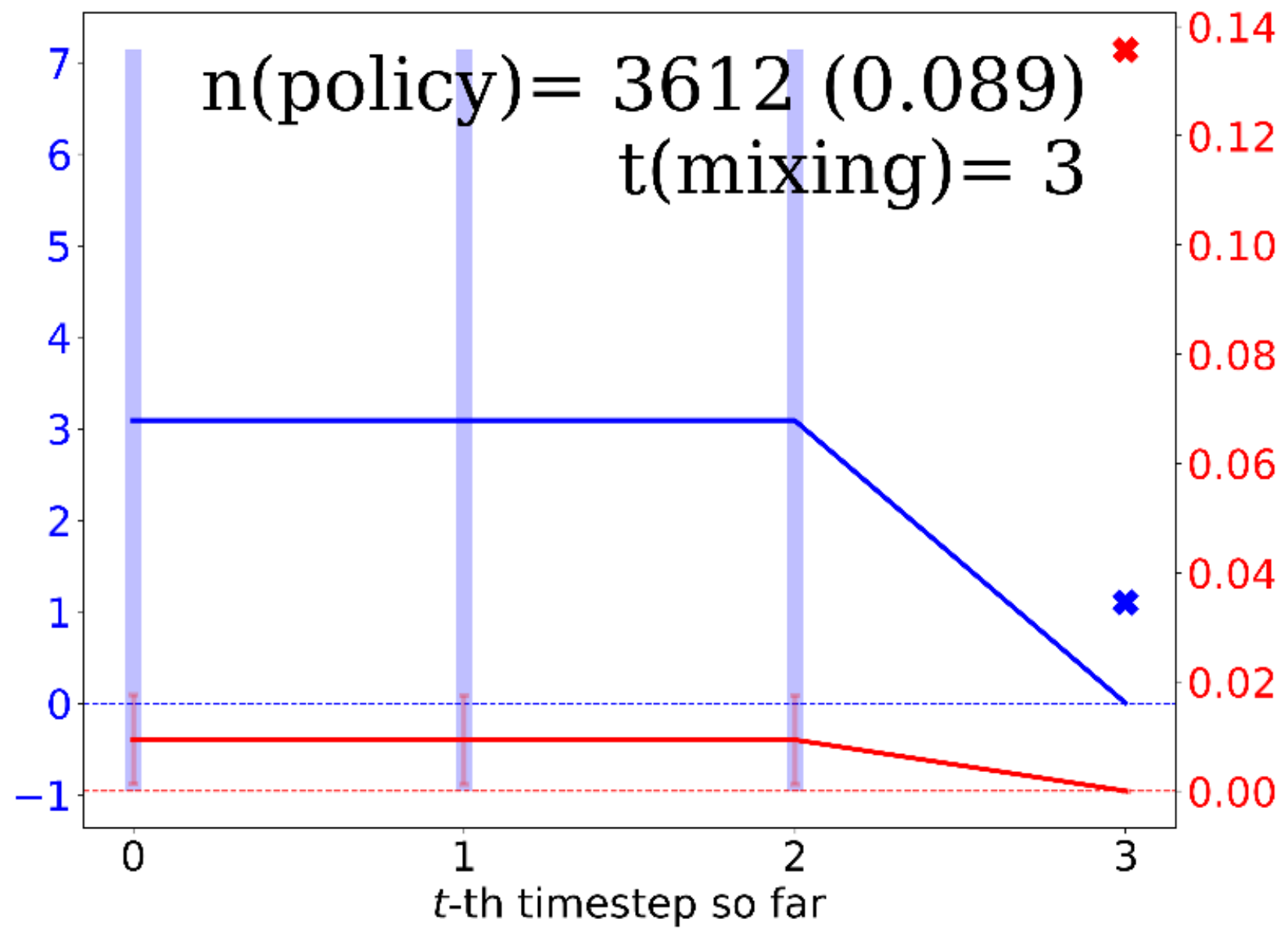}
\end{subfigure}
\begin{subfigure}{0.215\textwidth}
\includegraphics[width=\textwidth]{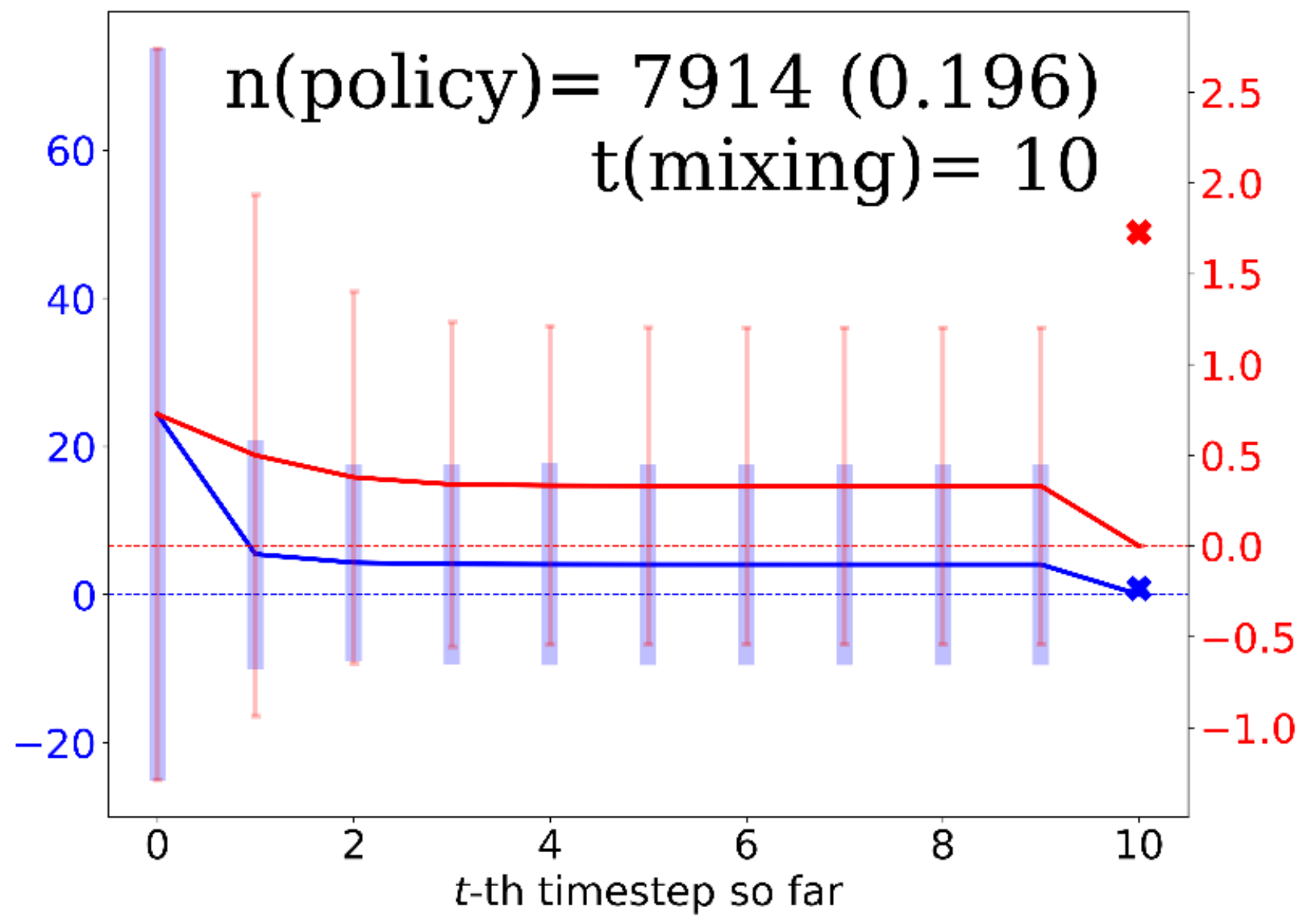}
\end{subfigure}
\begin{subfigure}{0.215\textwidth}
\includegraphics[width=\textwidth]{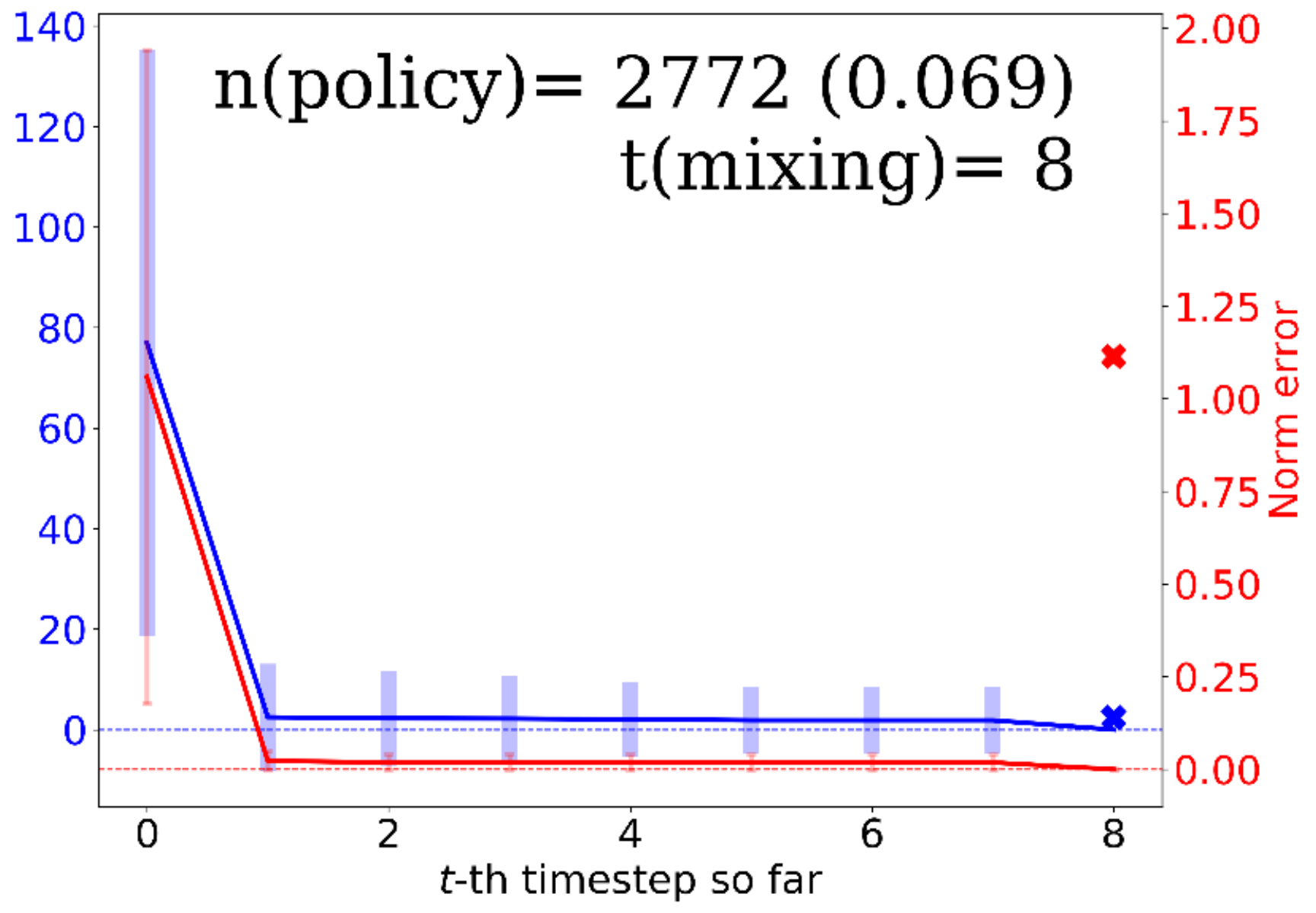}
\end{subfigure}

\begin{subfigure}{0.215\textwidth}
\includegraphics[width=\textwidth]{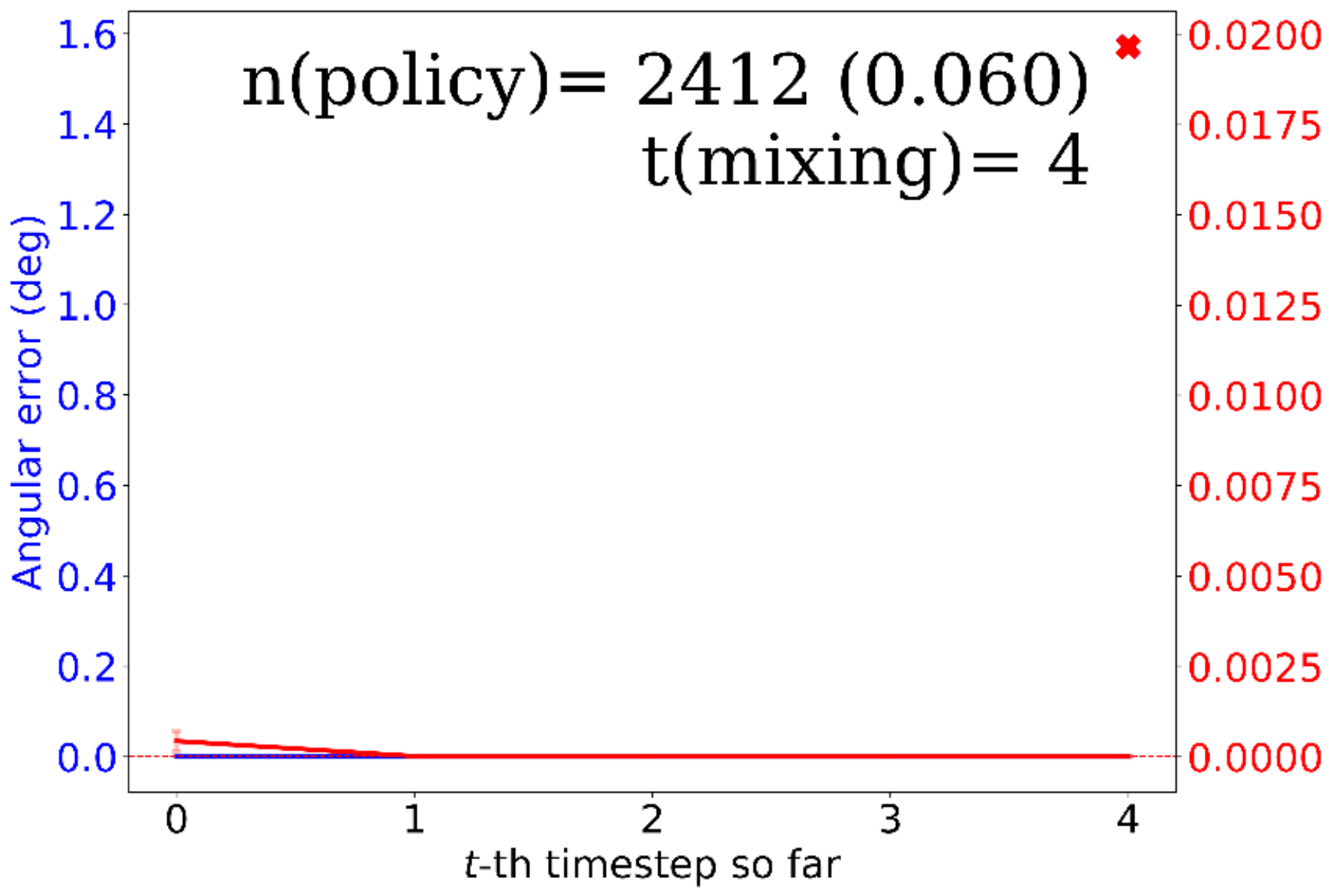}
\end{subfigure}
\begin{subfigure}{0.215\textwidth}
\includegraphics[width=\textwidth]{fig/gradbias-decompose/gradsamplingexact__1__OneDimensionStateAndTwoActionPolicyNetwork__Tor_20210307-v1}
\end{subfigure}
\begin{subfigure}{0.215\textwidth}
\includegraphics[width=\textwidth]{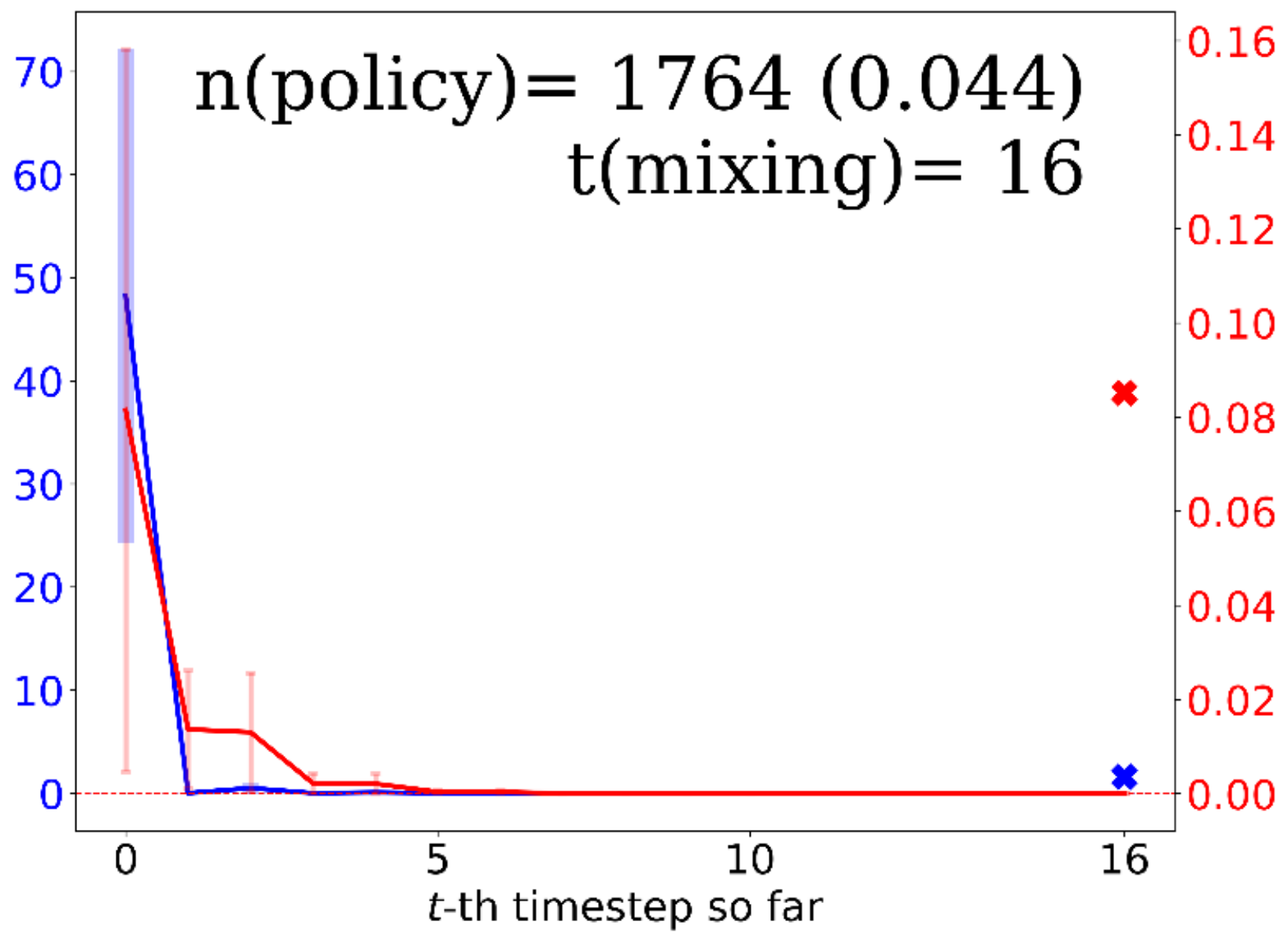}
\end{subfigure}
\begin{subfigure}{0.215\textwidth}
\includegraphics[width=\textwidth]{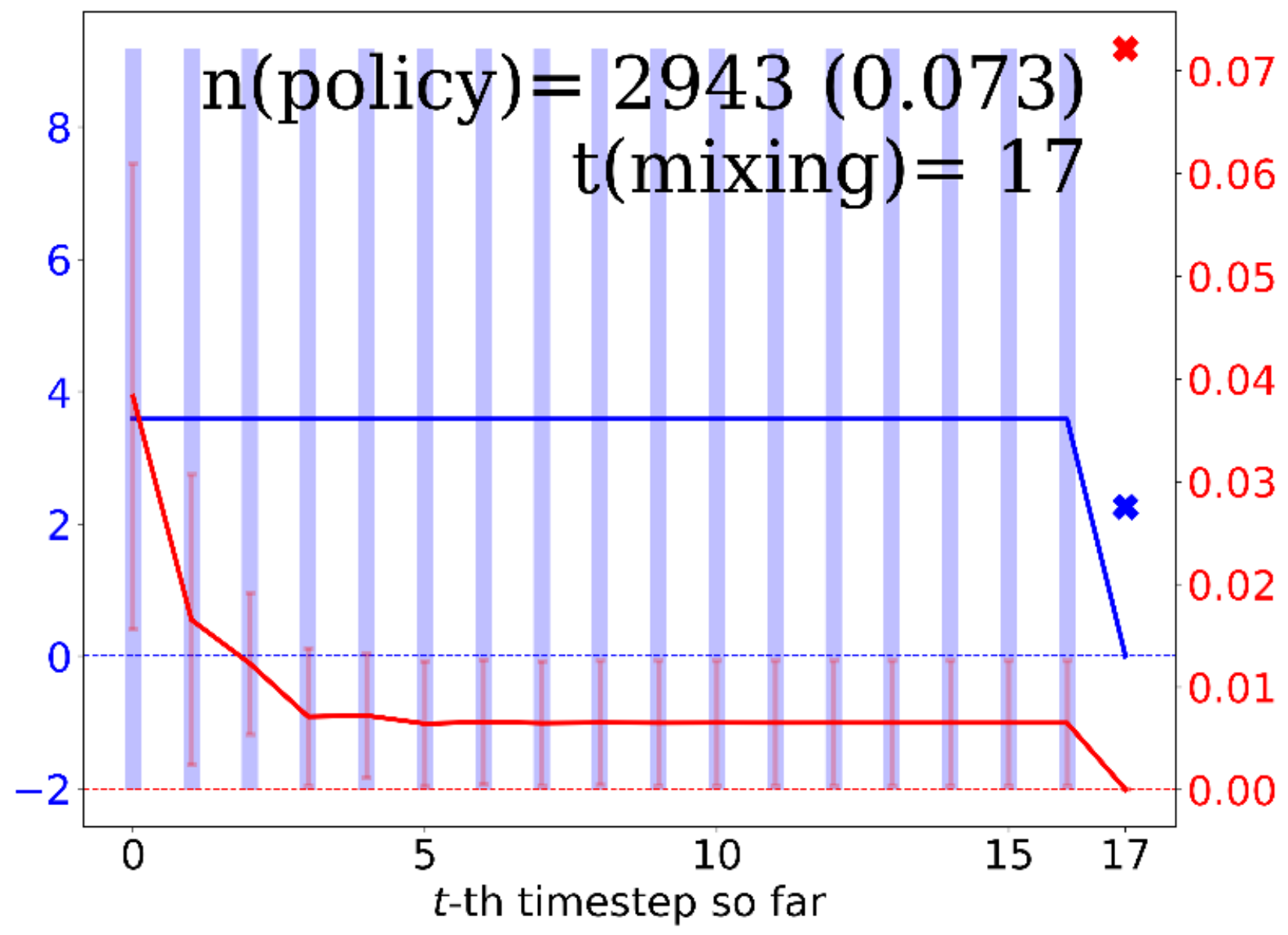}
\end{subfigure}
\begin{subfigure}{0.215\textwidth}
\includegraphics[width=\textwidth]{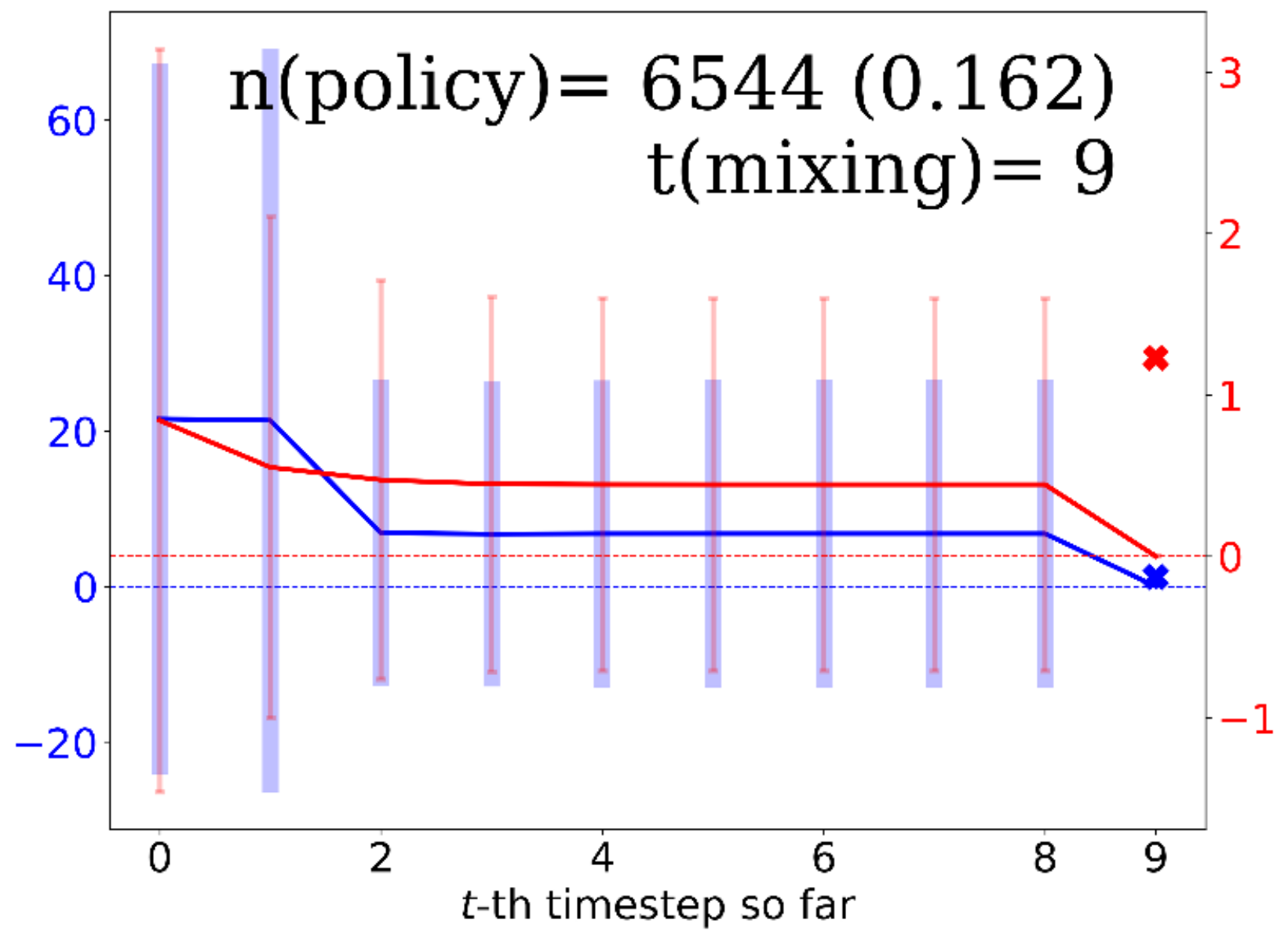}
\end{subfigure}
\begin{subfigure}{0.215\textwidth}
\includegraphics[width=\textwidth]{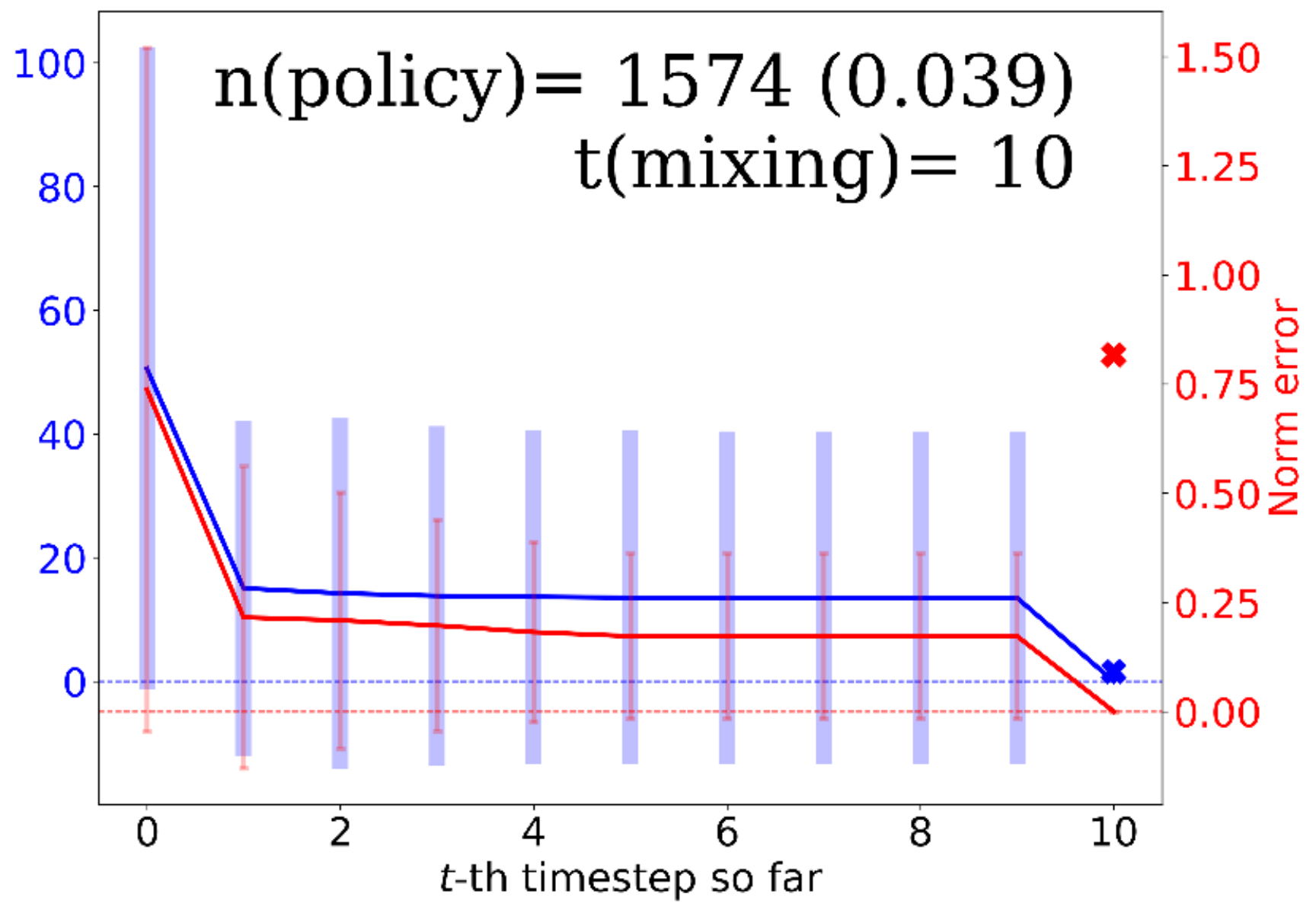}
\end{subfigure}

\begin{subfigure}{0.215\textwidth}
\includegraphics[width=\textwidth]{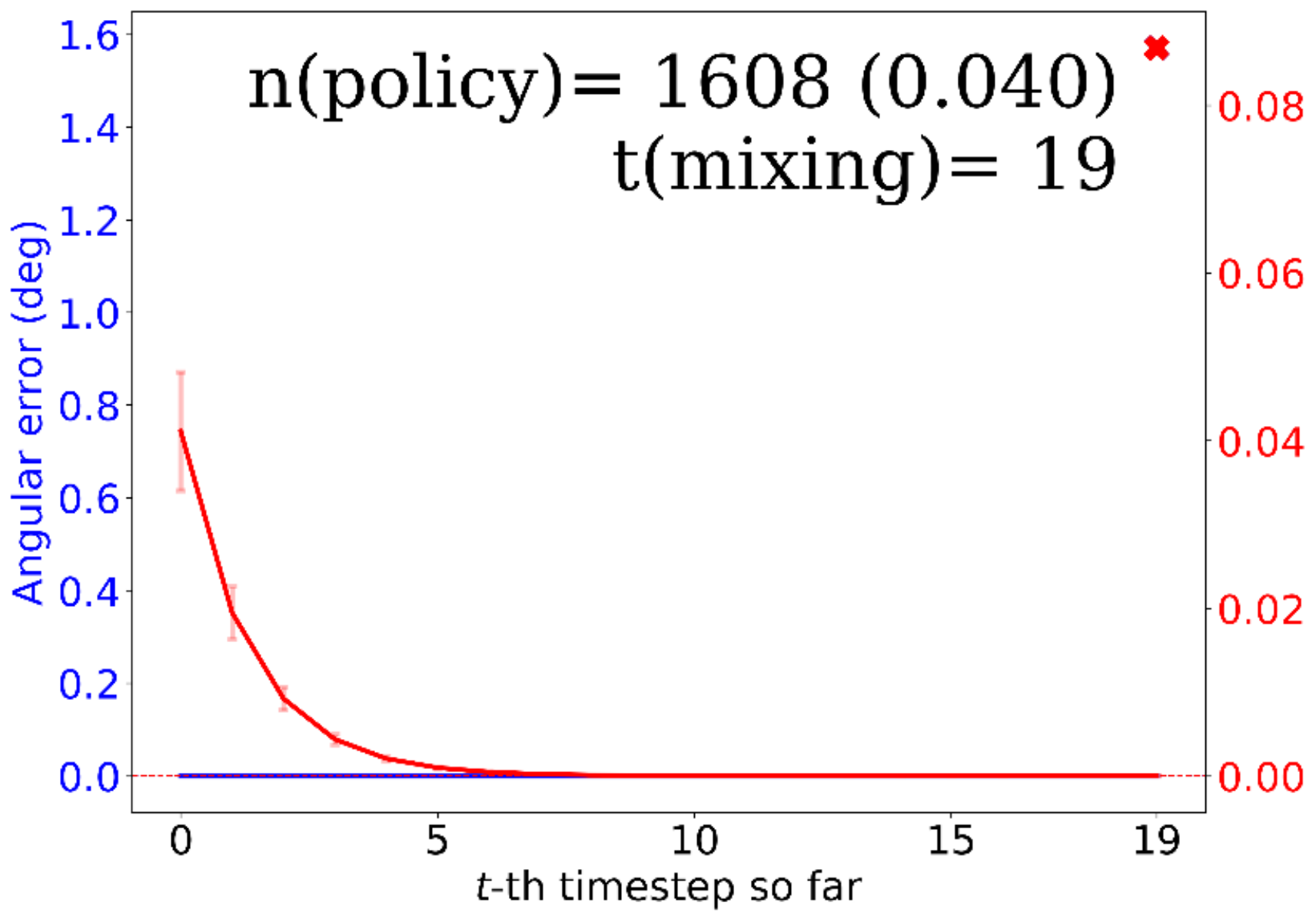}
\subcaption*{Env-A1}
\end{subfigure}
\begin{subfigure}{0.215\textwidth}
\includegraphics[width=\textwidth]{fig/gradbias-decompose/gradsamplingexact__1__OneDimensionStateAndTwoActionPolicyNetwork__Tor_20210307-v1}
\subcaption*{Env-A2}
\end{subfigure}
\begin{subfigure}{0.215\textwidth}
\includegraphics[width=\textwidth]{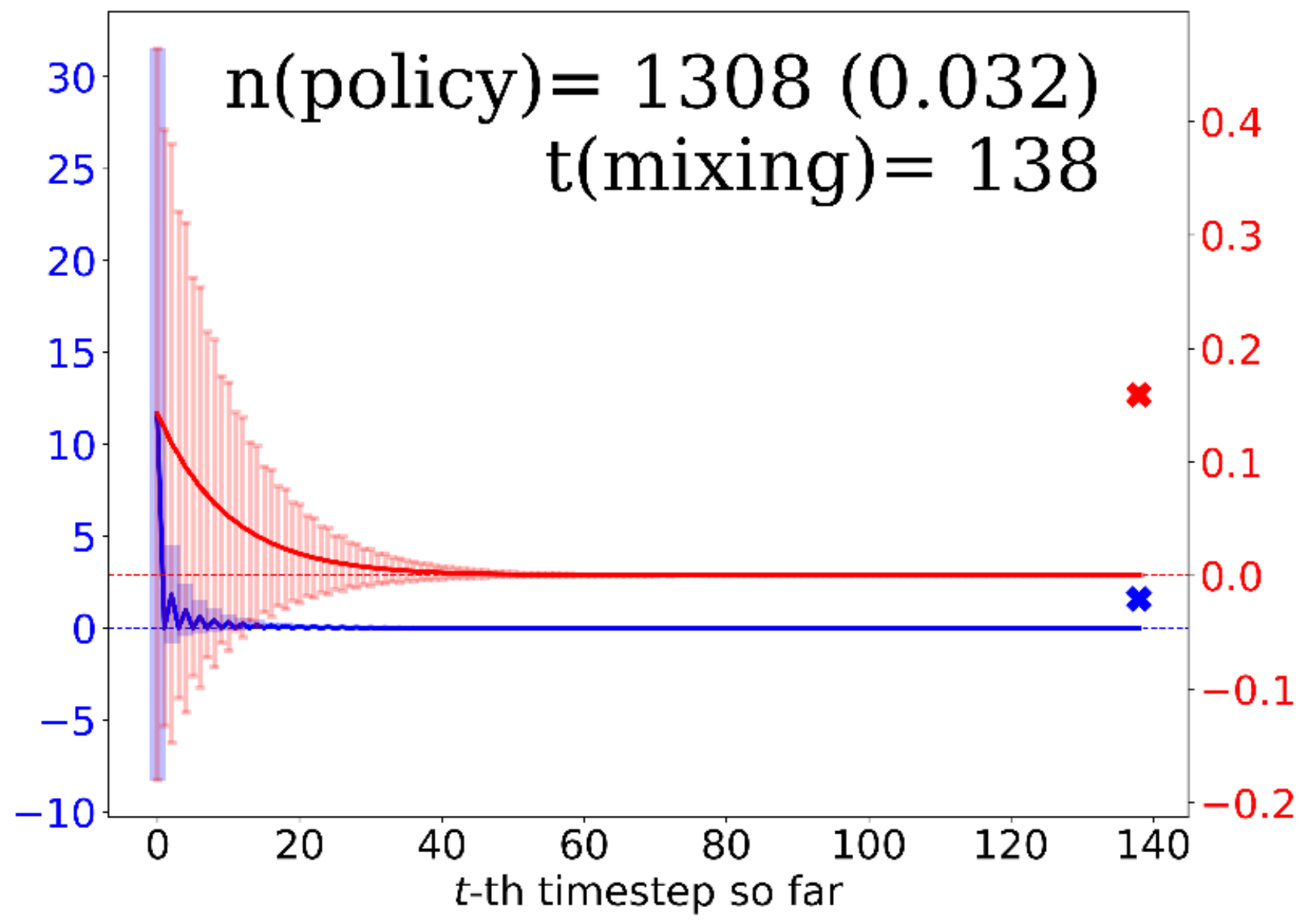}
\subcaption*{Env-A3}
\end{subfigure}
\begin{subfigure}{0.215\textwidth}
\includegraphics[width=\textwidth]{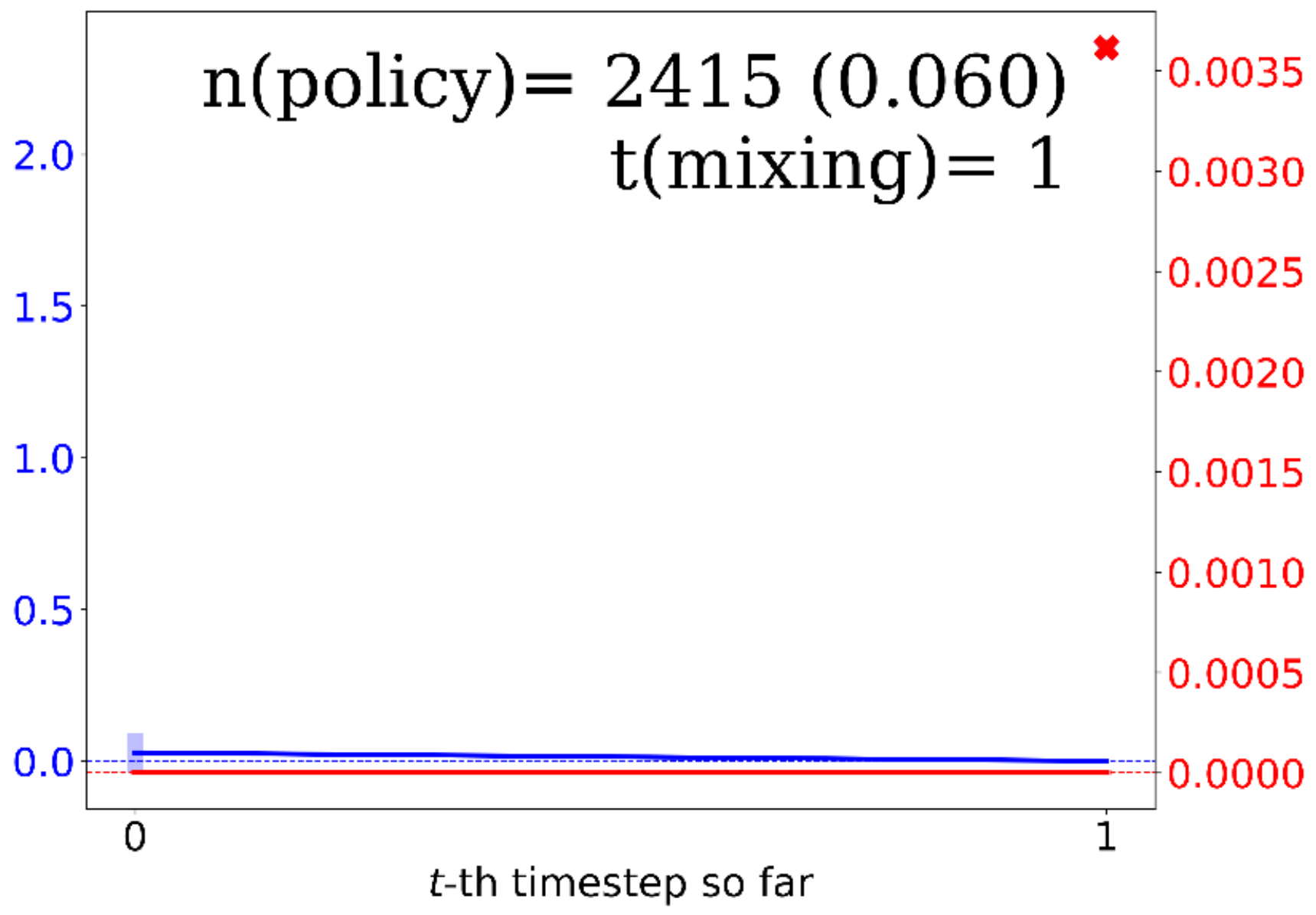}
\subcaption*{Env-B1}
\end{subfigure}
\begin{subfigure}{0.215\textwidth}
\includegraphics[width=\textwidth]{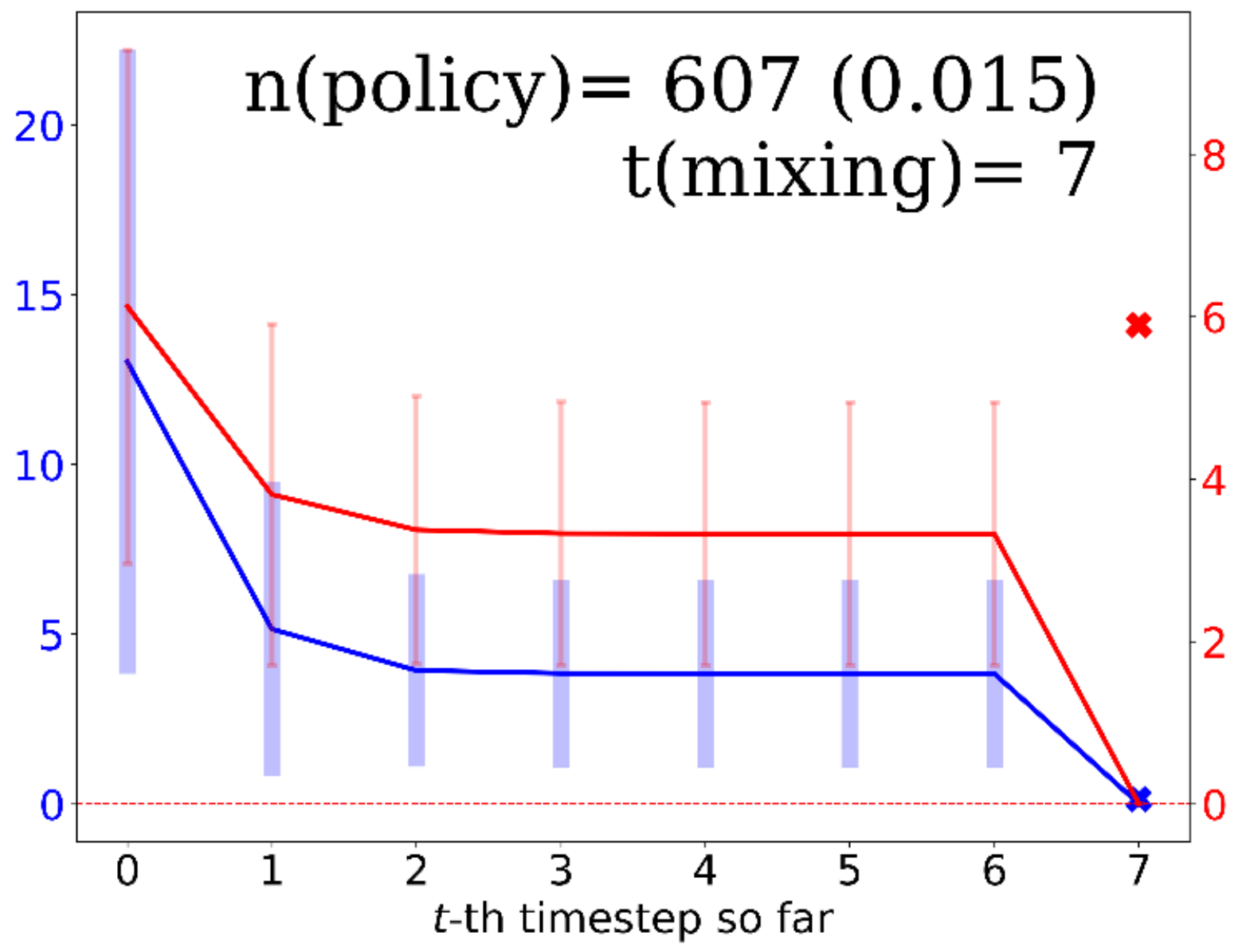}
\subcaption*{Env-B2}
\end{subfigure}
\begin{subfigure}{0.215\textwidth}
\includegraphics[width=\textwidth]{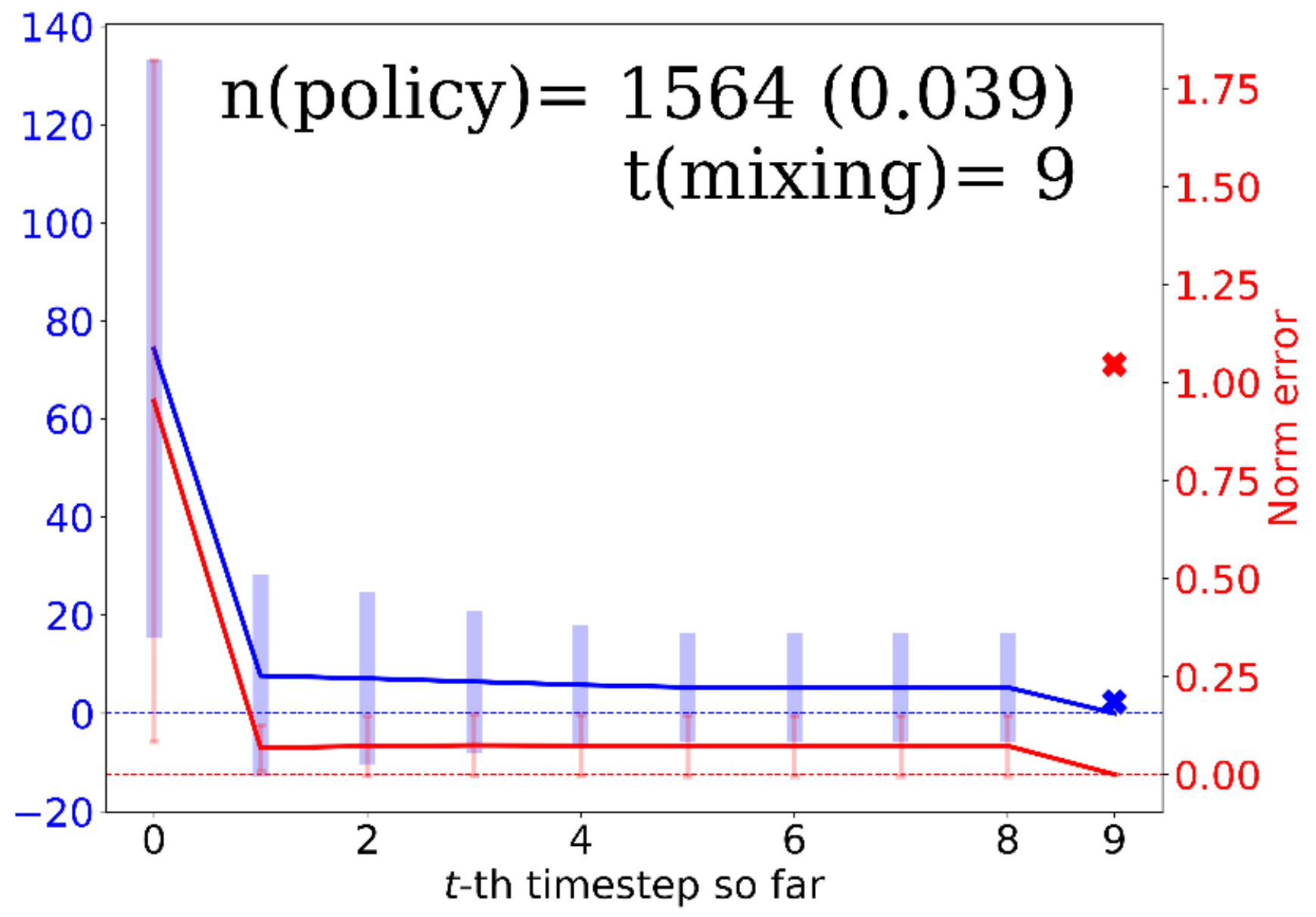}
\subcaption*{Env-B3}
\end{subfigure}

\caption{The error of the exact gradual summation of bias gradient terms
with respect to the exact full bias gradient \eqref{equ:grad_bias}.
The blue line indicates \textcolor{blue}{the angular error}, whereas
the red \textcolor{red}{the $\ell_2$-norm error}.
These bias gradients are of policies, whose mixing time is the most common (top row)
in that it is induced by most policies
(out of $40401$ policies in a $201$-by-$201$-grid policy parameter space
from $-10$ to $+10$ with $0.1$ uniform spacing in each parameter dimension).
Those of the 2nd, 3rd, up to 5th most common mixing times are in
the 2nd, 3rd, down to the bottom-most row, respectively.
The error bar indicates the standard deviation across
the corresponding error values of those many policies.
The blue and red crosses (on the right part of each subfigure) show
the angular and norm errors of the single post-mixing term.
In some subfigures, the blue cross coincides with the red one
(note however, their vertical axes are different).
For experimental setup, see \secref{sec:xprmt_setup_nbwpg}.
}
\label{fig:gradbias_decompose}
\end{figure}
\end{landscape}

\begin{landscape}
\begin{figure}
\centering

\begin{subfigure}{0.180\textwidth}
\includegraphics[width=\textwidth]{fig/envprop-3d/bs0__Example_10_1_2-v1__OneDimensionStateAndTwoActionPolicyNetwork}
\end{subfigure}
\begin{subfigure}{0.180\textwidth}
\includegraphics[width=\textwidth]{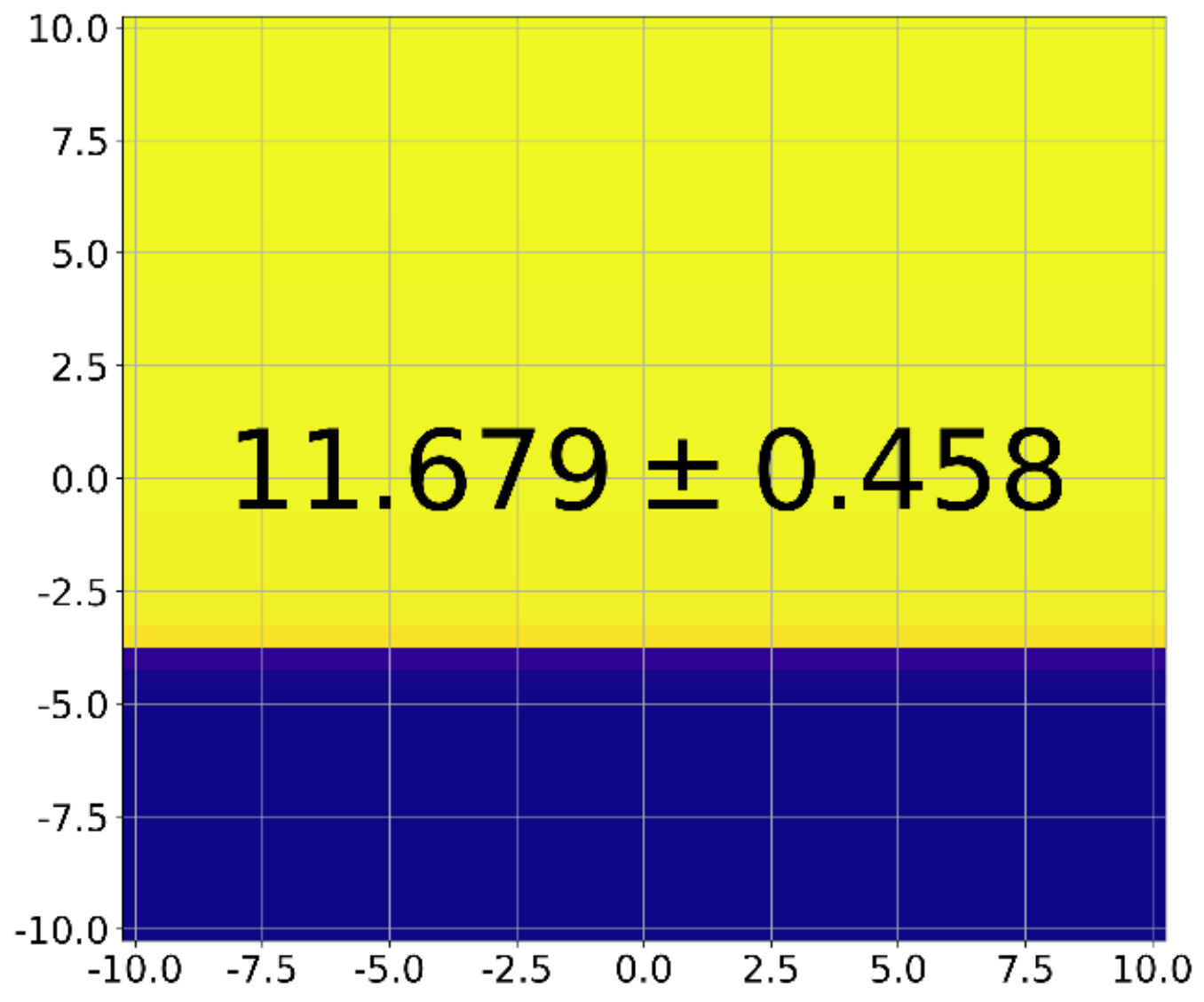}
\end{subfigure}
\begin{subfigure}{0.180\textwidth}
\includegraphics[width=\textwidth]{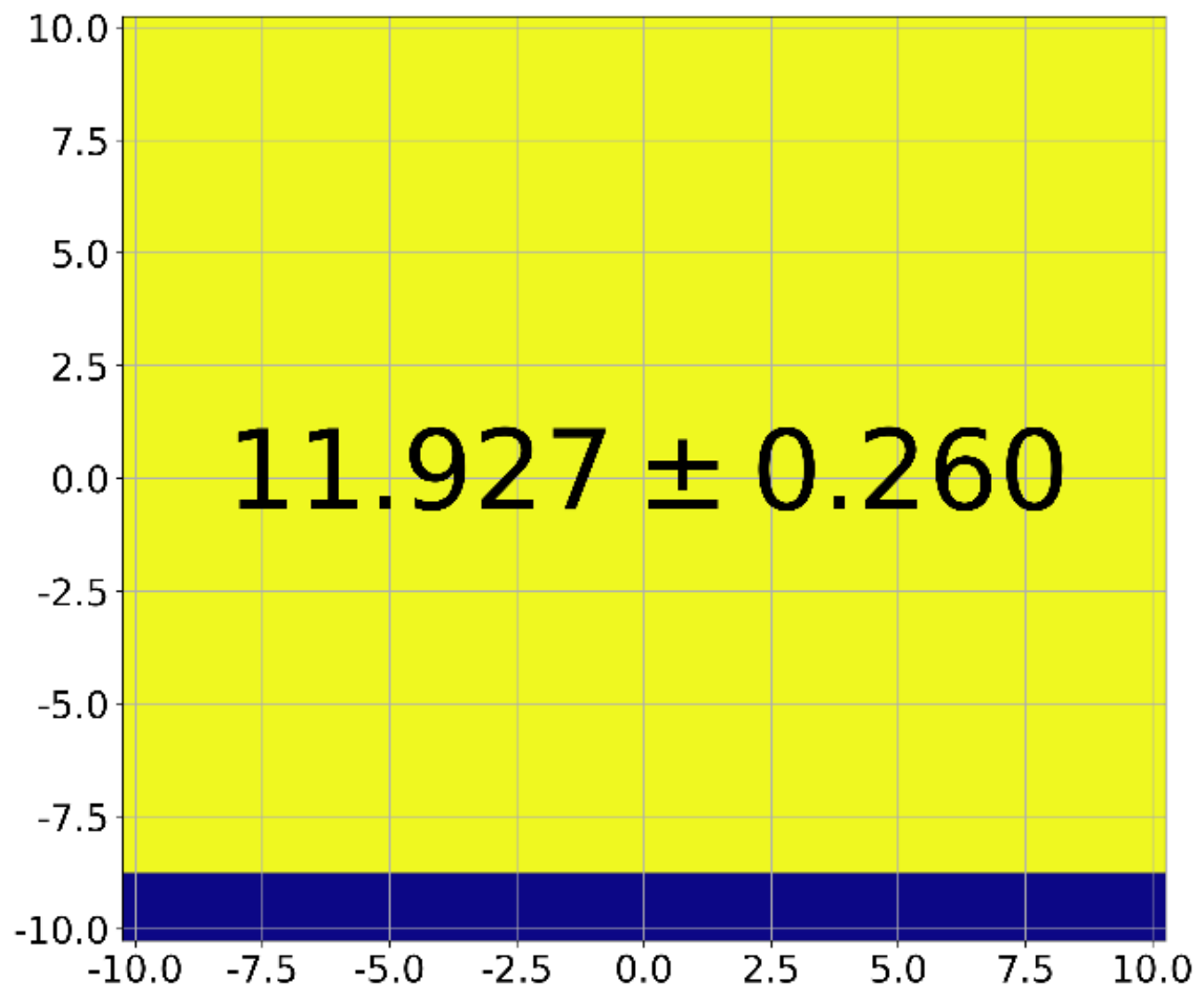}
\end{subfigure}
\begin{subfigure}{0.180\textwidth}
\includegraphics[width=\textwidth]{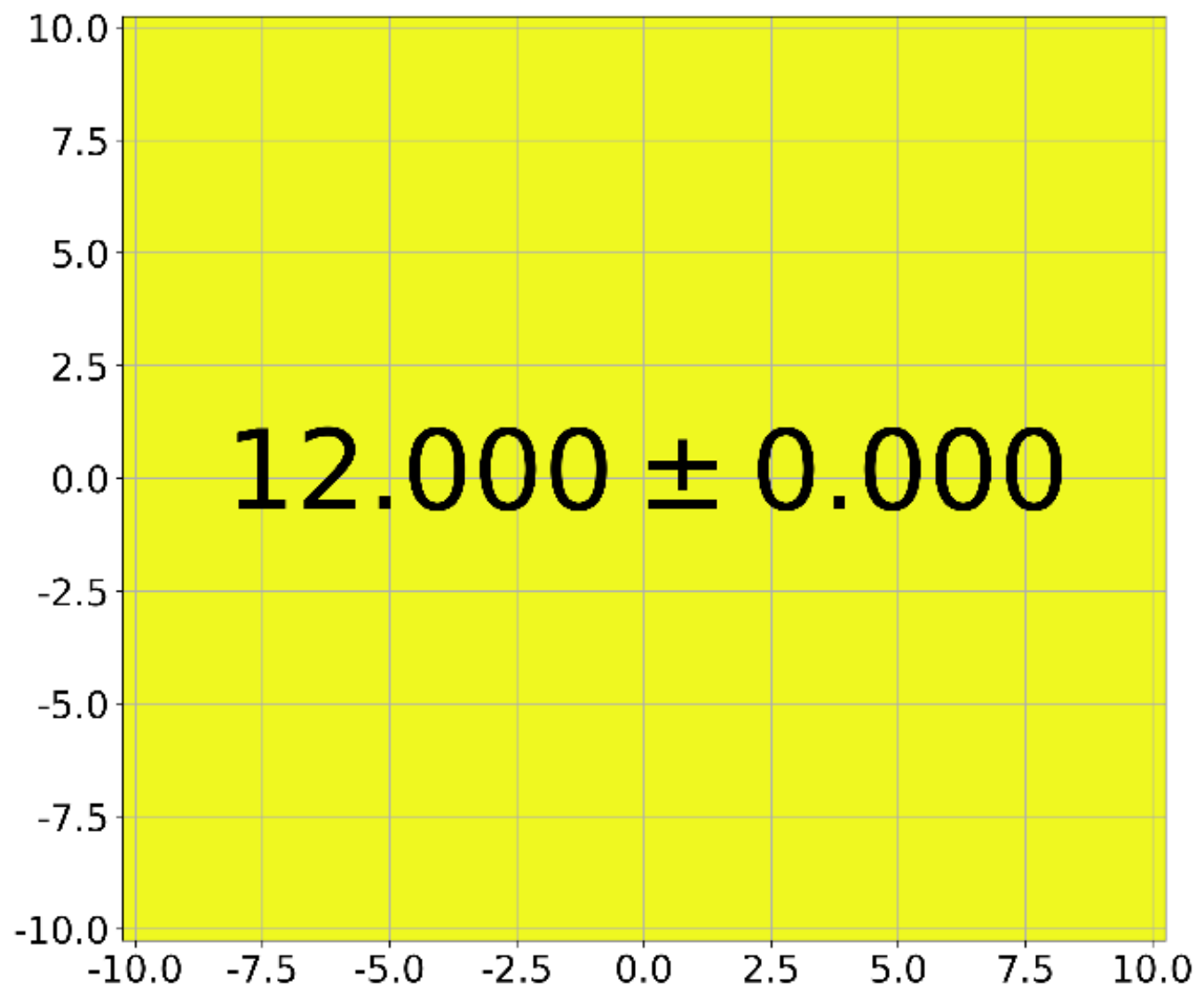}
\end{subfigure}
\begin{subfigure}{0.180\textwidth}
\includegraphics[width=\textwidth]{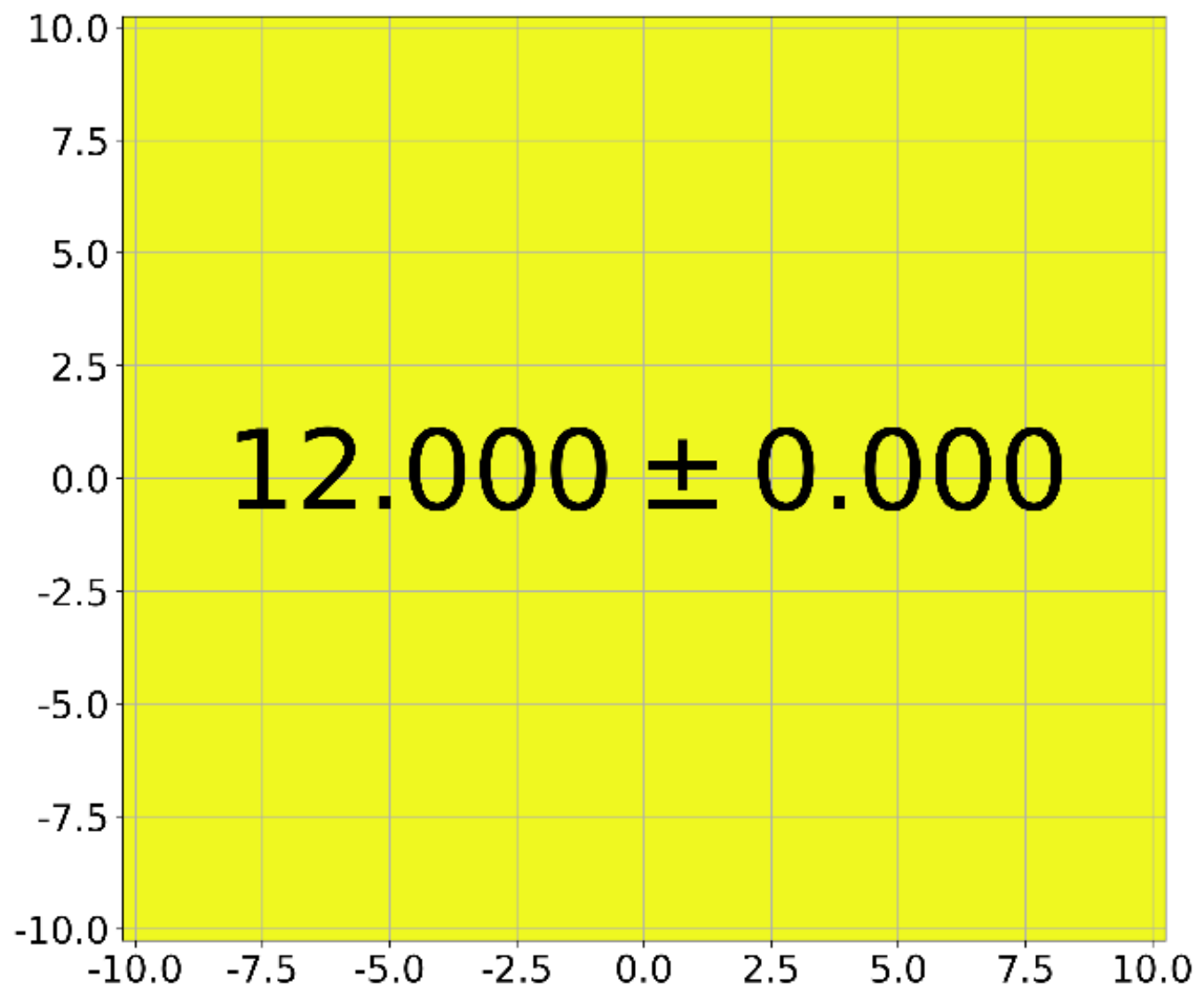}
\end{subfigure}
\begin{subfigure}{0.180\textwidth}
\includegraphics[width=\textwidth]{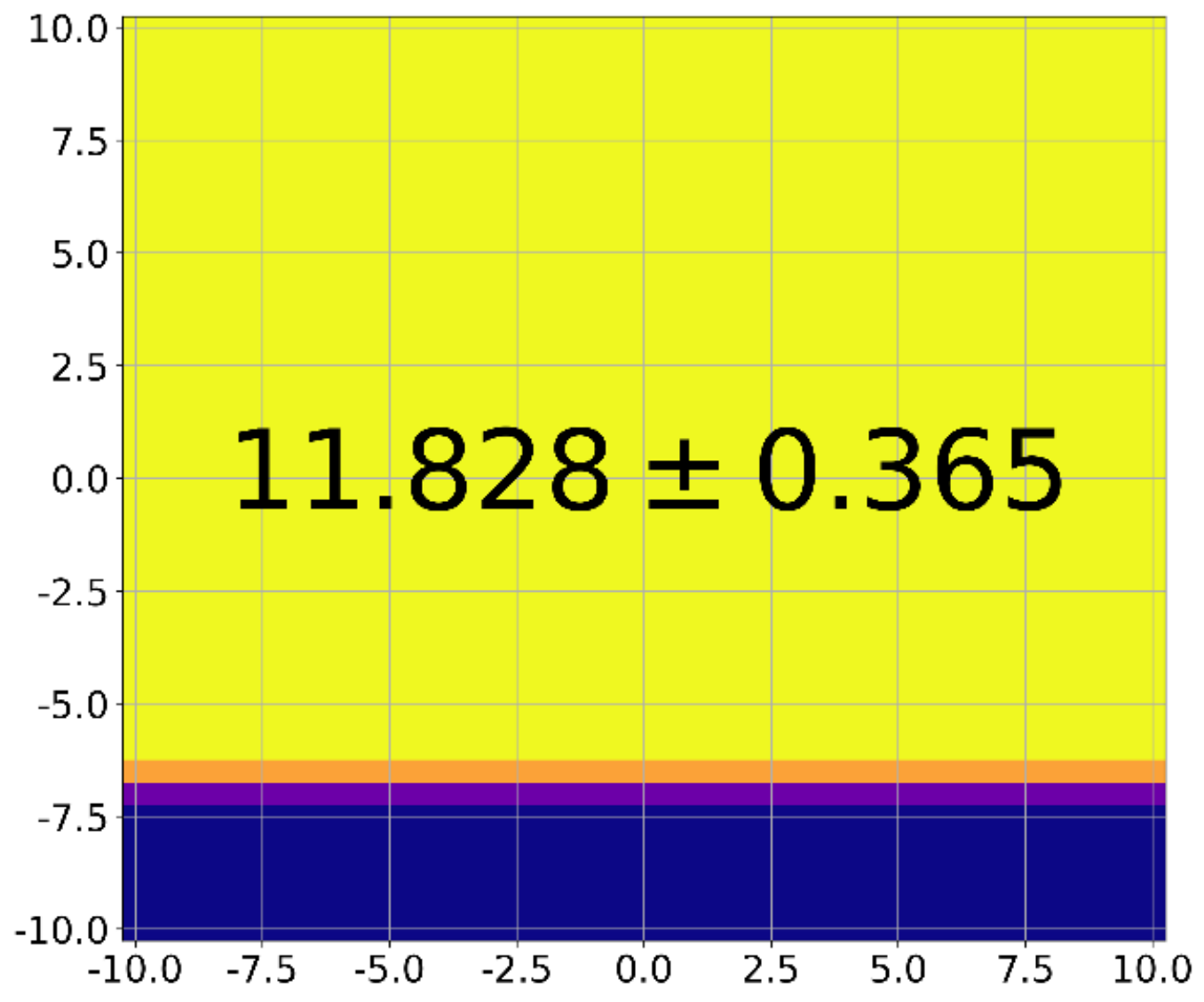}
\end{subfigure}
\begin{subfigure}{0.180\textwidth}
\includegraphics[width=\textwidth]{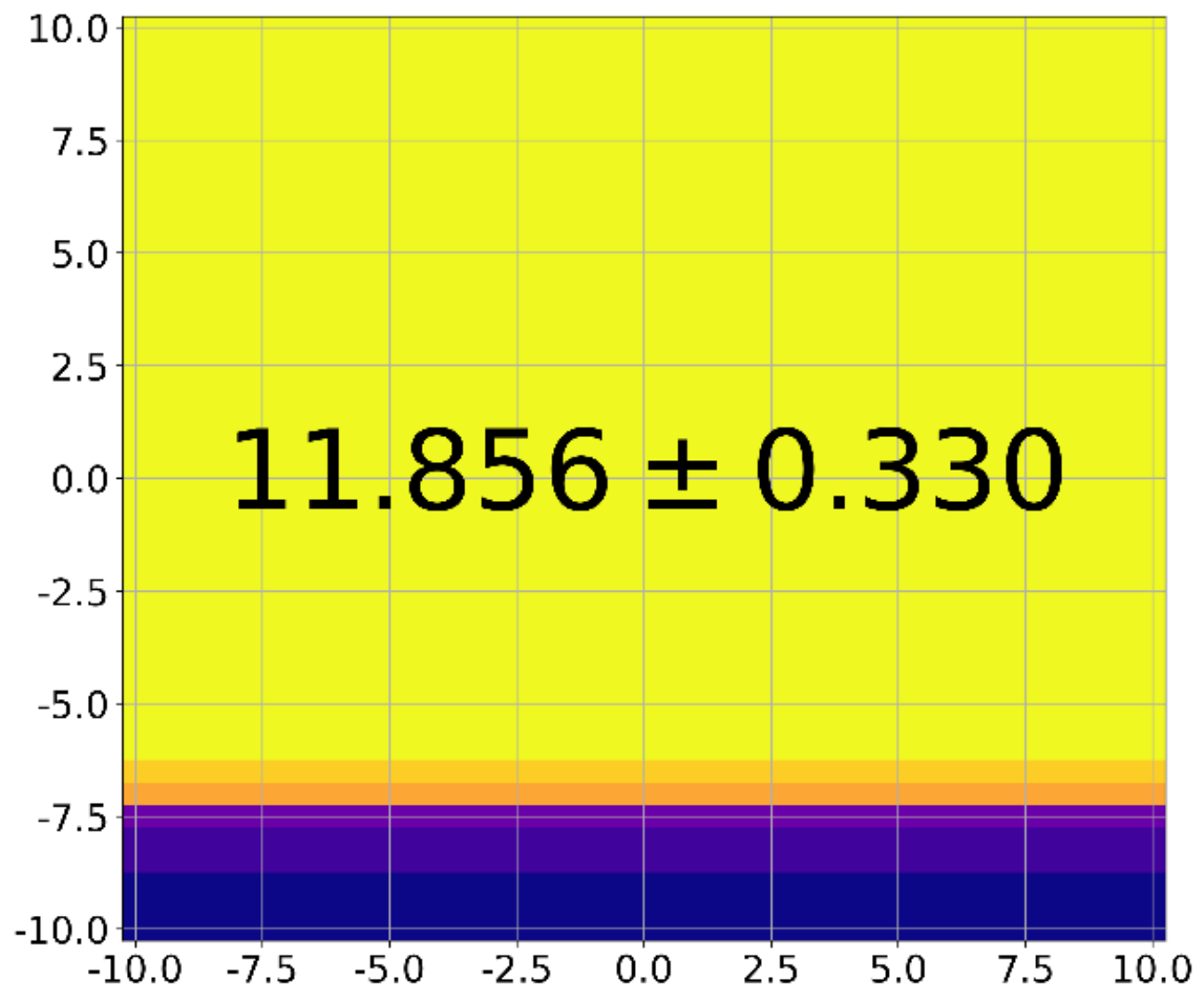}
\end{subfigure}

\begin{subfigure}{0.180\textwidth}
\includegraphics[width=\textwidth]{fig/envprop-3d/bs0__Tor_20210307-v1__OneDimensionStateAndTwoActionPolicyNetwork}
\end{subfigure}
\begin{subfigure}{0.180\textwidth}
\includegraphics[width=\textwidth]{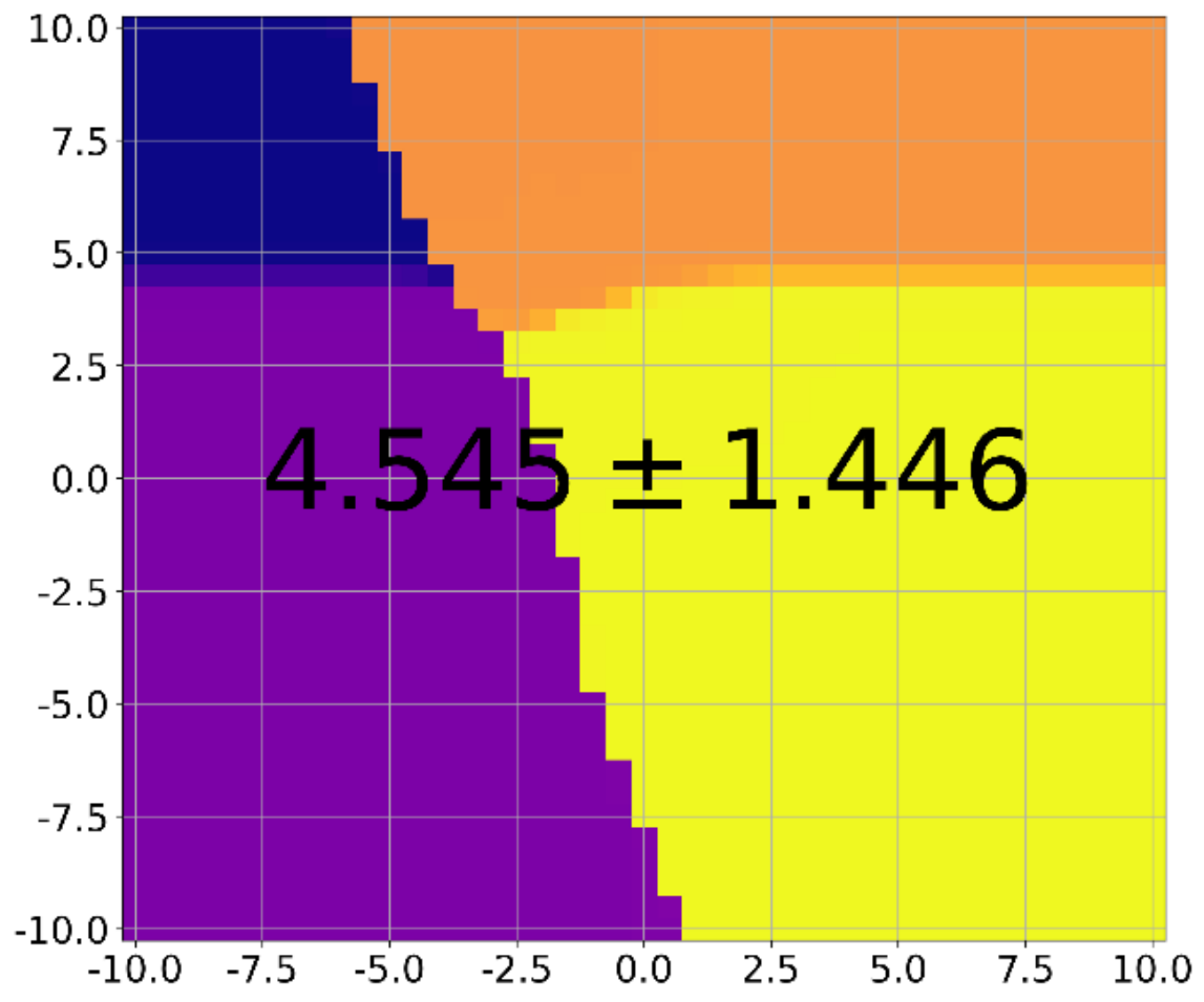}
\end{subfigure}
\begin{subfigure}{0.180\textwidth}
\includegraphics[width=\textwidth]{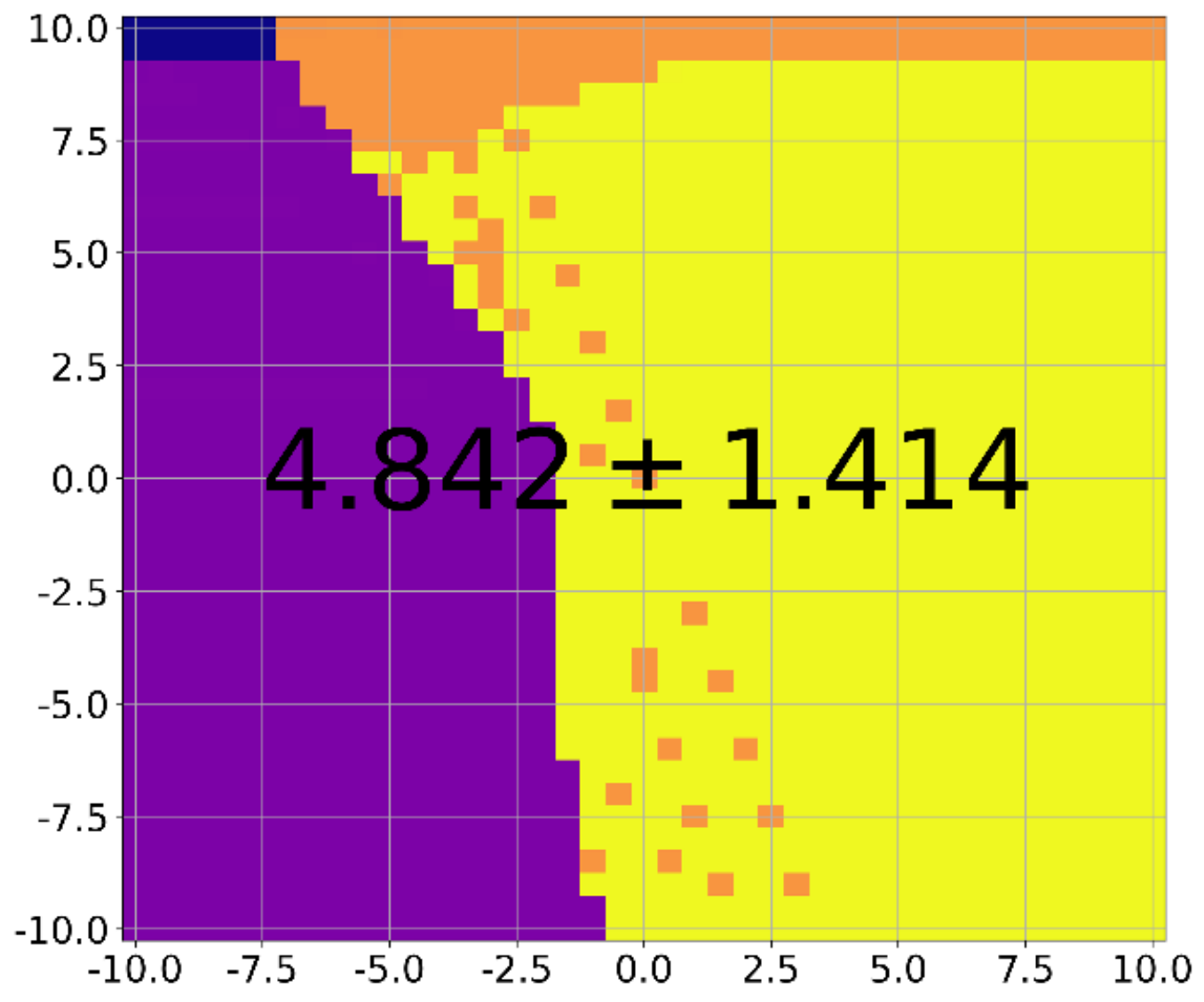}
\end{subfigure}
\begin{subfigure}{0.180\textwidth}
\includegraphics[width=\textwidth]{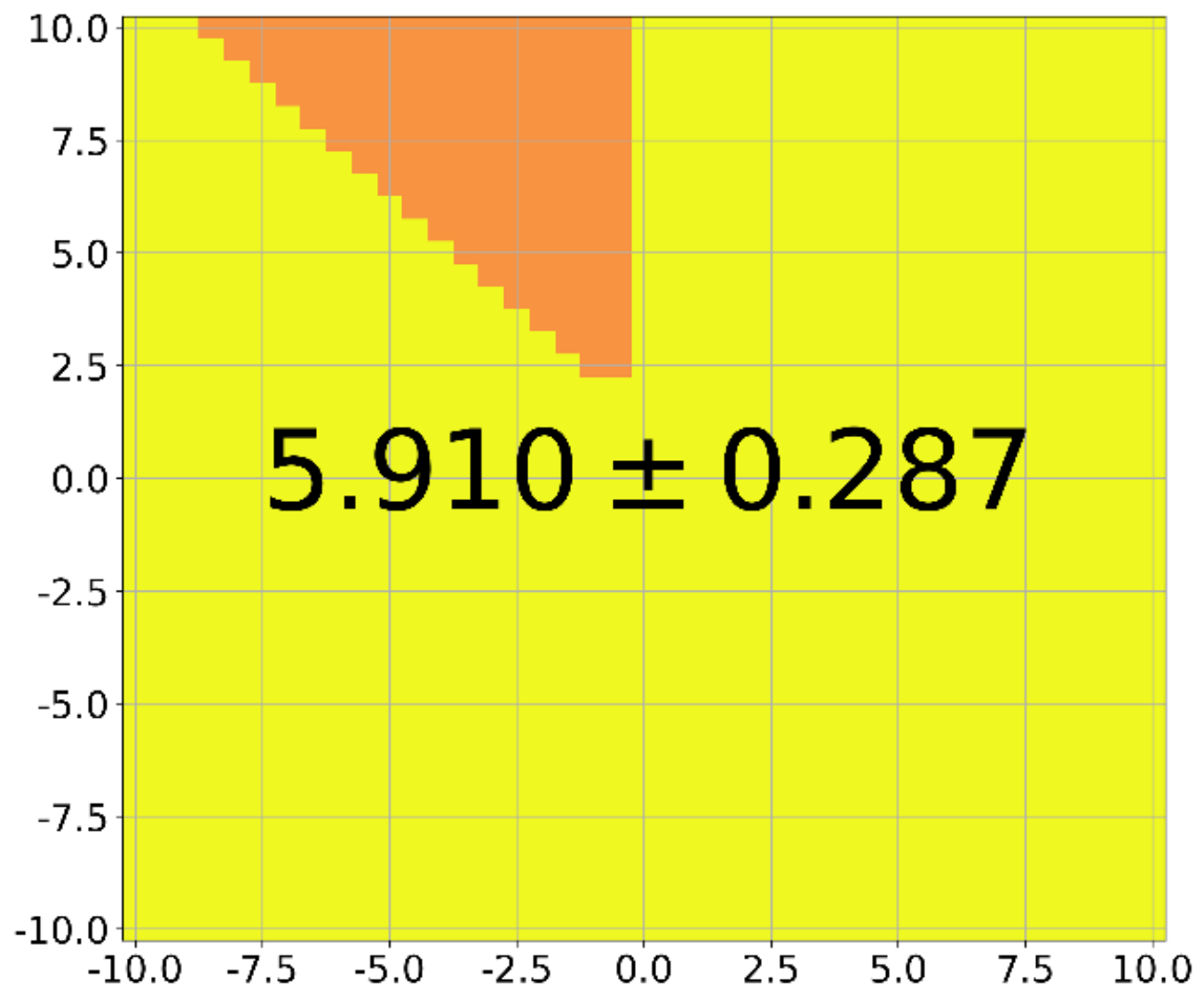}
\end{subfigure}
\begin{subfigure}{0.180\textwidth}
\includegraphics[width=\textwidth]{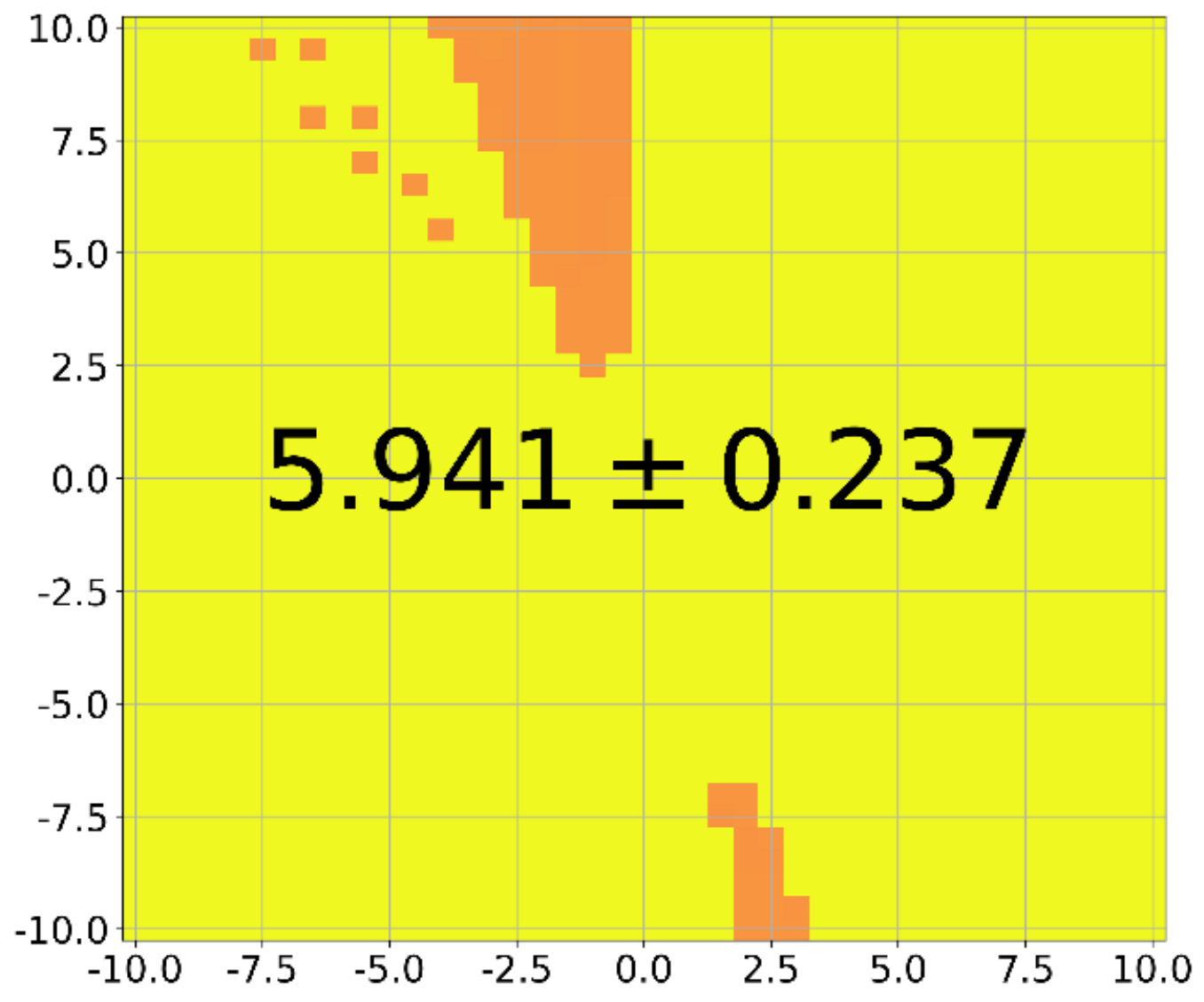}
\end{subfigure}
\begin{subfigure}{0.180\textwidth}
\includegraphics[width=\textwidth]{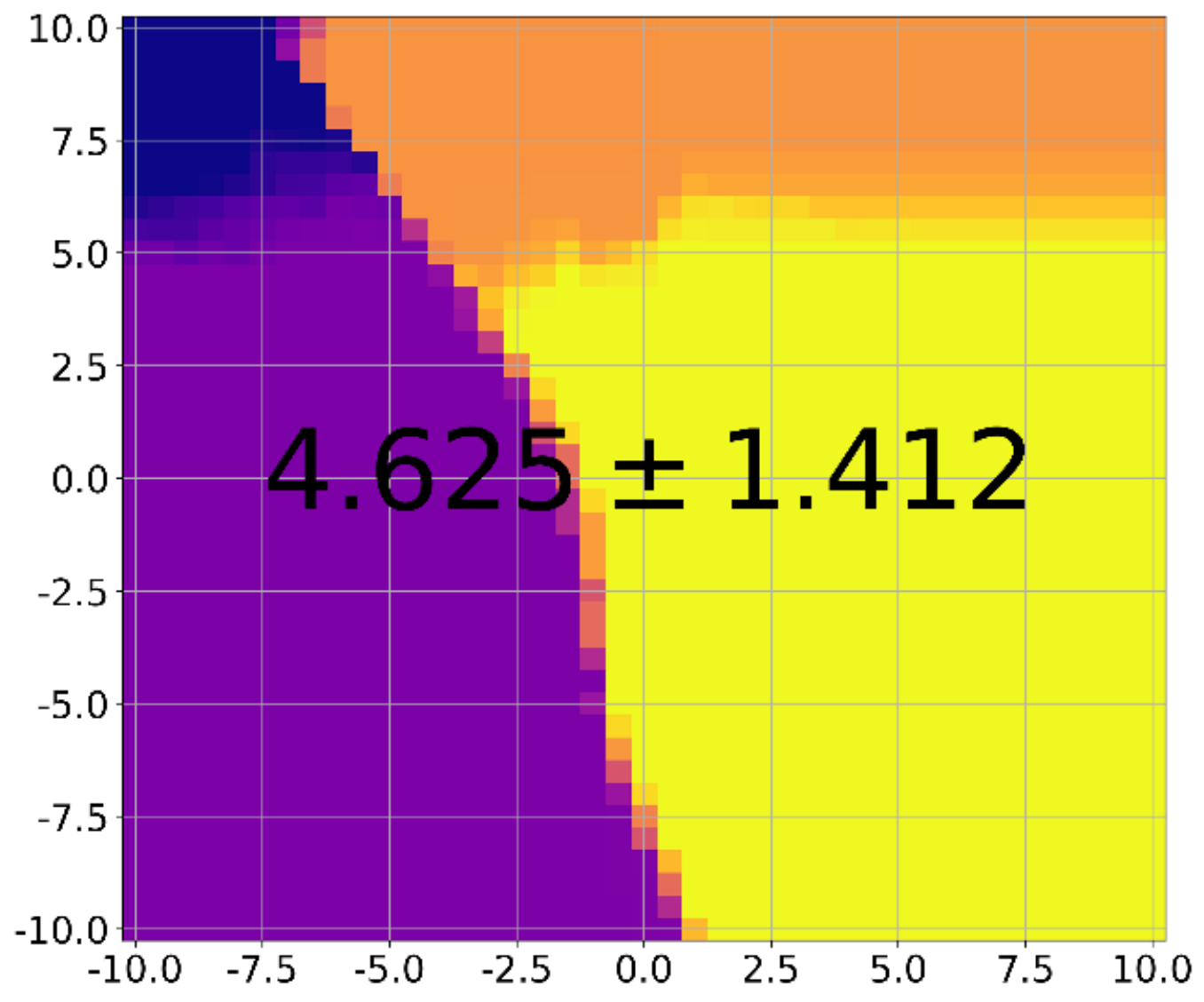}
\end{subfigure}
\begin{subfigure}{0.180\textwidth}
\includegraphics[width=\textwidth]{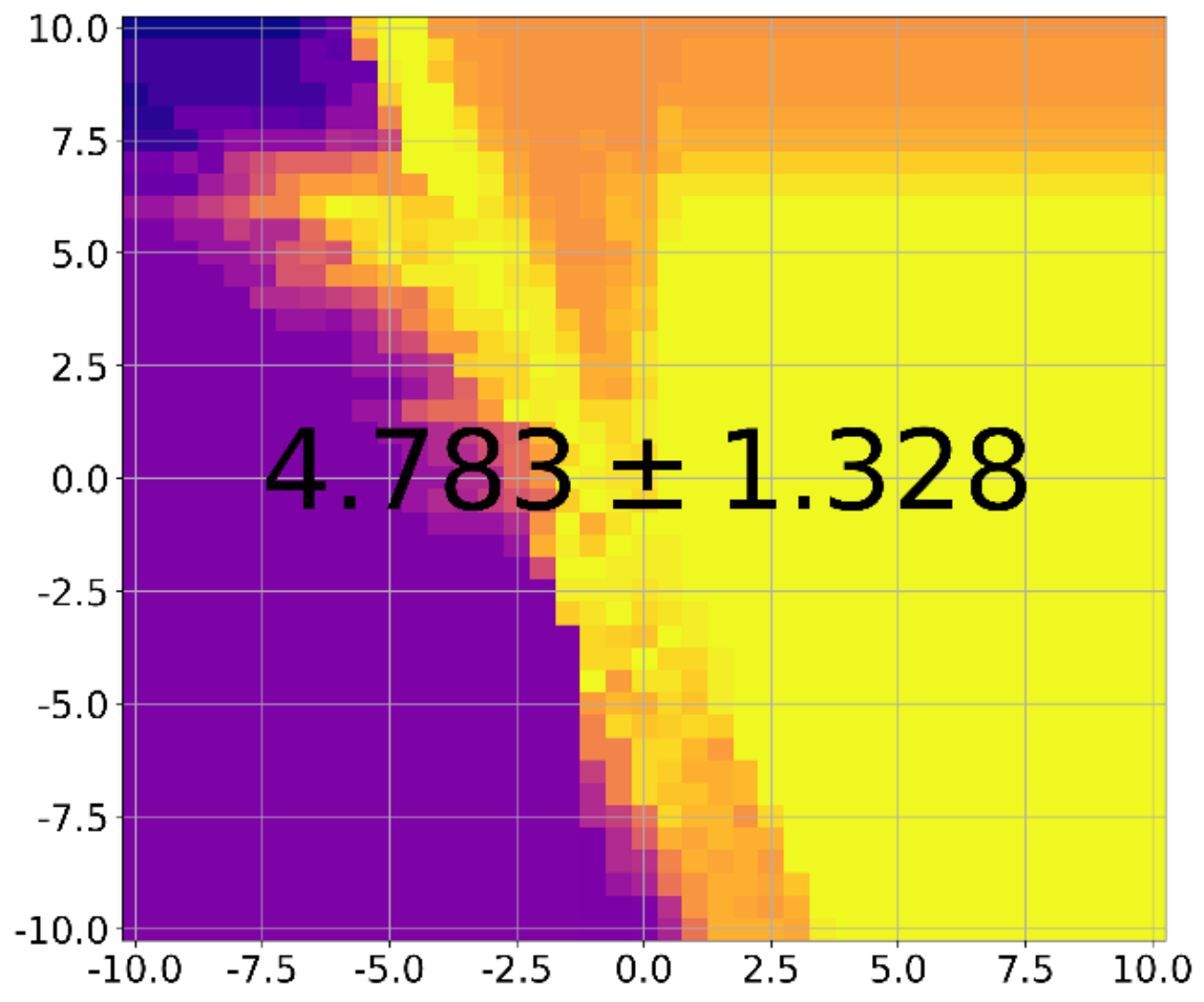}
\end{subfigure}

\begin{subfigure}{0.180\textwidth}
\includegraphics[width=\textwidth]{fig/envprop-3d/bs0__GridNav_2-v0__OneDimensionStateAndTwoActionPolicyNetwork}
\end{subfigure}
\begin{subfigure}{0.180\textwidth}
\includegraphics[width=\textwidth]{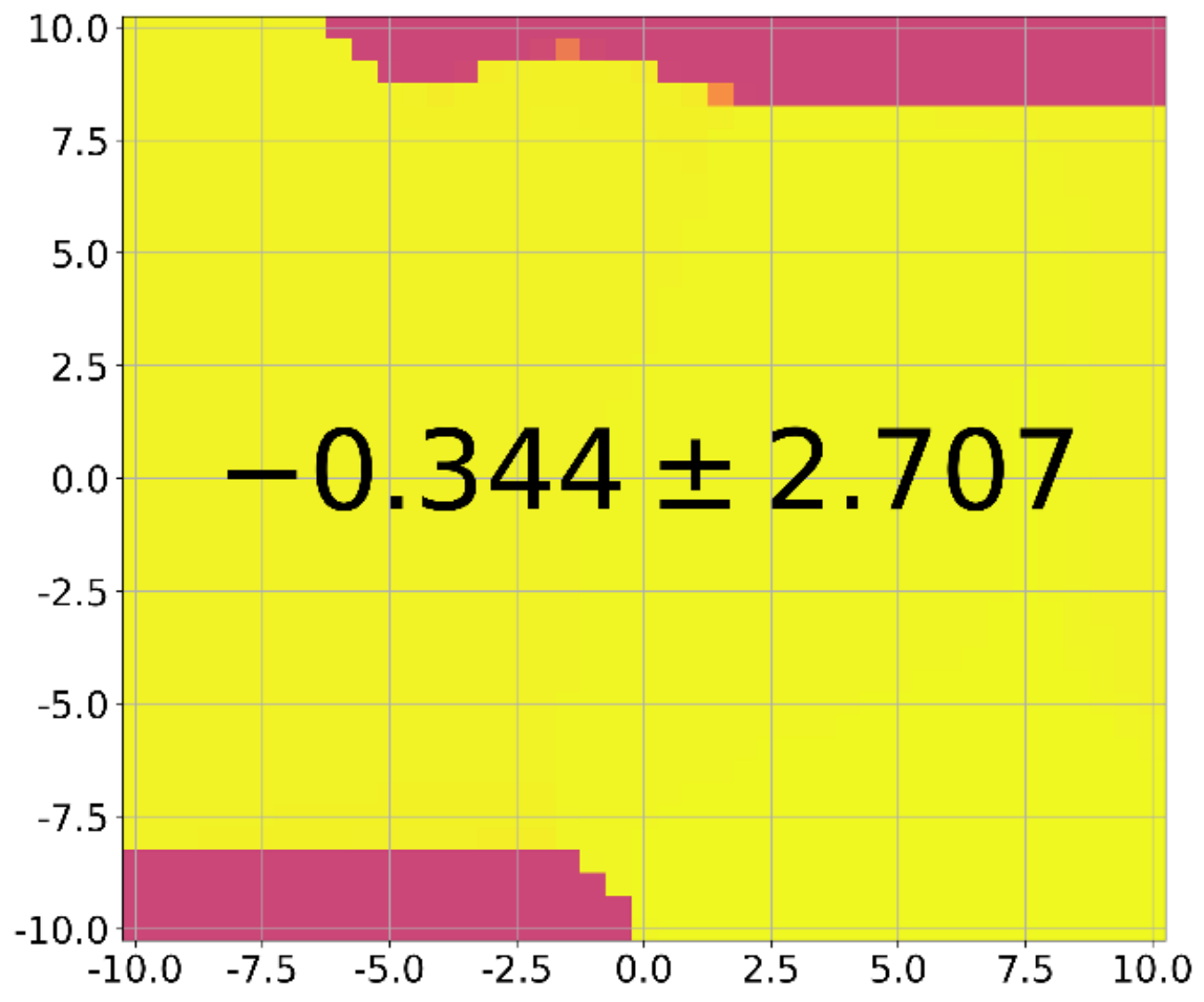}
\end{subfigure}
\begin{subfigure}{0.180\textwidth}
\includegraphics[width=\textwidth]{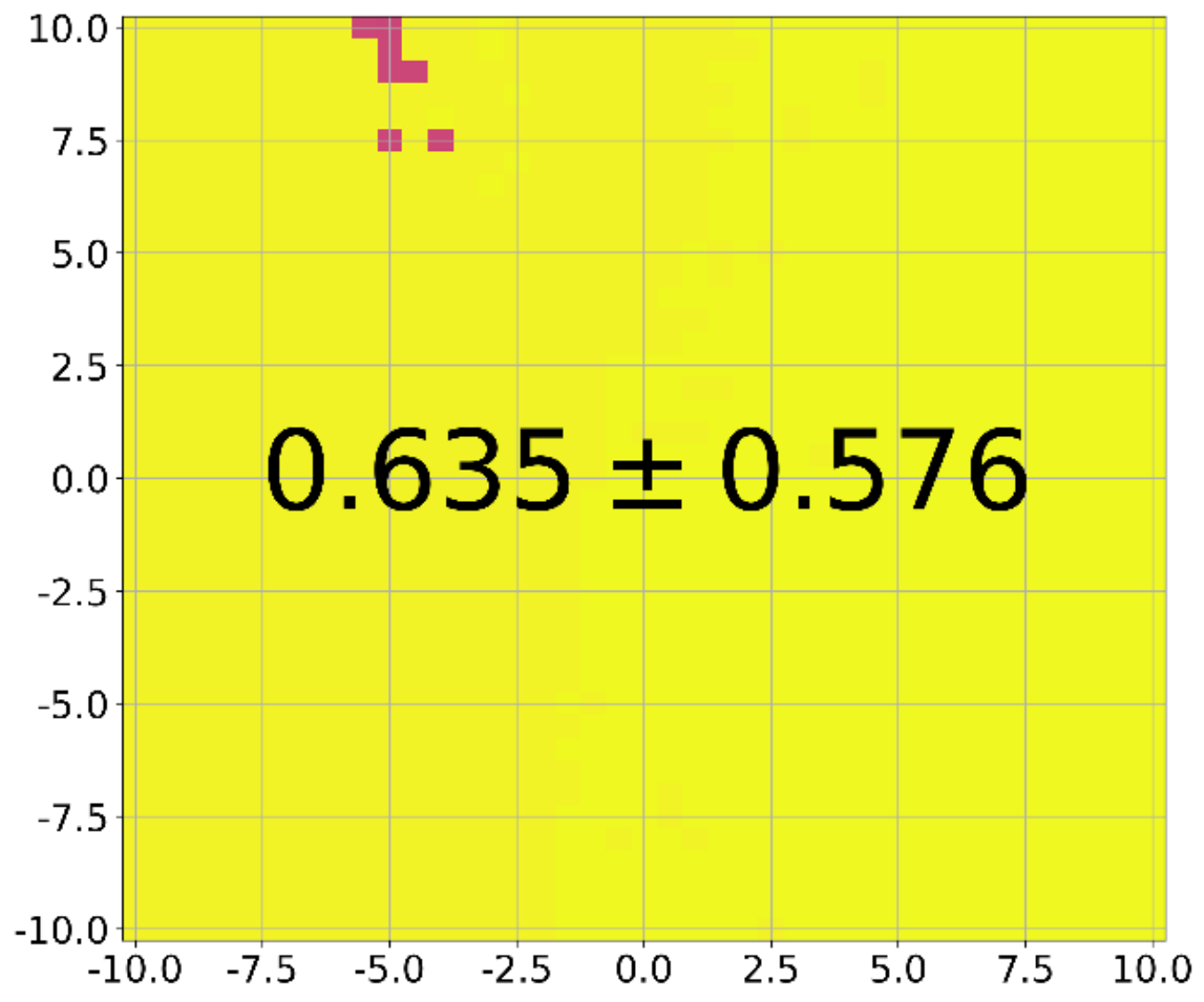}
\end{subfigure}
\begin{subfigure}{0.180\textwidth}
\includegraphics[width=\textwidth]{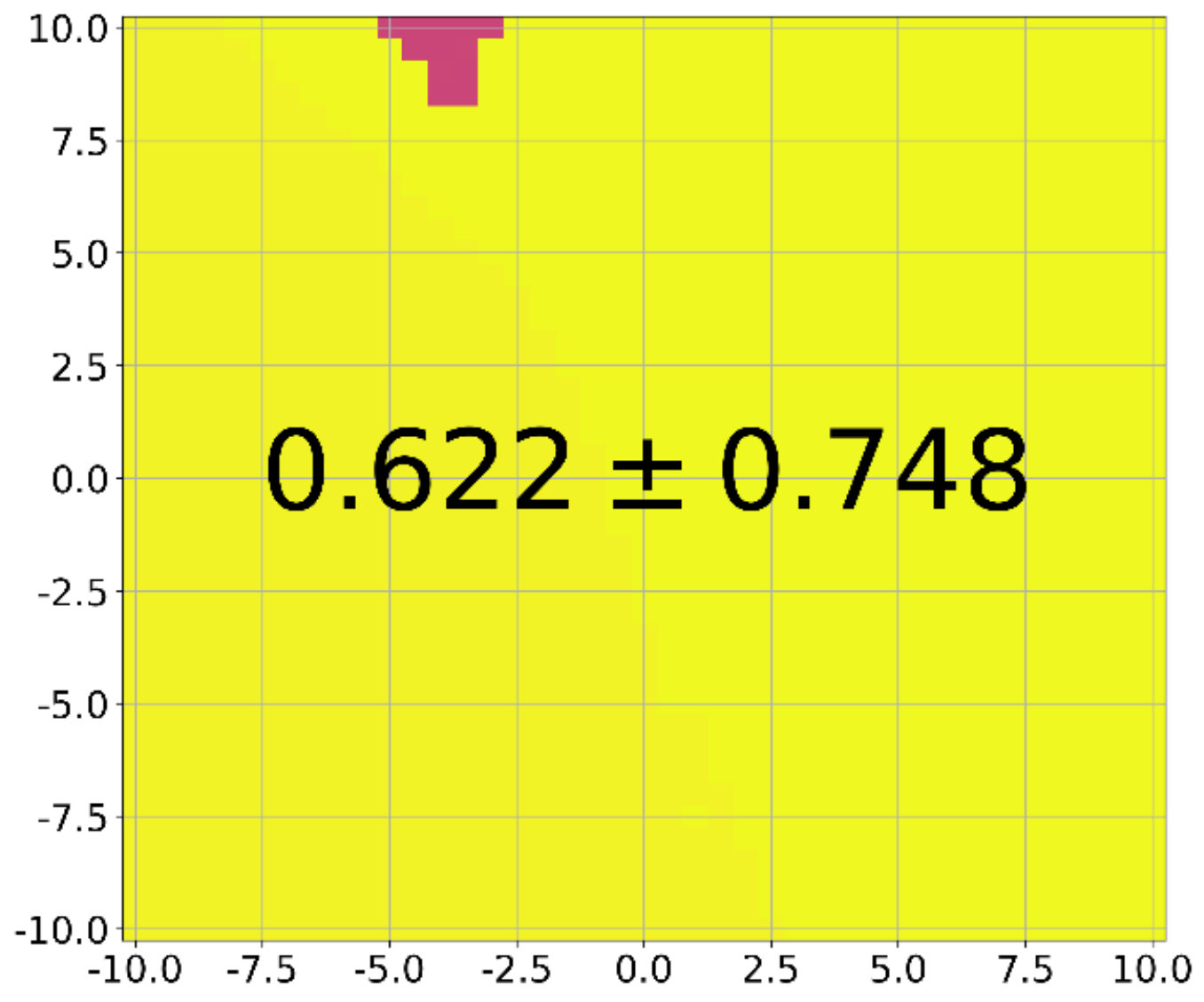}
\end{subfigure}
\begin{subfigure}{0.180\textwidth}
\includegraphics[width=\textwidth]{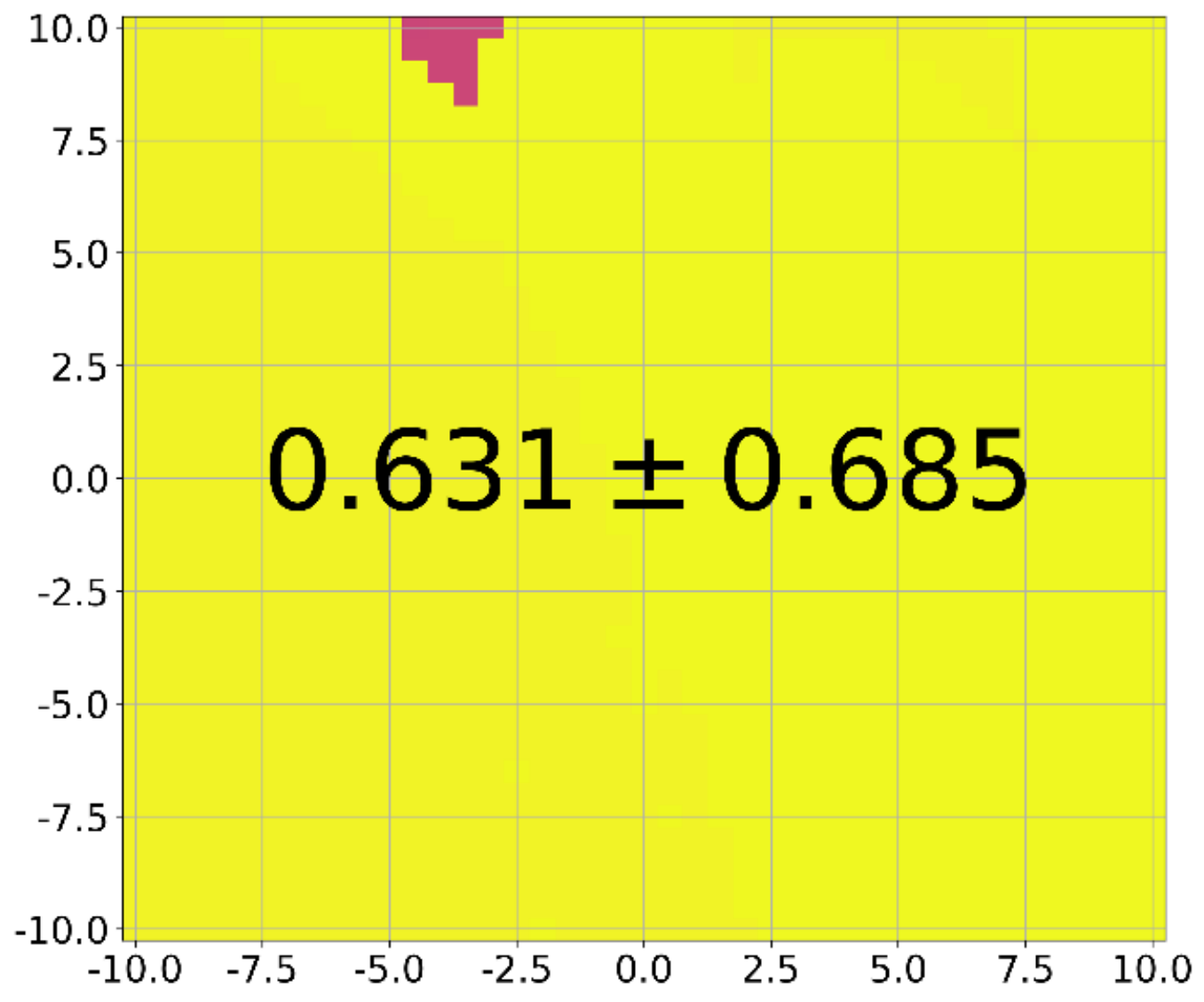}
\end{subfigure}
\begin{subfigure}{0.180\textwidth}
\includegraphics[width=\textwidth]{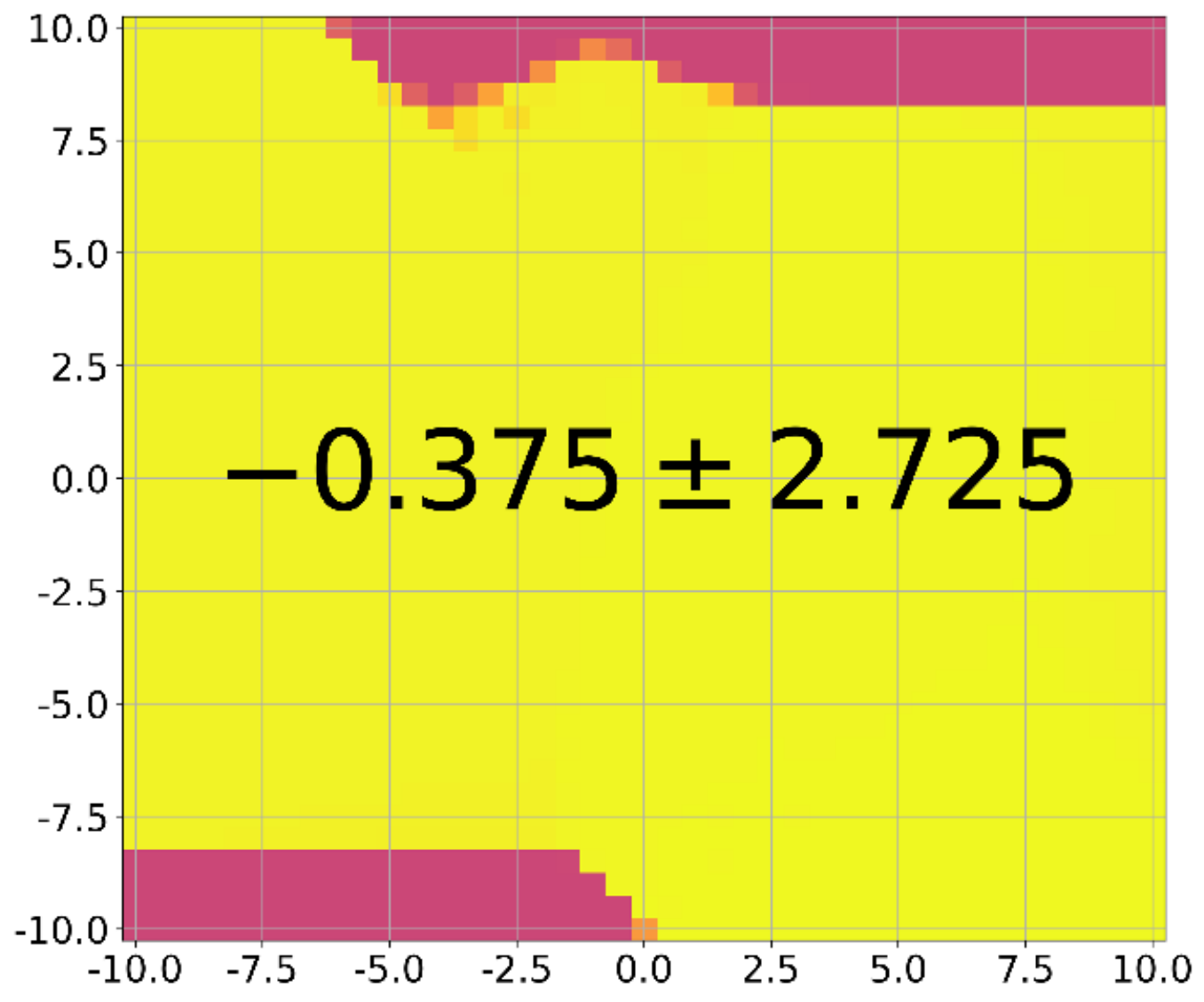}
\end{subfigure}
\begin{subfigure}{0.180\textwidth}
\includegraphics[width=\textwidth]{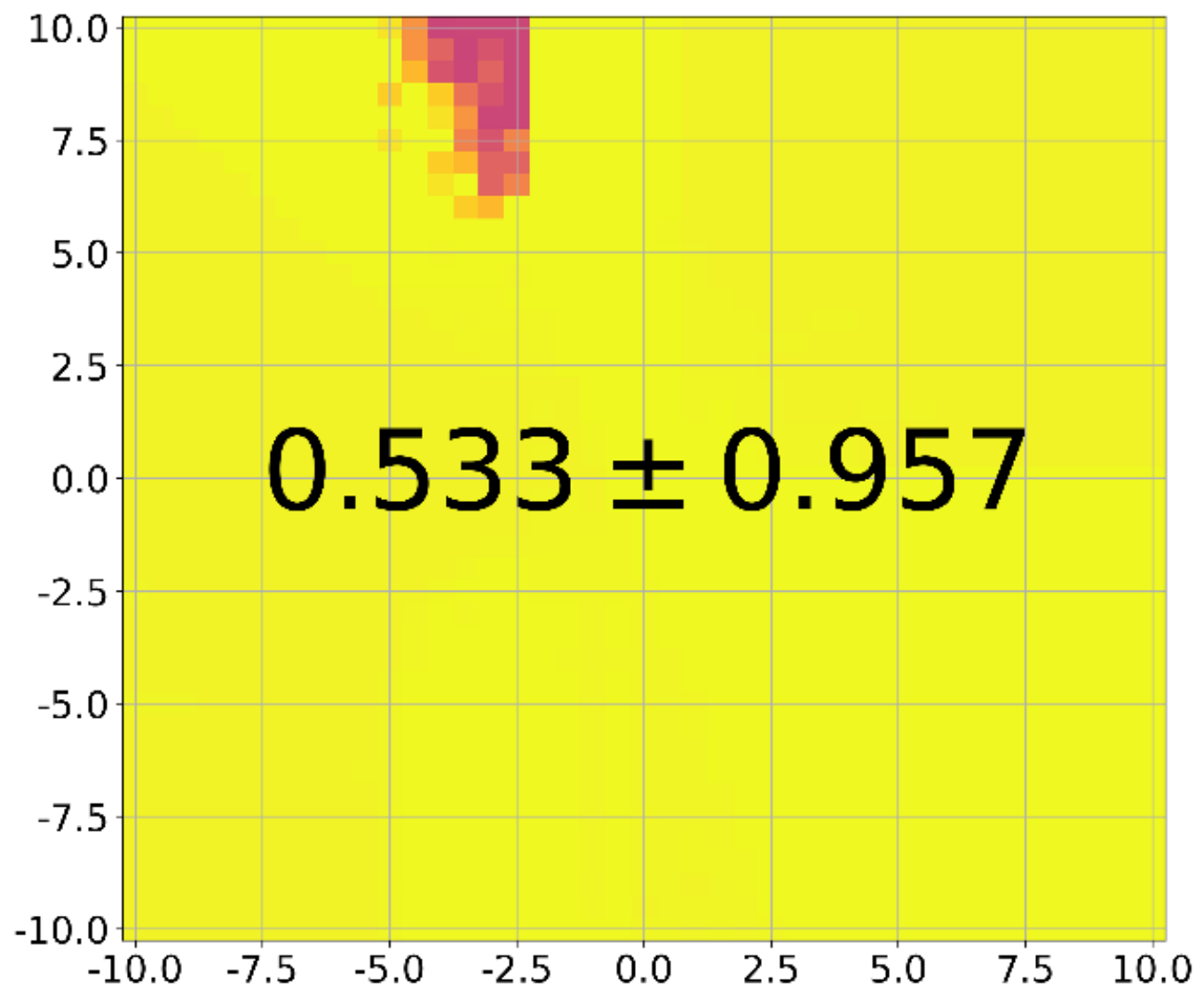}
\end{subfigure}

\begin{subfigure}{0.180\textwidth}
\includegraphics[width=\textwidth]{fig/envprop-3d/bs0__Tor_20201121a-v1__OneDimensionStateAndTwoActionPolicyNetwork}
\end{subfigure}
\begin{subfigure}{0.180\textwidth}
\includegraphics[width=\textwidth]{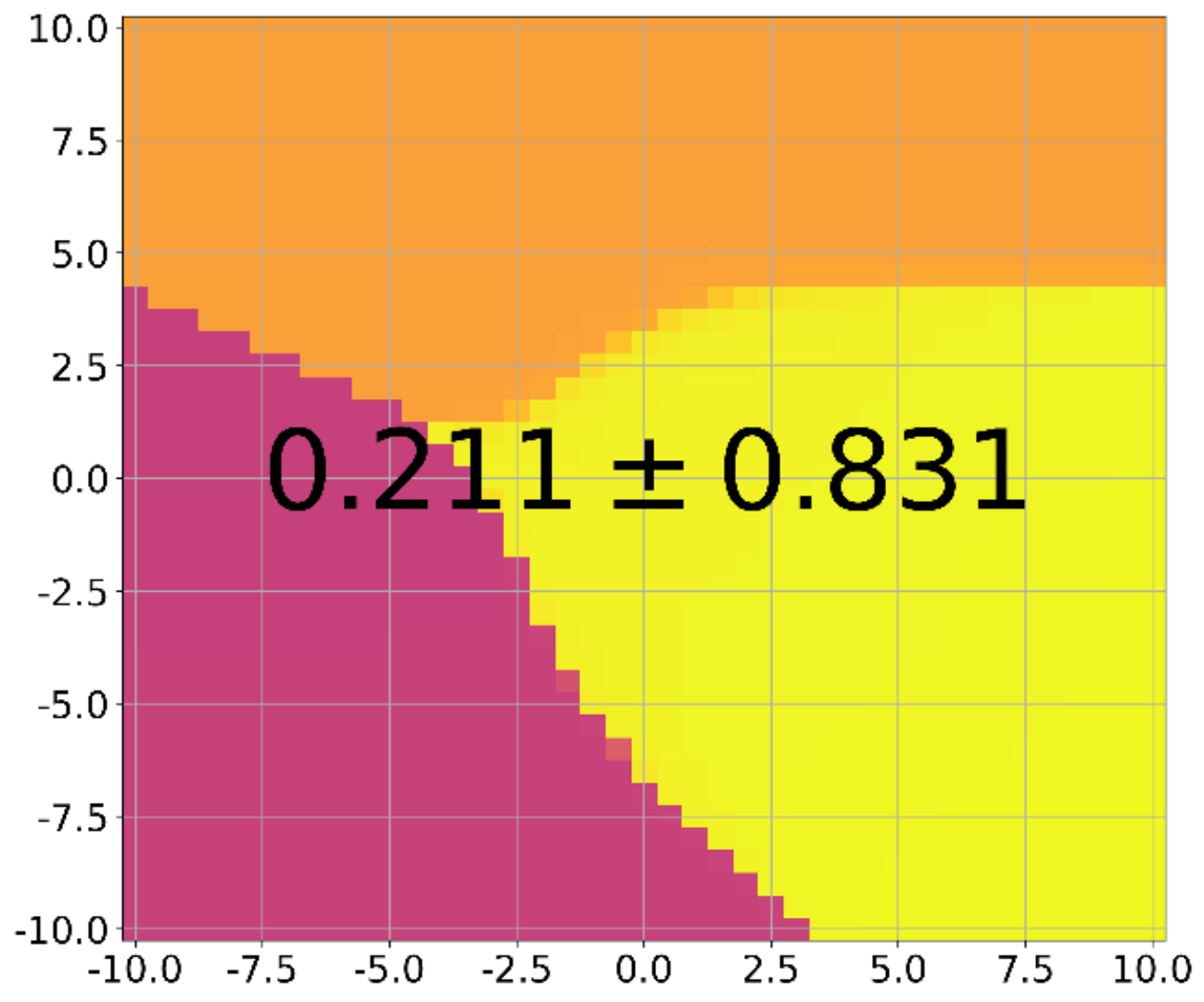}
\end{subfigure}
\begin{subfigure}{0.180\textwidth}
\includegraphics[width=\textwidth]{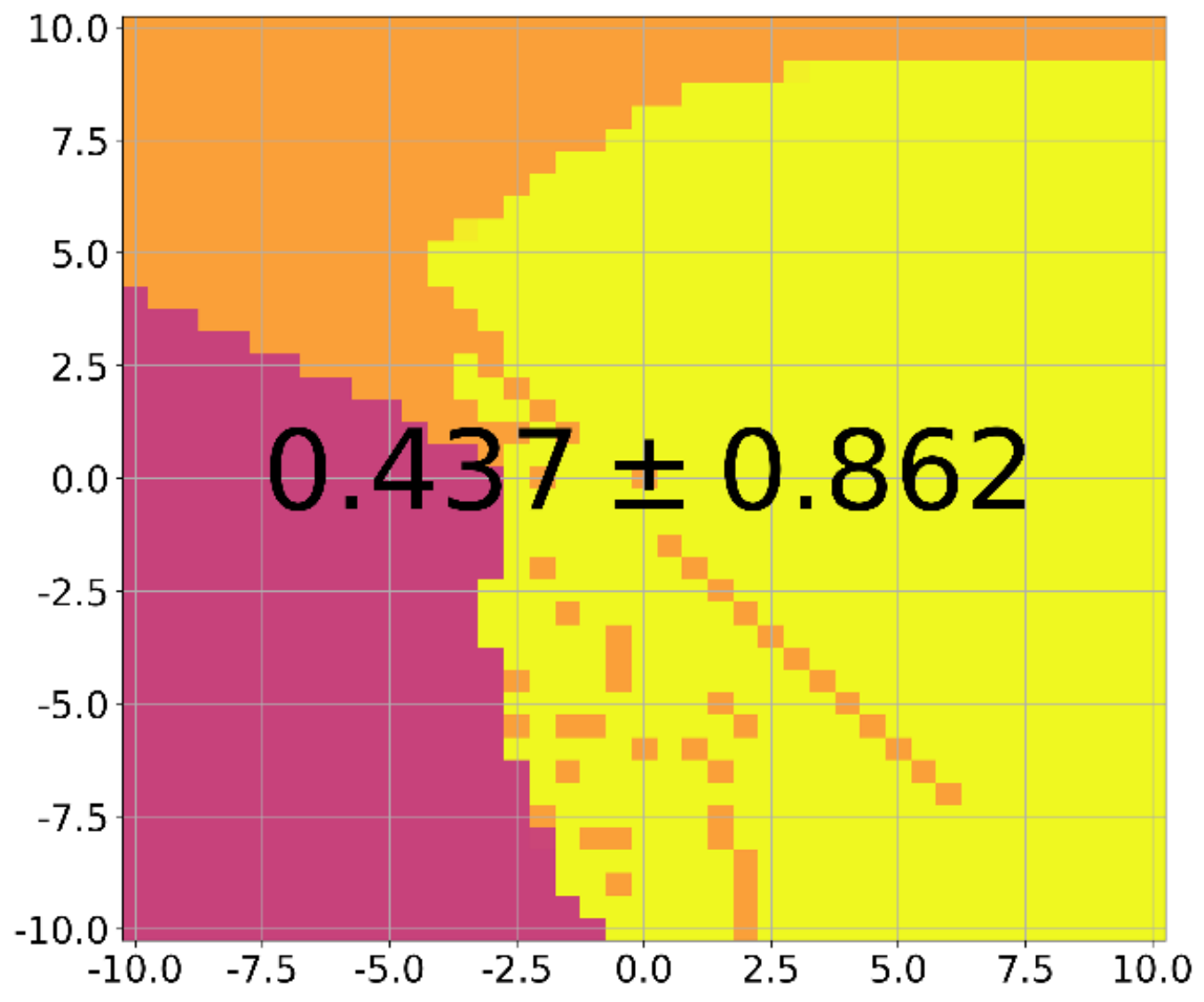}
\end{subfigure}
\begin{subfigure}{0.180\textwidth}
\includegraphics[width=\textwidth]{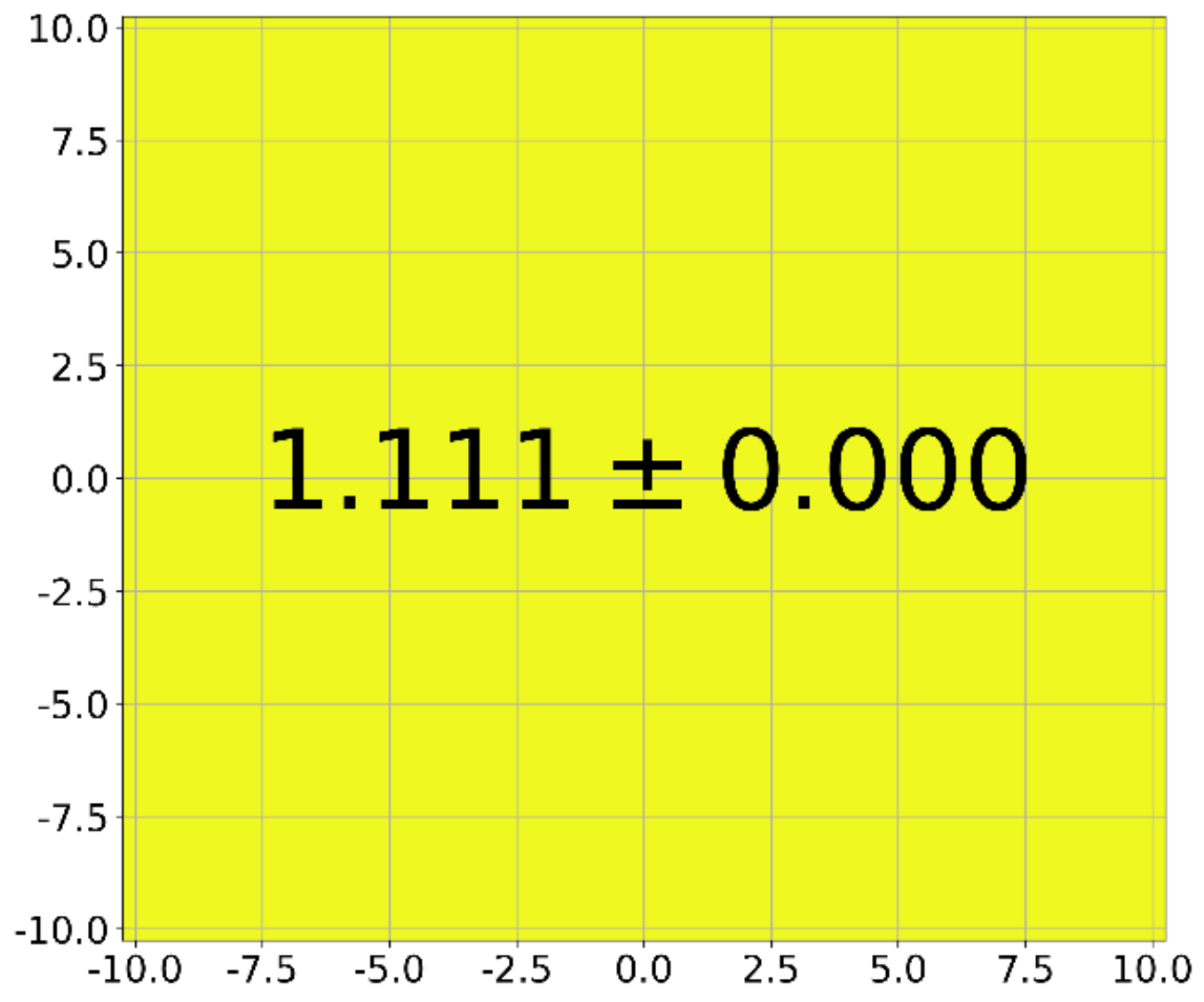}
\end{subfigure}
\begin{subfigure}{0.180\textwidth}
\includegraphics[width=\textwidth]{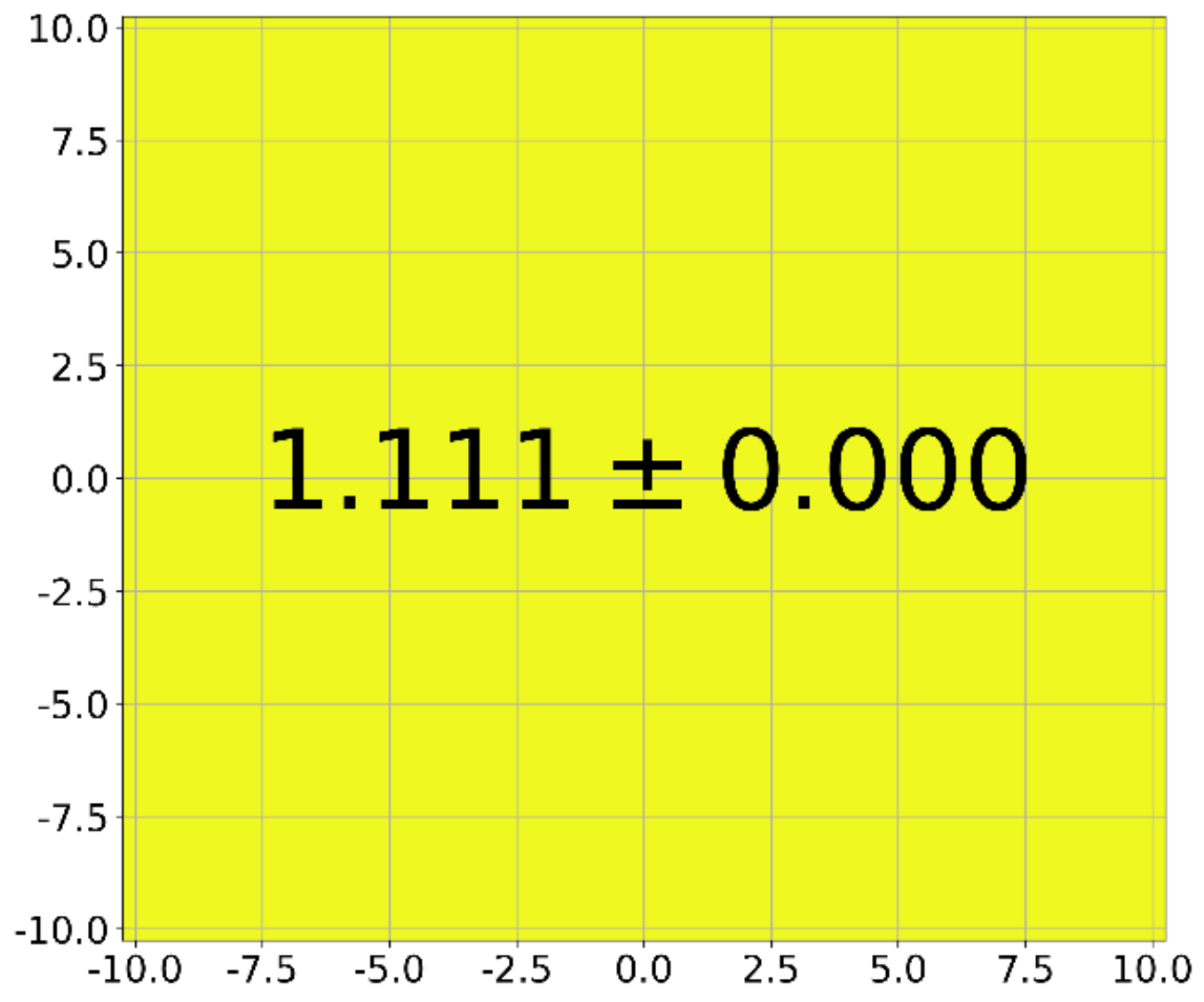}
\end{subfigure}
\begin{subfigure}{0.180\textwidth}
\includegraphics[width=\textwidth]{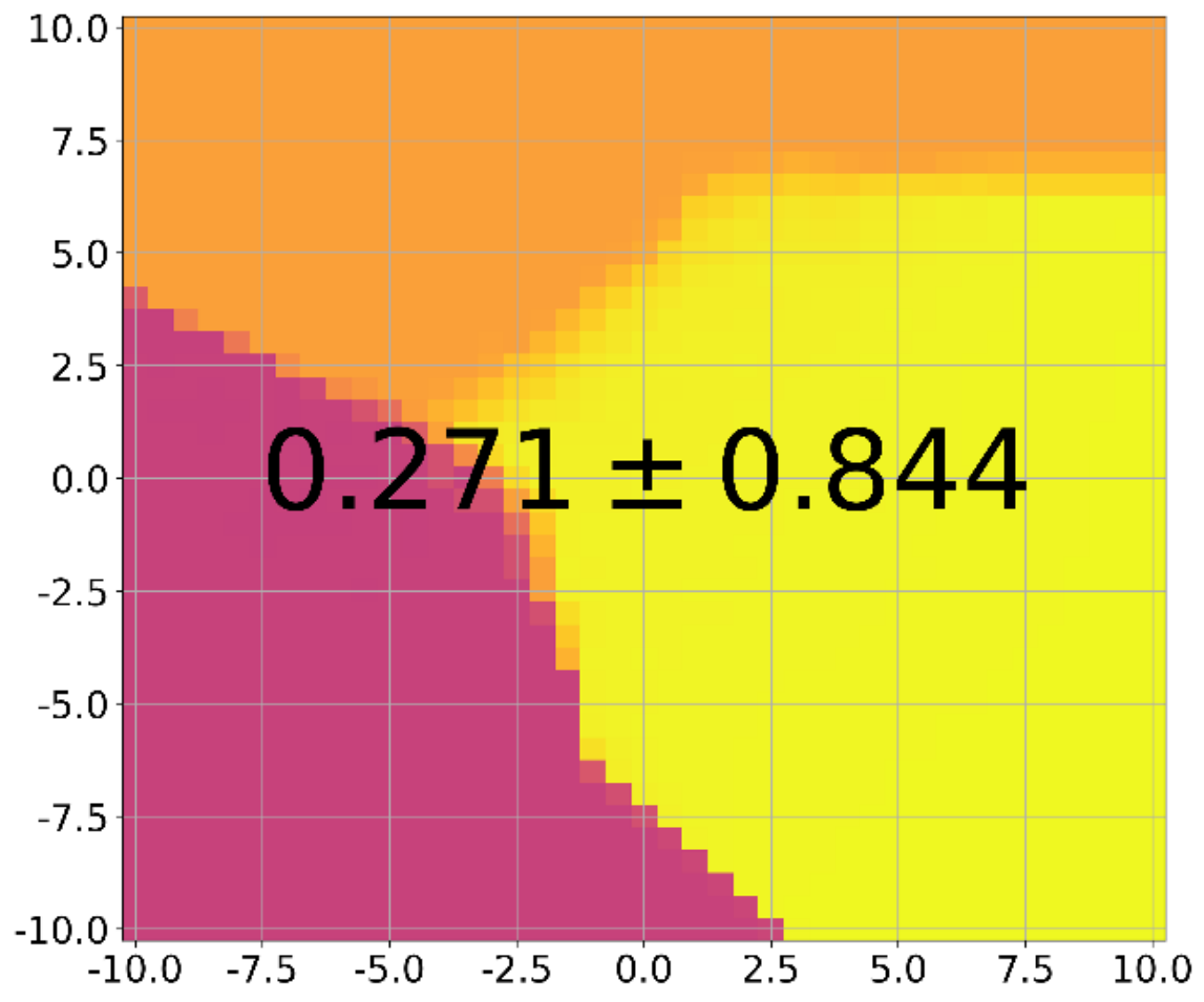}
\end{subfigure}
\begin{subfigure}{0.180\textwidth}
\includegraphics[width=\textwidth]{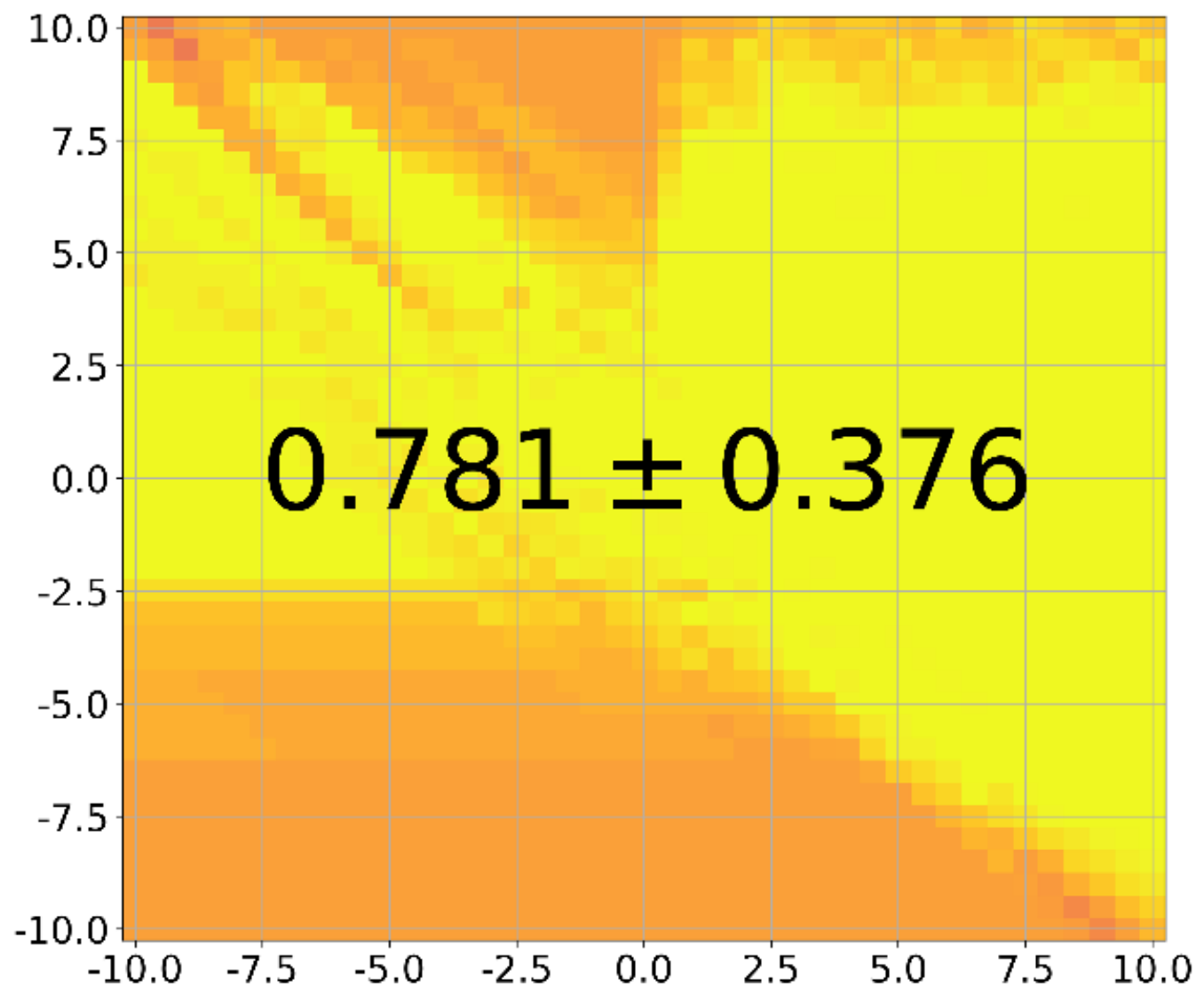}
\end{subfigure}

\begin{subfigure}{0.180\textwidth}
\includegraphics[width=\textwidth]{fig/envprop-3d/bs0__Hordijk_example-v4__OneDimensionStateAndTwoActionPolicyNetwork}
\end{subfigure}
\begin{subfigure}{0.180\textwidth}
\includegraphics[width=\textwidth]{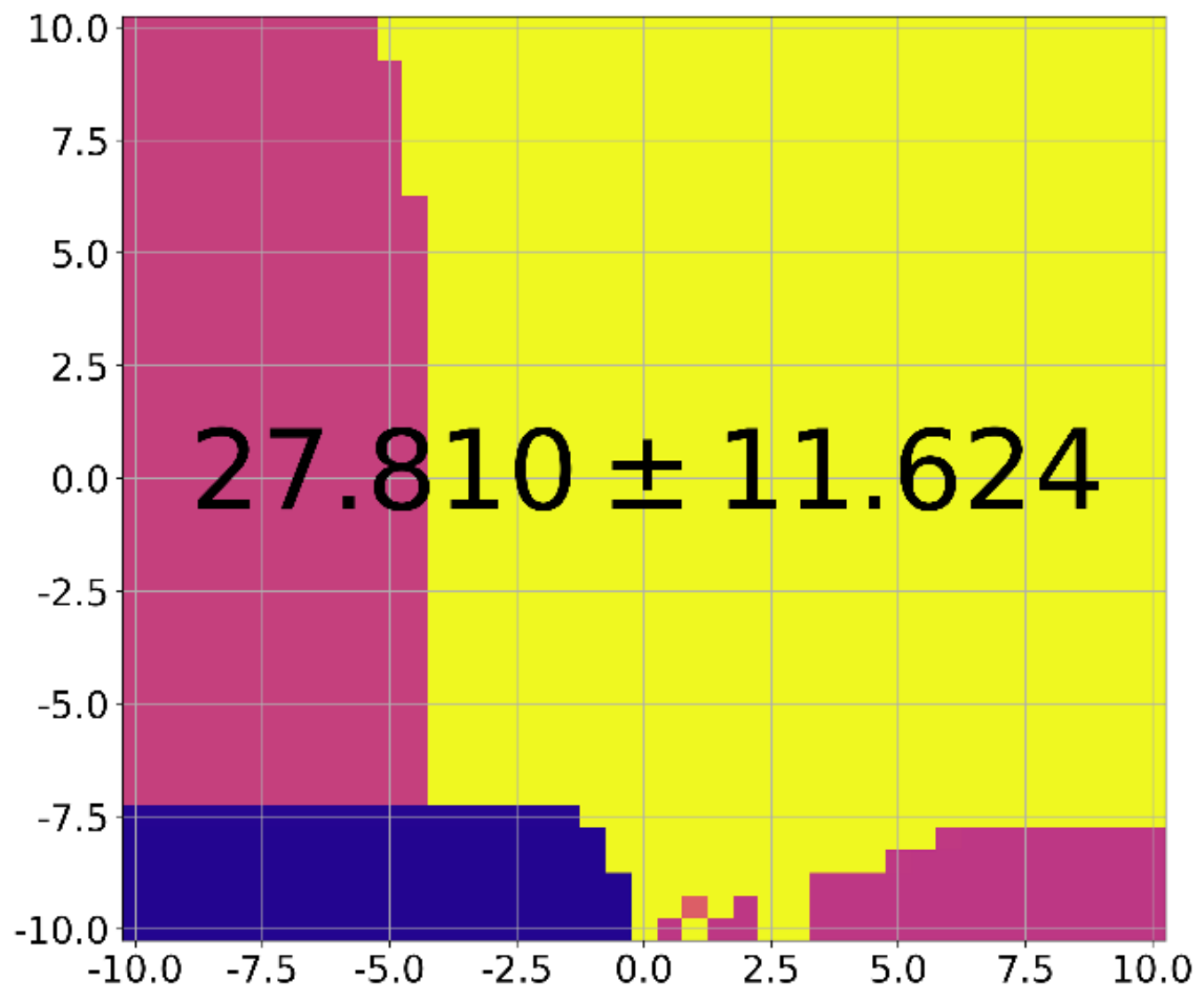}
\end{subfigure}
\begin{subfigure}{0.180\textwidth}
\includegraphics[width=\textwidth]{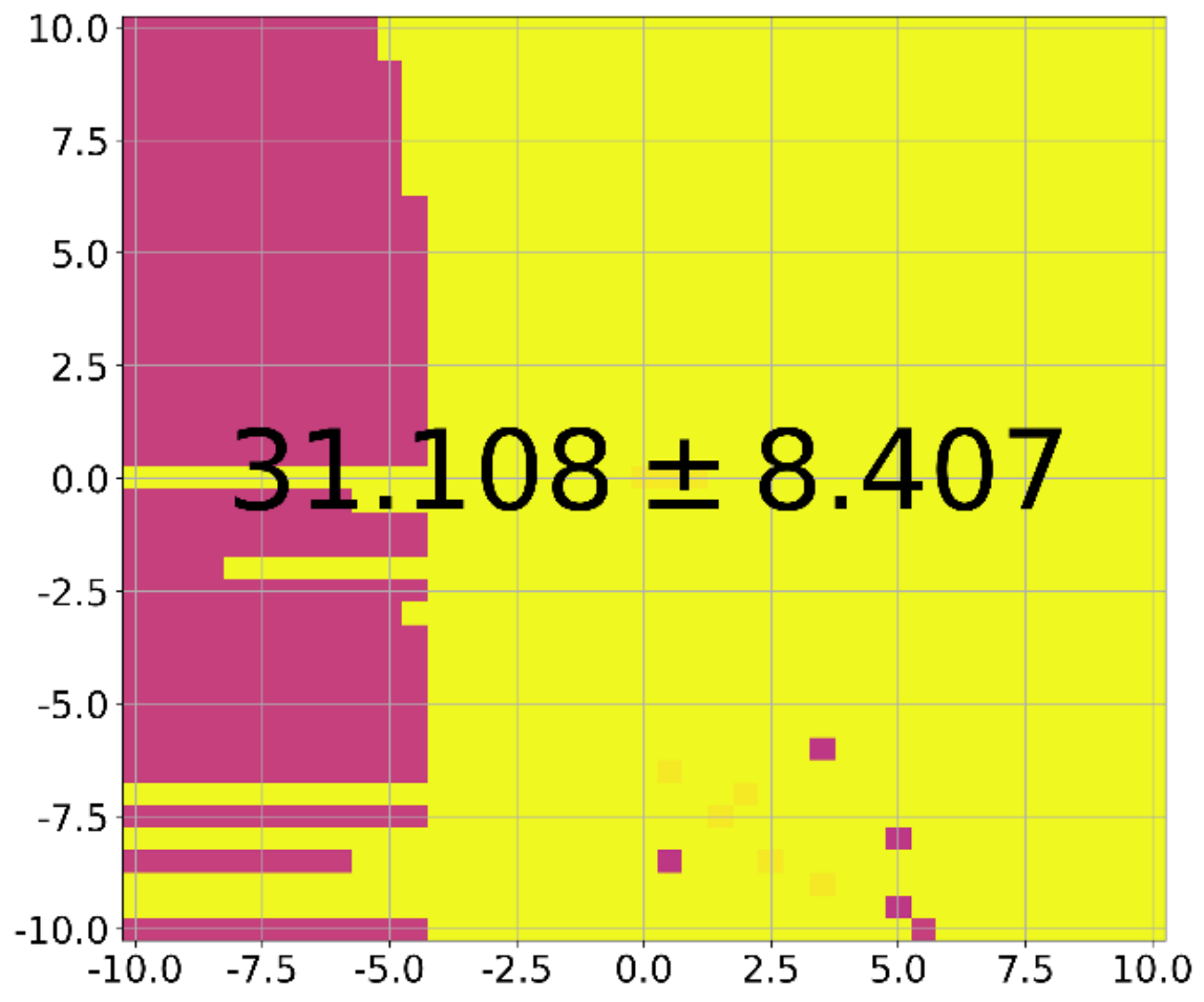}
\end{subfigure}
\begin{subfigure}{0.180\textwidth}
\includegraphics[width=\textwidth]{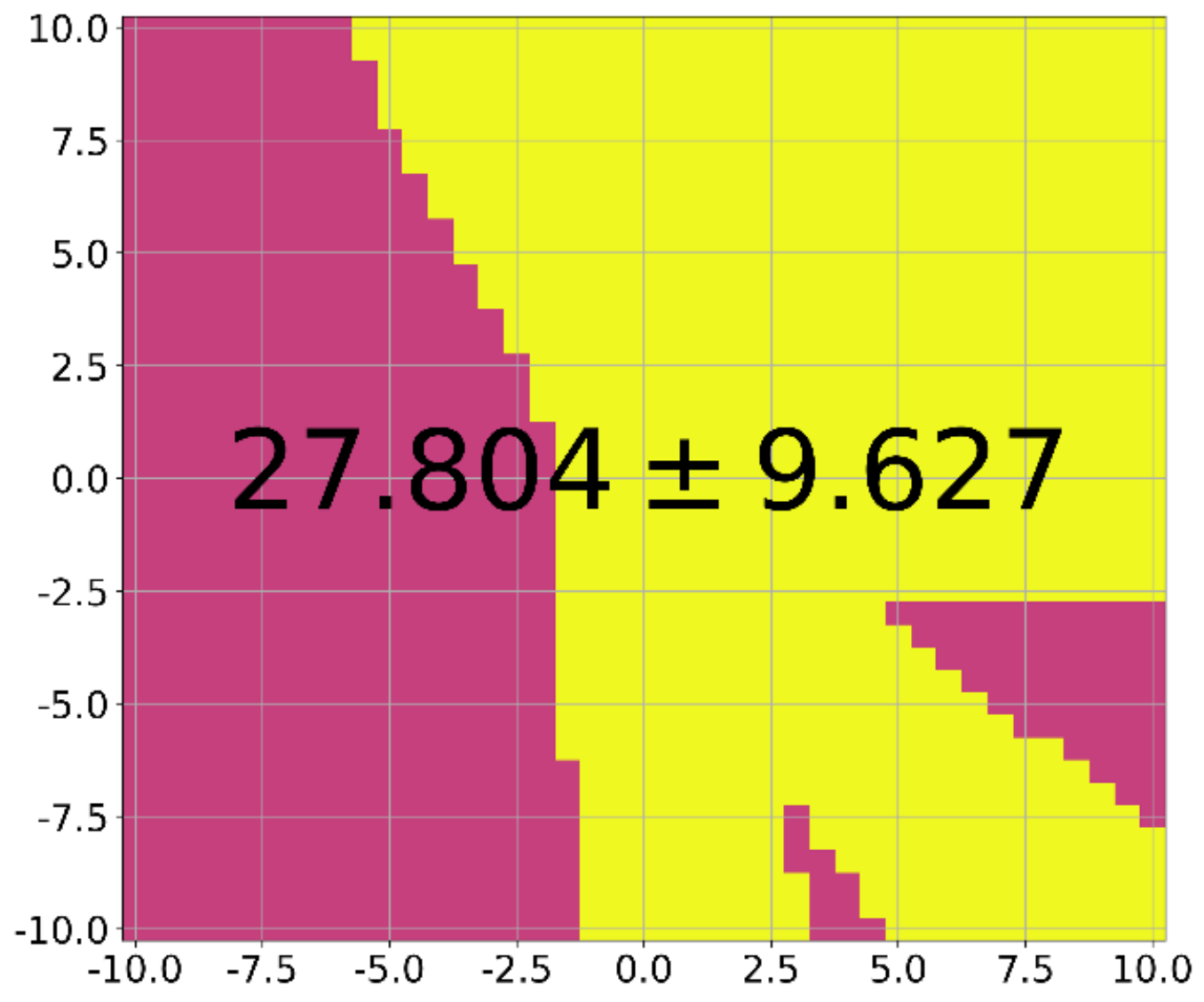}
\end{subfigure}
\begin{subfigure}{0.180\textwidth}
\includegraphics[width=\textwidth]{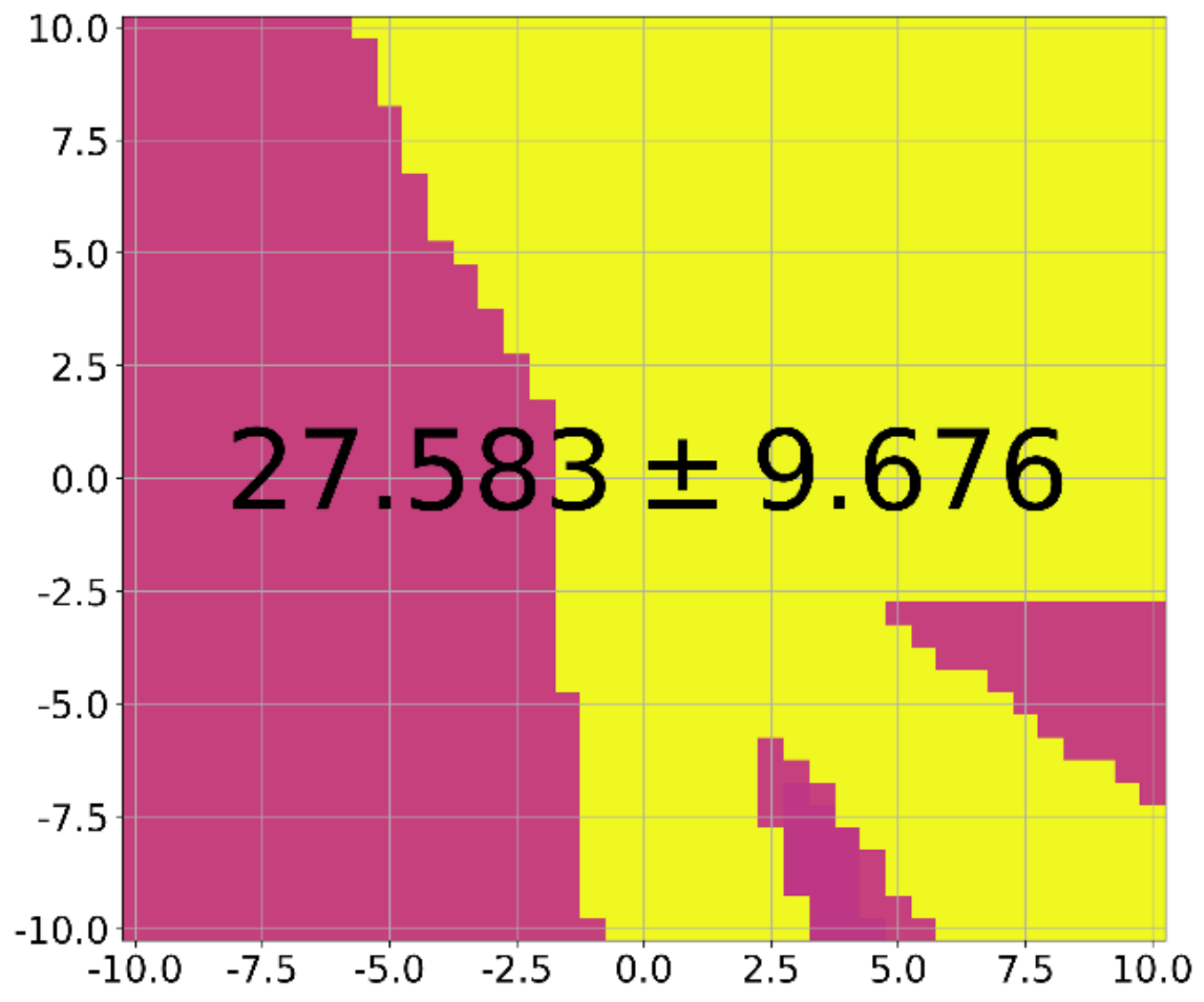}
\end{subfigure}
\begin{subfigure}{0.180\textwidth}
\includegraphics[width=\textwidth]{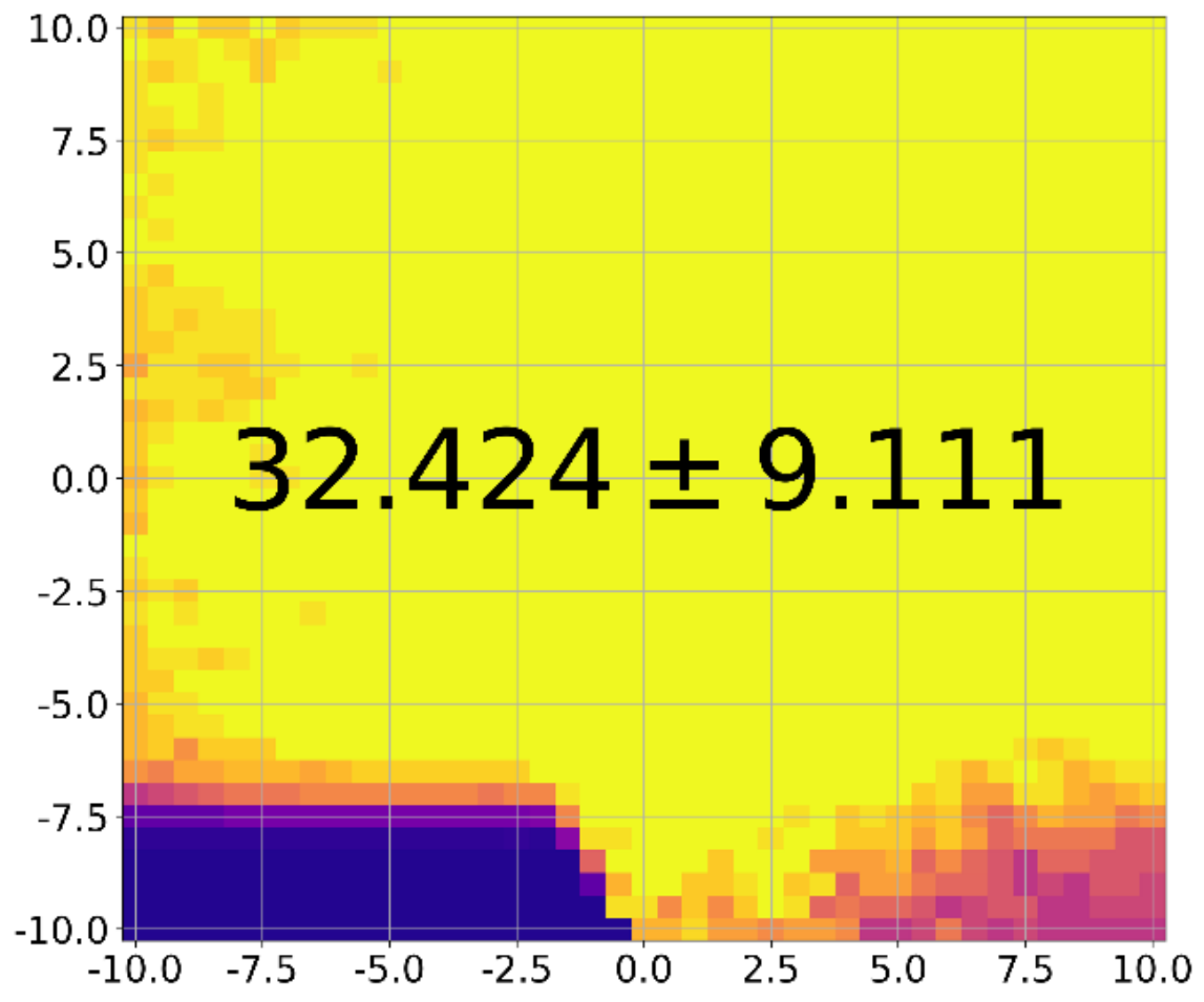}
\end{subfigure}
\begin{subfigure}{0.180\textwidth}
\includegraphics[width=\textwidth]{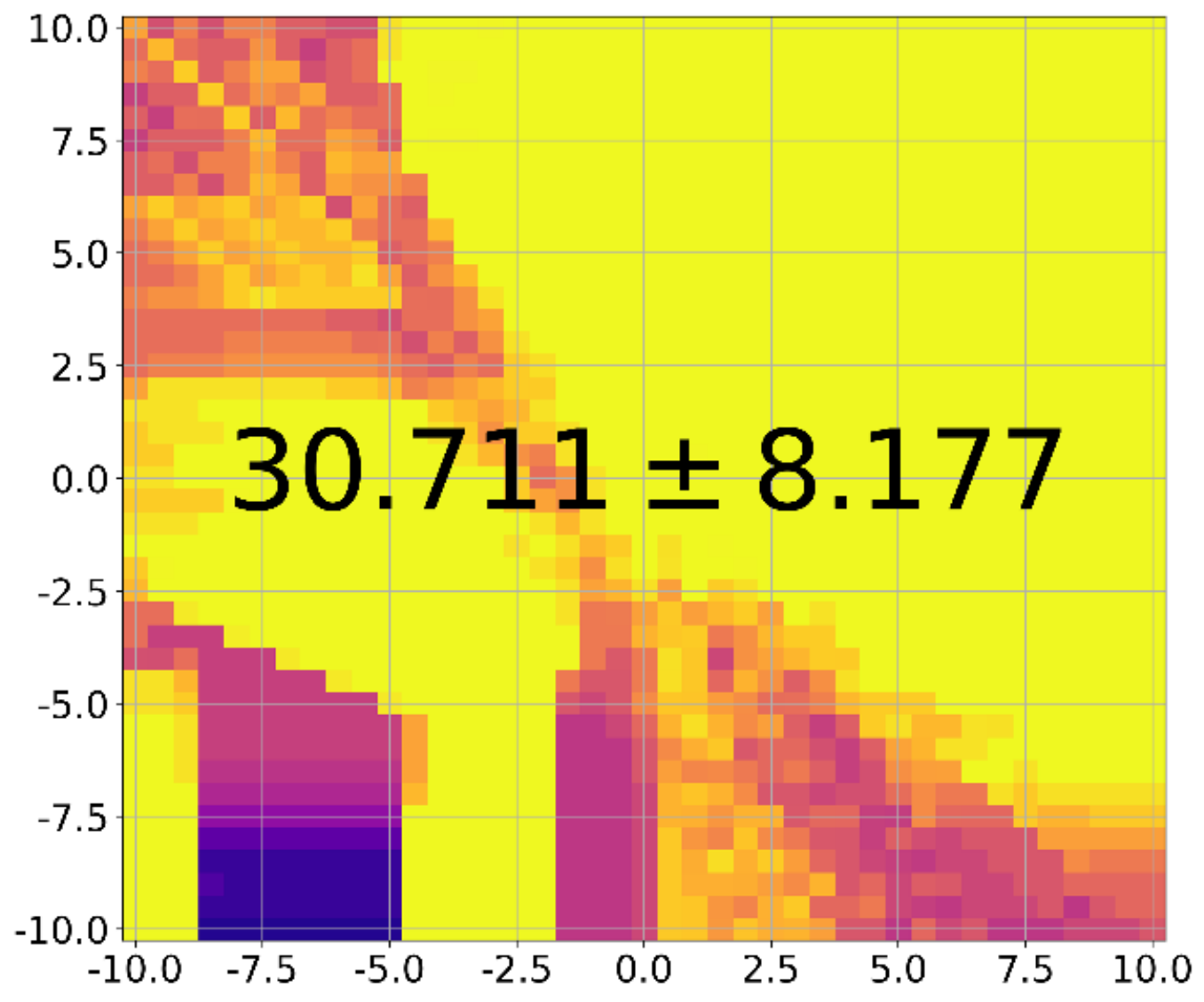}
\end{subfigure}

\begin{subfigure}{0.180\textwidth}
\includegraphics[width=\textwidth]{fig/envprop-3d/bs0__NChain_mod-v1__OneDimensionStateAndTwoActionPolicyNetwork}
\subcaption*{Bias landscape}
\end{subfigure}
\begin{subfigure}{0.180\textwidth}
\includegraphics[width=\textwidth]{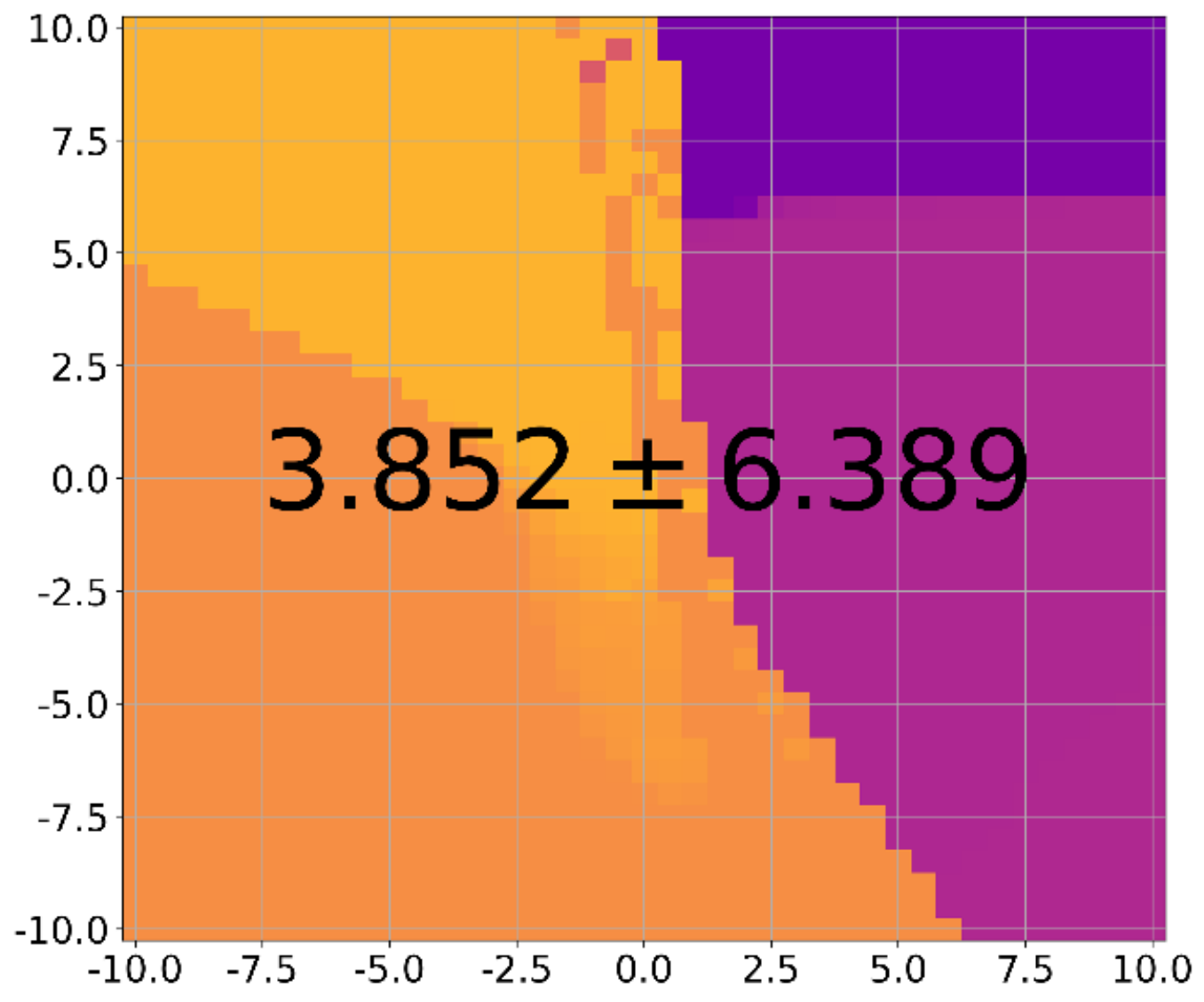}
\subcaption*{Identity}
\end{subfigure}
\begin{subfigure}{0.180\textwidth}
\includegraphics[width=\textwidth]{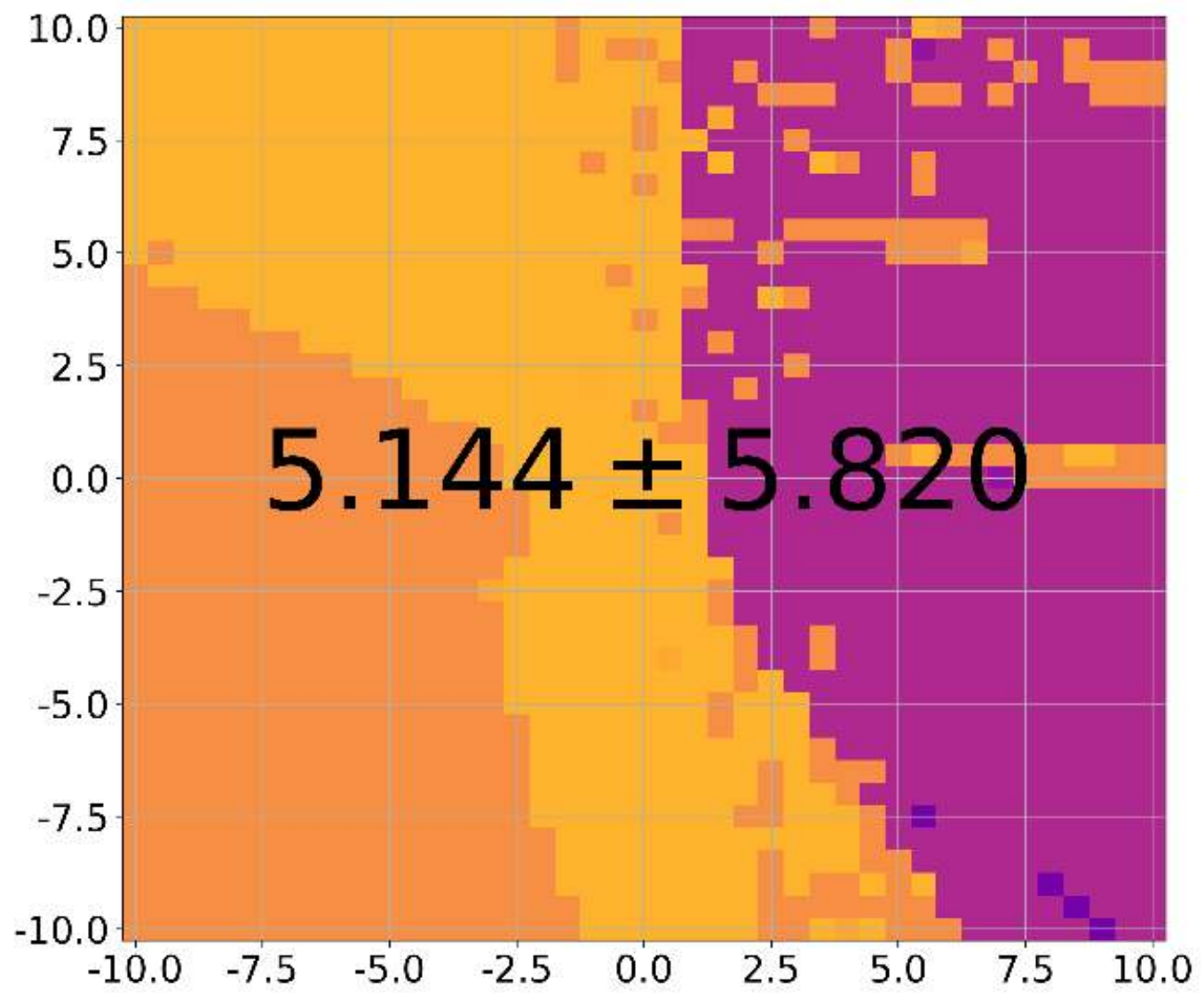}
\subcaption*{{\scriptsize (Modified) Hessian}}
\end{subfigure}
\begin{subfigure}{0.180\textwidth}
\includegraphics[width=\textwidth]{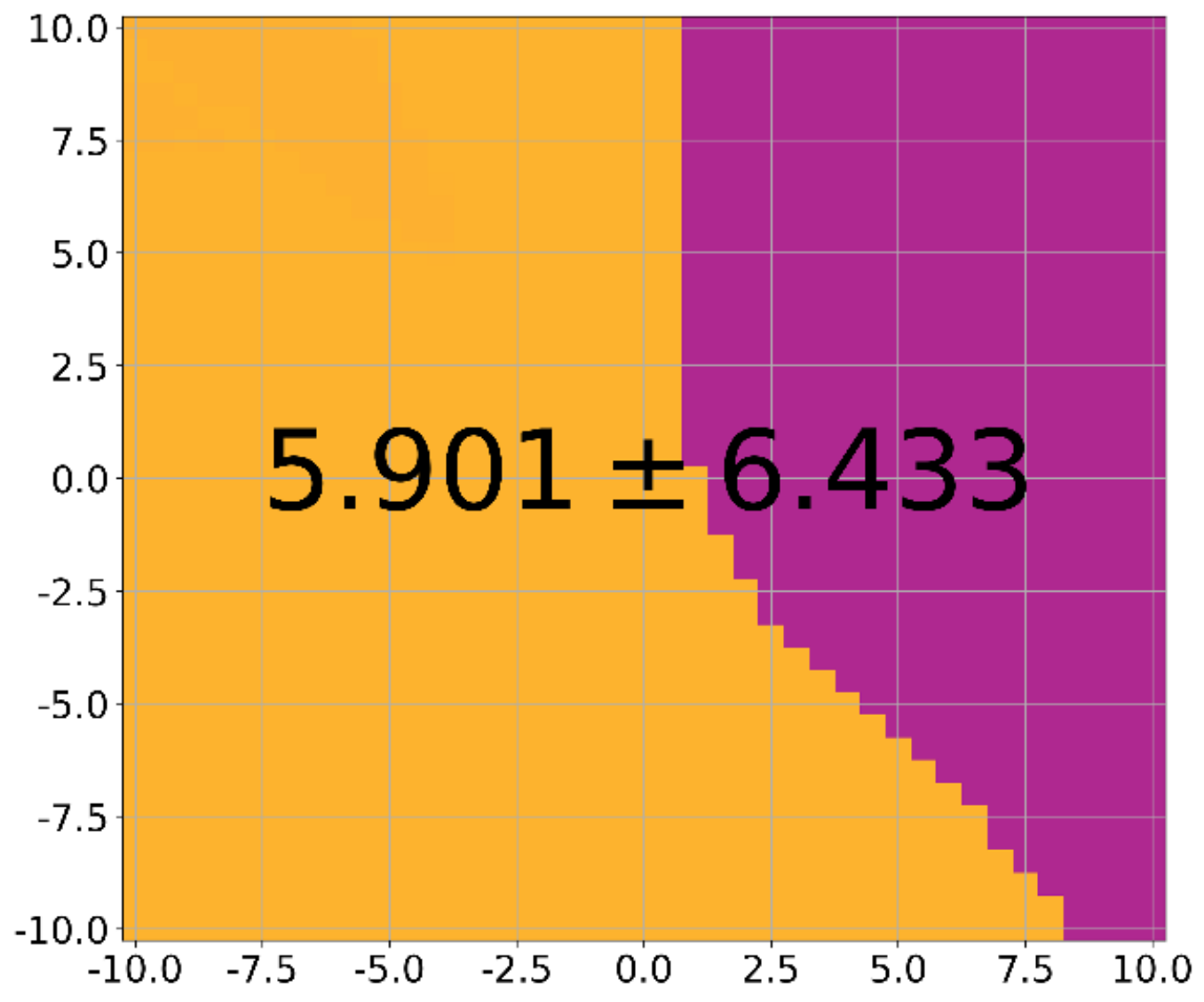}
\subcaption*{Analytic Fisher}
\end{subfigure}
\begin{subfigure}{0.180\textwidth}
\includegraphics[width=\textwidth]{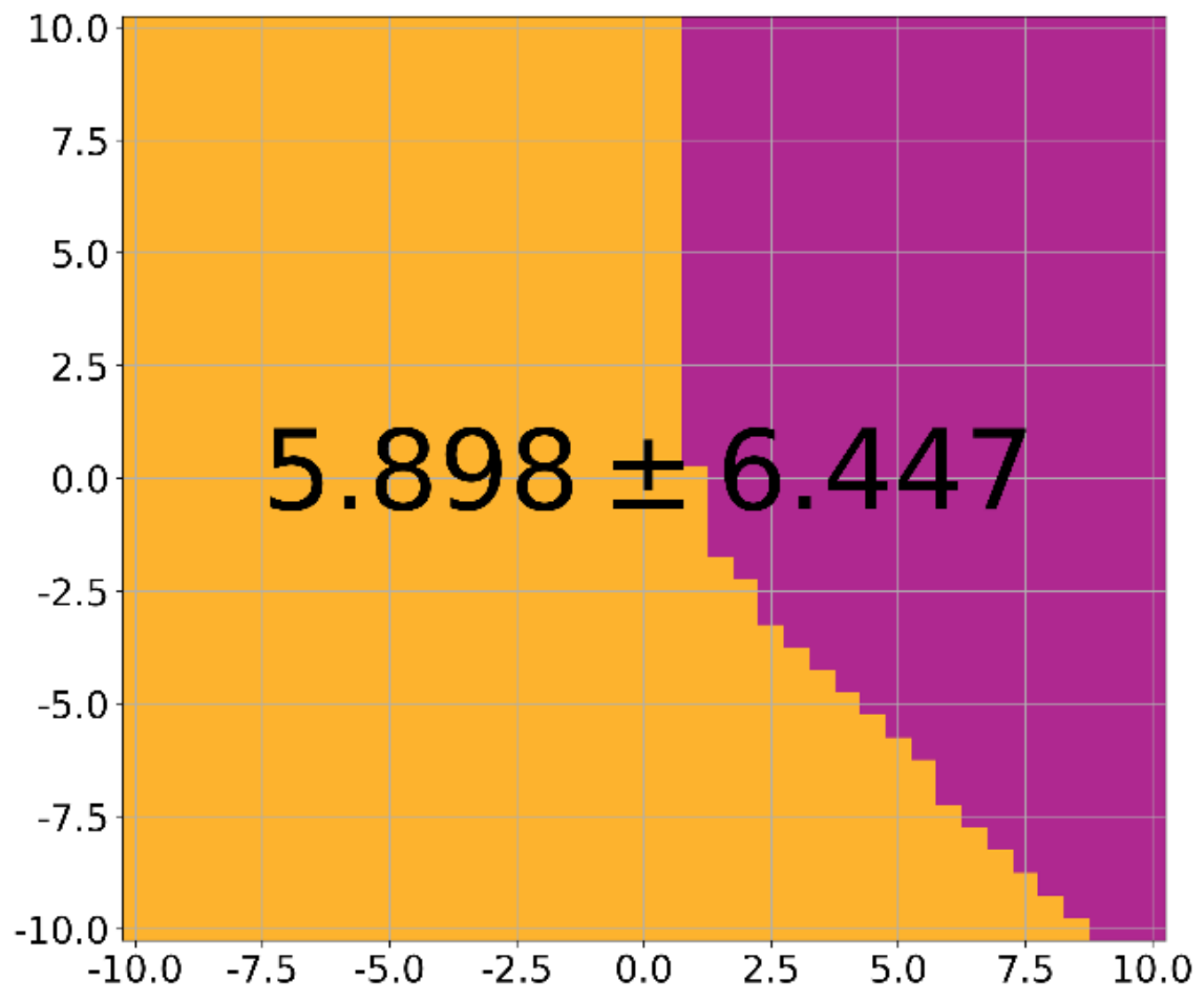}
\subcaption*{\small{Sampling Fisher}} %
\end{subfigure}
\begin{subfigure}{0.180\textwidth}
\includegraphics[width=\textwidth]{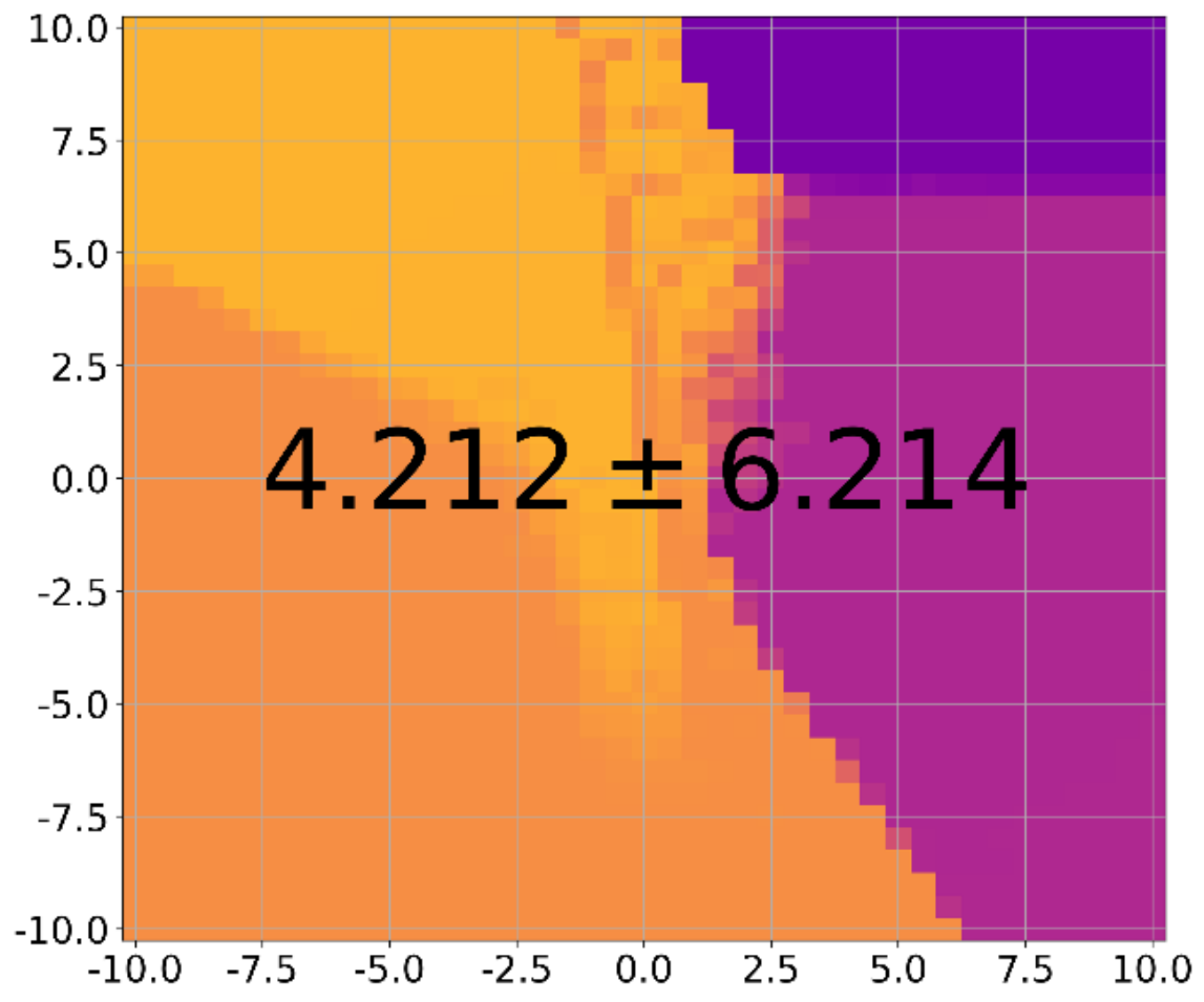}
\subcaption*{\scriptsize 16-sample Identity}
\end{subfigure}
\begin{subfigure}{0.180\textwidth}
\includegraphics[width=\textwidth]{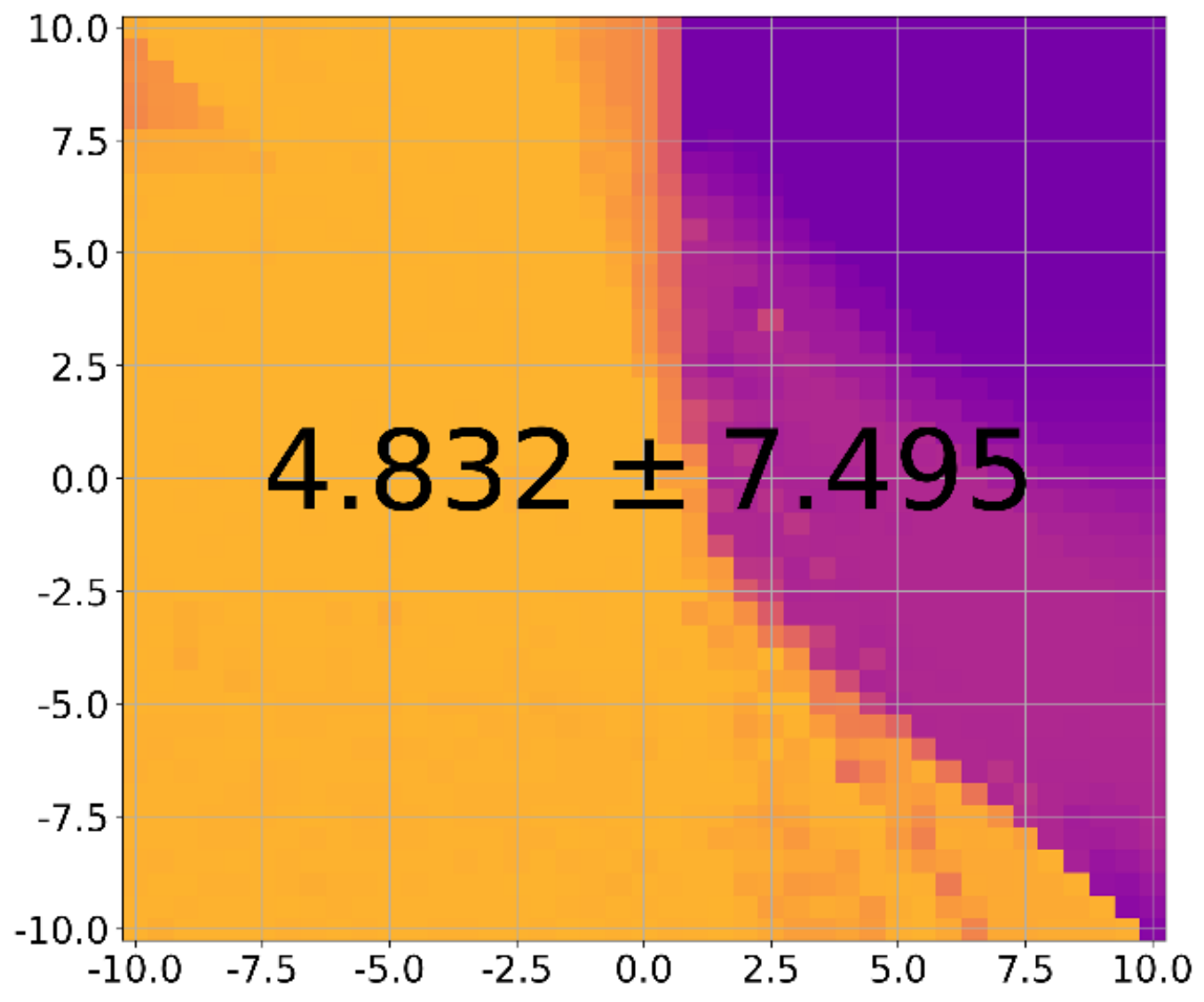}
\subcaption*{16-sample Fisher}
\end{subfigure}

\caption{Bias-only optimization on Env-A1, A2, A3, B1, B2, and B3 (top to bottom rows).
The 2nd to 5th columns from the left correspond to exact optimization of
Schemes i, ii, iii, and iv, whereas the two right-most columns correspond to
sampling-based approximation of Schemes i and iv (with 10 random seeds),
as explained in \secref{sec:biasopt}.
The color indicates the bias $v_b(\vecb{\theta}, s_0)$ of
the final policy parameter $\vecb{\theta}$ from optimization
initialized at the corresponding 2D-coordinate as $\vecb{\theta}_0$.
It maps low bias to dark-blue, whereas high bias to yellow.
Each subplot has $1681$ final bias grid-values, whose mean and standard deviation
are indicated by the number in the middle.
Here, the highest final bias (yellow color) is desirable.
For experimental setup, see \secref{sec:xprmt_setup_nbwpg}.
}
\label{fig:biasoptcmp}
\end{figure}
\end{landscape}

\begin{landscape}
\begin{figure}
\centering

\begin{subfigure}{0.170\textwidth}
\includegraphics[width=\textwidth]{fig/envprop-3d/bs0__Example_10_1_2-v1__OneDimensionStateAndTwoActionPolicyNetwork}
\end{subfigure}
\begin{subfigure}{0.170\textwidth}
\includegraphics[width=\textwidth]{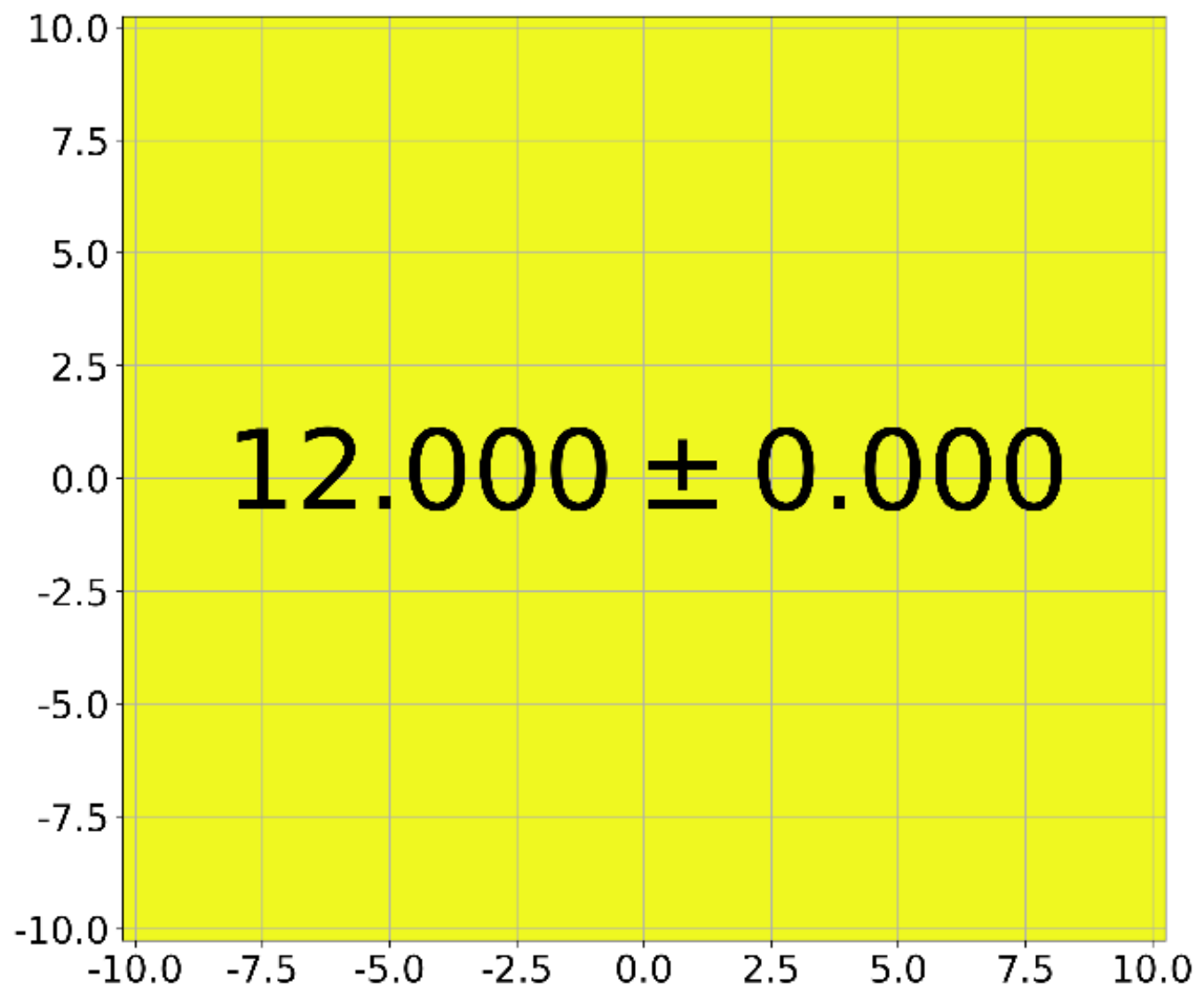}
\end{subfigure}
\begin{subfigure}{0.170\textwidth}
\includegraphics[width=\textwidth]{fig/biasopt/bias__obj=bias__conditioner_mode=fisher_probtransient__nseed1__OneDimensionStateAndTwoActionPolicyNetwork__Example_10_1_2-v1}
\end{subfigure}
\begin{subfigure}{0.170\textwidth}
\includegraphics[width=\textwidth]{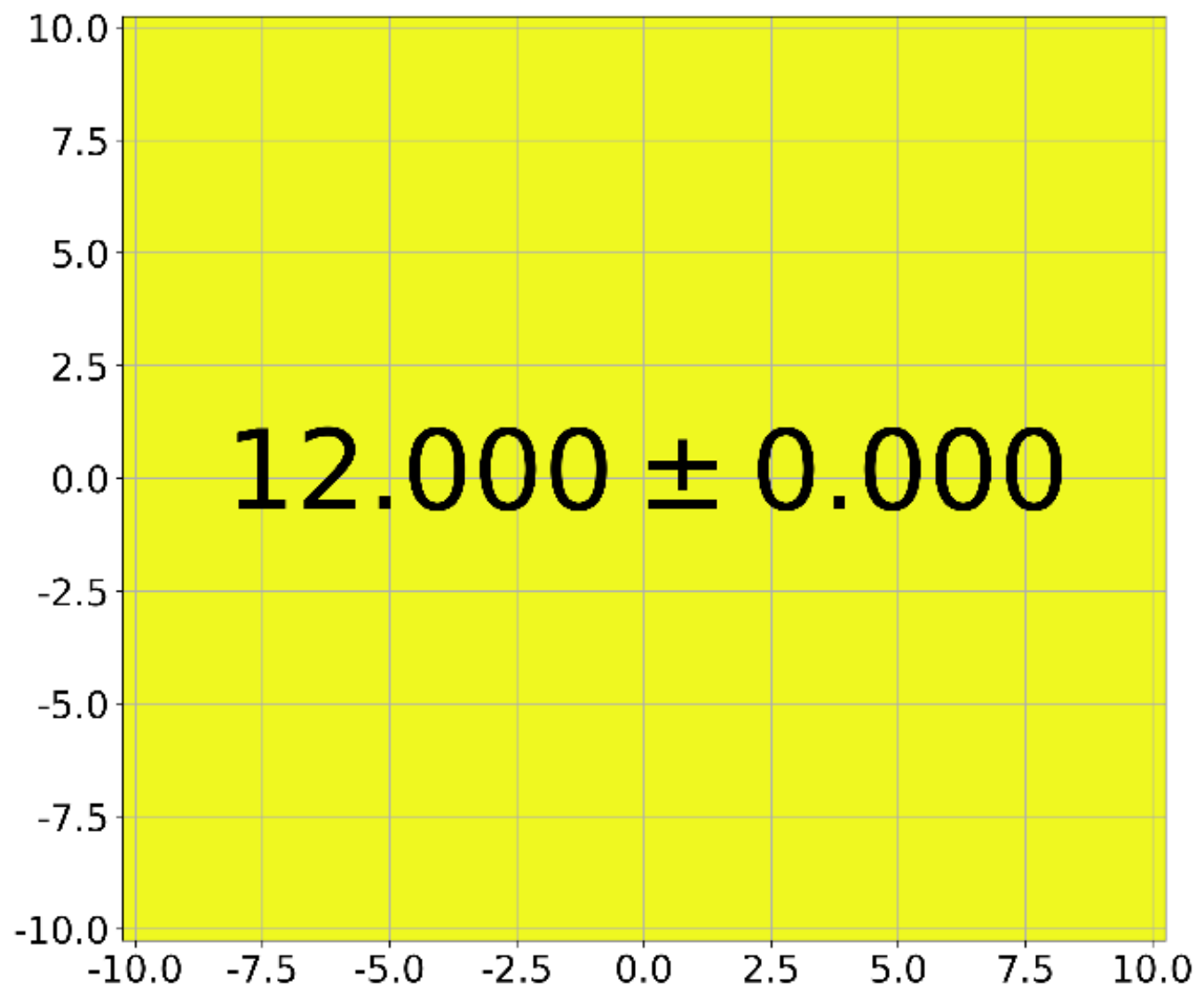}
\end{subfigure}
\begin{subfigure}{0.170\textwidth}
\includegraphics[width=\textwidth]{fig/biasopt/bias__obj=bias__conditioner_mode=fisher_transient_withsteadymul_upto_t1__nseed1__OneDimensionStateAndTwoActionPolicyNetwork__Example_10_1_2-v1}
\end{subfigure}
\begin{subfigure}{0.170\textwidth}
\includegraphics[width=\textwidth]{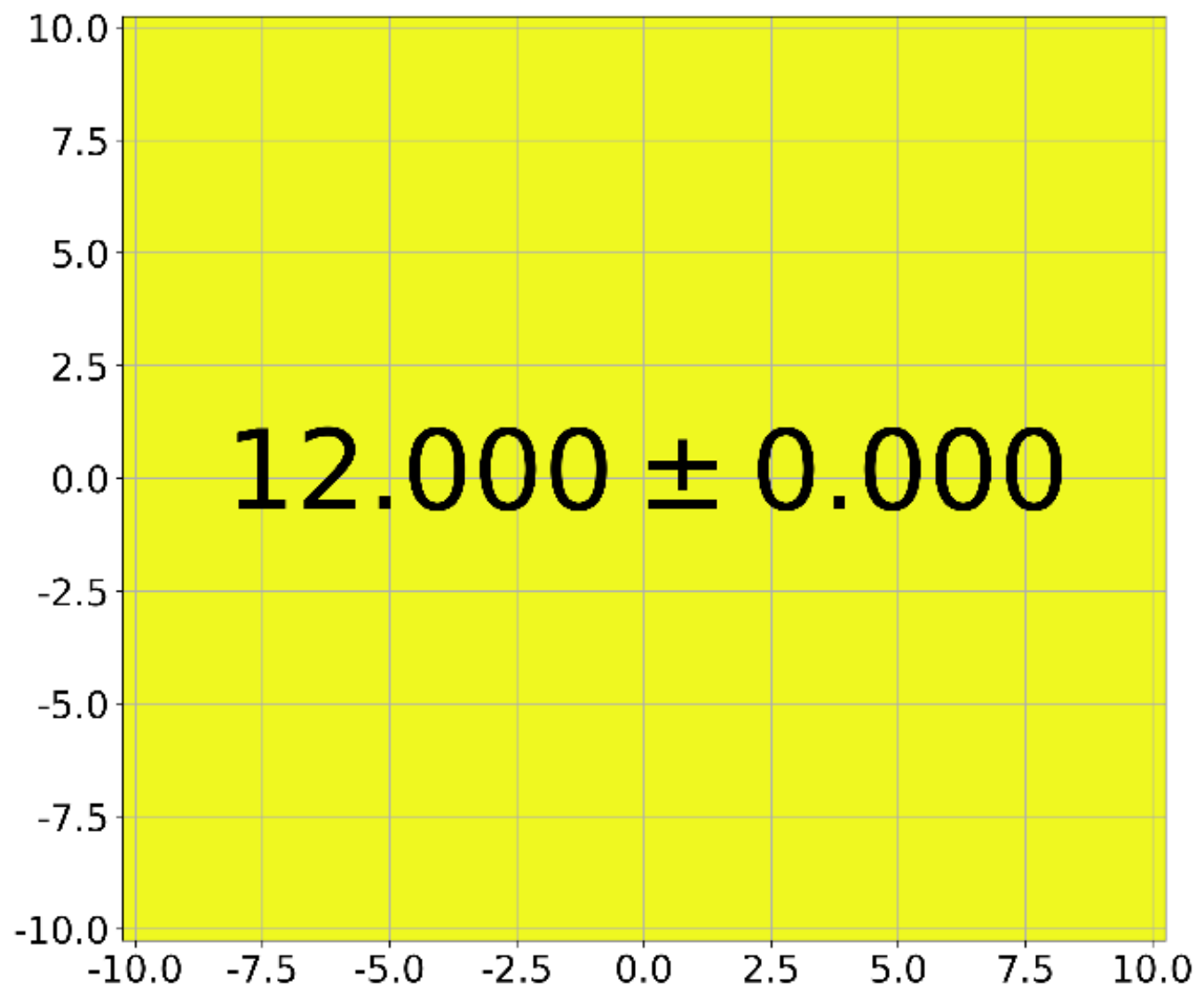}
\end{subfigure}
\begin{subfigure}{0.170\textwidth}
\includegraphics[width=\textwidth]{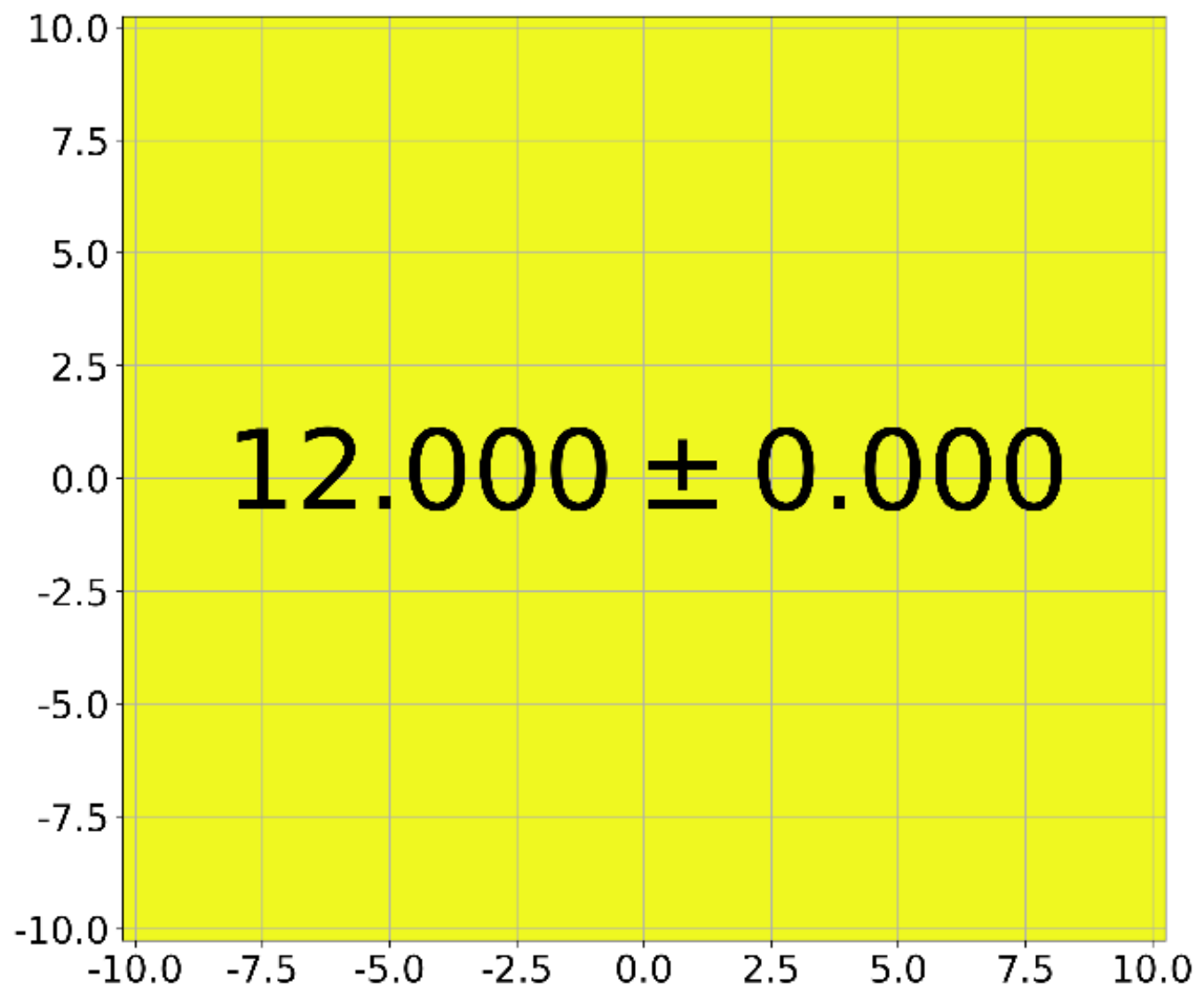}
\end{subfigure}
\begin{subfigure}{0.170\textwidth}
\includegraphics[width=\textwidth]{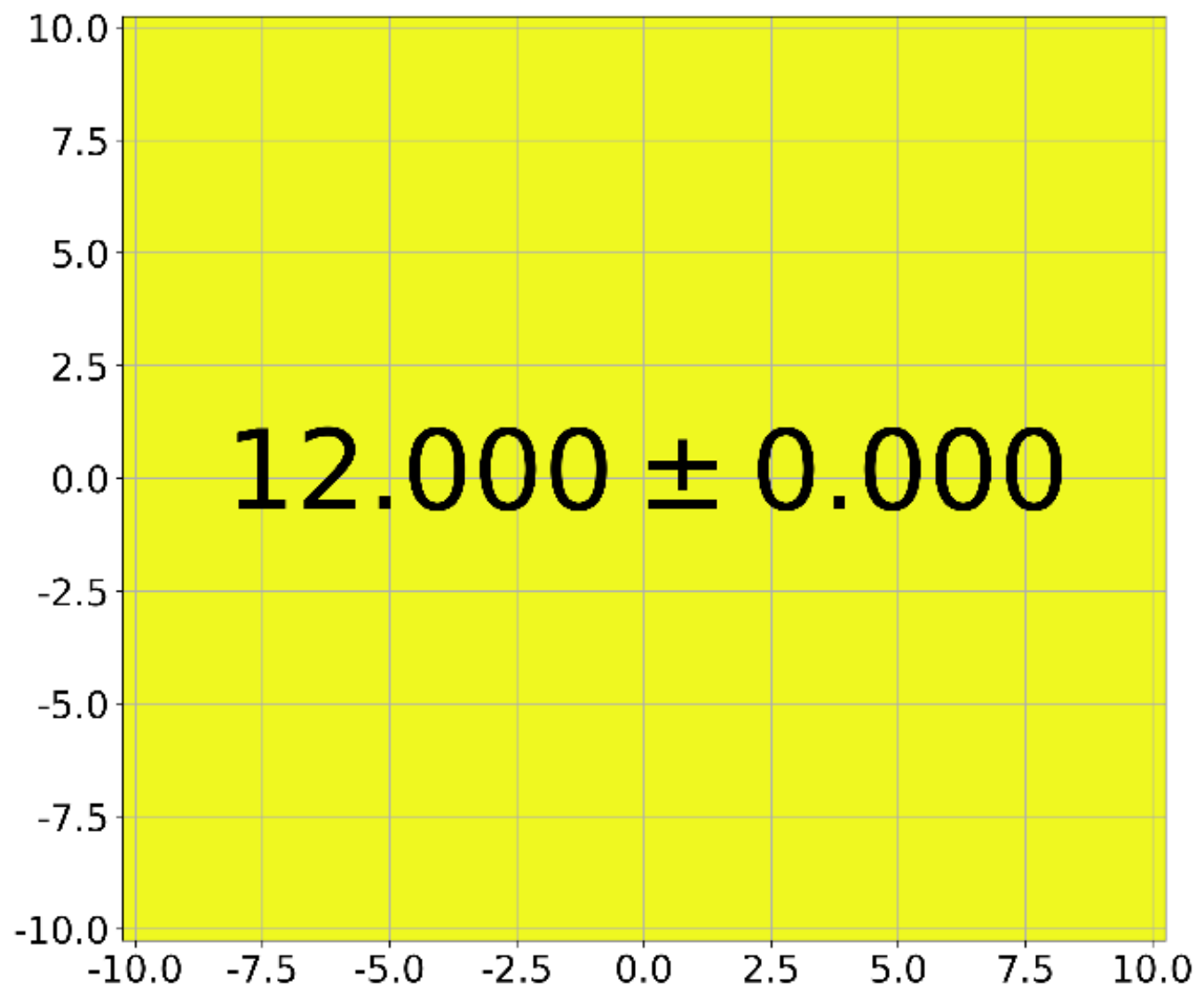}
\end{subfigure}

\begin{subfigure}{0.170\textwidth}
\includegraphics[width=\textwidth]{fig/envprop-3d/bs0__Tor_20210307-v1__OneDimensionStateAndTwoActionPolicyNetwork}
\end{subfigure}
\begin{subfigure}{0.170\textwidth}
\includegraphics[width=\textwidth]{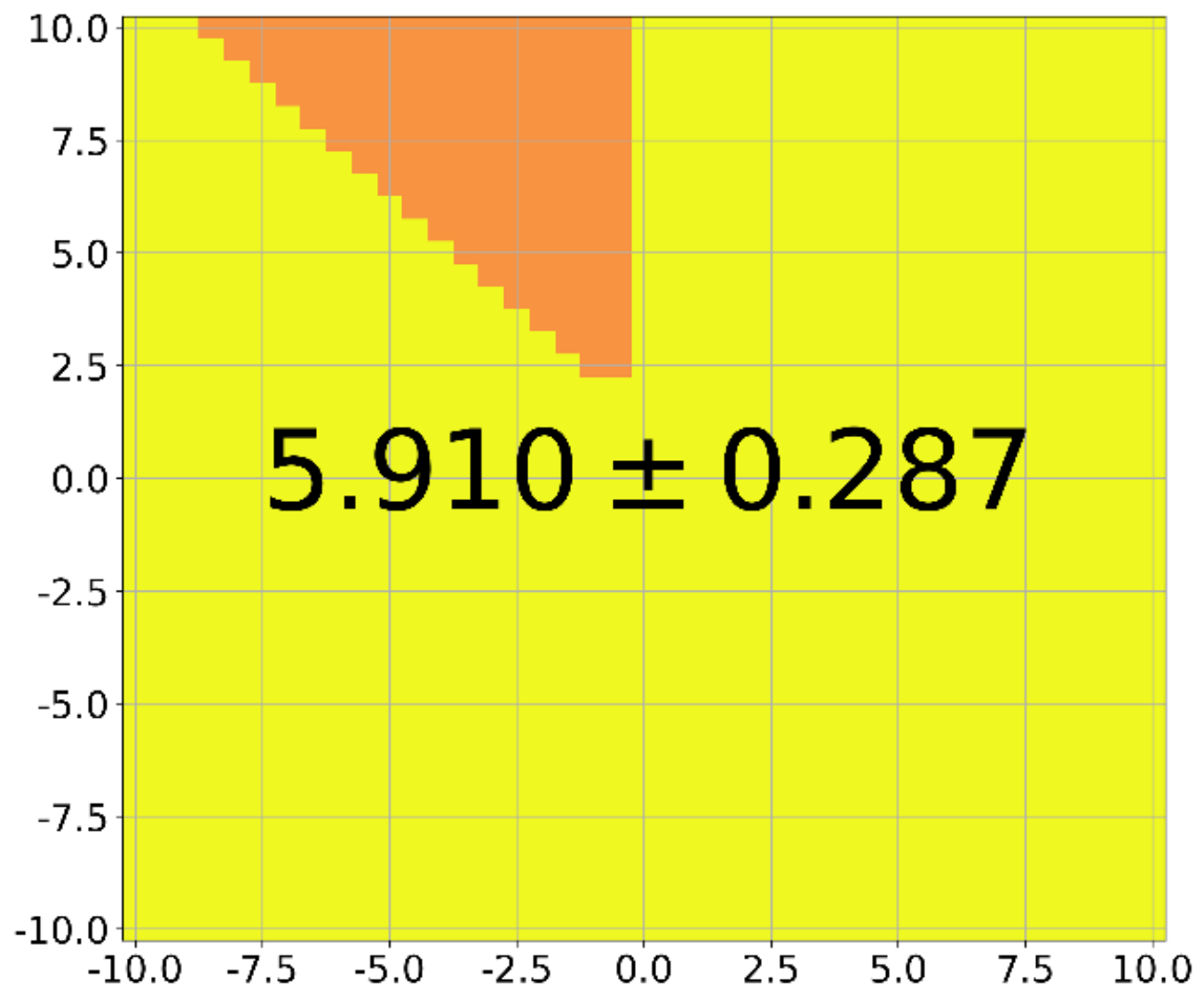}
\end{subfigure}
\begin{subfigure}{0.170\textwidth}
\includegraphics[width=\textwidth]{fig/biasopt/bias__obj=bias__conditioner_mode=fisher_probtransient__nseed1__OneDimensionStateAndTwoActionPolicyNetwork__Tor_20210307-v1}
\end{subfigure}
\begin{subfigure}{0.170\textwidth}
\includegraphics[width=\textwidth]{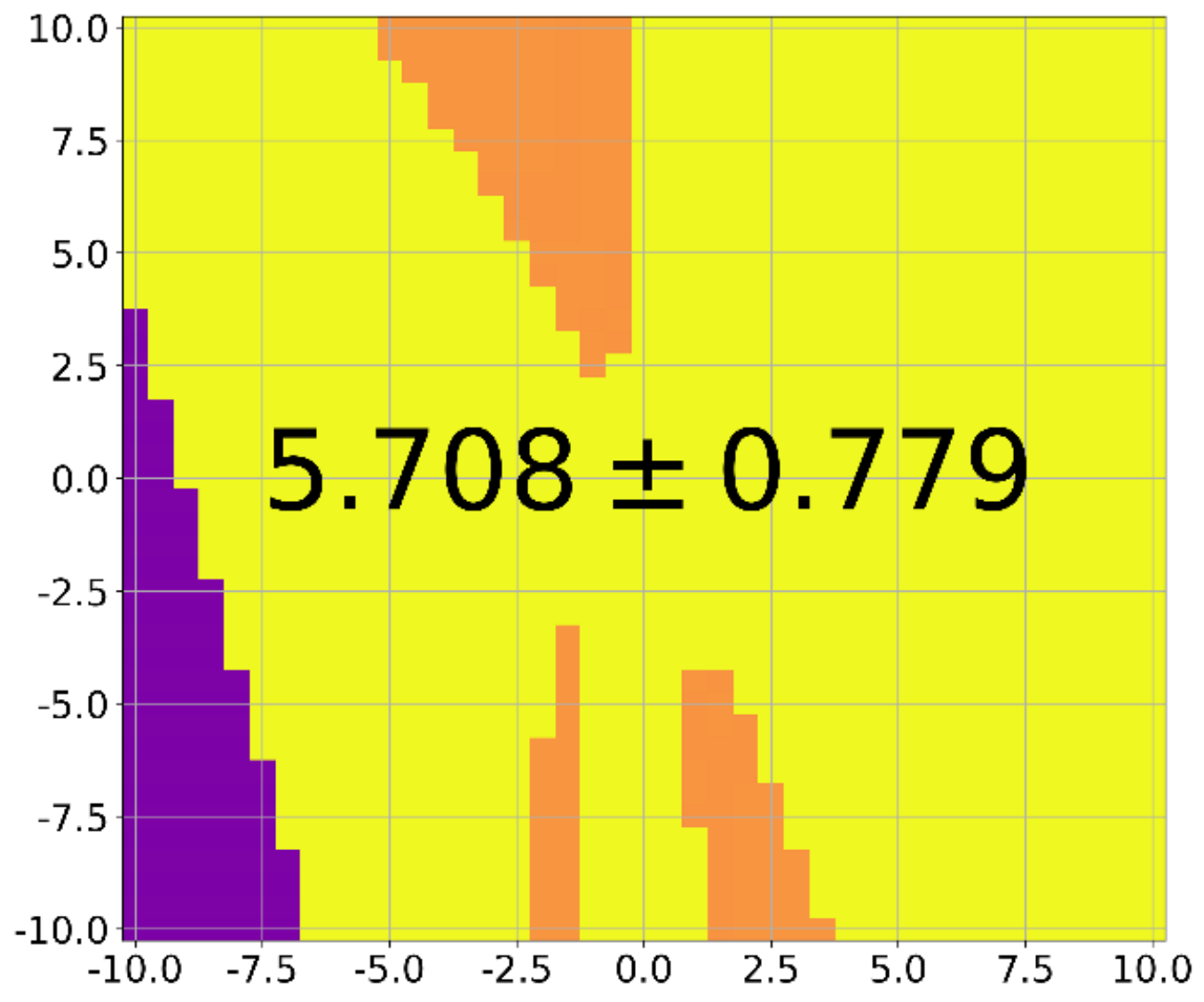}
\end{subfigure}
\begin{subfigure}{0.170\textwidth}
\includegraphics[width=\textwidth]{fig/biasopt/bias__obj=bias__conditioner_mode=fisher_transient_withsteadymul_upto_t1__nseed1__OneDimensionStateAndTwoActionPolicyNetwork__Tor_20210307-v1}
\end{subfigure}
\begin{subfigure}{0.170\textwidth}
\includegraphics[width=\textwidth]{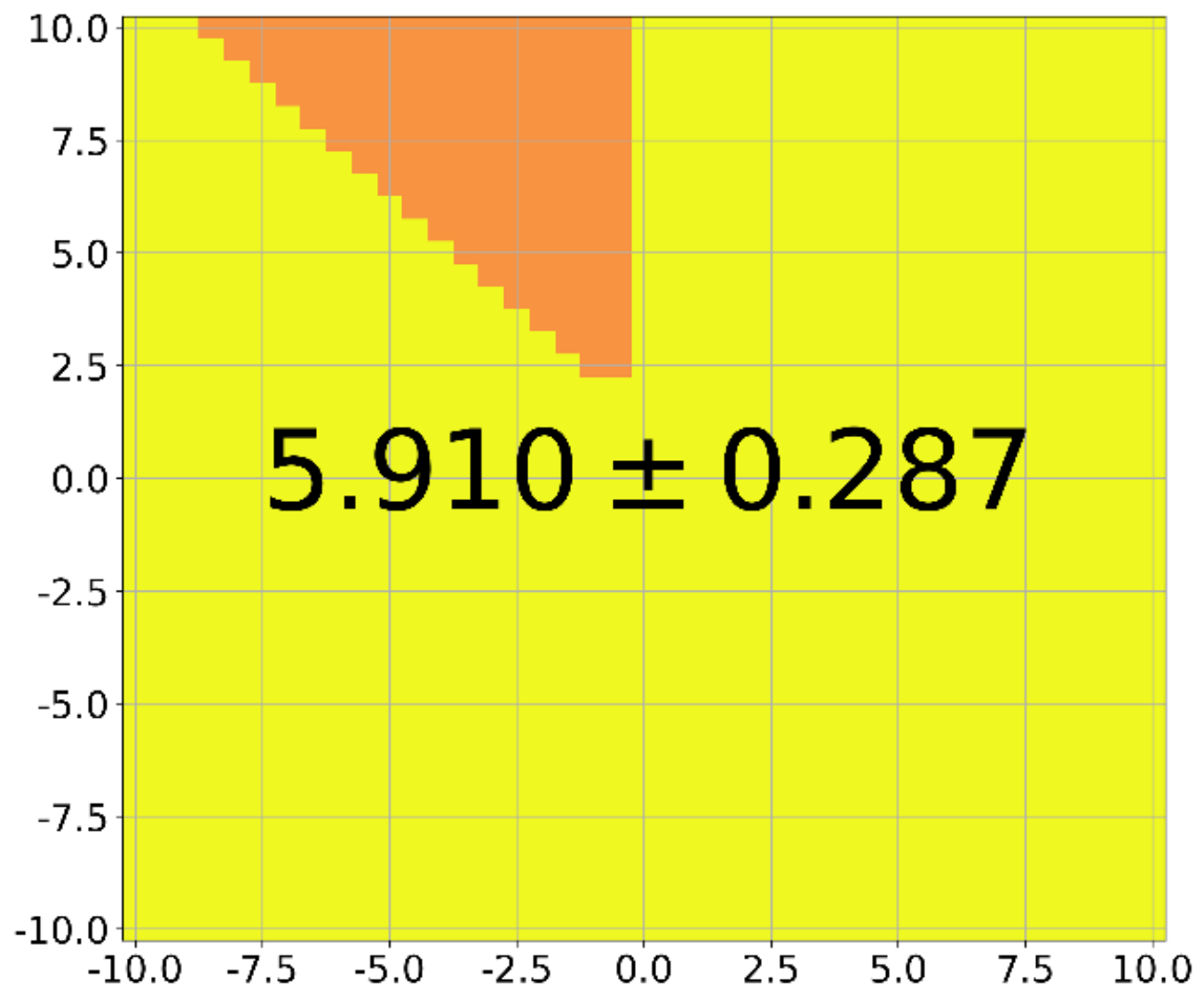}
\end{subfigure}
\begin{subfigure}{0.170\textwidth}
\includegraphics[width=\textwidth]{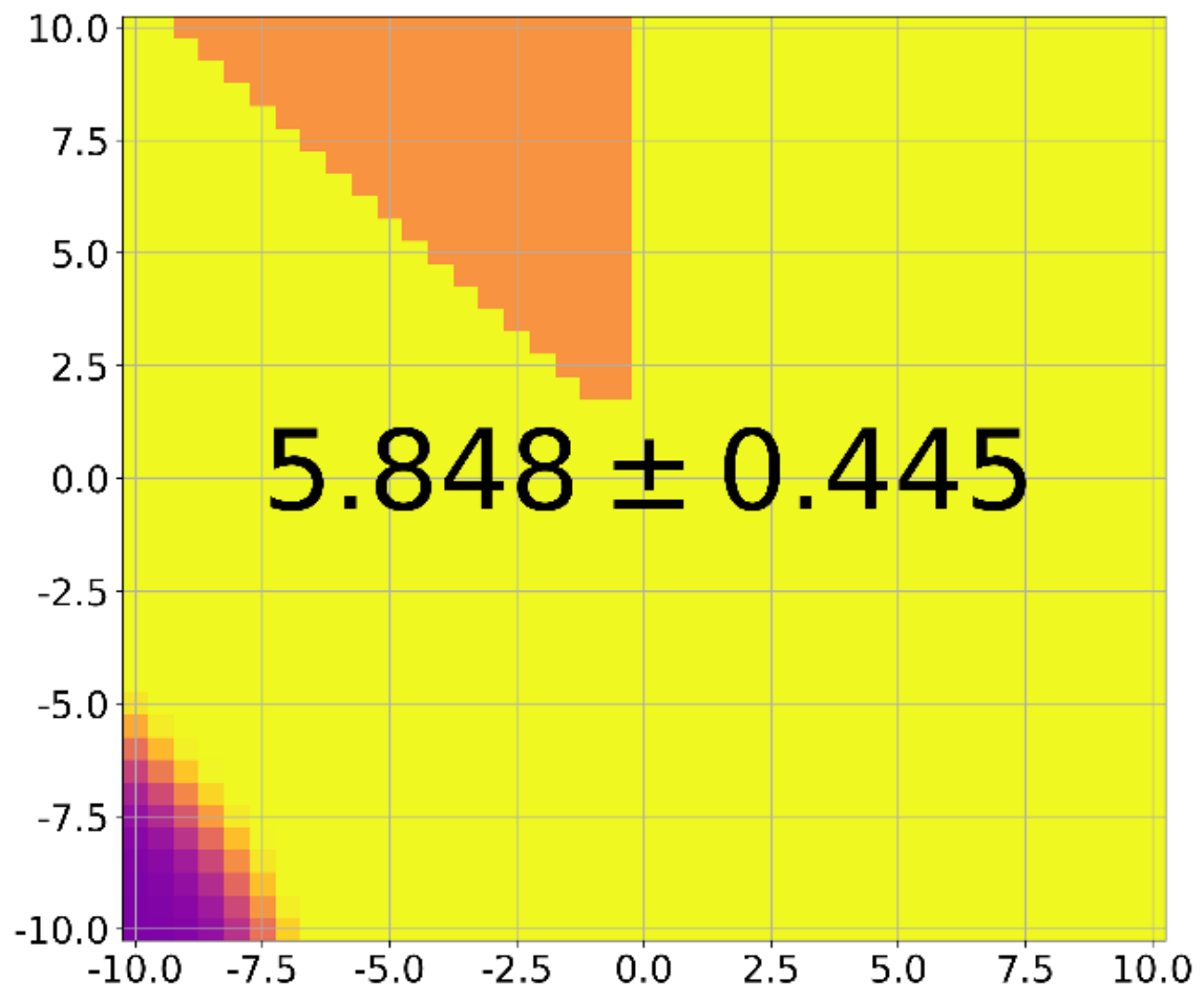}
\end{subfigure}
\begin{subfigure}{0.170\textwidth}
\includegraphics[width=\textwidth]{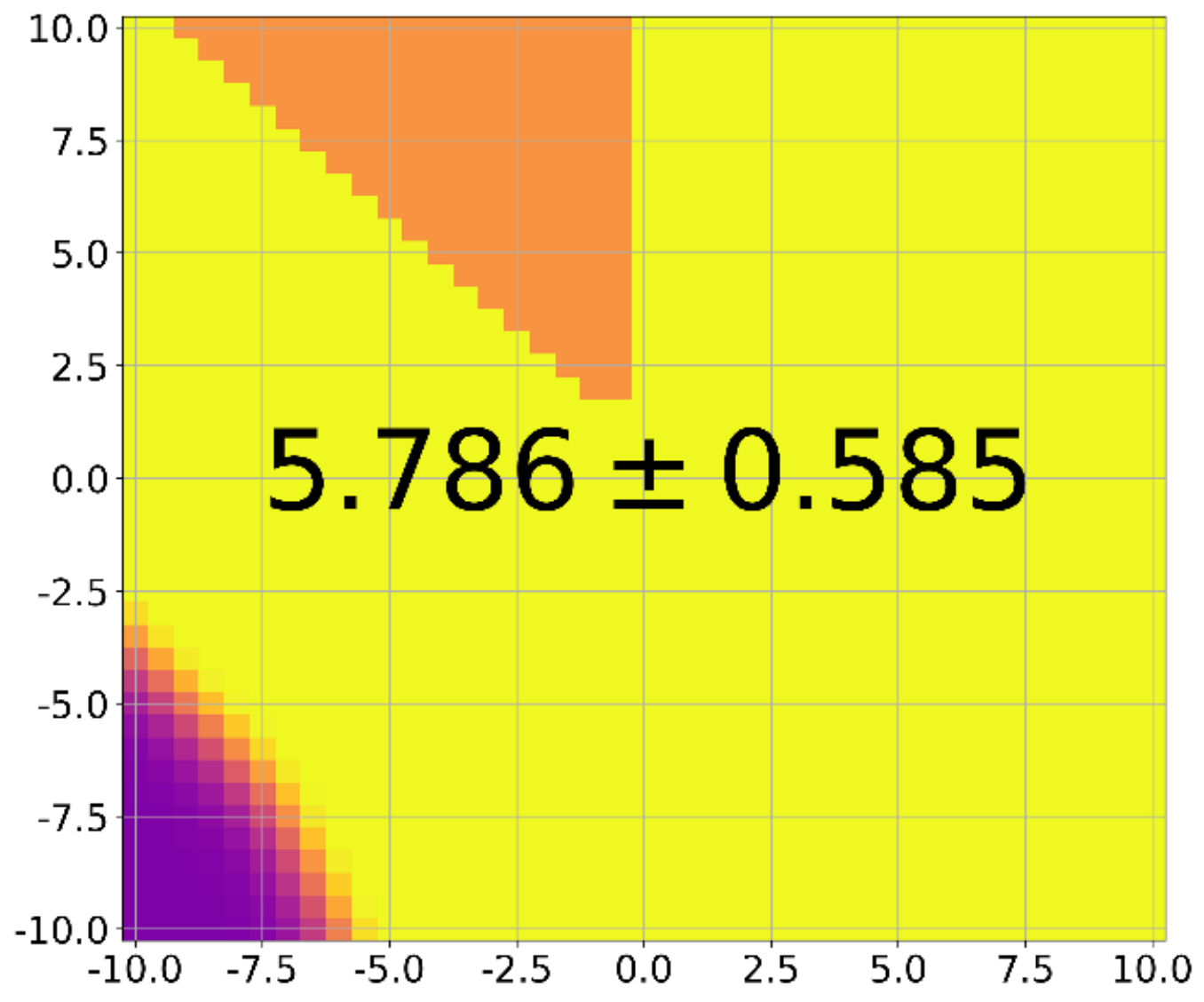}
\end{subfigure}

\begin{subfigure}{0.170\textwidth}
\includegraphics[width=\textwidth]{fig/envprop-3d/bs0__GridNav_2-v0__OneDimensionStateAndTwoActionPolicyNetwork}
\end{subfigure}
\begin{subfigure}{0.170\textwidth}
\includegraphics[width=\textwidth]{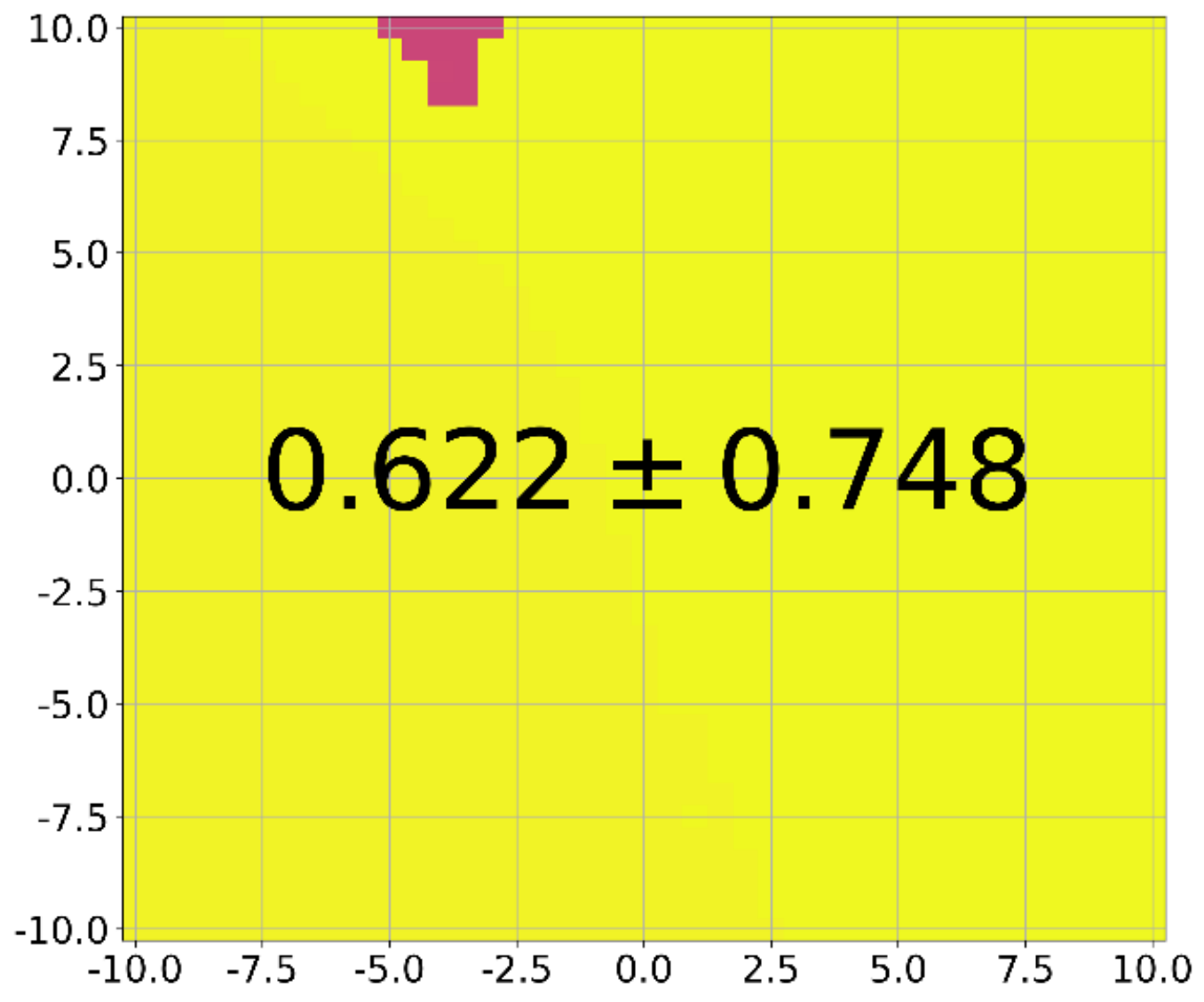}
\end{subfigure}
\begin{subfigure}{0.170\textwidth}
\includegraphics[width=\textwidth]{fig/biasopt/bias__obj=bias__conditioner_mode=fisher_probtransient__nseed1__OneDimensionStateAndTwoActionPolicyNetwork__GridNav_2-v0}
\end{subfigure}
\begin{subfigure}{0.170\textwidth}
\includegraphics[width=\textwidth]{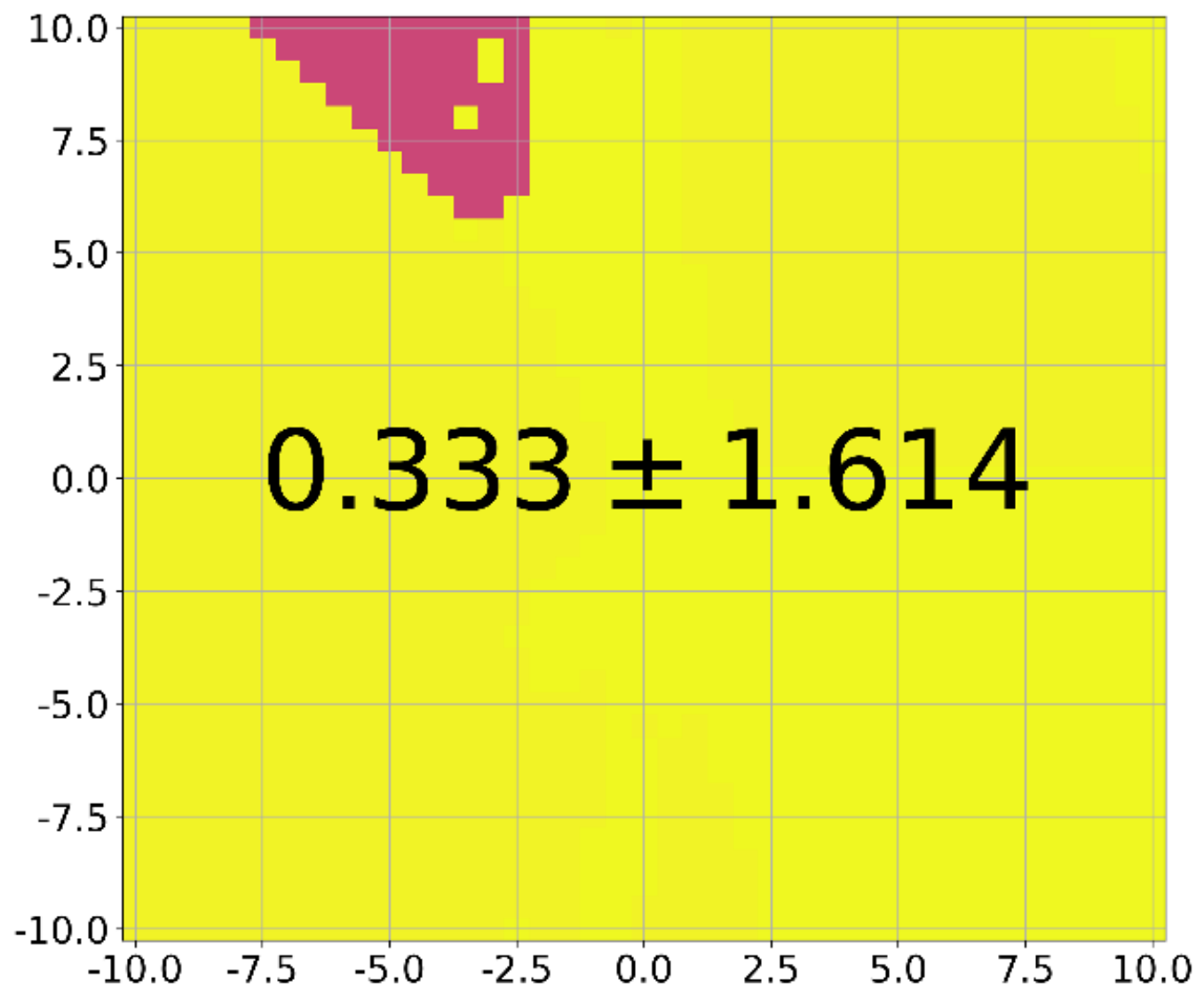}
\end{subfigure}
\begin{subfigure}{0.170\textwidth}
\includegraphics[width=\textwidth]{fig/biasopt/bias__obj=bias__conditioner_mode=fisher_transient_withsteadymul_upto_t1__nseed1__OneDimensionStateAndTwoActionPolicyNetwork__GridNav_2-v0}
\end{subfigure}
\begin{subfigure}{0.170\textwidth}
\includegraphics[width=\textwidth]{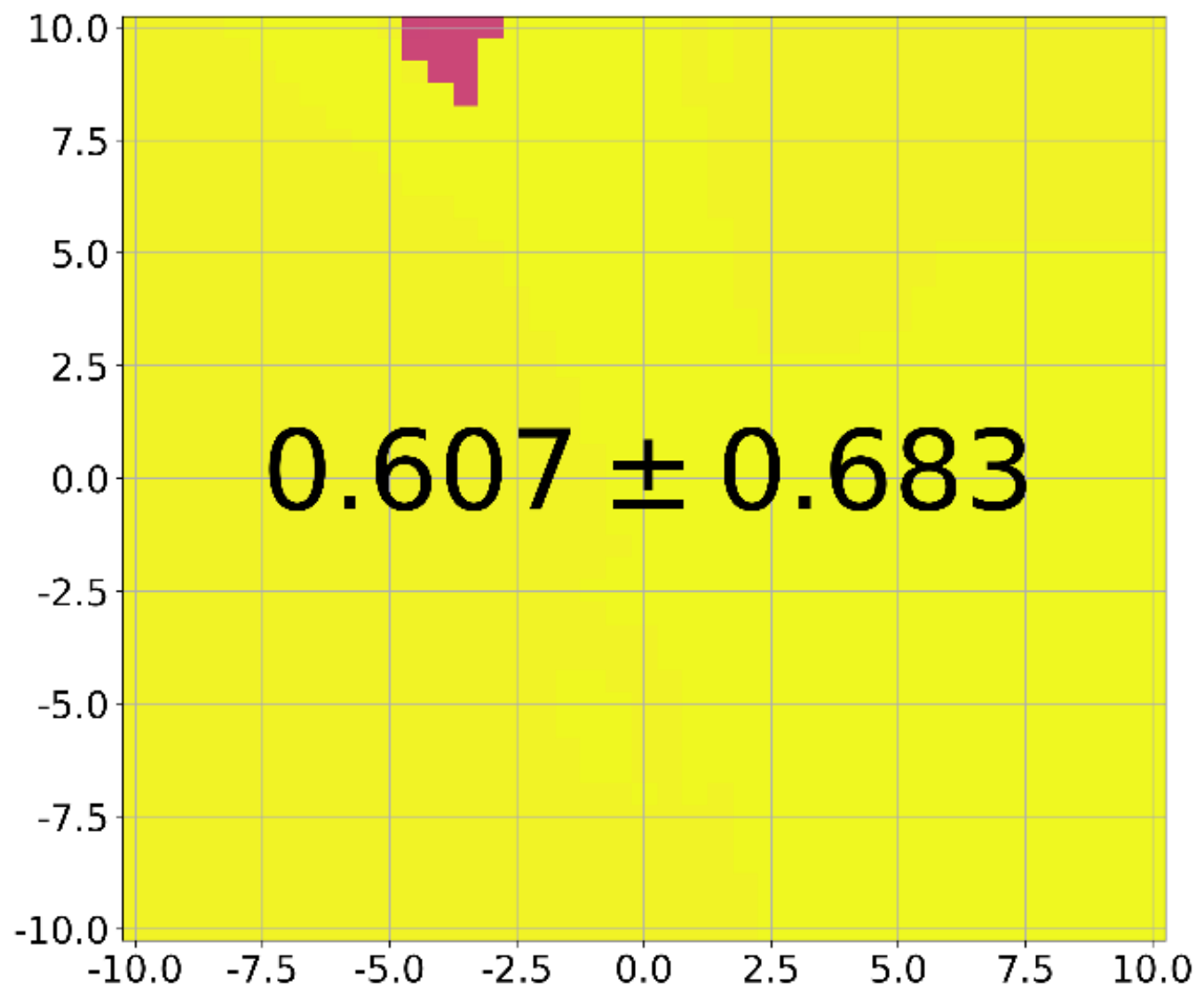}
\end{subfigure}
\begin{subfigure}{0.170\textwidth}
\includegraphics[width=\textwidth]{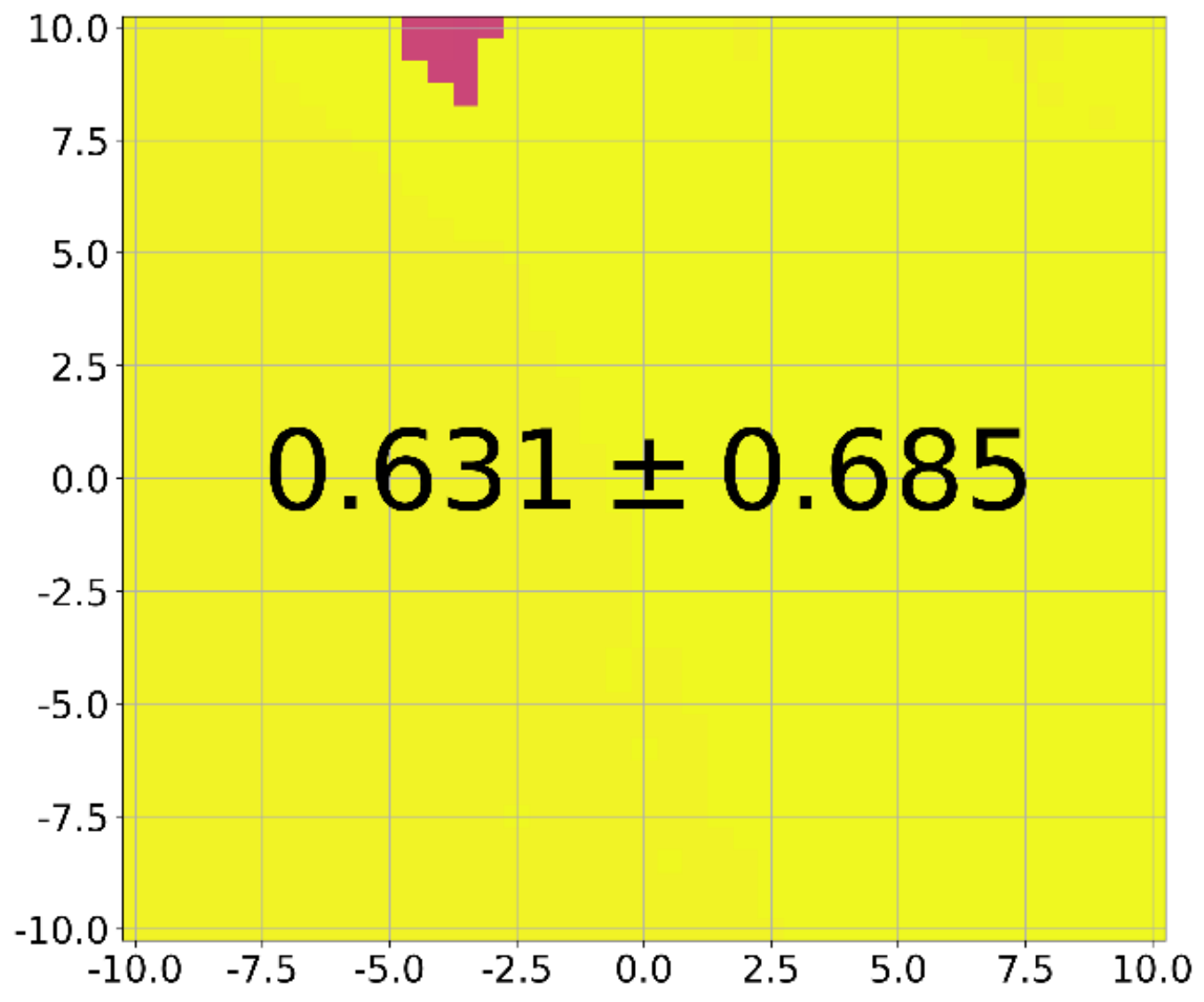}
\end{subfigure}
\begin{subfigure}{0.170\textwidth}
\includegraphics[width=\textwidth]{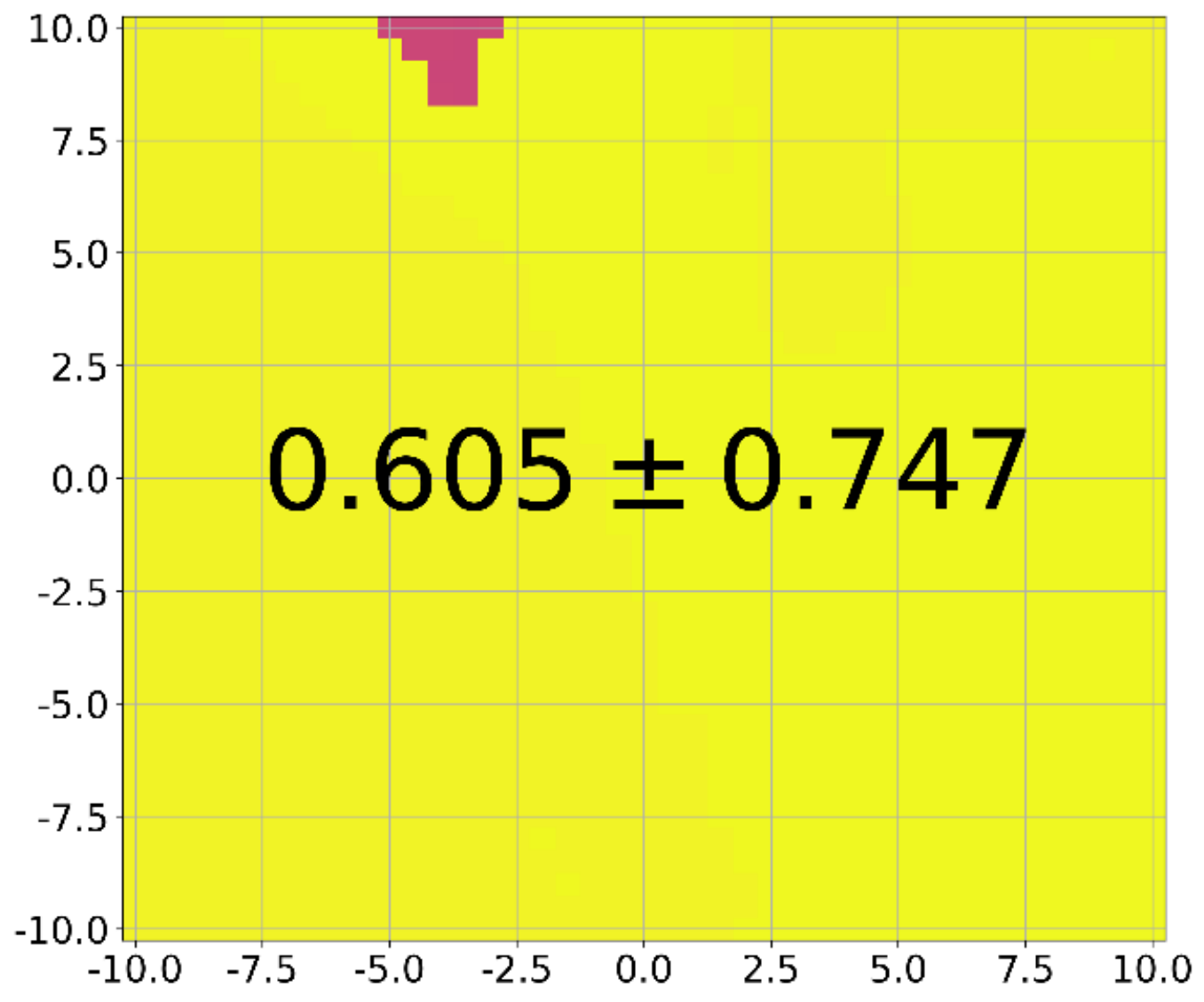}
\end{subfigure}

\begin{subfigure}{0.170\textwidth}
\includegraphics[width=\textwidth]{fig/envprop-3d/bs0__Tor_20201121a-v1__OneDimensionStateAndTwoActionPolicyNetwork}
\end{subfigure}
\begin{subfigure}{0.170\textwidth}
\includegraphics[width=\textwidth]{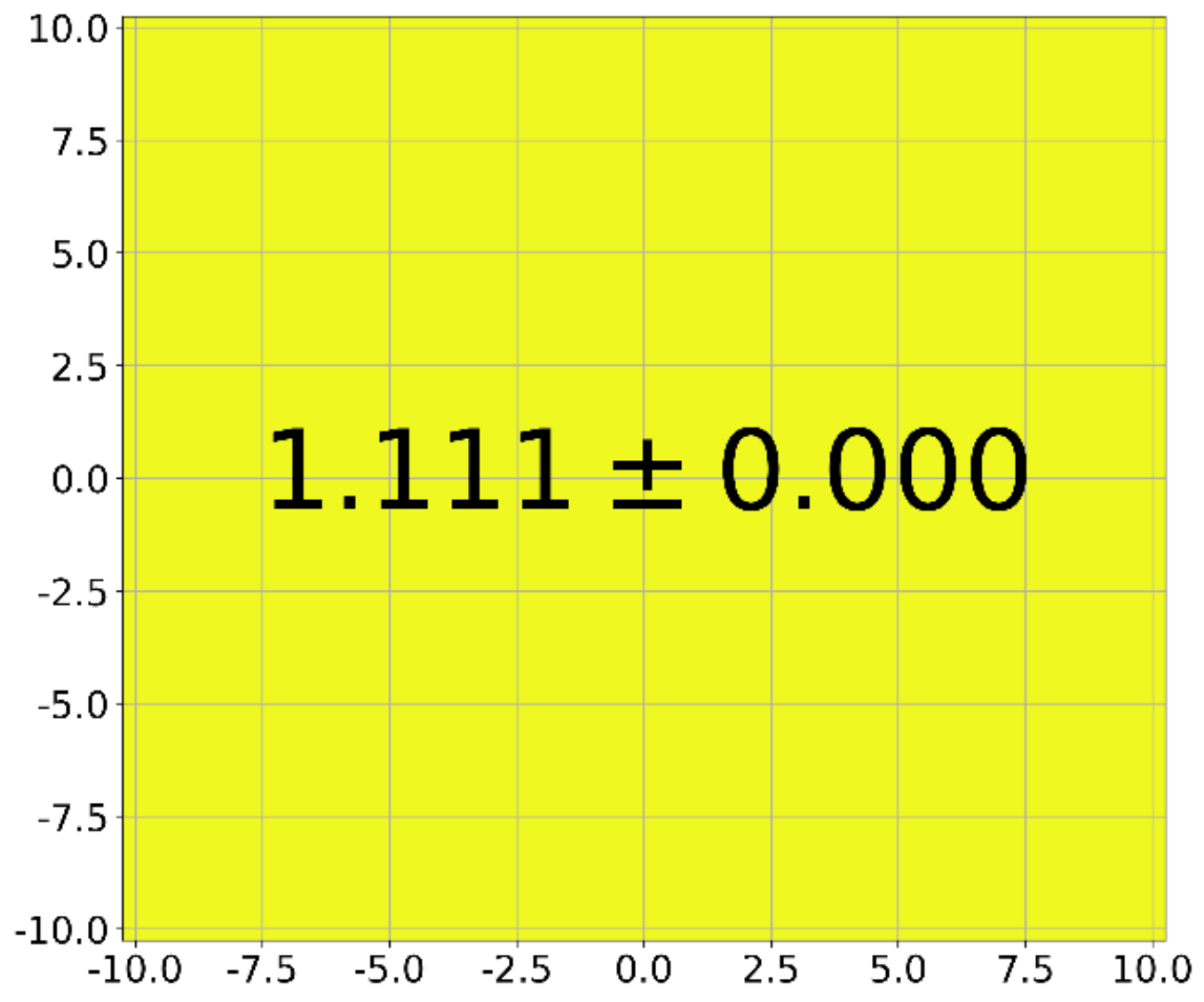}
\end{subfigure}
\begin{subfigure}{0.170\textwidth}
\includegraphics[width=\textwidth]{fig/biasopt/bias__obj=bias__conditioner_mode=fisher_probtransient__nseed1__OneDimensionStateAndTwoActionPolicyNetwork__Tor_20201121a-v1}
\end{subfigure}
\begin{subfigure}{0.170\textwidth}
\includegraphics[width=\textwidth]{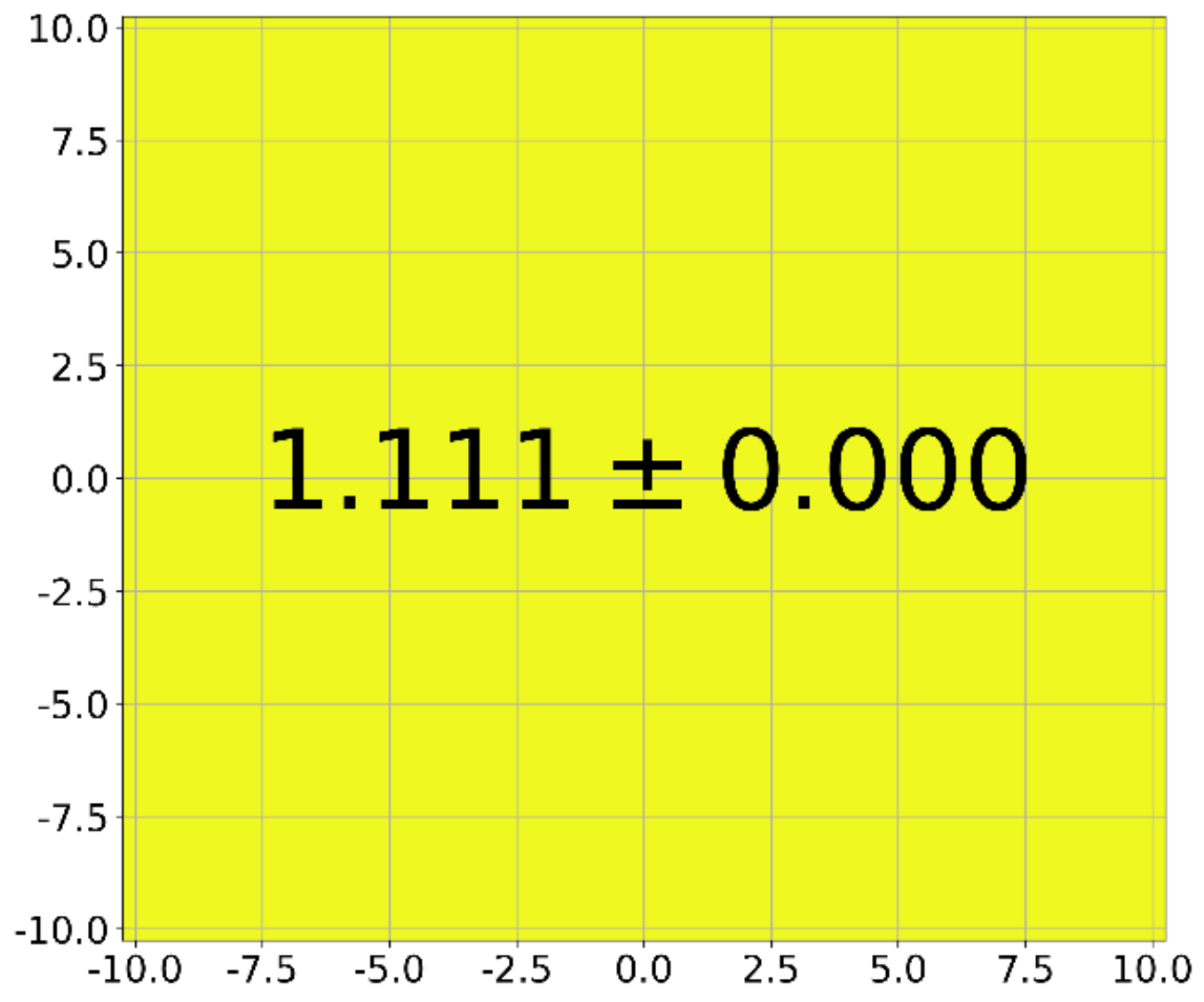}
\end{subfigure}
\begin{subfigure}{0.170\textwidth}
\includegraphics[width=\textwidth]{fig/biasopt/bias__obj=bias__conditioner_mode=fisher_transient_withsteadymul_upto_t1__nseed1__OneDimensionStateAndTwoActionPolicyNetwork__Tor_20201121a-v1}
\end{subfigure}
\begin{subfigure}{0.170\textwidth}
\includegraphics[width=\textwidth]{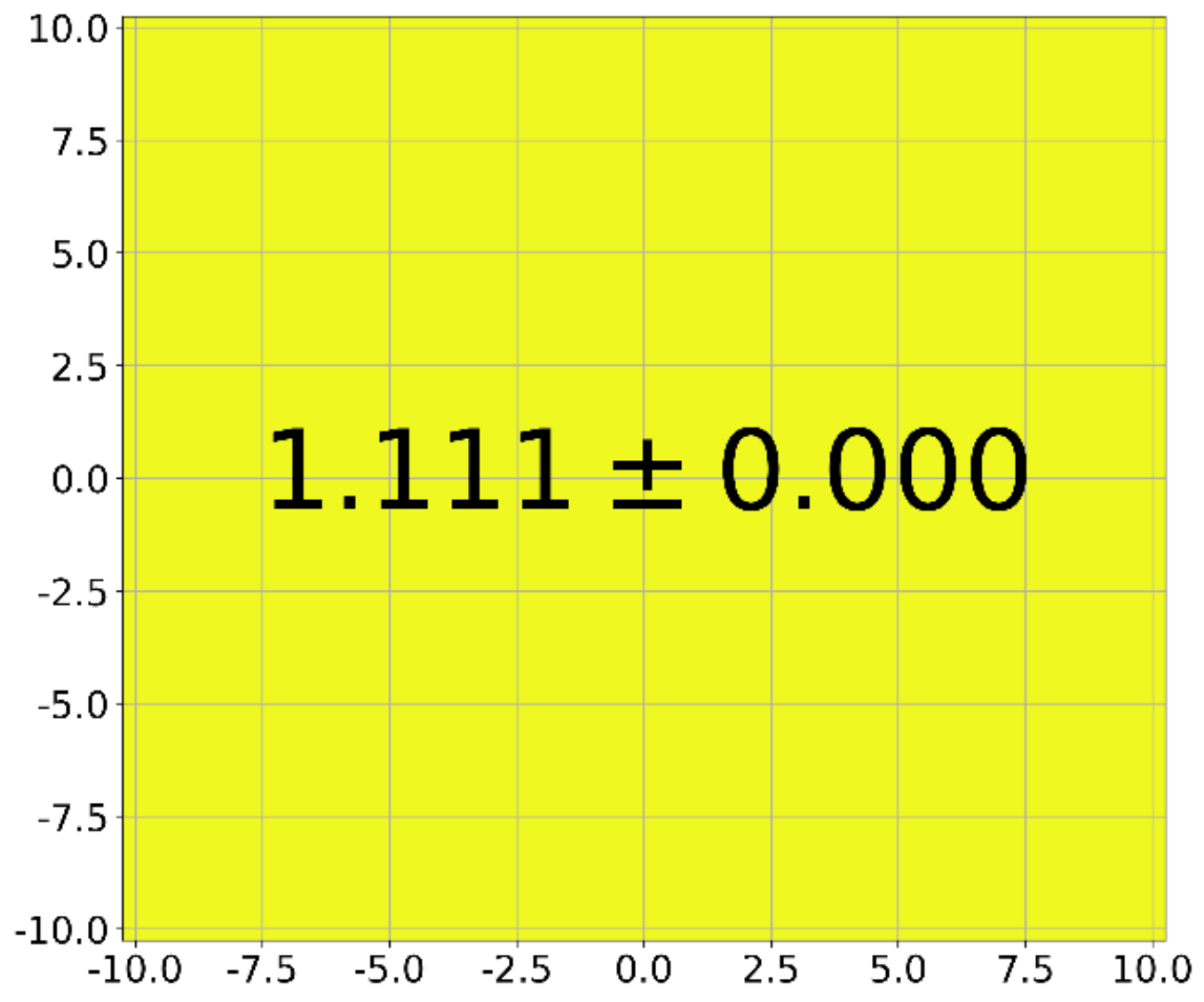}
\end{subfigure}
\begin{subfigure}{0.170\textwidth}
\includegraphics[width=\textwidth]{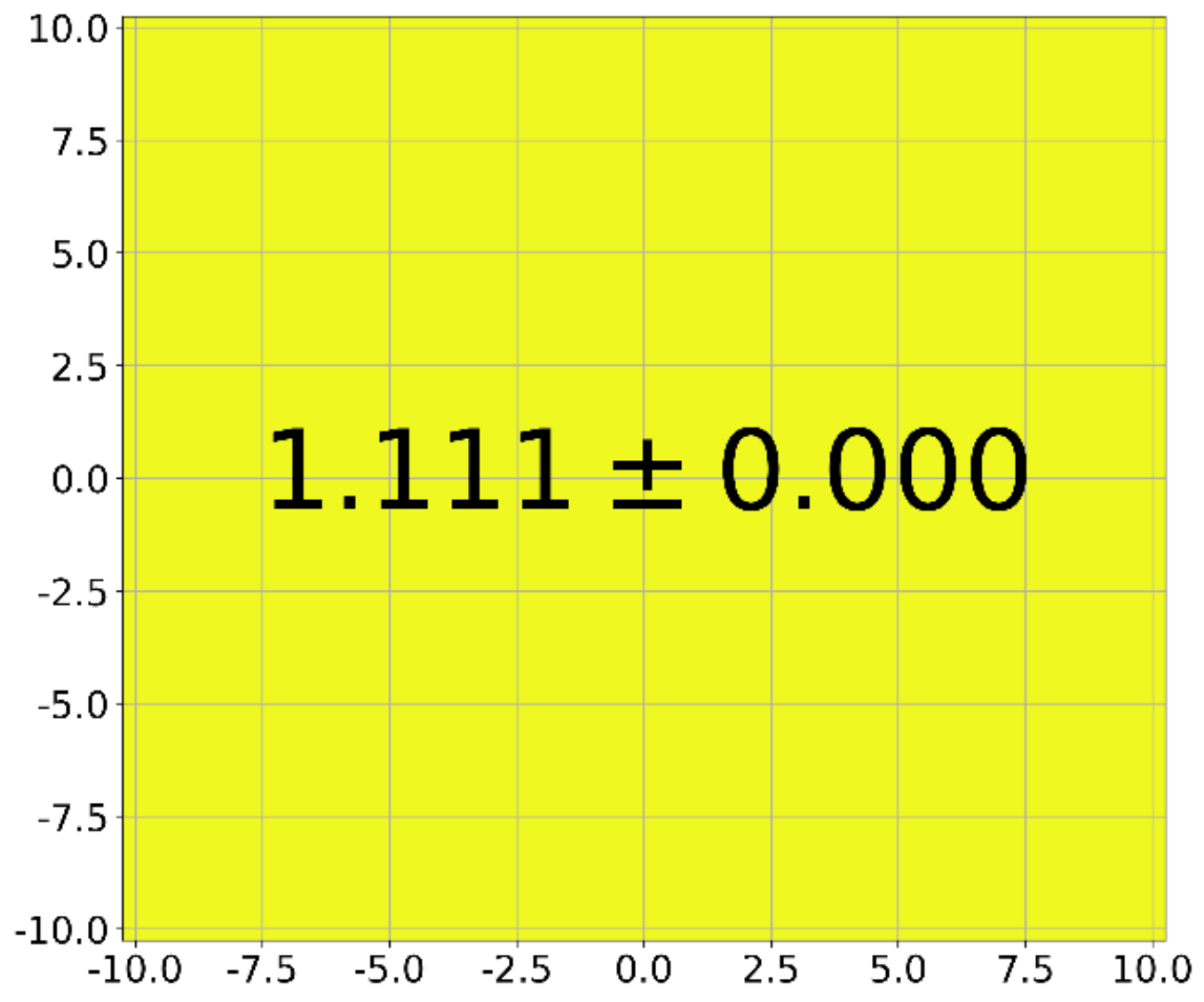}
\end{subfigure}
\begin{subfigure}{0.170\textwidth}
\includegraphics[width=\textwidth]{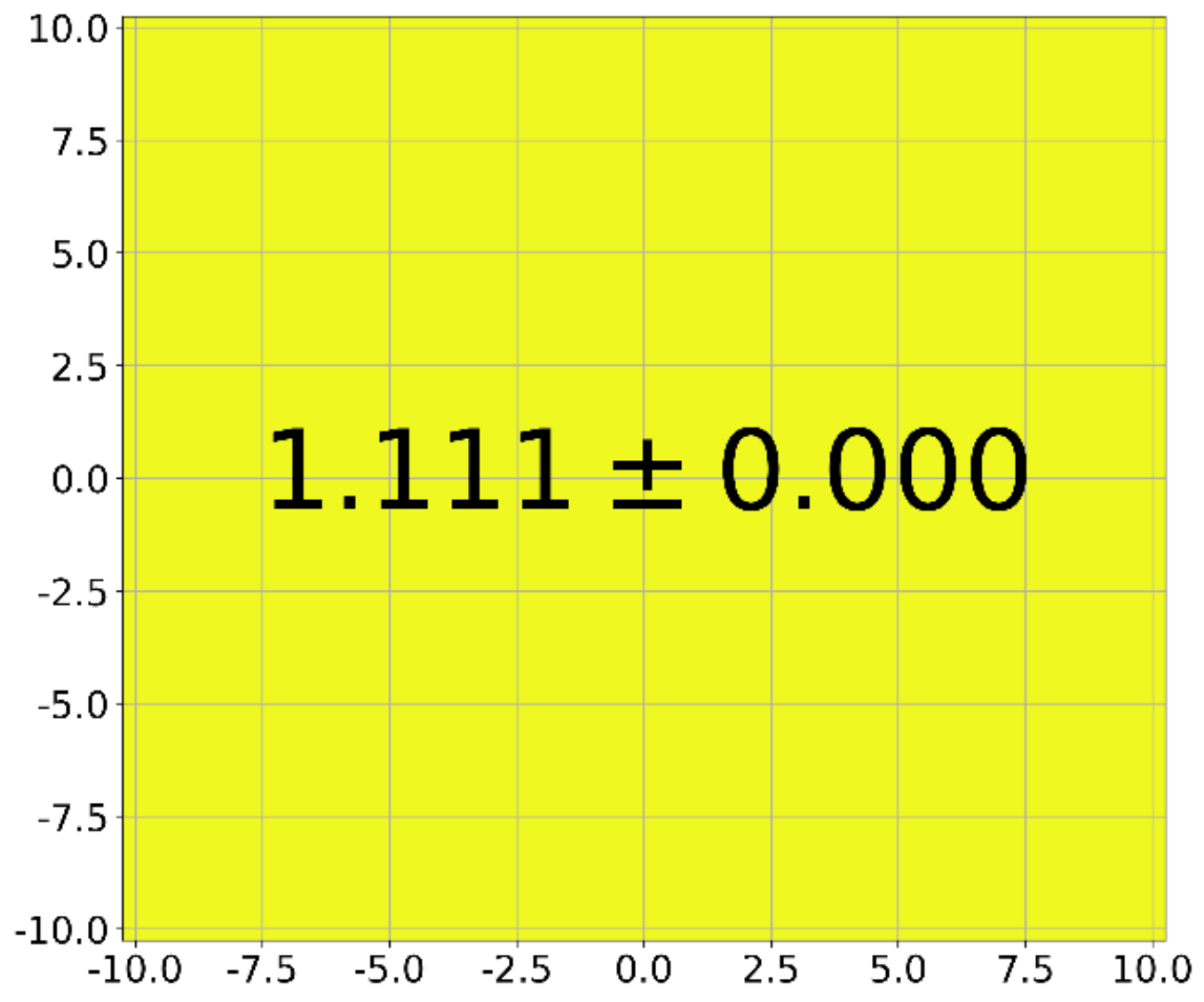}
\end{subfigure}

\begin{subfigure}{0.170\textwidth}
\includegraphics[width=\textwidth]{fig/envprop-3d/bs0__Hordijk_example-v4__OneDimensionStateAndTwoActionPolicyNetwork}
\end{subfigure}
\begin{subfigure}{0.170\textwidth}
\includegraphics[width=\textwidth]{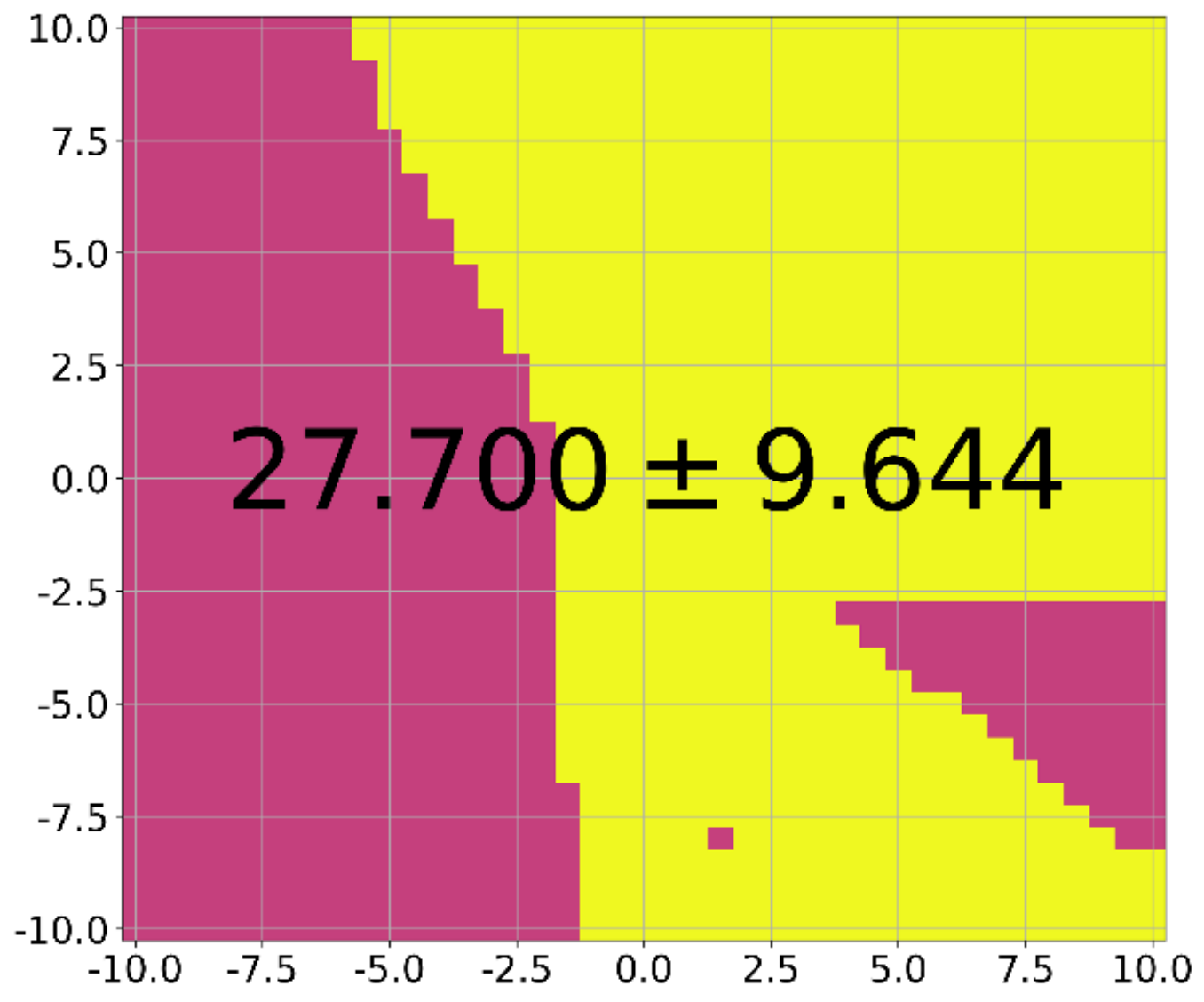}
\end{subfigure}
\begin{subfigure}{0.170\textwidth}
\includegraphics[width=\textwidth]{fig/biasopt/bias__obj=bias__conditioner_mode=fisher_probtransient__nseed1__OneDimensionStateAndTwoActionPolicyNetwork__Hordijk_example-v4}
\end{subfigure}
\begin{subfigure}{0.170\textwidth}
\includegraphics[width=\textwidth]{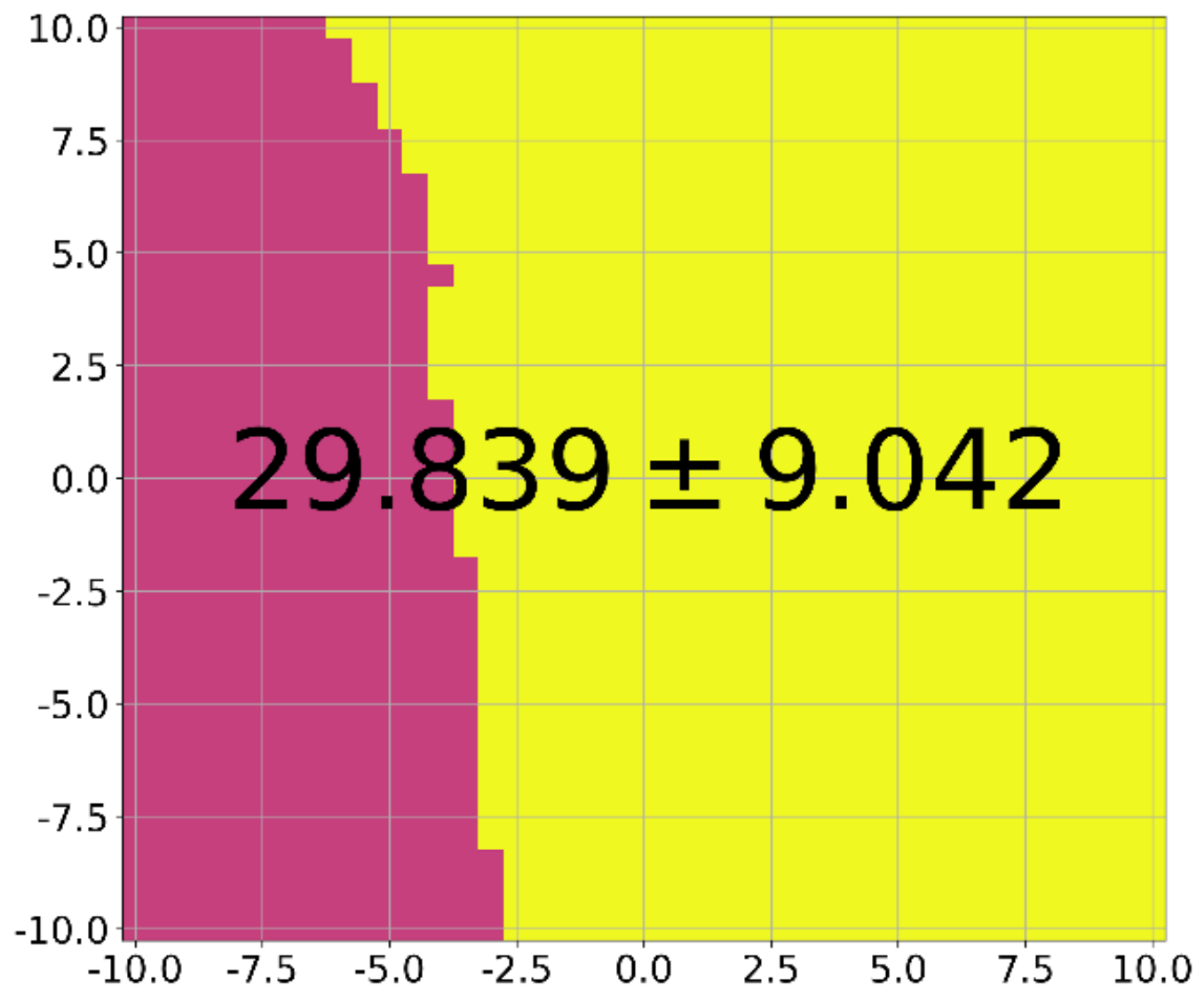}
\end{subfigure}
\begin{subfigure}{0.170\textwidth}
\includegraphics[width=\textwidth]{fig/biasopt/bias__obj=bias__conditioner_mode=fisher_transient_withsteadymul_upto_t1__nseed1__OneDimensionStateAndTwoActionPolicyNetwork__Hordijk_example-v4}
\end{subfigure}
\begin{subfigure}{0.170\textwidth}
\includegraphics[width=\textwidth]{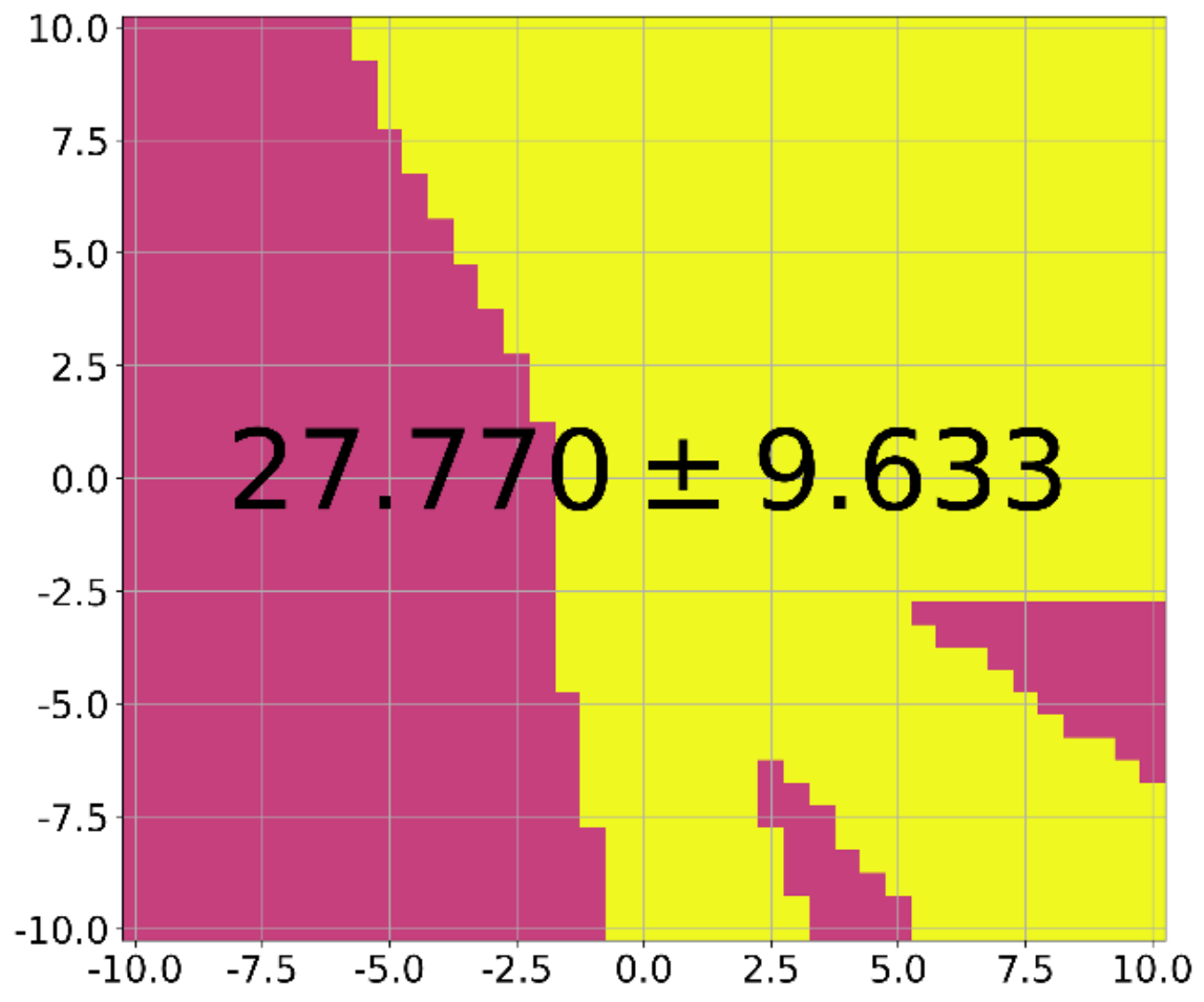}
\end{subfigure}
\begin{subfigure}{0.170\textwidth}
\includegraphics[width=\textwidth]{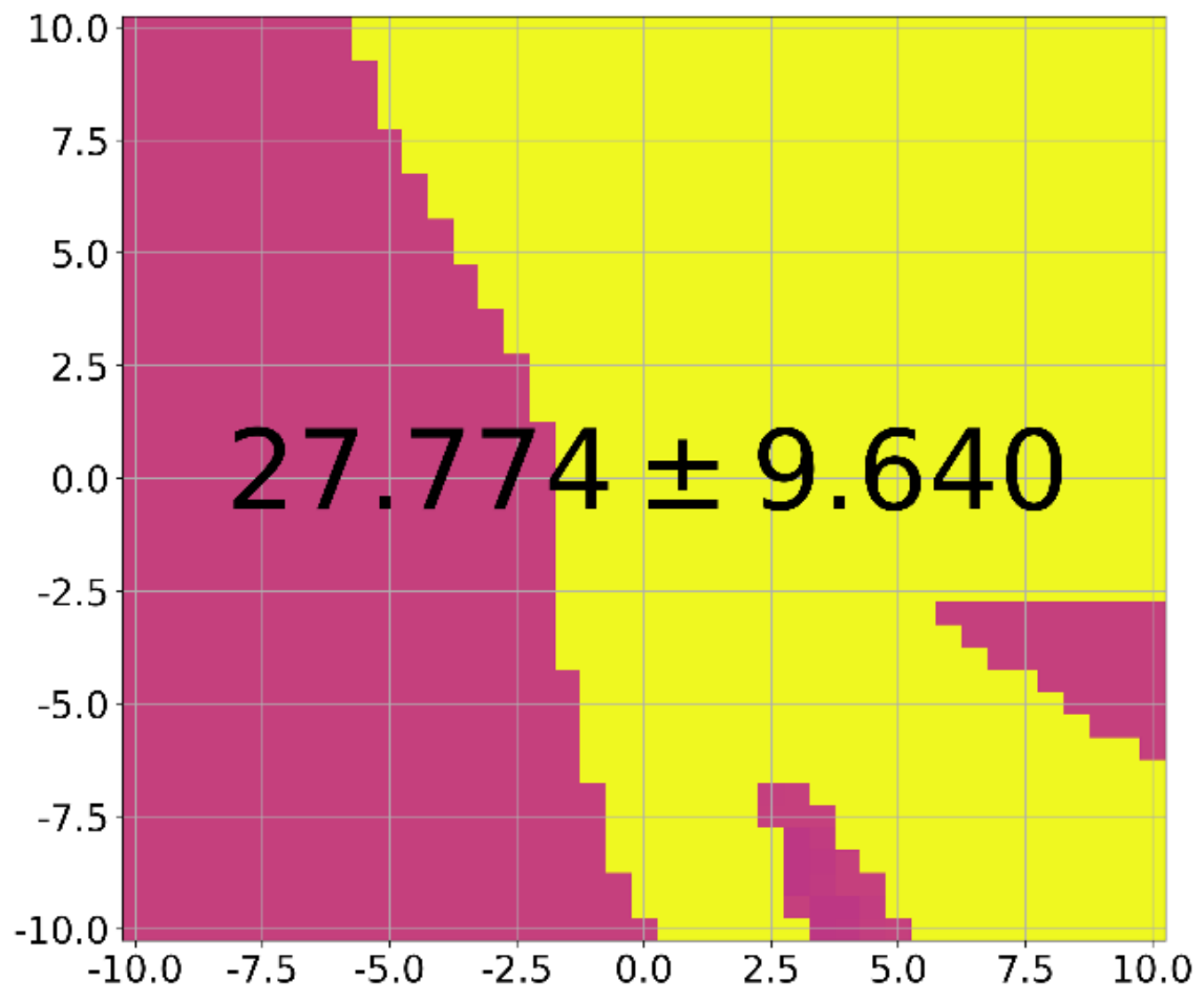}
\end{subfigure}
\begin{subfigure}{0.170\textwidth}
\includegraphics[width=\textwidth]{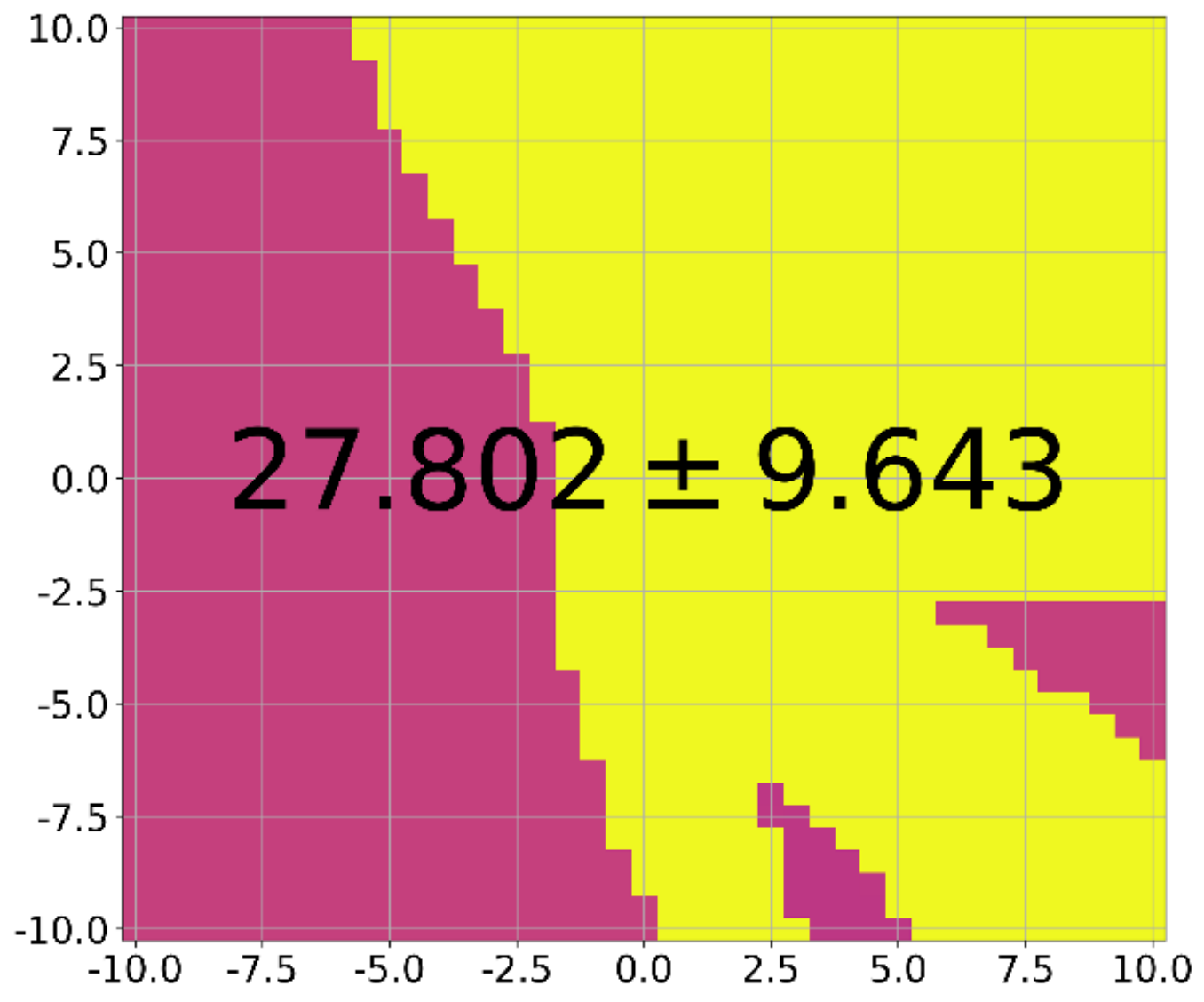}
\end{subfigure}

\begin{subfigure}{0.170\textwidth}
\includegraphics[width=\textwidth]{fig/envprop-3d/bs0__NChain_mod-v1__OneDimensionStateAndTwoActionPolicyNetwork}
\subcaption*{Bias landscape}
\end{subfigure}
\begin{subfigure}{0.170\textwidth}
\includegraphics[width=\textwidth]{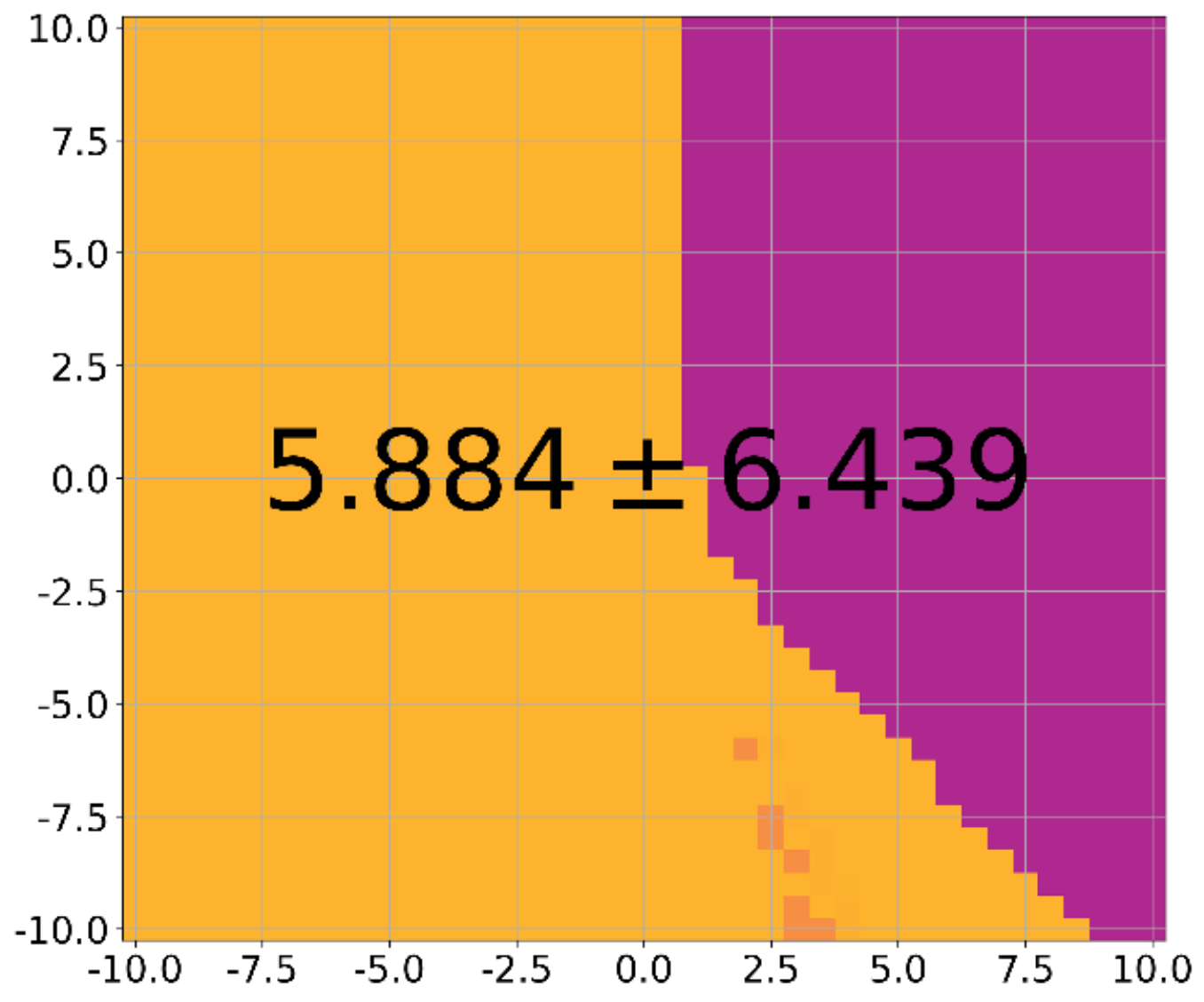}
\subcaption*{Devmat}
\end{subfigure}
\begin{subfigure}{0.170\textwidth}
\includegraphics[width=\textwidth]{fig/biasopt/bias__obj=bias__conditioner_mode=fisher_probtransient__nseed1__OneDimensionStateAndTwoActionPolicyNetwork__NChain_mod-v1}
\subcaption*{Analytic \eqref{equ:fisher_b}}
\end{subfigure}
\begin{subfigure}{0.170\textwidth}
\includegraphics[width=\textwidth]{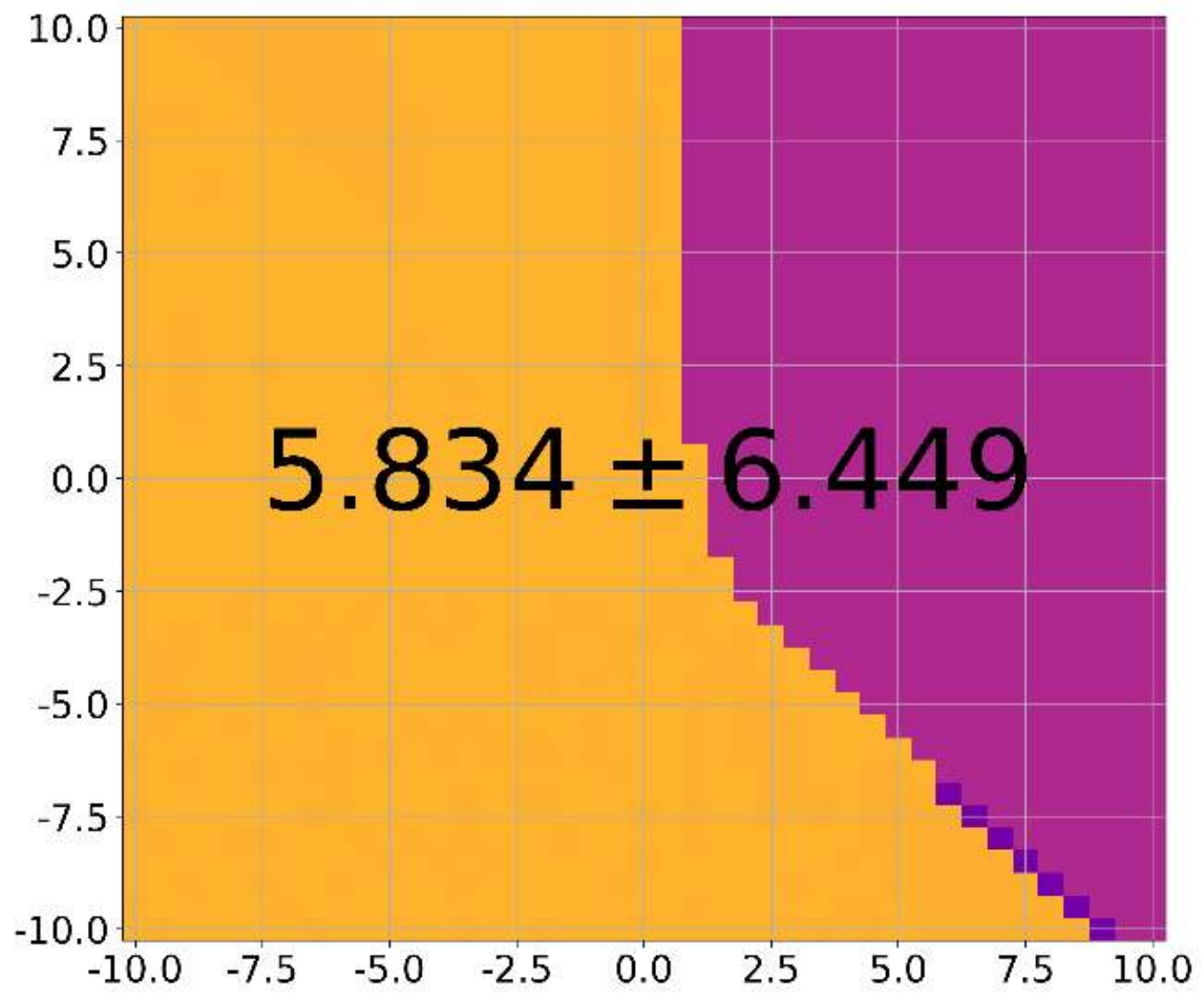}
\subcaption*{$\tabsminhat = 1$}
\end{subfigure}
\begin{subfigure}{0.170\textwidth}
\includegraphics[width=\textwidth]{fig/biasopt/bias__obj=bias__conditioner_mode=fisher_transient_withsteadymul_upto_t1__nseed1__OneDimensionStateAndTwoActionPolicyNetwork__NChain_mod-v1}
\subcaption*{$\tabsminhat = 2$}
\end{subfigure}
\begin{subfigure}{0.170\textwidth}
\includegraphics[width=\textwidth]{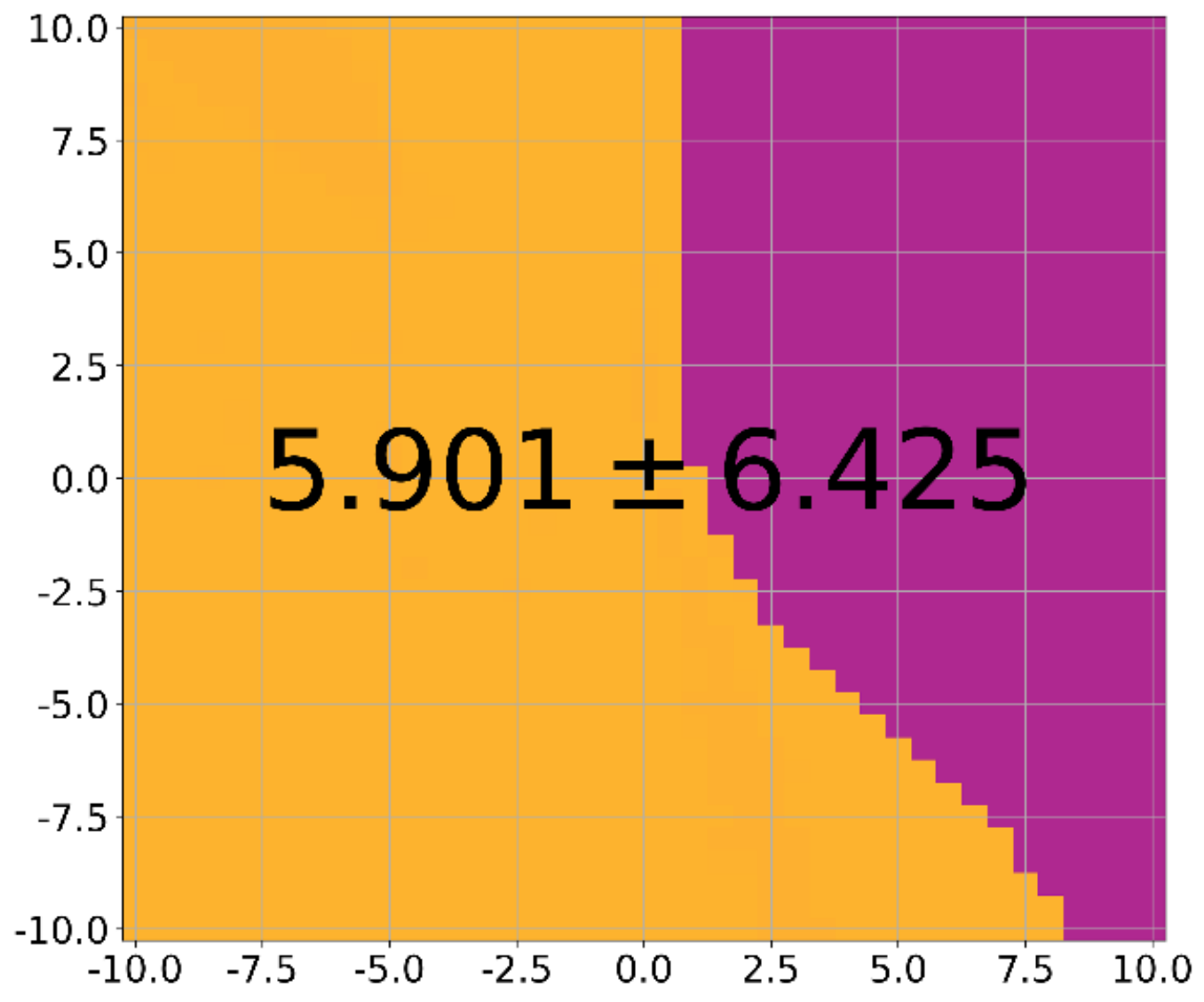}
\subcaption*{$\tabsminhat = 3$}
\end{subfigure}
\begin{subfigure}{0.170\textwidth}
\includegraphics[width=\textwidth]{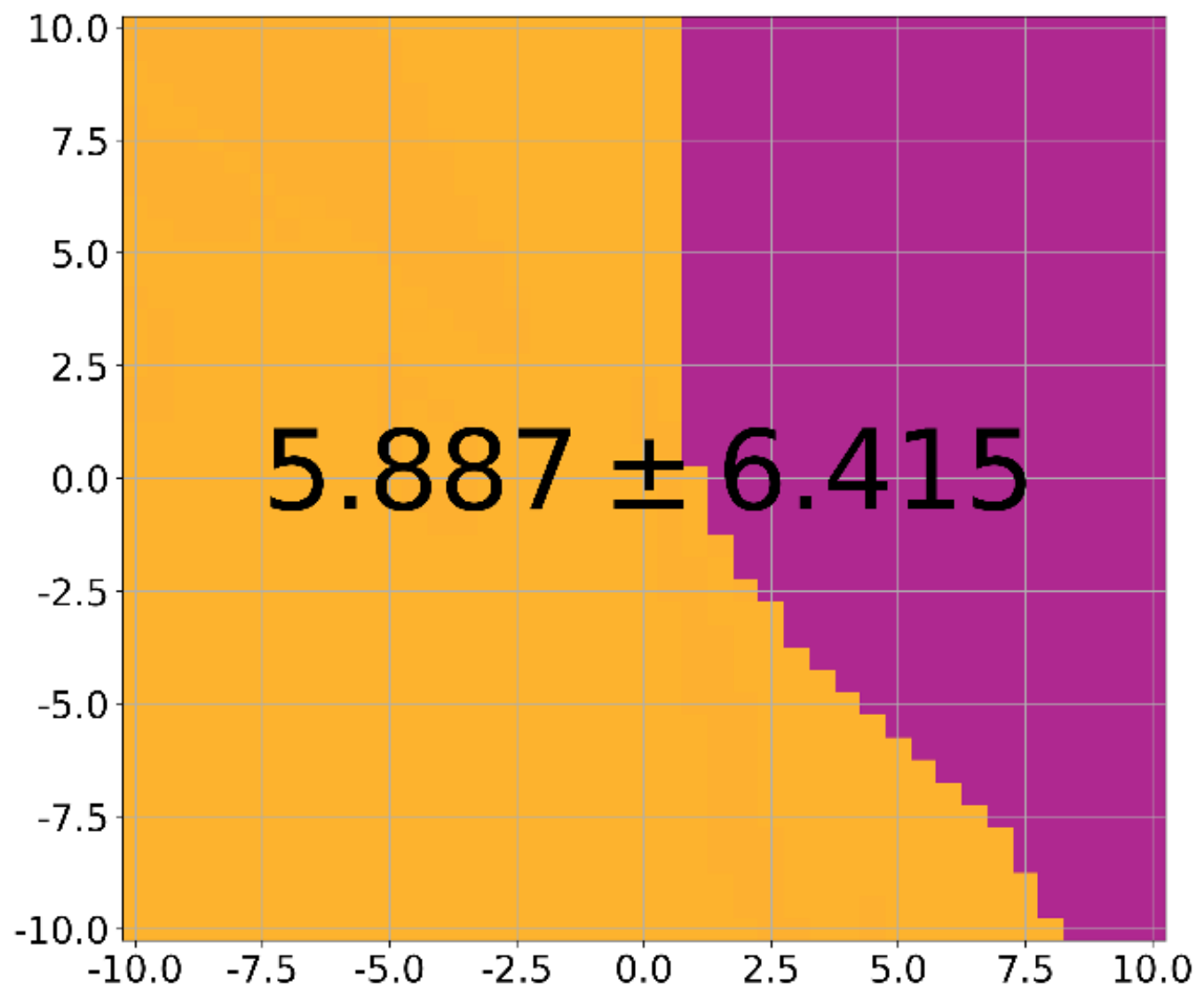}
\subcaption*{$\tabsminhat = 4$}
\end{subfigure}
\begin{subfigure}{0.170\textwidth}
\includegraphics[width=\textwidth]{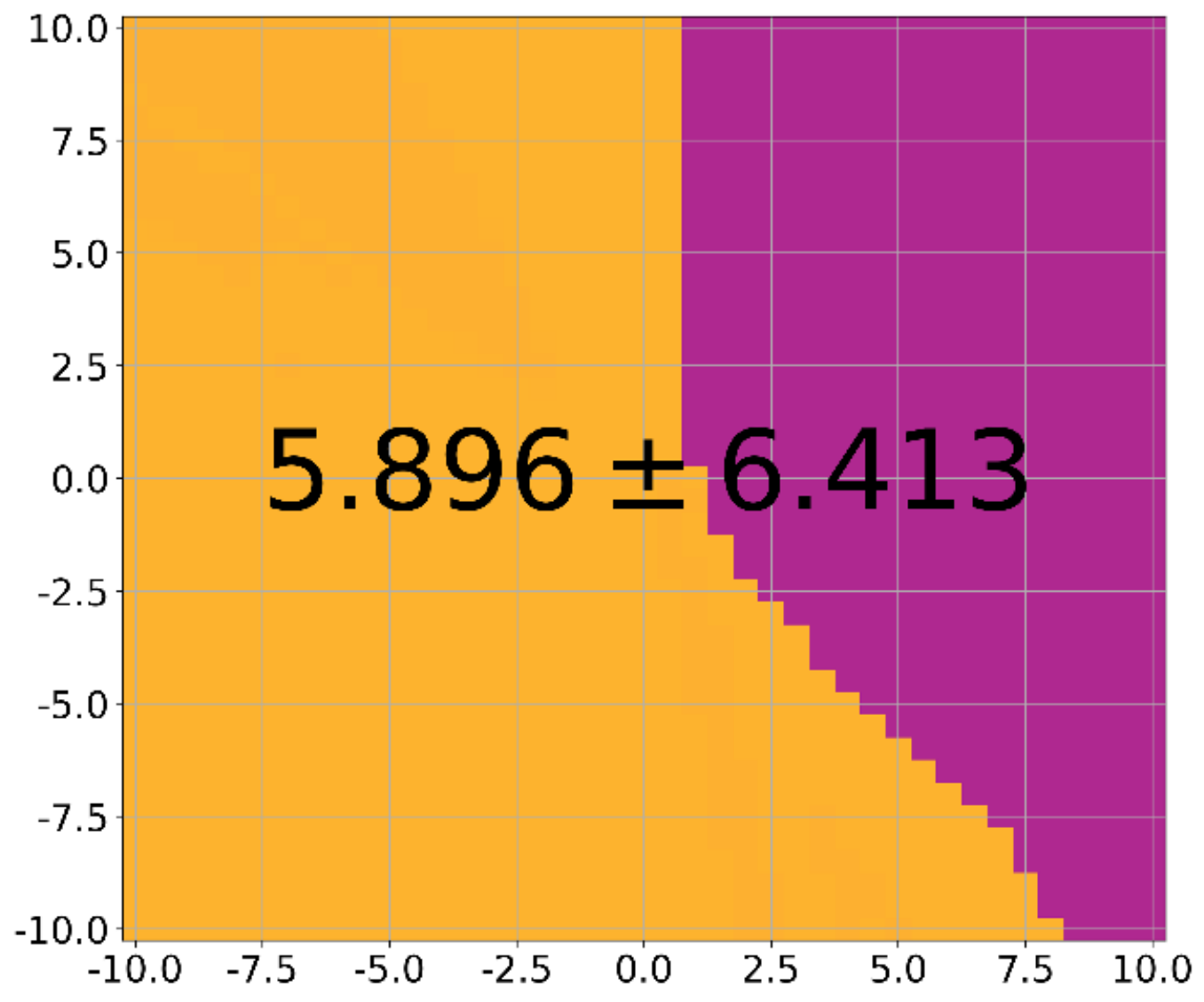}
\subcaption*{$\tabsminhat = 5$}
\end{subfigure}

\caption{A comparison of Fisher candidates for exact bias-only optimization
on Env-A1, A2, A3, B1, B2, and B3 (top to bottom rows).
The Devmat Fisher uses the entry-wise absolute deviation matrix, that is,
$\sum_{s \in \setname{S}} | h_\pi(s|s_0) | \famat(\vecb{\theta}, s)$.
The Analytic Fisher refers to Scheme~\ref{item:natgrad} (\secref{sec:biasopt}).
Both Devmat and Analytic candidates do not enable sampling-based approximation.
Contrarily, the other five candidates are sampling-enabler expressions
(but computed exactly here).
They are derived from \eqref{equ:fisher_b_sampling} with
varying minimum absorption time estimates $\tabsminhatb = 1, 2, \ldots, 5$
(each is invariant to policies and initial states).
Like in \figref{fig:biasoptcmp}, the color indicates the bias of
the final policy parameter from optimization initialized
at the corresponding 2D-coordinate.
It maps low bias to dark-blue, whereas high bias to yellow.
Each subplot has $1681$ final bias grid-values, whose
mean and standard deviation are indicated by the number in the middle.
Here, the highest final bias (yellow color) is desirable.
For experimental setup, see \secref{sec:xprmt_setup_nbwpg}.
}
\label{fig:biasfishercmp}
\end{figure}
\end{landscape}

\begin{landscape}
\begin{figure}
\centering

\begin{subfigure}{0.170\textwidth}
\includegraphics[width=\textwidth]{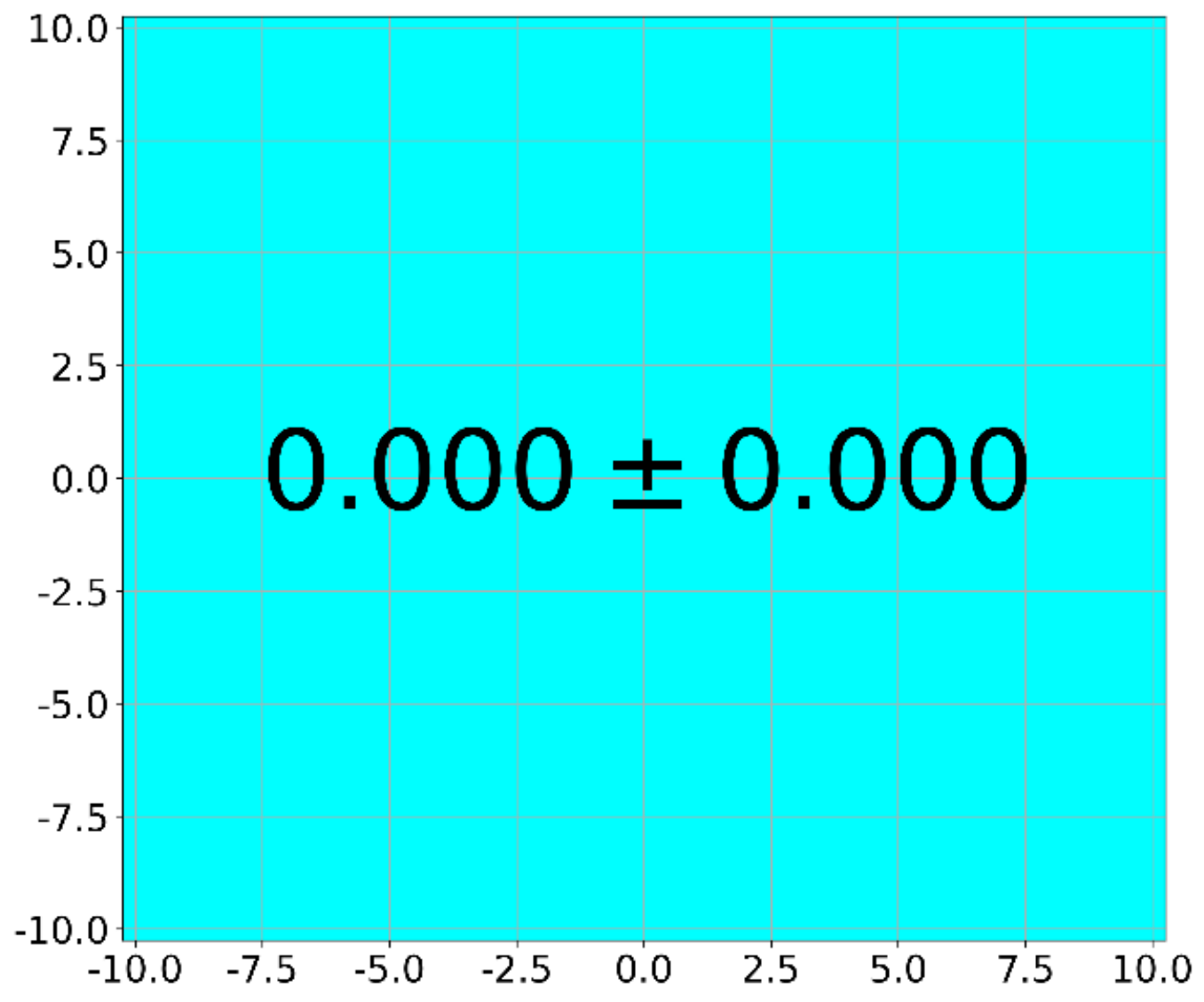}
\end{subfigure}
\begin{subfigure}{0.170\textwidth}
\includegraphics[width=\textwidth]{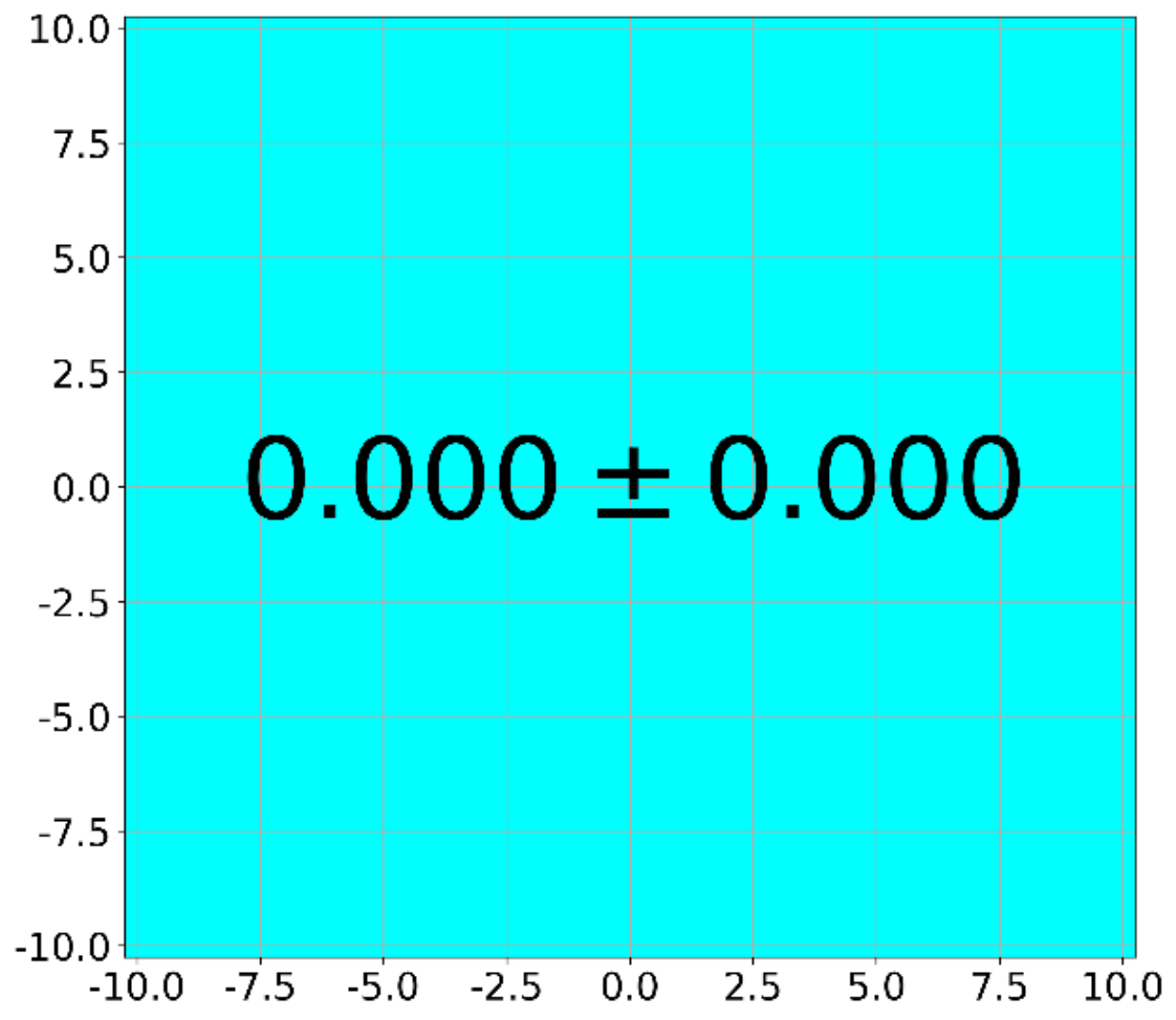}
\end{subfigure}
\begin{subfigure}{0.170\textwidth}
\includegraphics[width=\textwidth]{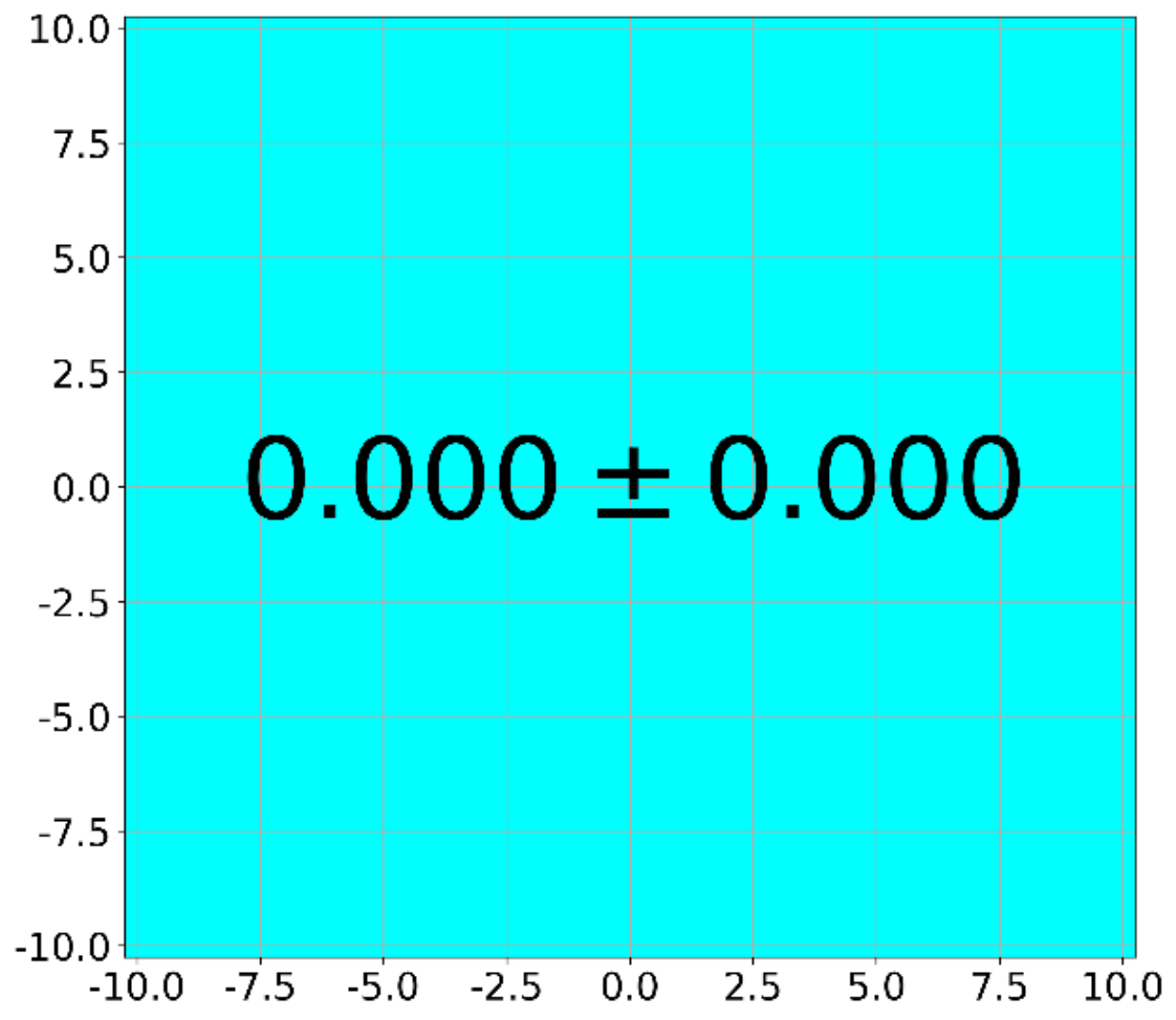}
\end{subfigure}
\begin{subfigure}{0.170\textwidth}
\includegraphics[width=\textwidth]{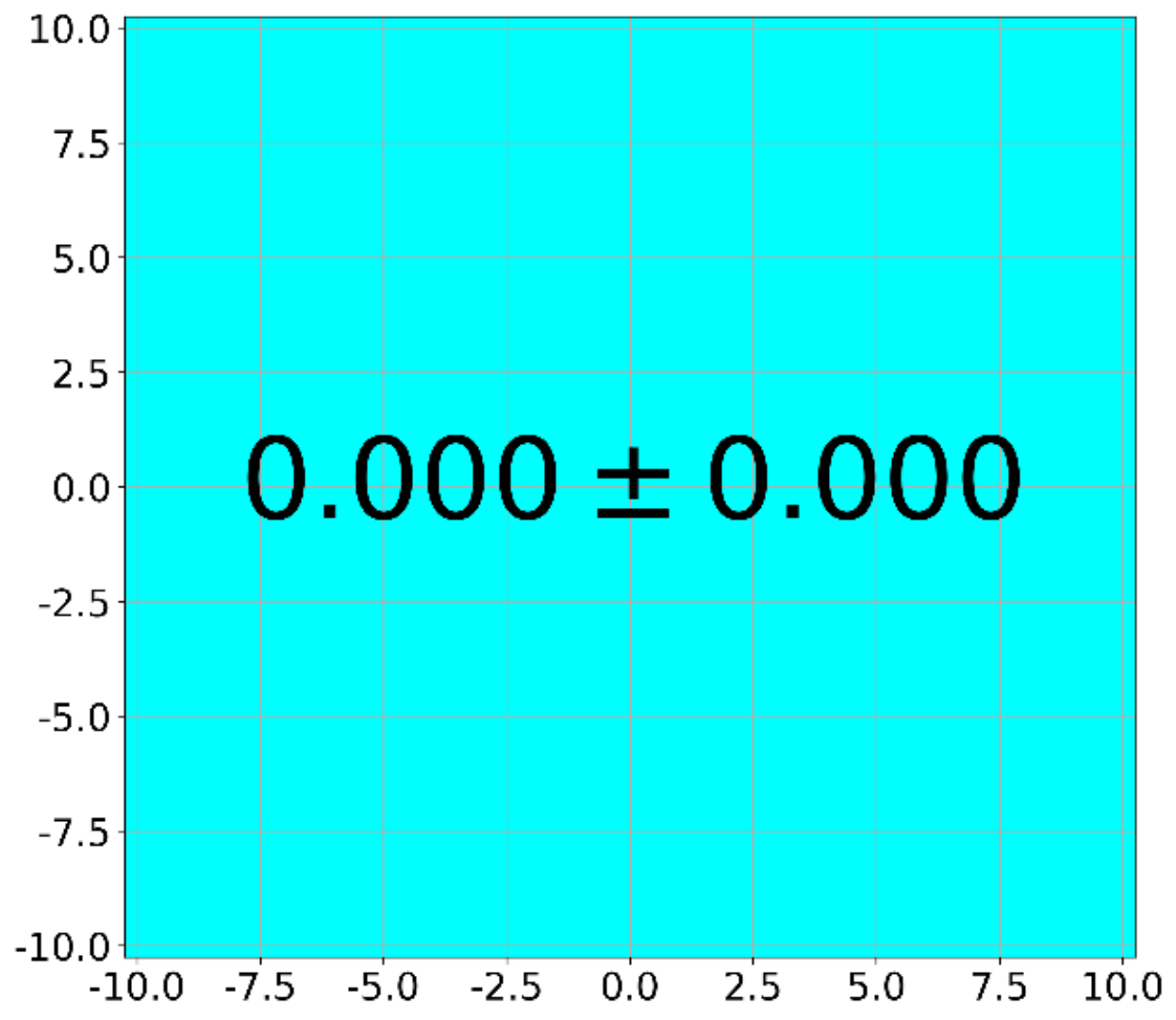}
\end{subfigure}
\begin{subfigure}{0.170\textwidth}
\includegraphics[width=\textwidth]{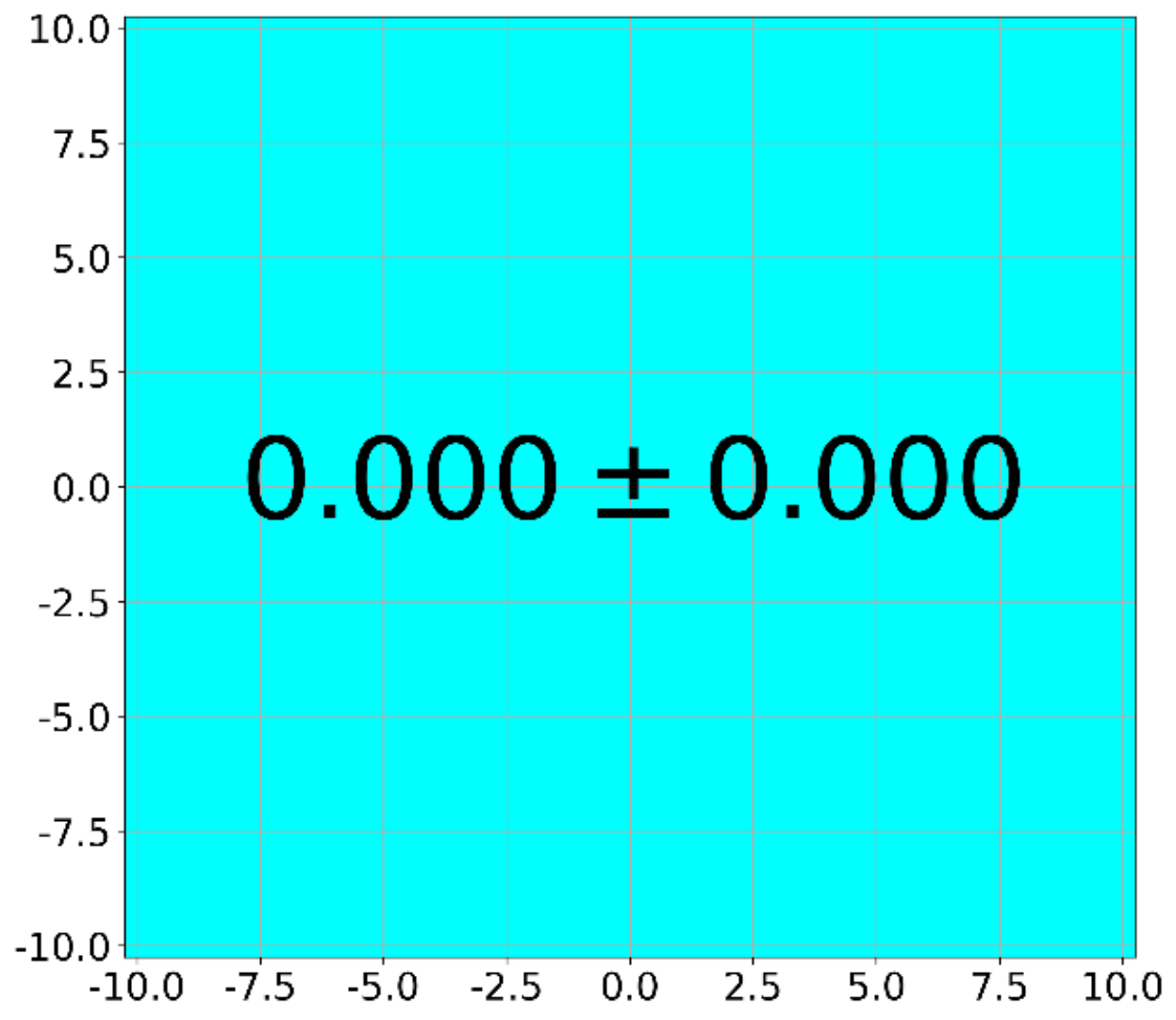}
\end{subfigure}
\begin{subfigure}{0.170\textwidth}
\includegraphics[width=\textwidth]{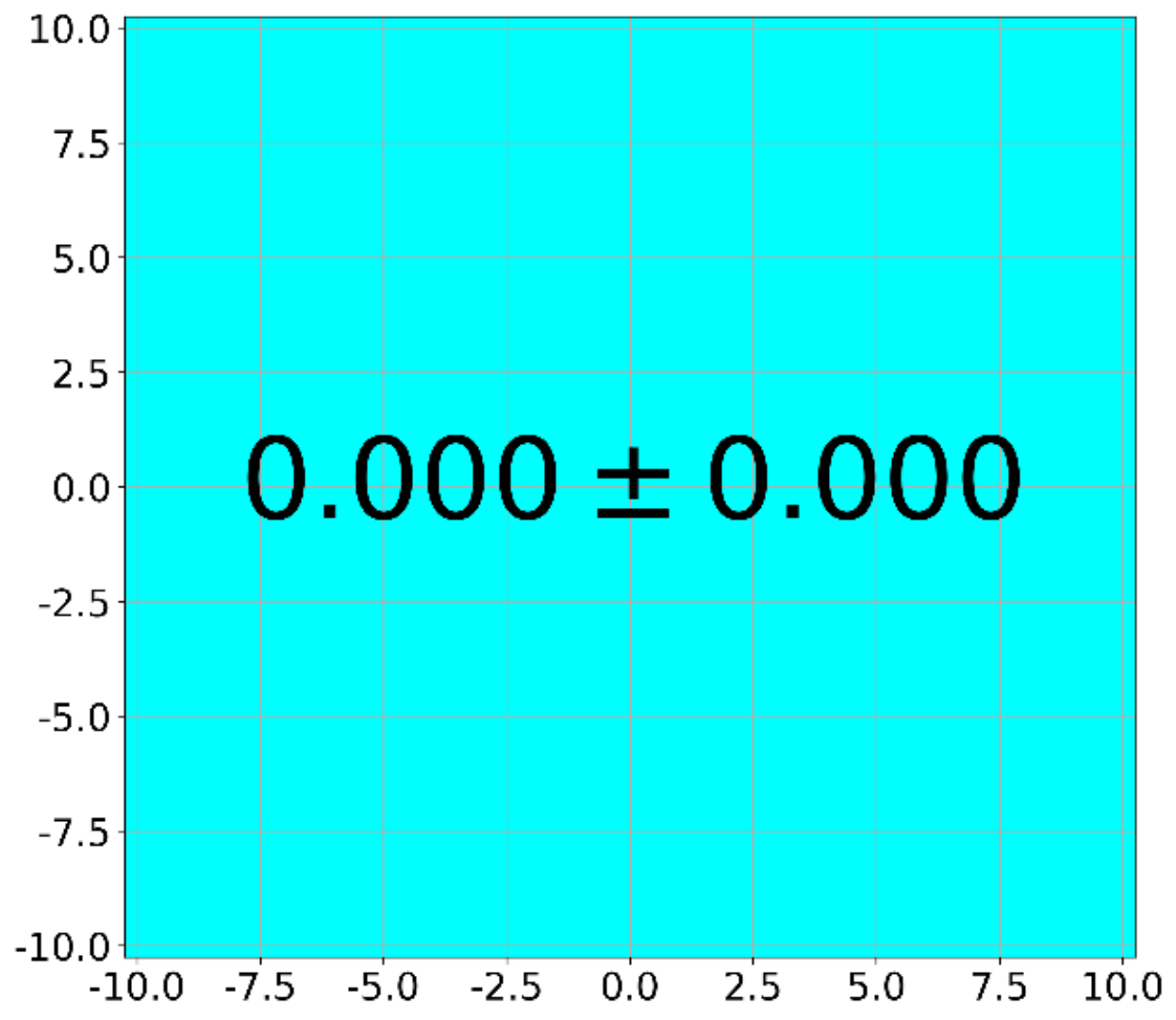}
\end{subfigure}
\begin{subfigure}{0.170\textwidth}
\includegraphics[width=\textwidth]{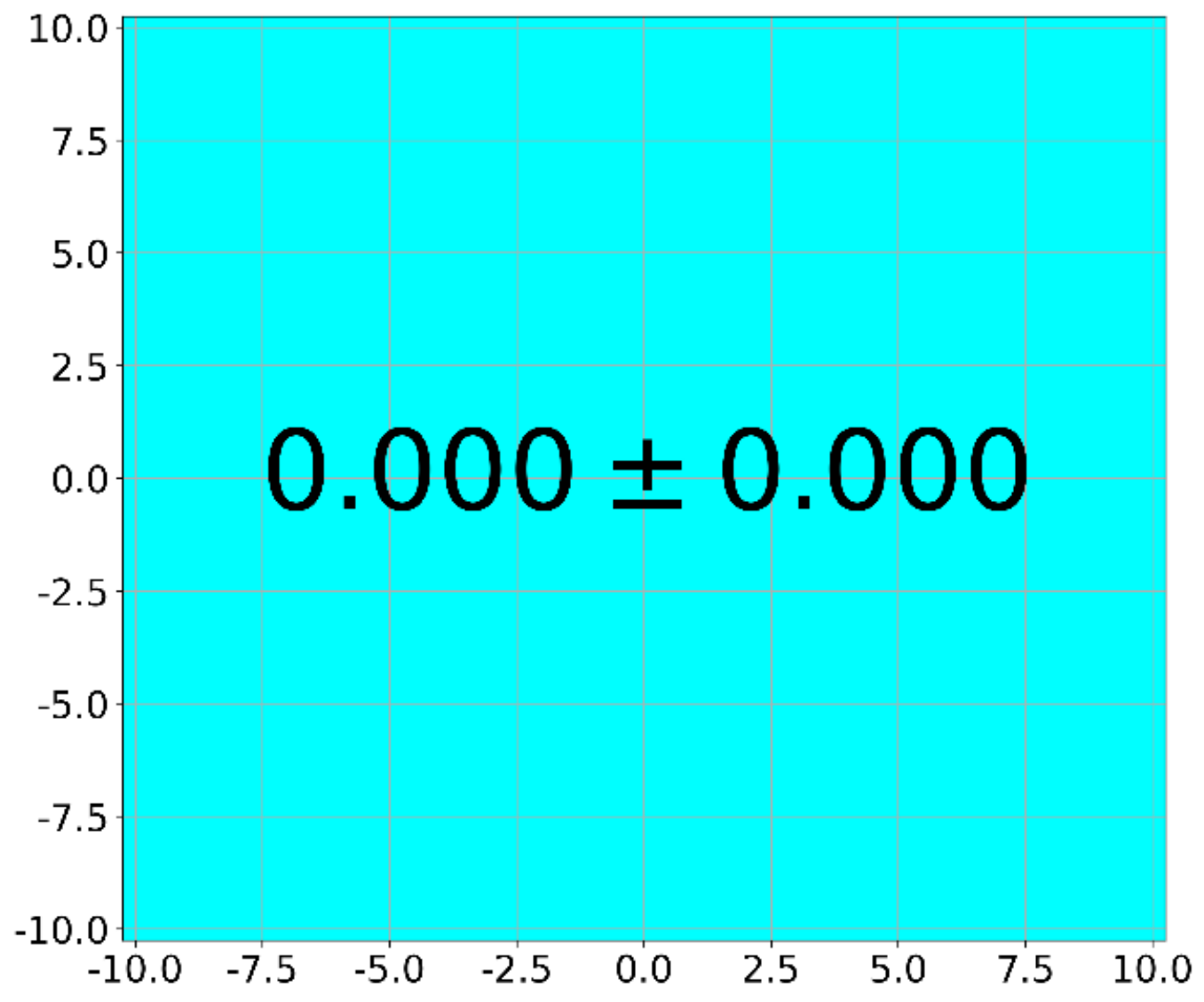}
\end{subfigure}

\begin{subfigure}{0.170\textwidth}
\includegraphics[width=\textwidth]{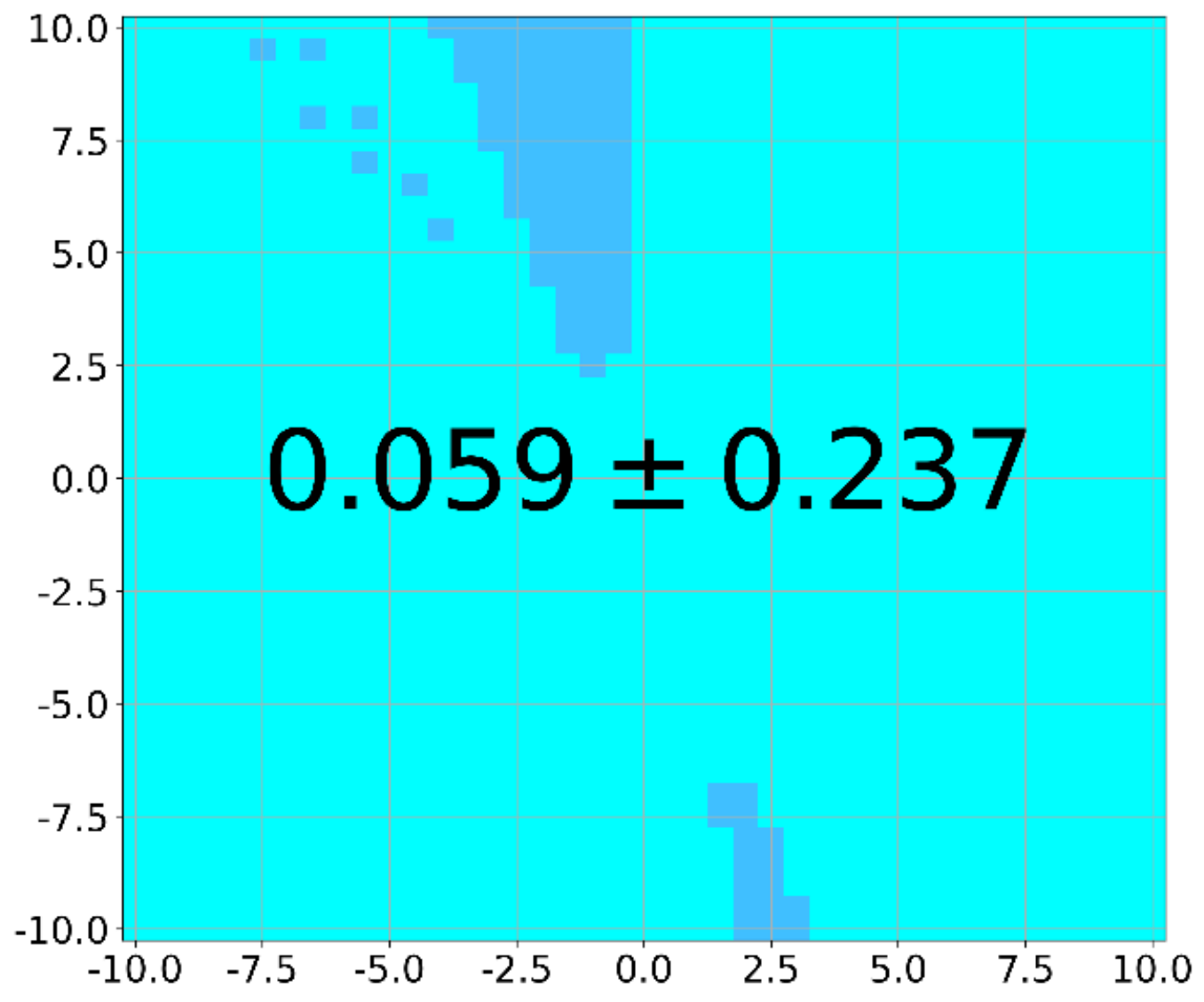}
\end{subfigure}
\begin{subfigure}{0.170\textwidth}
\includegraphics[width=\textwidth]{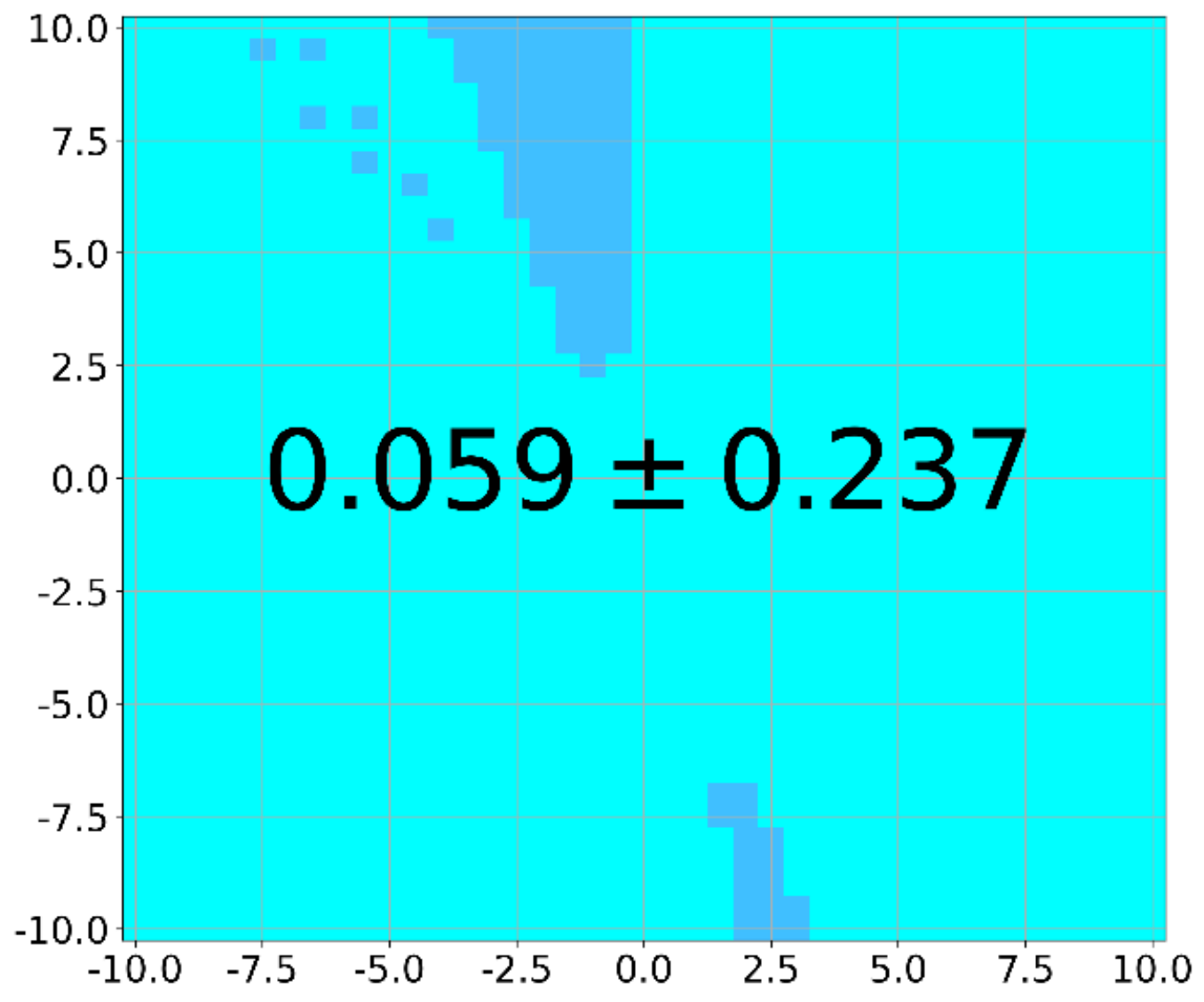}
\end{subfigure}
\begin{subfigure}{0.170\textwidth}
\includegraphics[width=\textwidth]{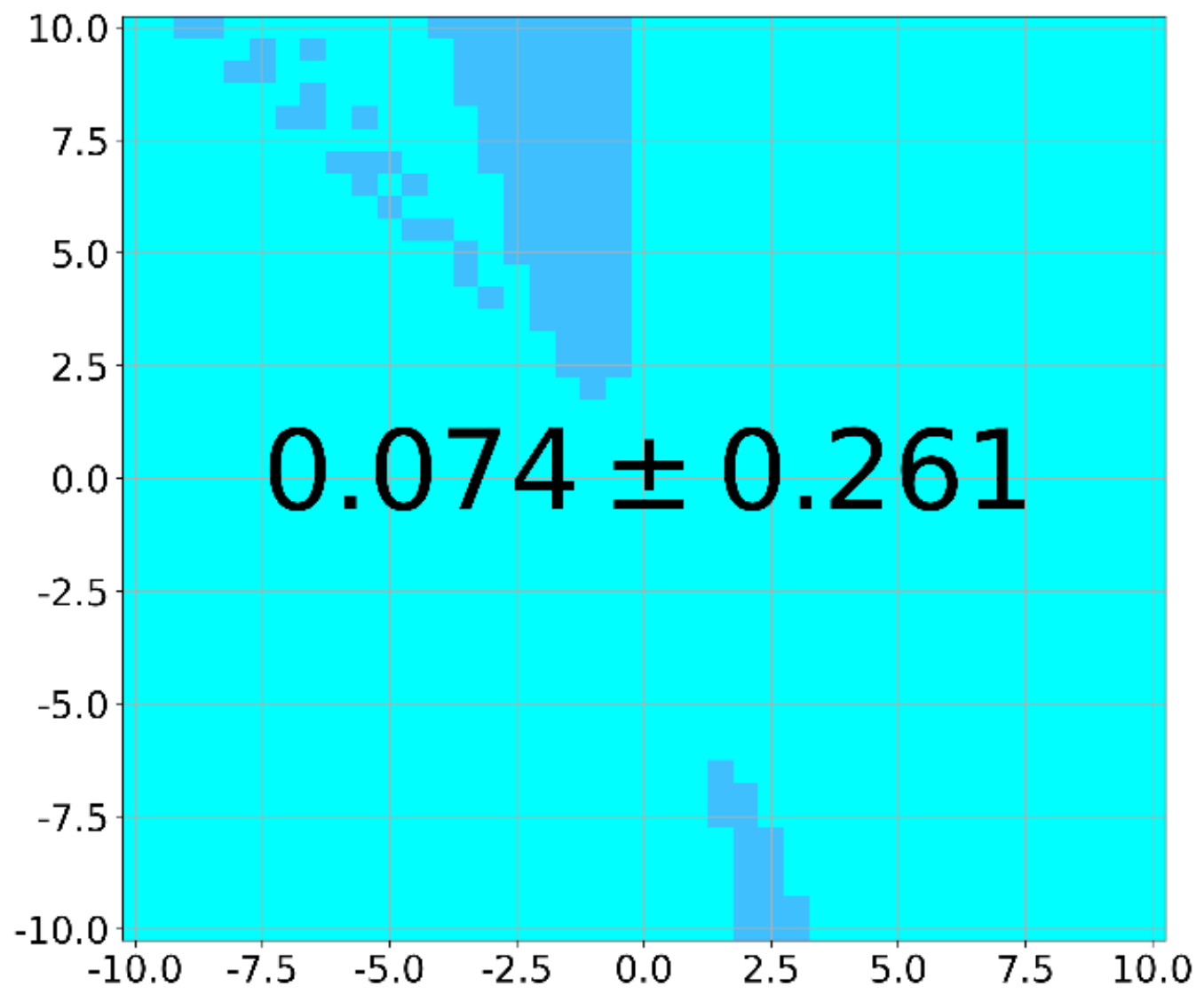}
\end{subfigure}
\begin{subfigure}{0.170\textwidth}
\includegraphics[width=\textwidth]{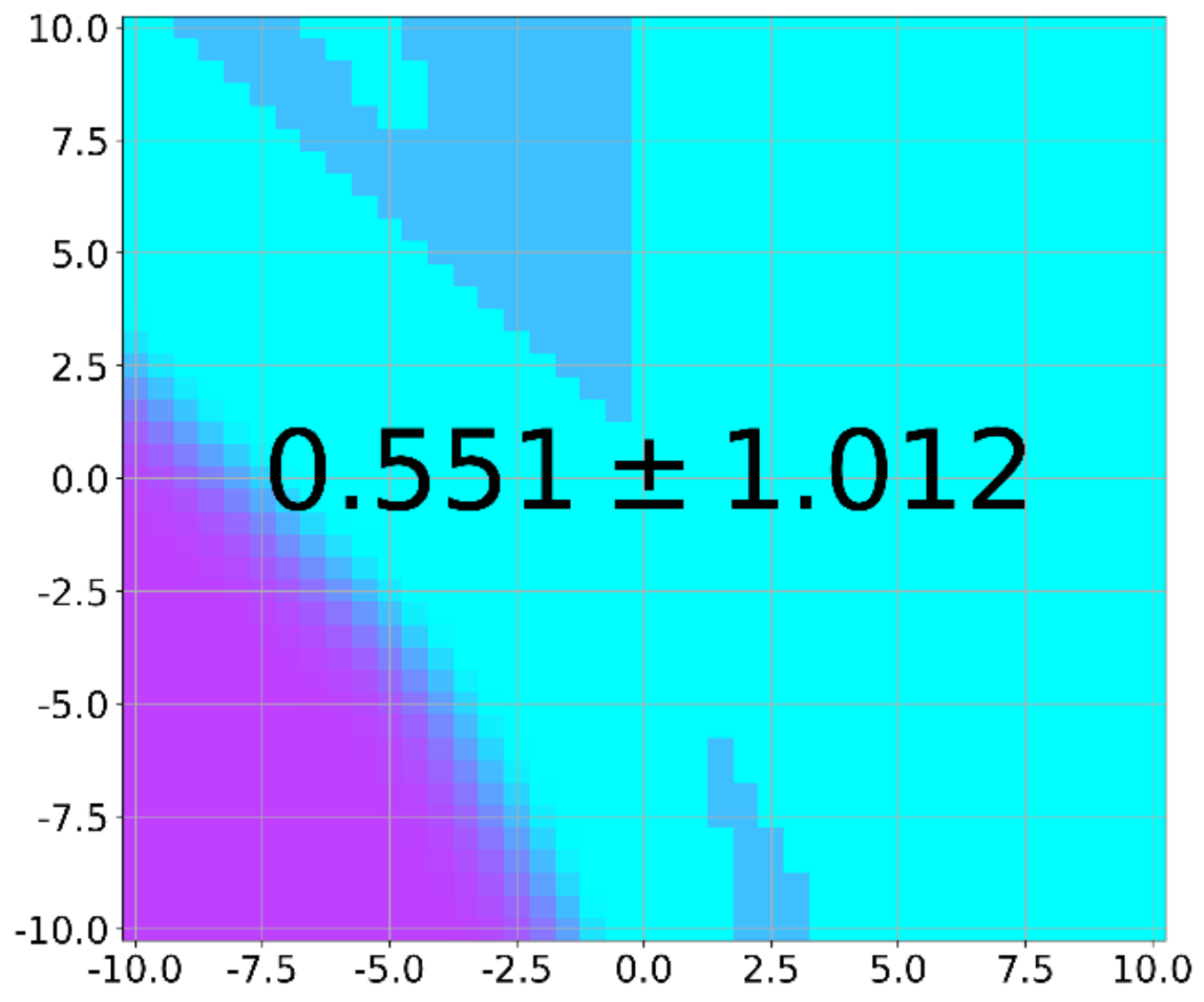}
\end{subfigure}
\begin{subfigure}{0.170\textwidth}
\includegraphics[width=\textwidth]{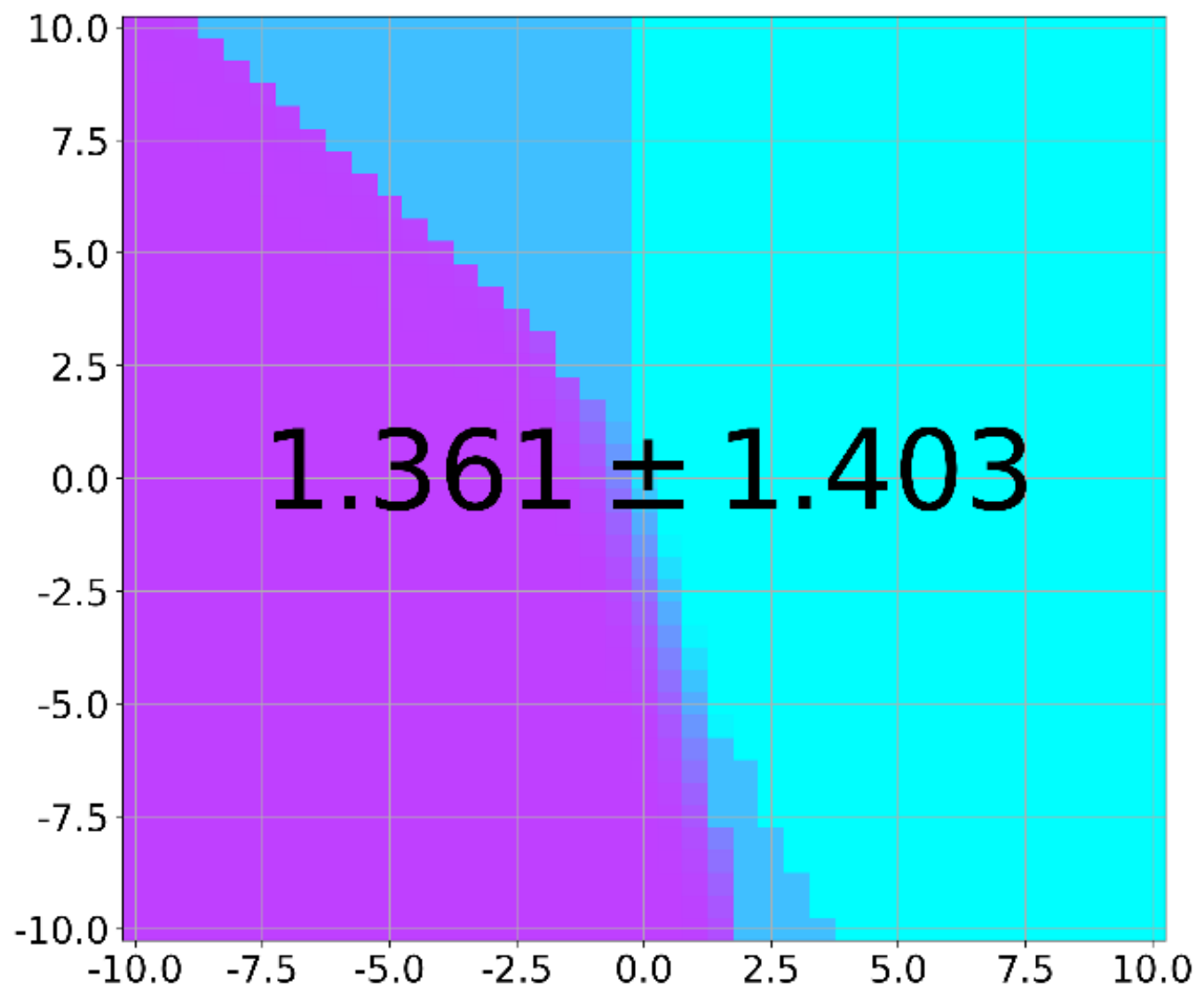}
\end{subfigure}
\begin{subfigure}{0.170\textwidth}
\includegraphics[width=\textwidth]{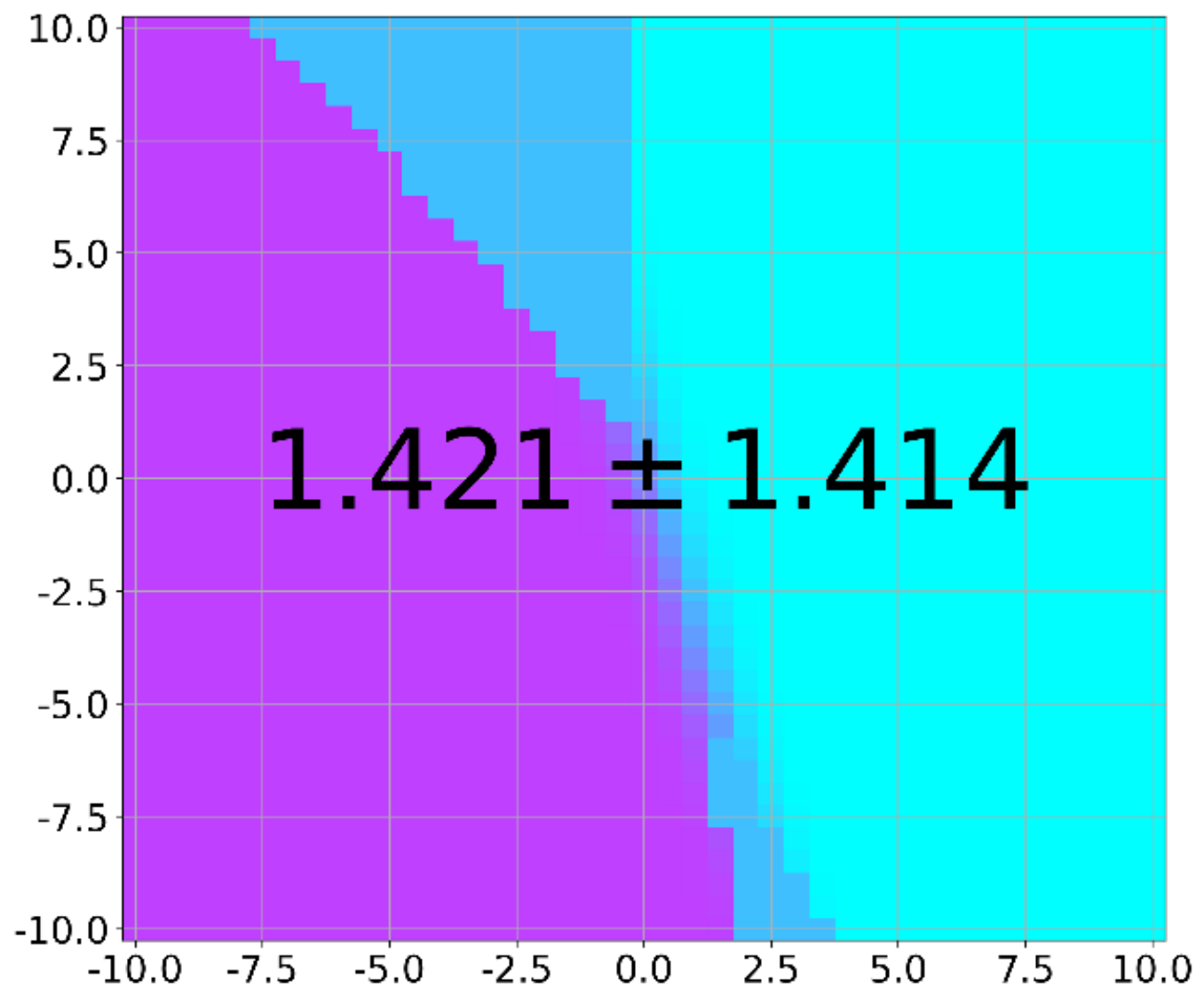}
\end{subfigure}
\begin{subfigure}{0.170\textwidth}
\includegraphics[width=\textwidth]{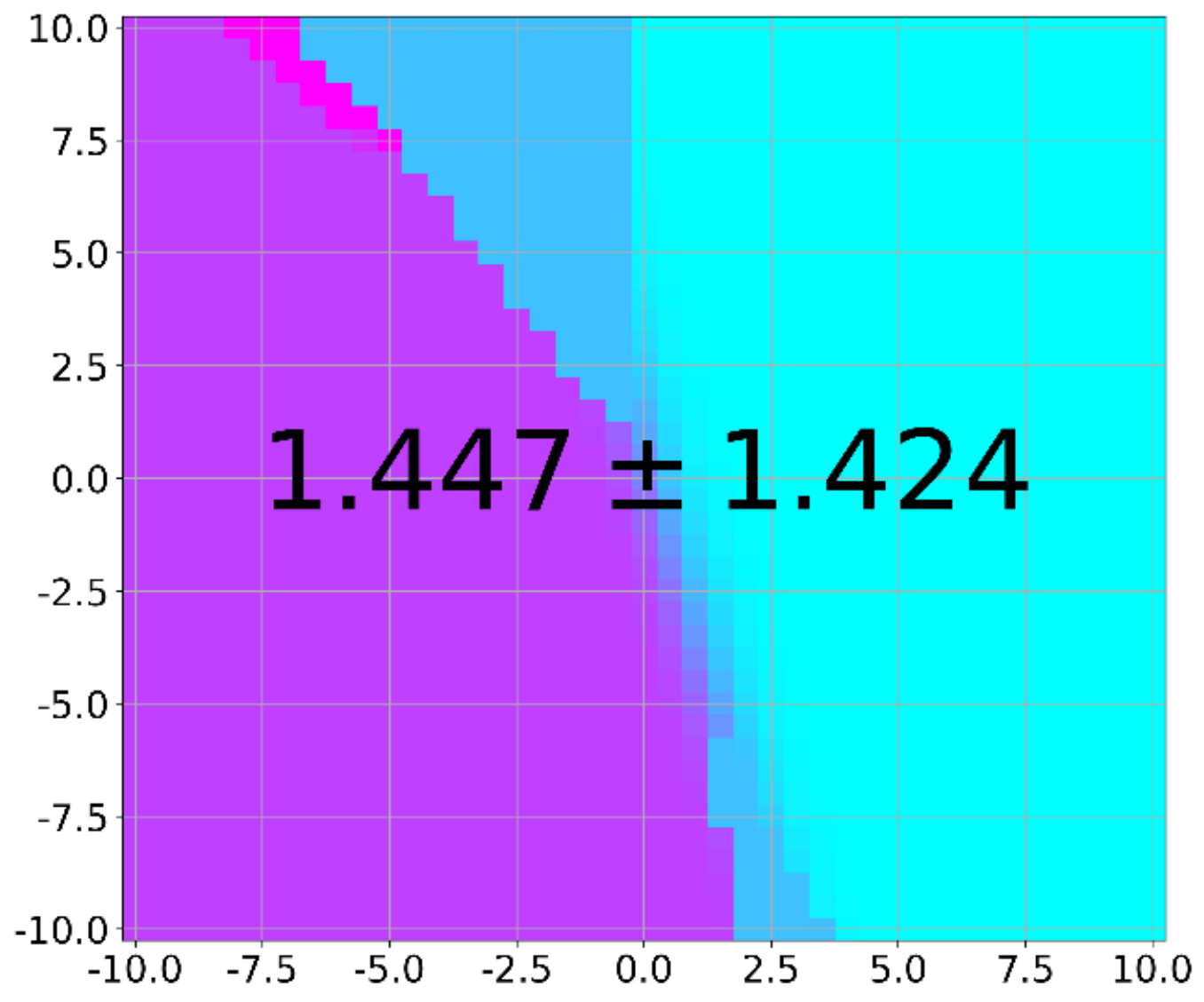}
\end{subfigure}

\begin{subfigure}{0.170\textwidth}
\includegraphics[width=\textwidth]{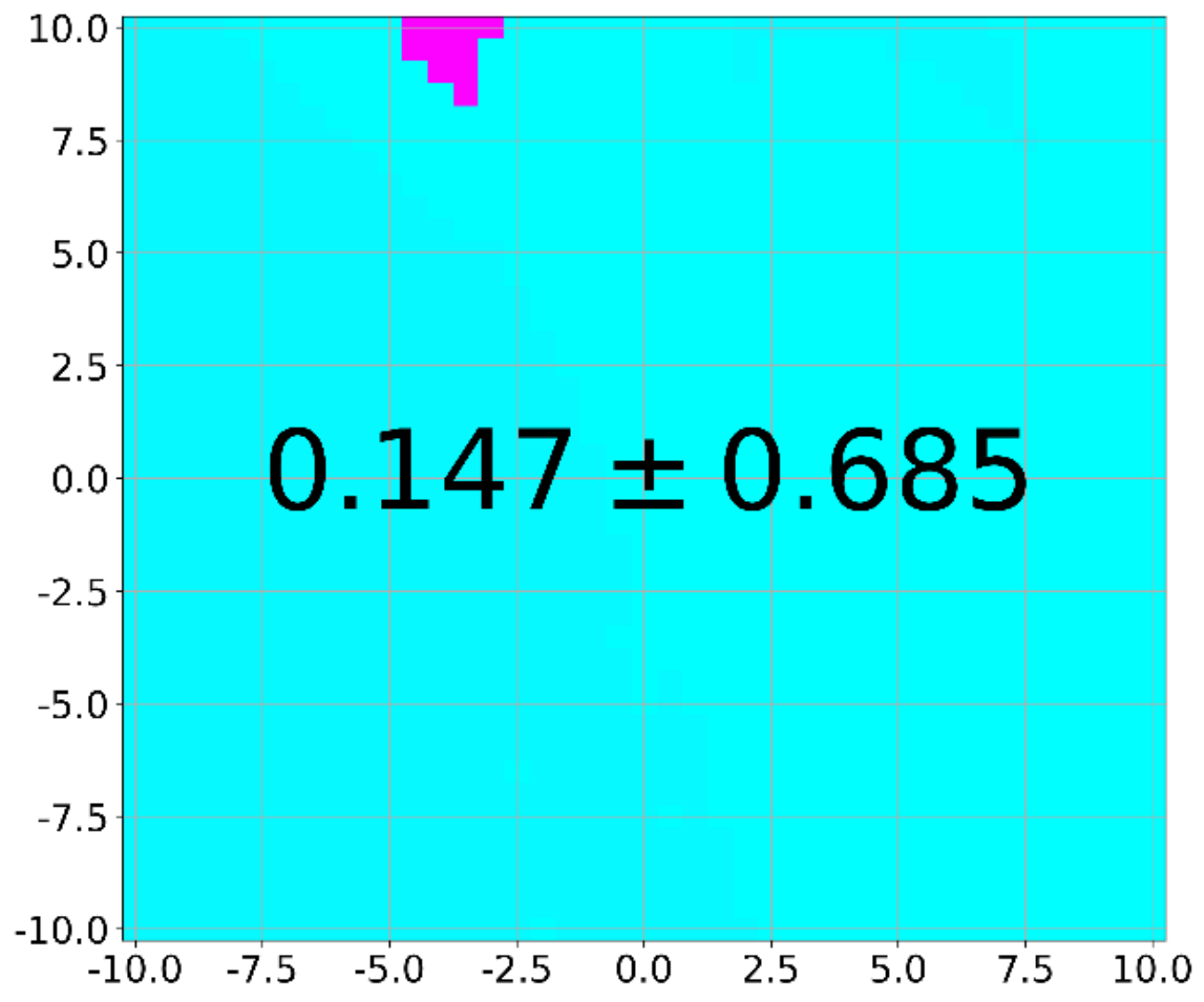}
\end{subfigure}
\begin{subfigure}{0.170\textwidth}
\includegraphics[width=\textwidth]{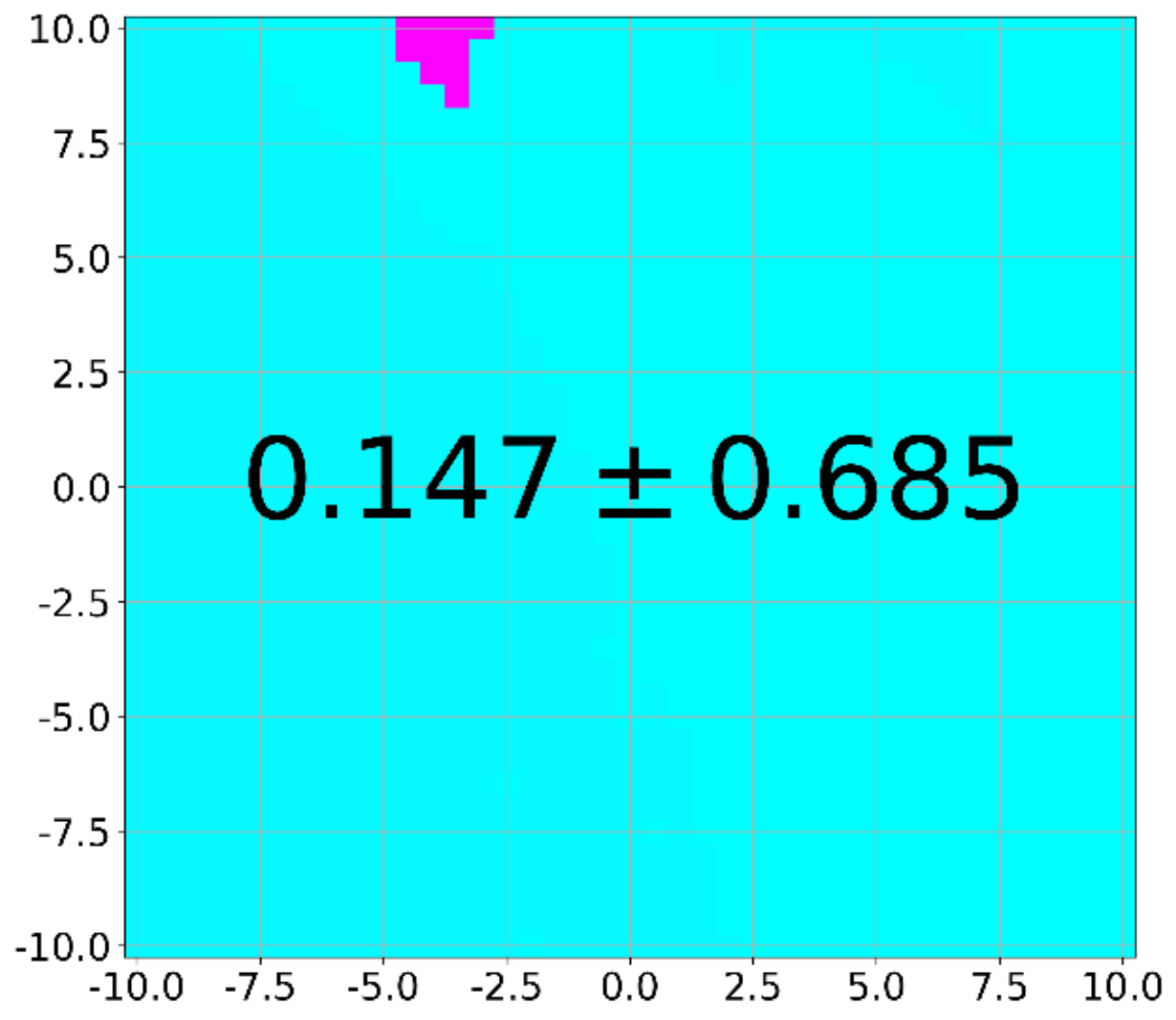}
\end{subfigure}
\begin{subfigure}{0.170\textwidth}
\includegraphics[width=\textwidth]{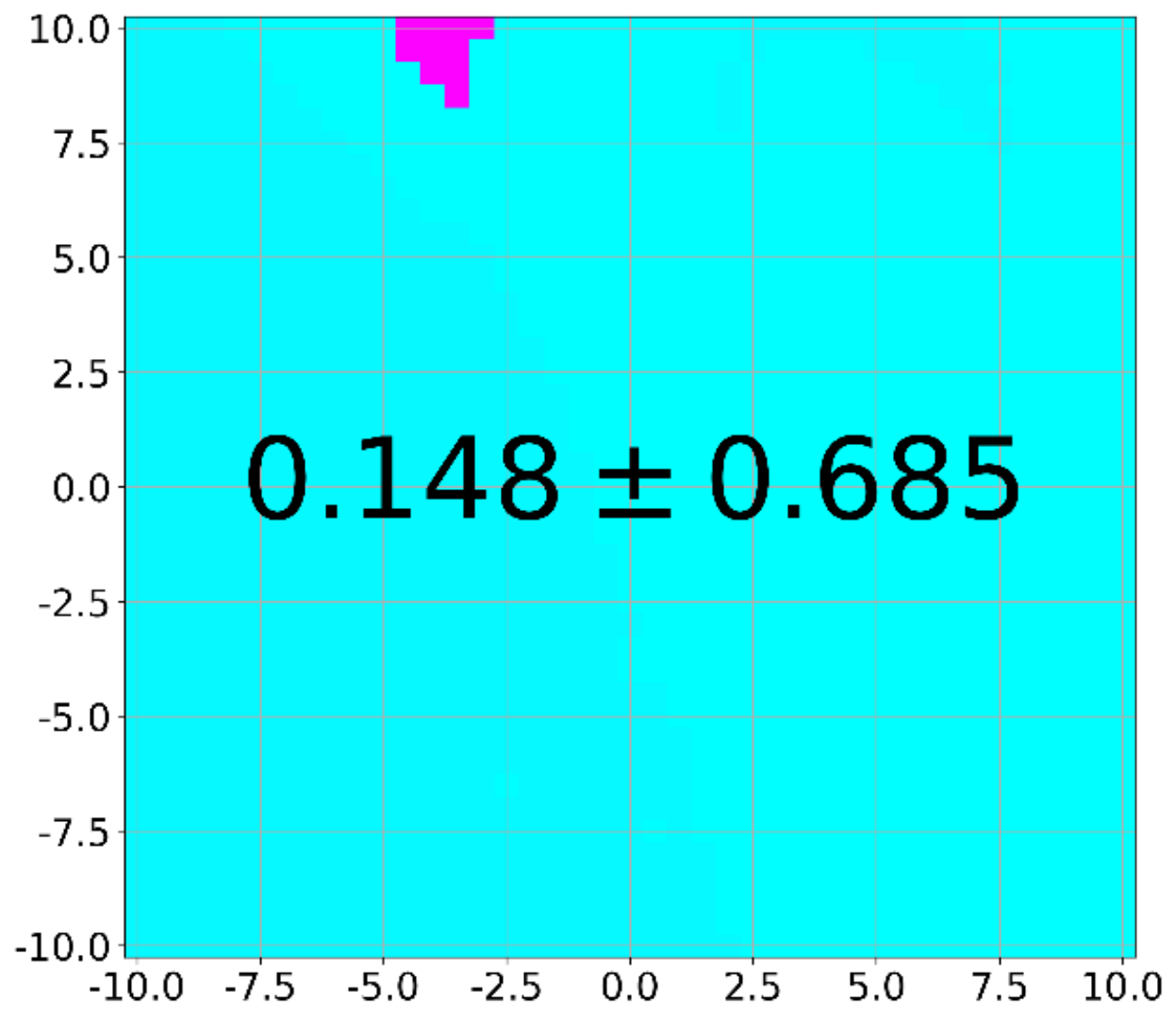}
\end{subfigure}
\begin{subfigure}{0.170\textwidth}
\includegraphics[width=\textwidth]{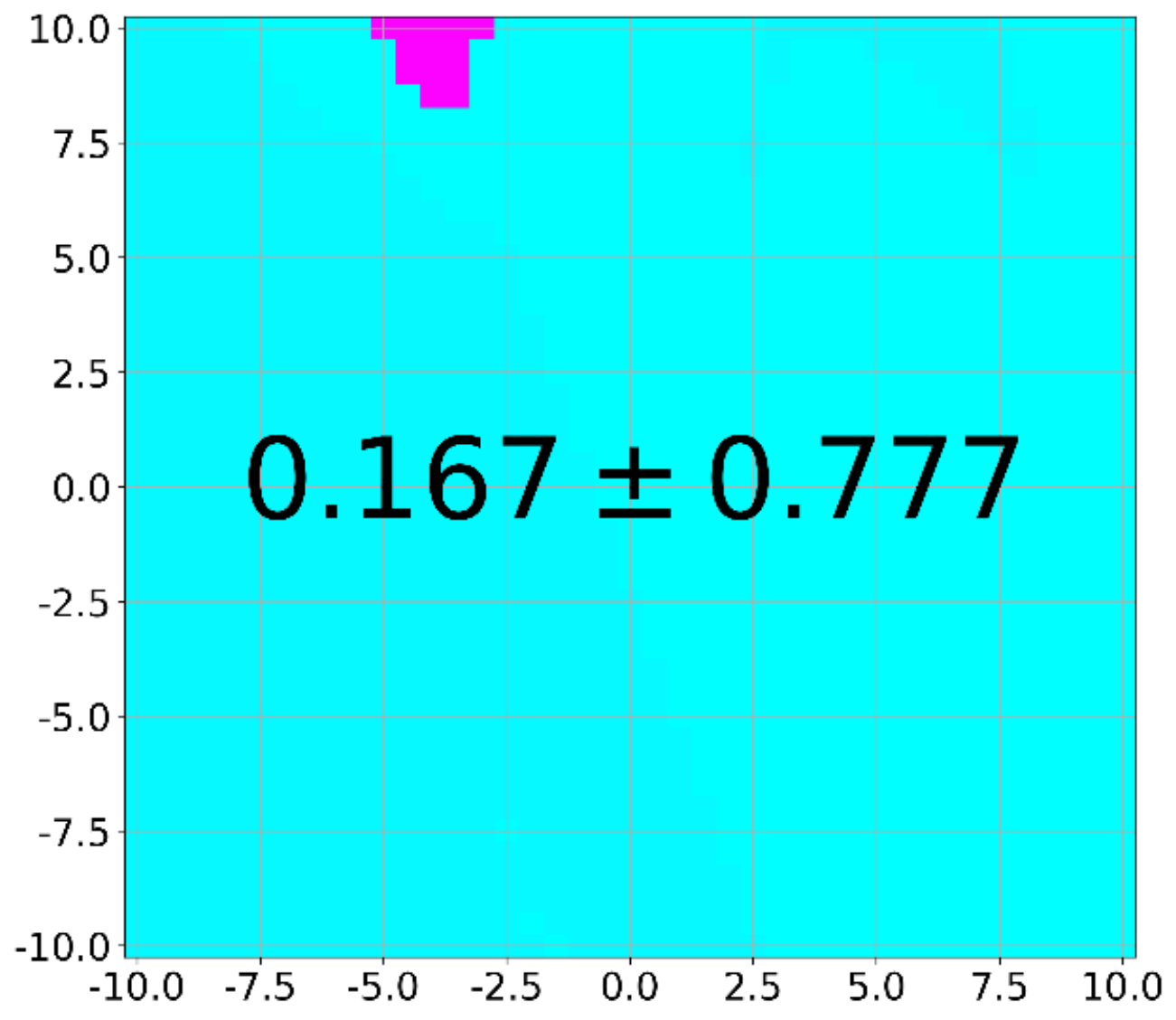}
\end{subfigure}
\begin{subfigure}{0.170\textwidth}
\includegraphics[width=\textwidth]{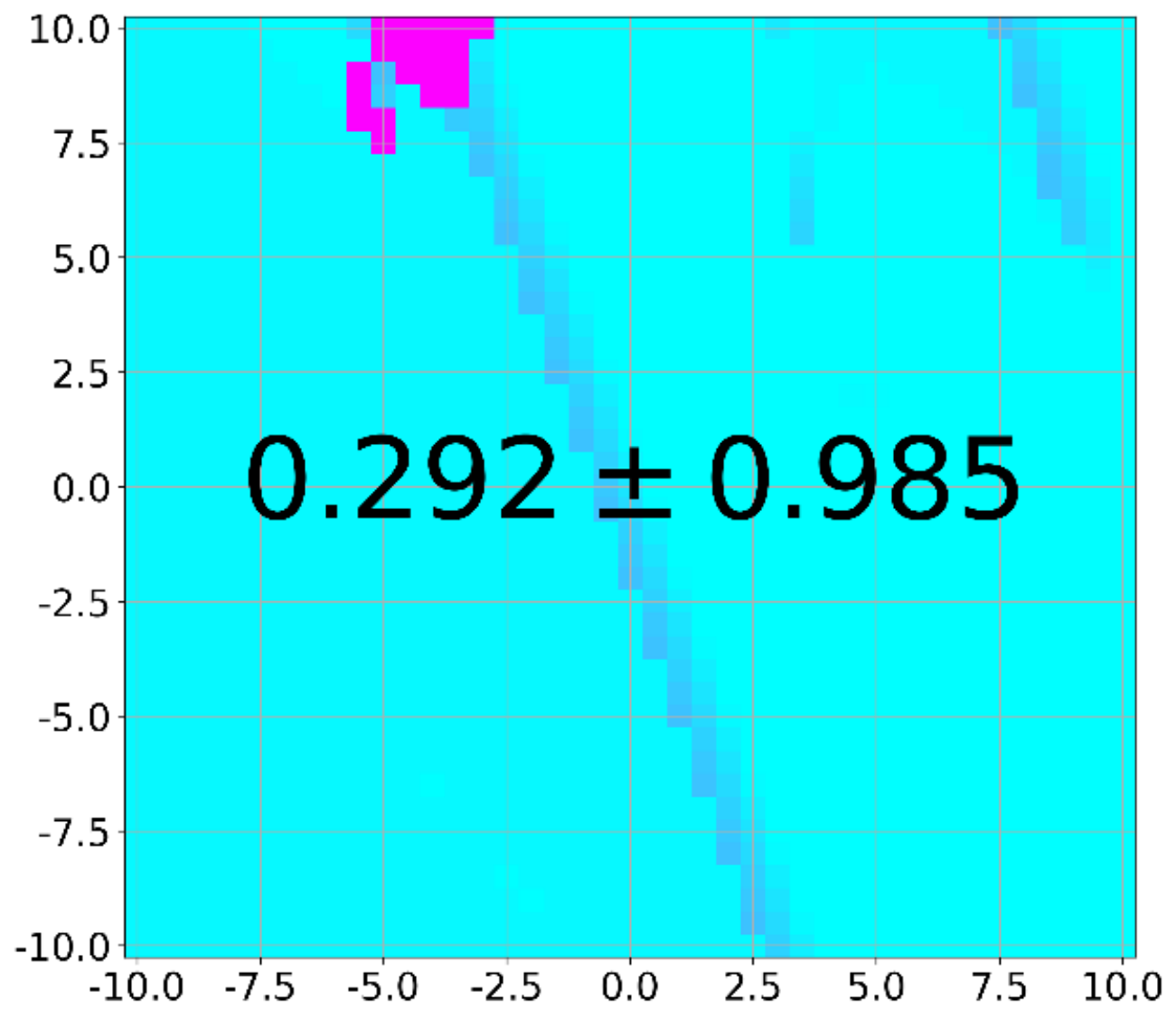}
\end{subfigure}
\begin{subfigure}{0.170\textwidth}
\includegraphics[width=\textwidth]{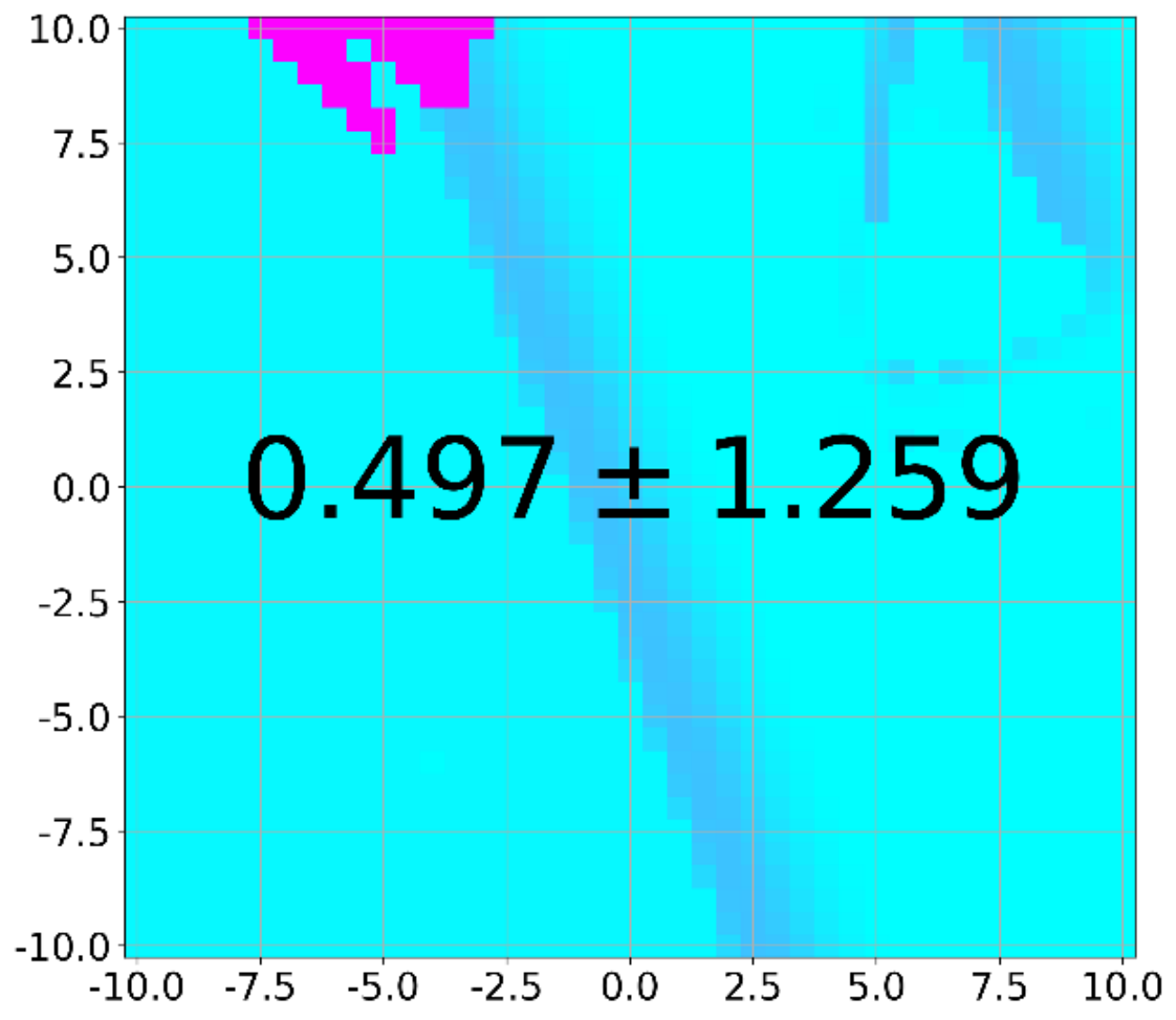}
\end{subfigure}
\begin{subfigure}{0.170\textwidth}
\includegraphics[width=\textwidth]{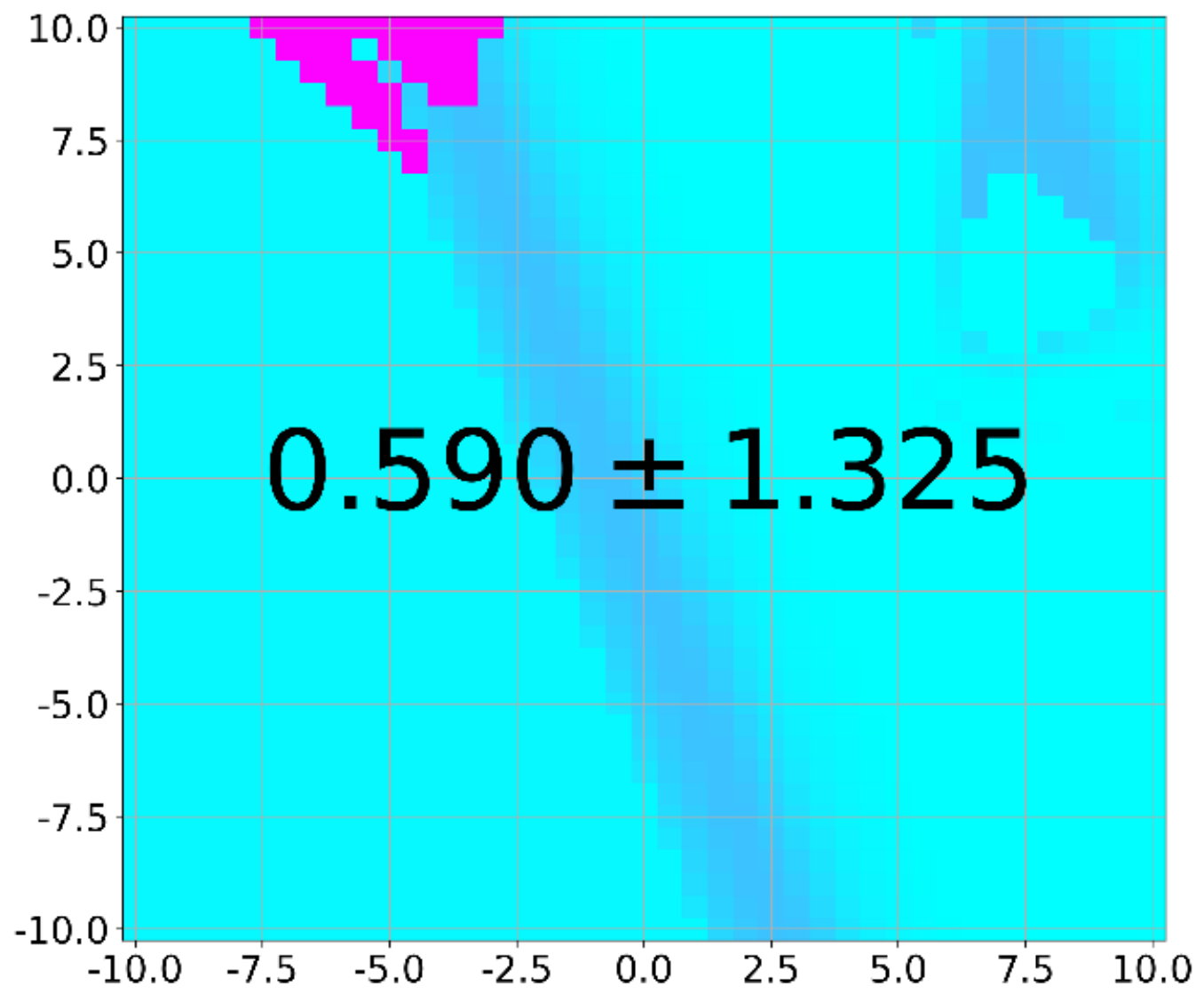}
\end{subfigure}

\begin{subfigure}{0.170\textwidth}
\includegraphics[width=\textwidth]{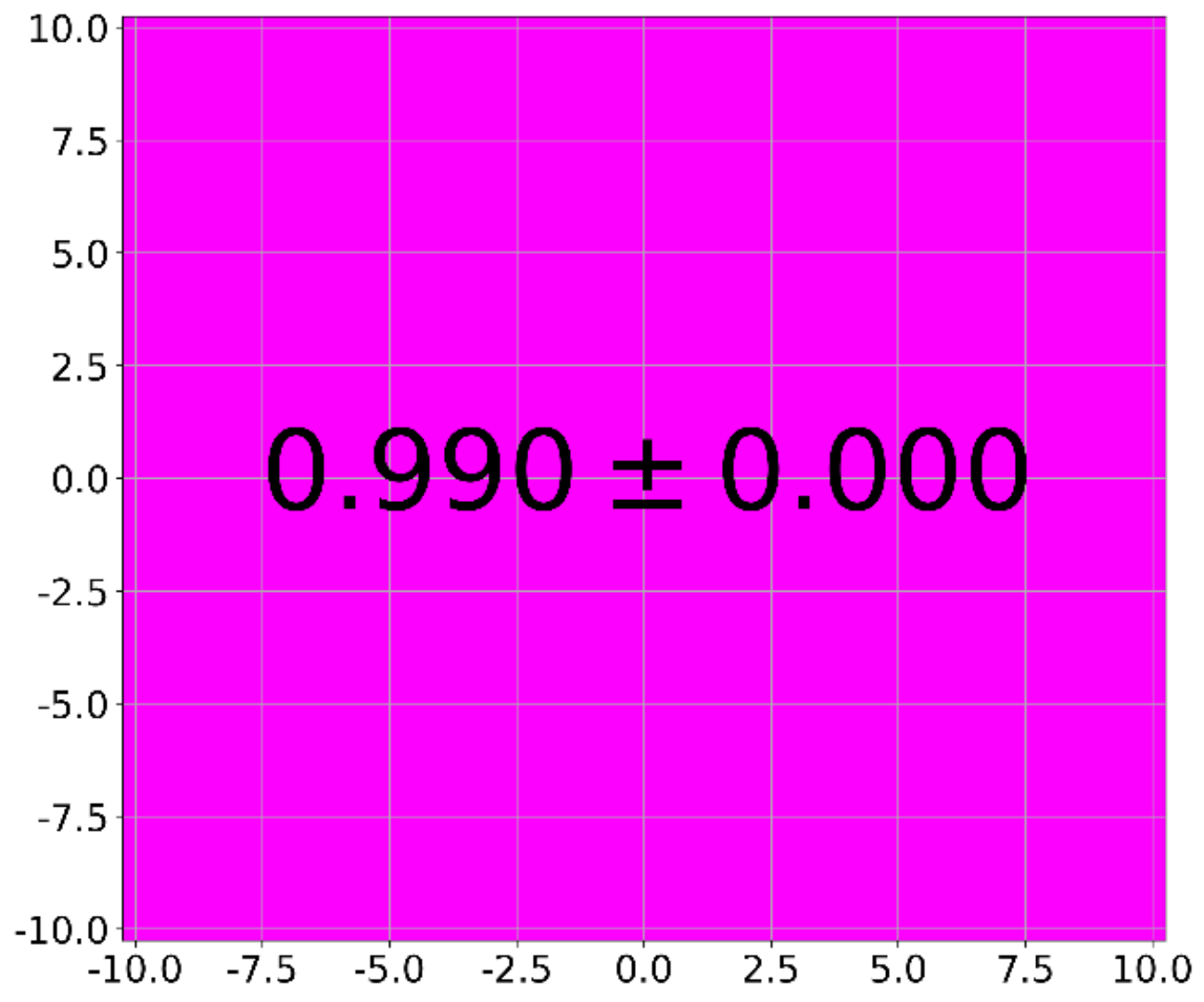}
\end{subfigure}
\begin{subfigure}{0.170\textwidth}
\includegraphics[width=\textwidth]{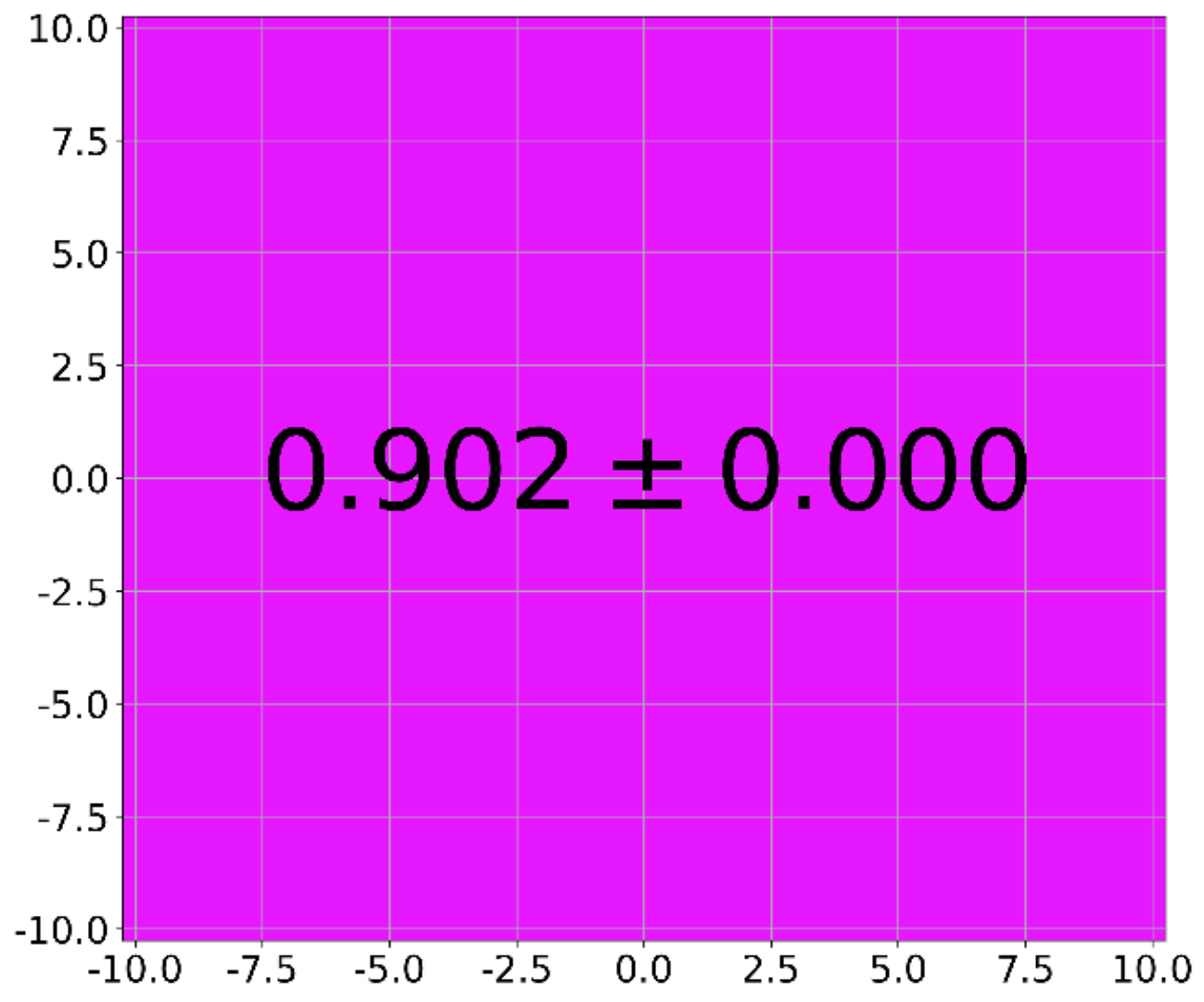}
\end{subfigure}
\begin{subfigure}{0.170\textwidth}
\includegraphics[width=\textwidth]{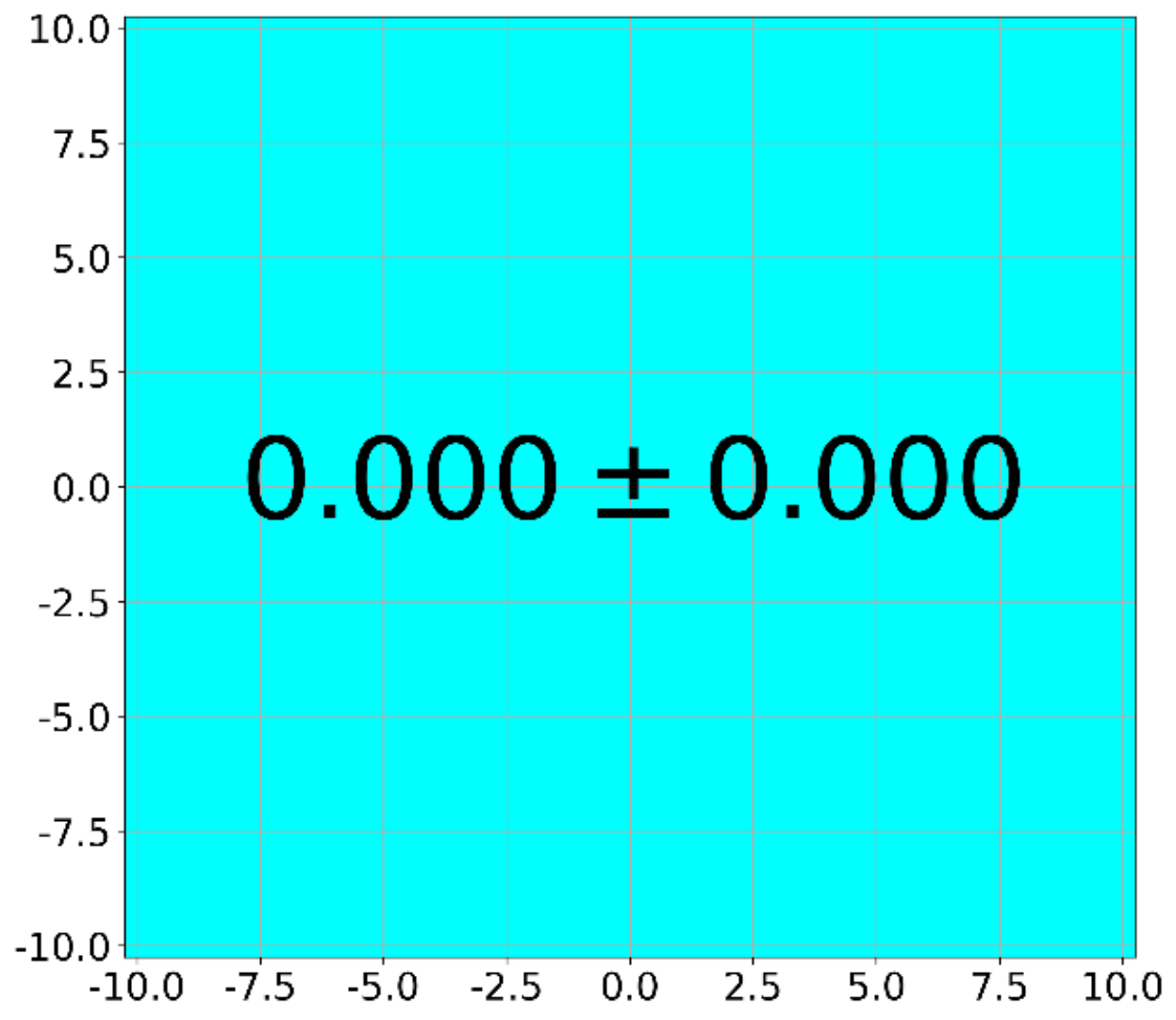}
\end{subfigure}
\begin{subfigure}{0.170\textwidth}
\includegraphics[width=\textwidth]{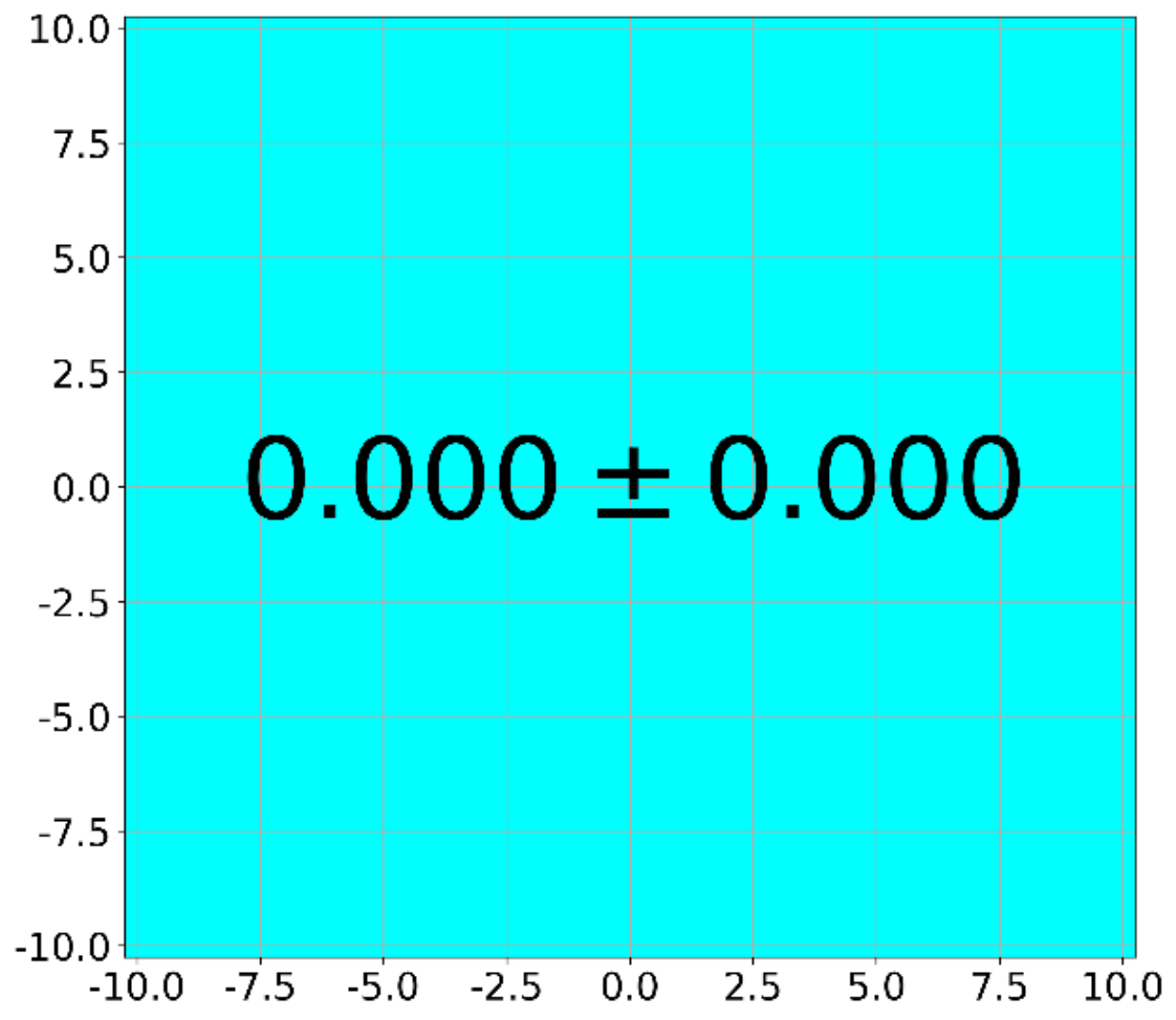}
\end{subfigure}
\begin{subfigure}{0.170\textwidth}
\includegraphics[width=\textwidth]{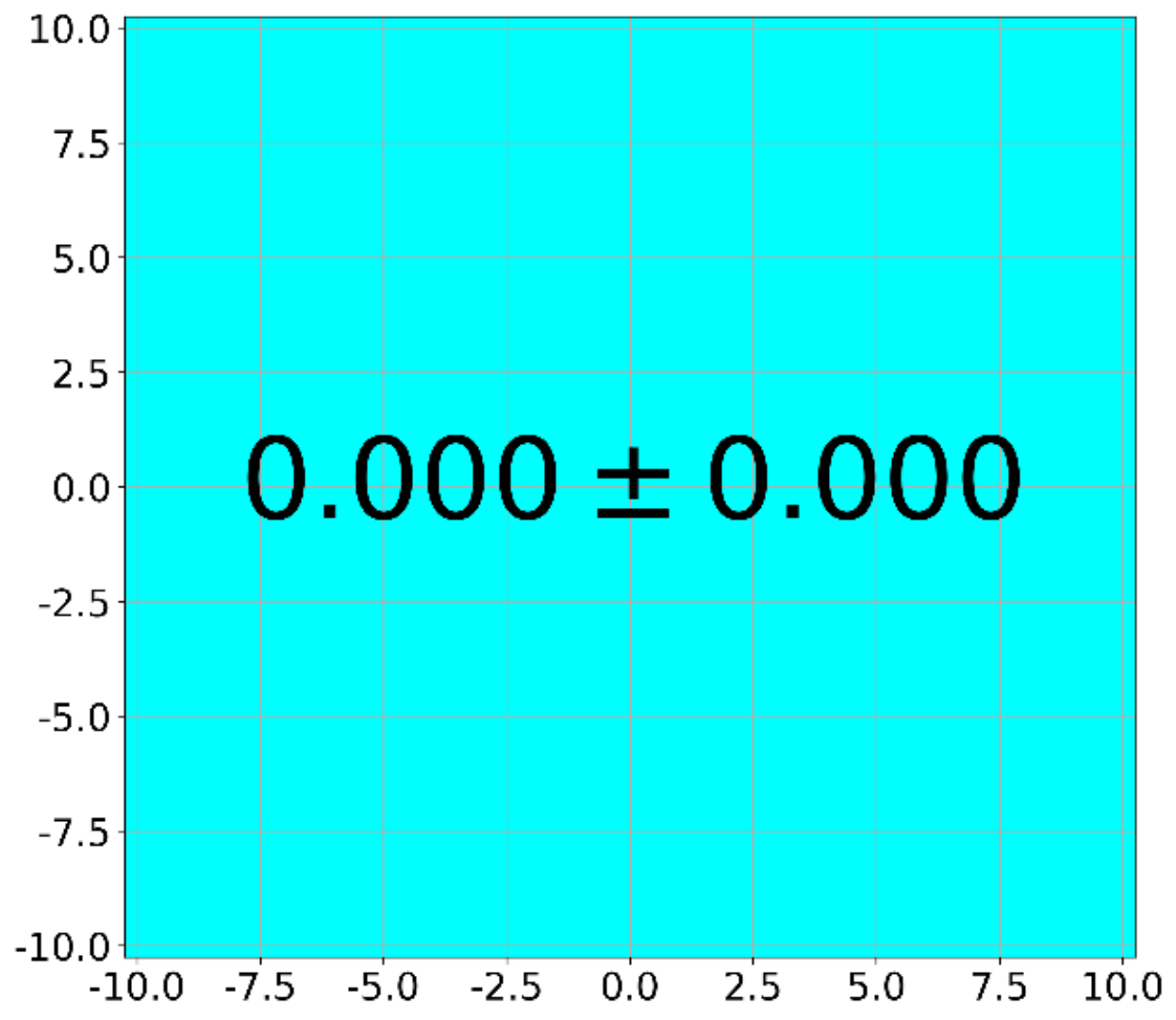}
\end{subfigure}
\begin{subfigure}{0.170\textwidth}
\includegraphics[width=\textwidth]{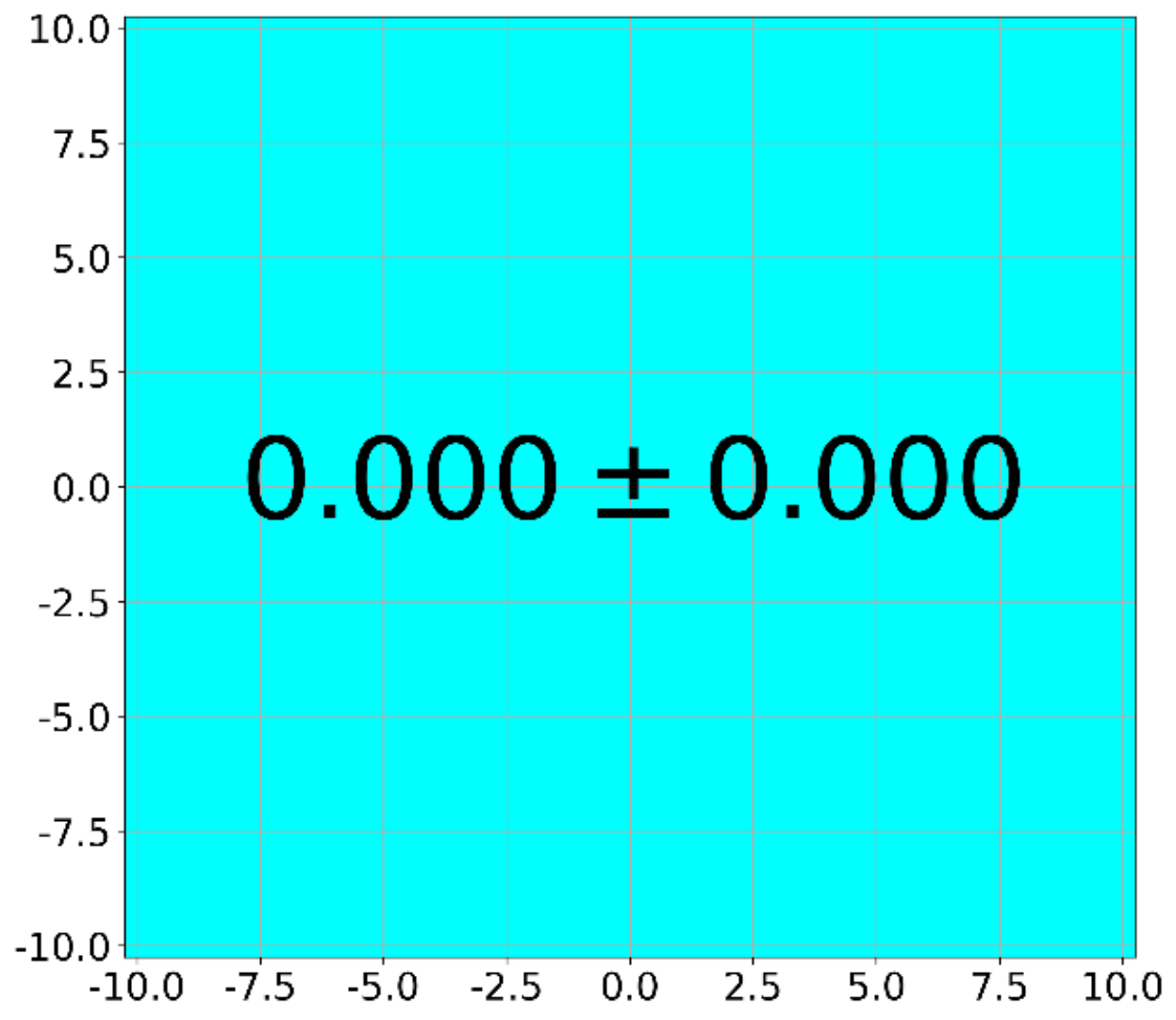}
\end{subfigure}
\begin{subfigure}{0.170\textwidth}
\includegraphics[width=\textwidth]{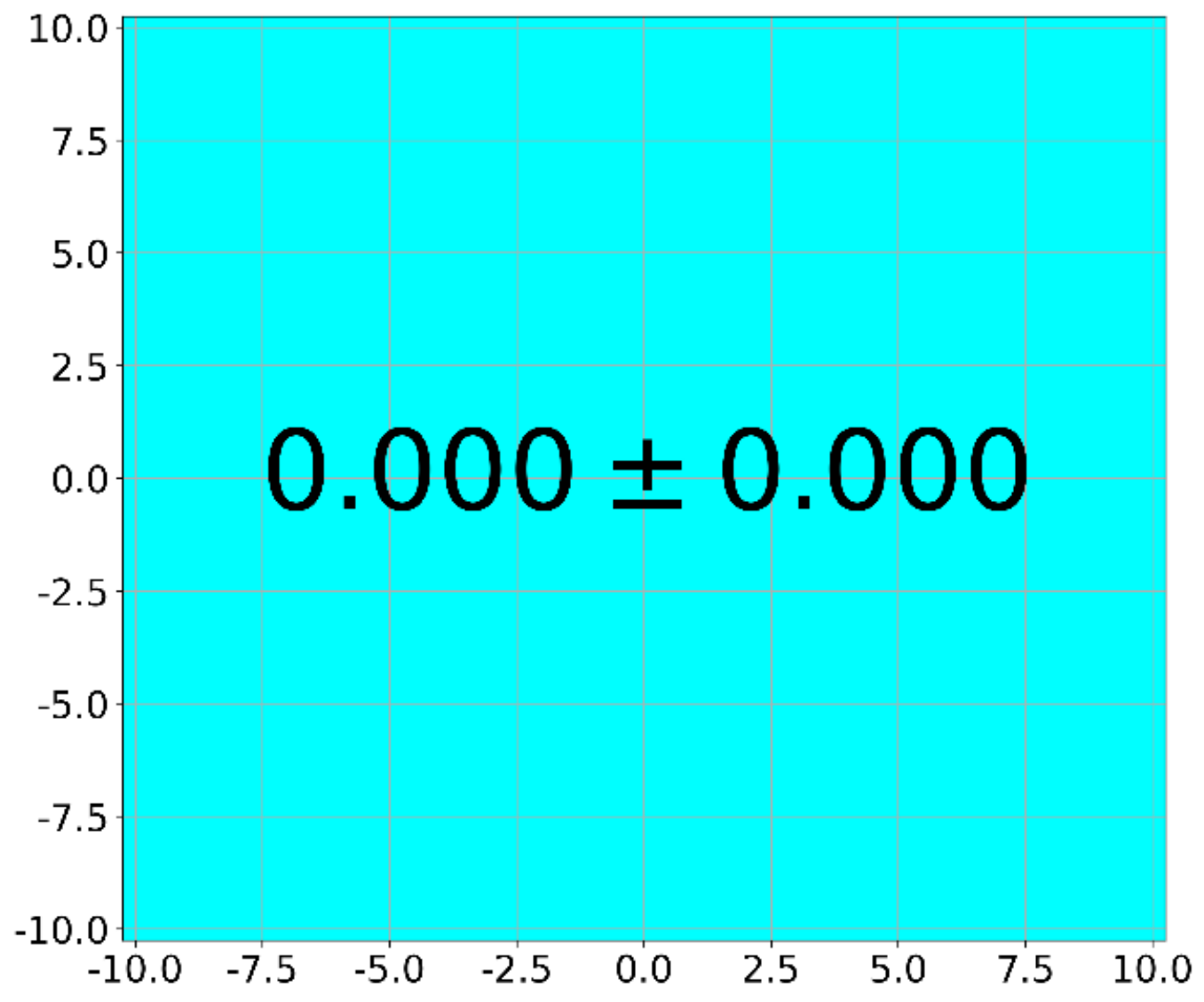}
\end{subfigure}

\begin{subfigure}{0.170\textwidth}
\includegraphics[width=\textwidth]{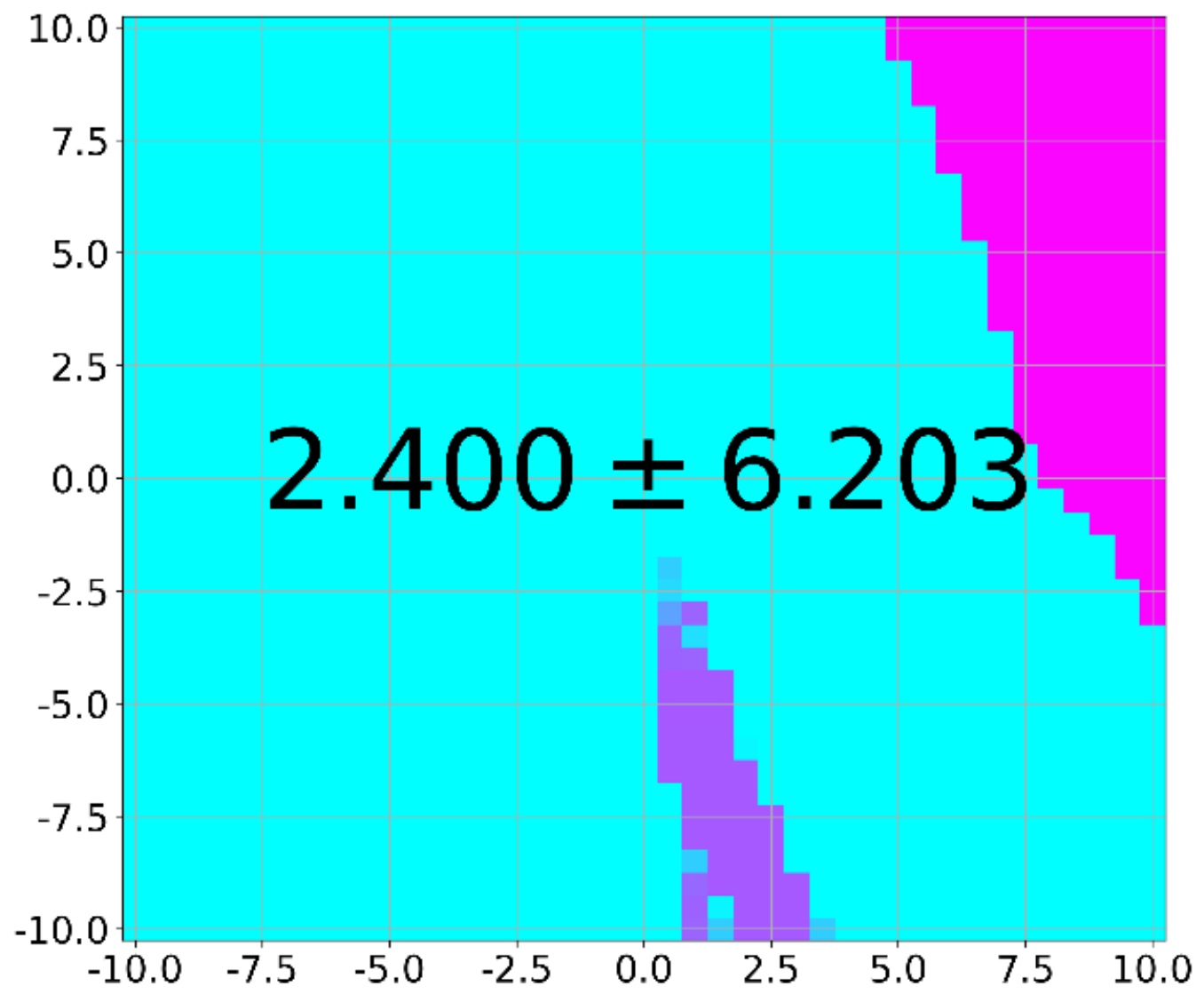}
\end{subfigure}
\begin{subfigure}{0.170\textwidth}
\includegraphics[width=\textwidth]{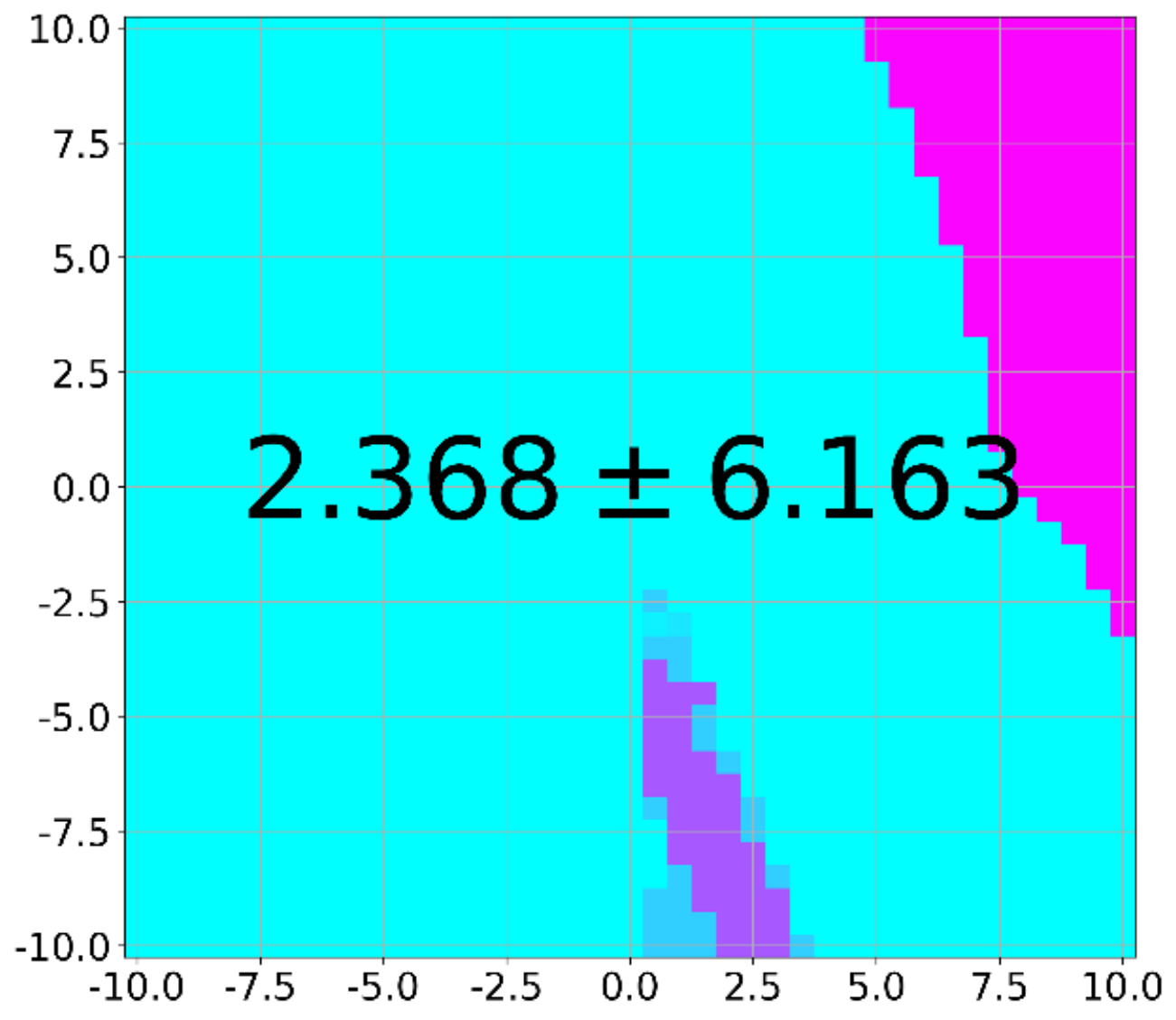}
\end{subfigure}
\begin{subfigure}{0.170\textwidth}
\includegraphics[width=\textwidth]{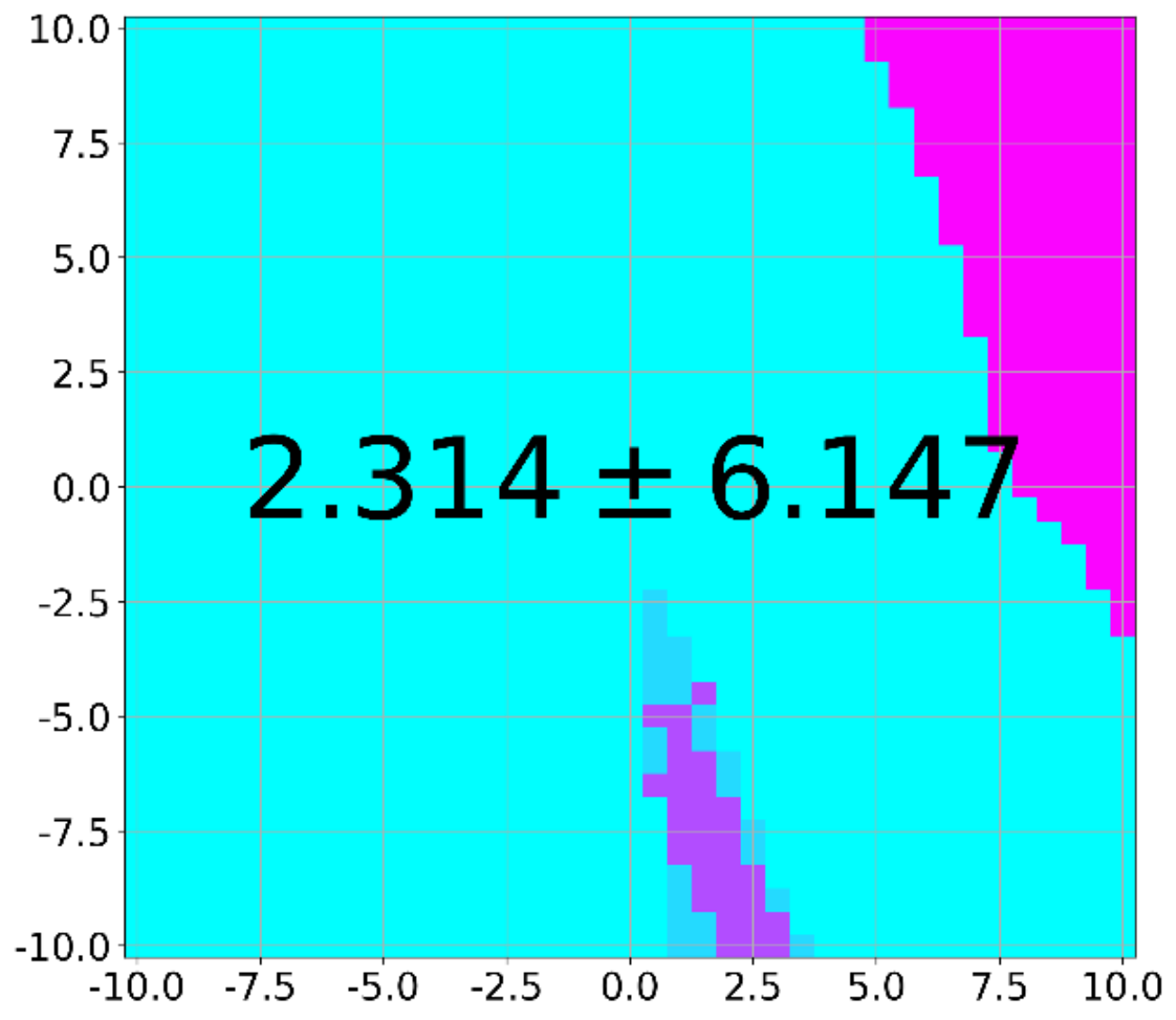}
\end{subfigure}
\begin{subfigure}{0.170\textwidth}
\includegraphics[width=\textwidth]{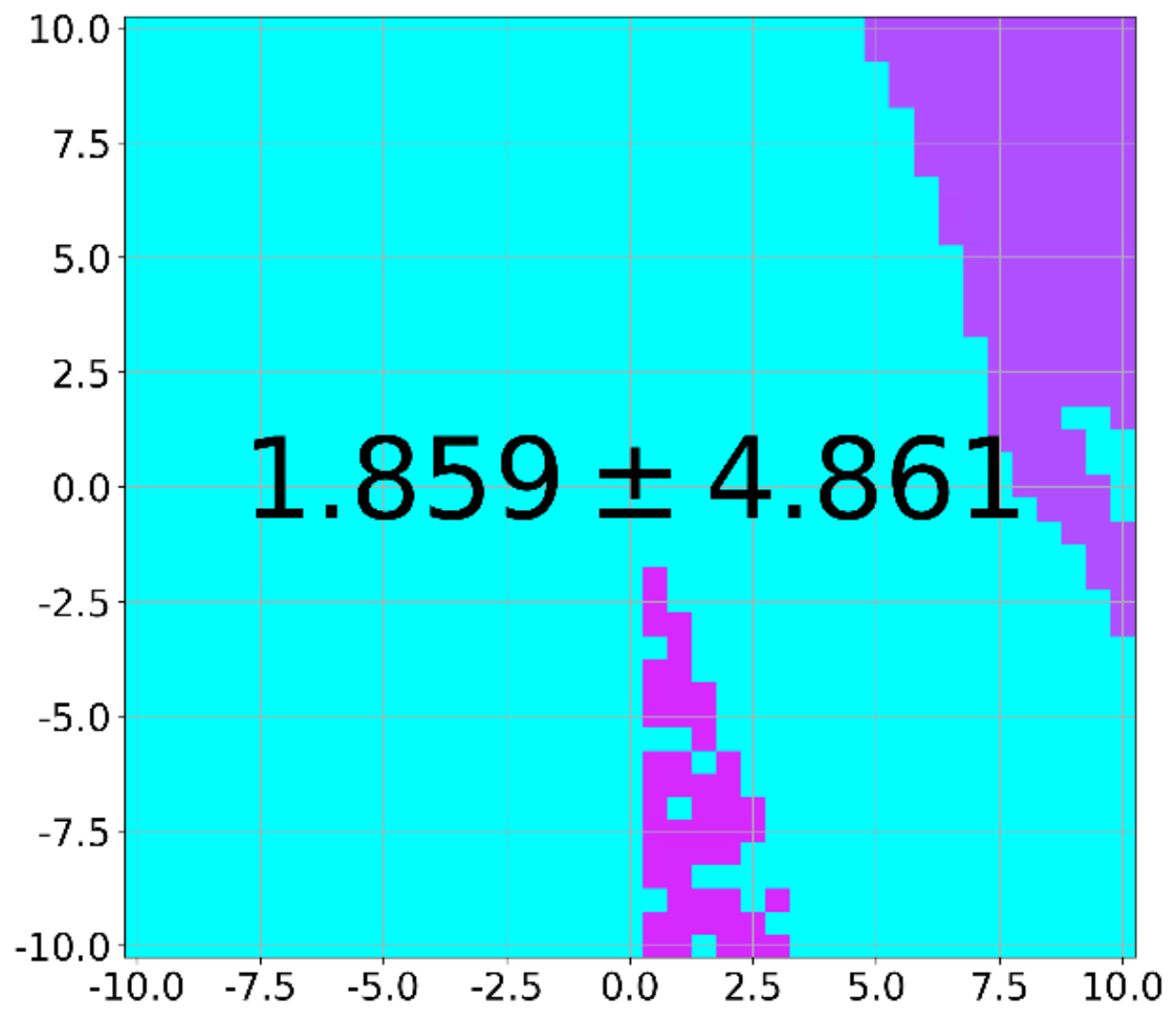}
\end{subfigure}
\begin{subfigure}{0.170\textwidth}
\includegraphics[width=\textwidth]{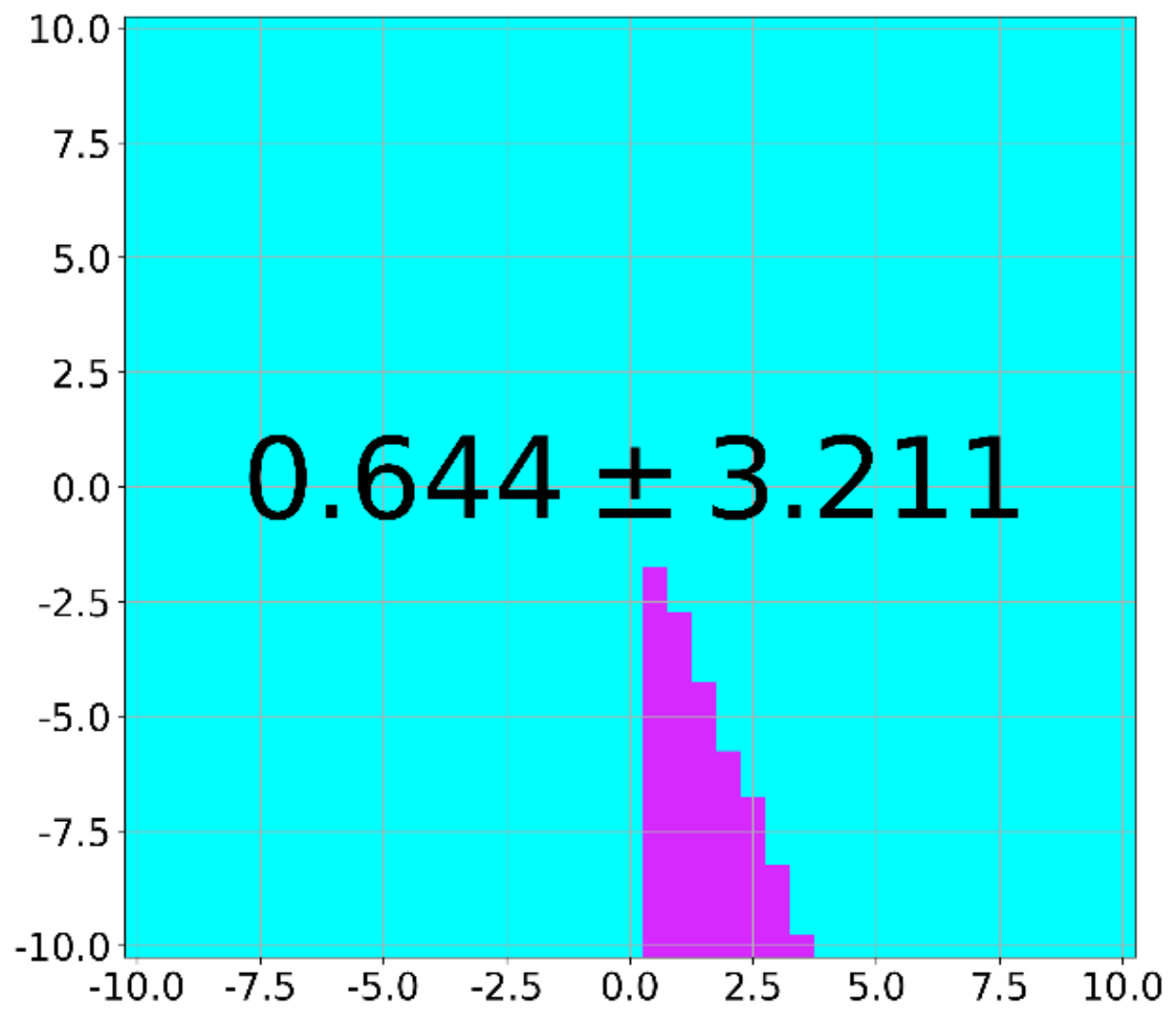}
\end{subfigure}
\begin{subfigure}{0.170\textwidth}
\includegraphics[width=\textwidth]{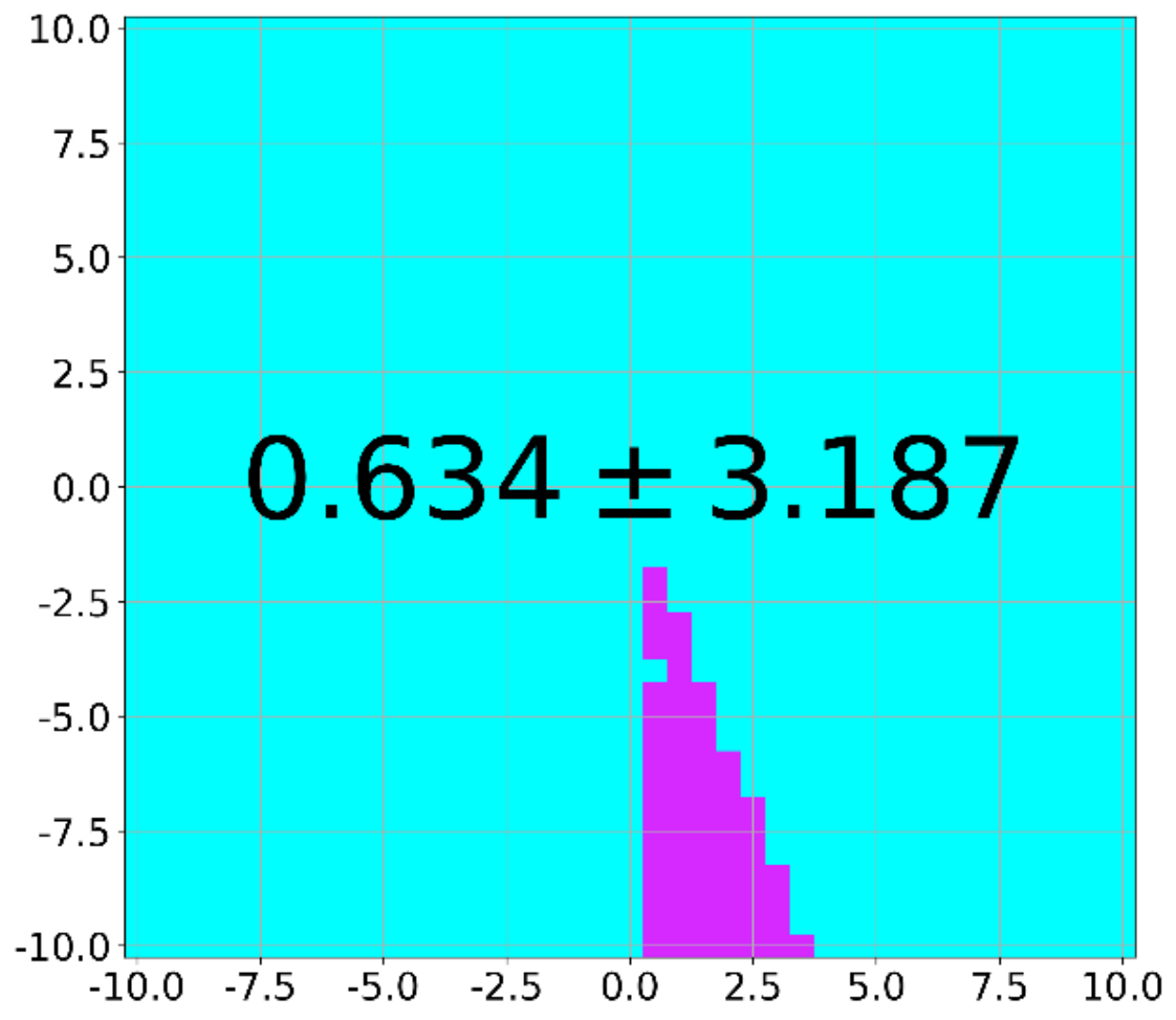}
\end{subfigure}
\begin{subfigure}{0.170\textwidth}
\includegraphics[width=\textwidth]{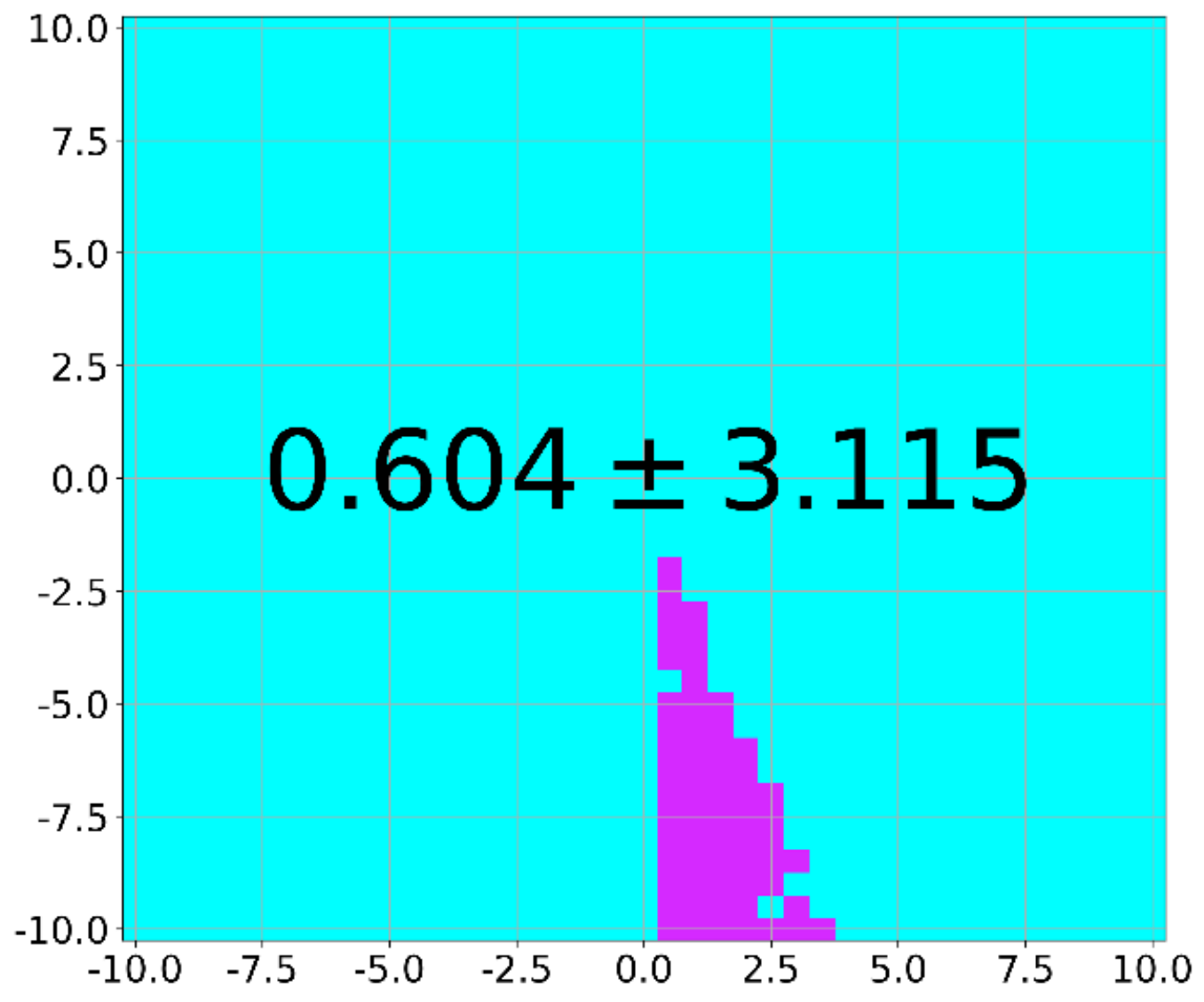}
\end{subfigure}

\begin{subfigure}{0.170\textwidth}
\includegraphics[width=\textwidth]{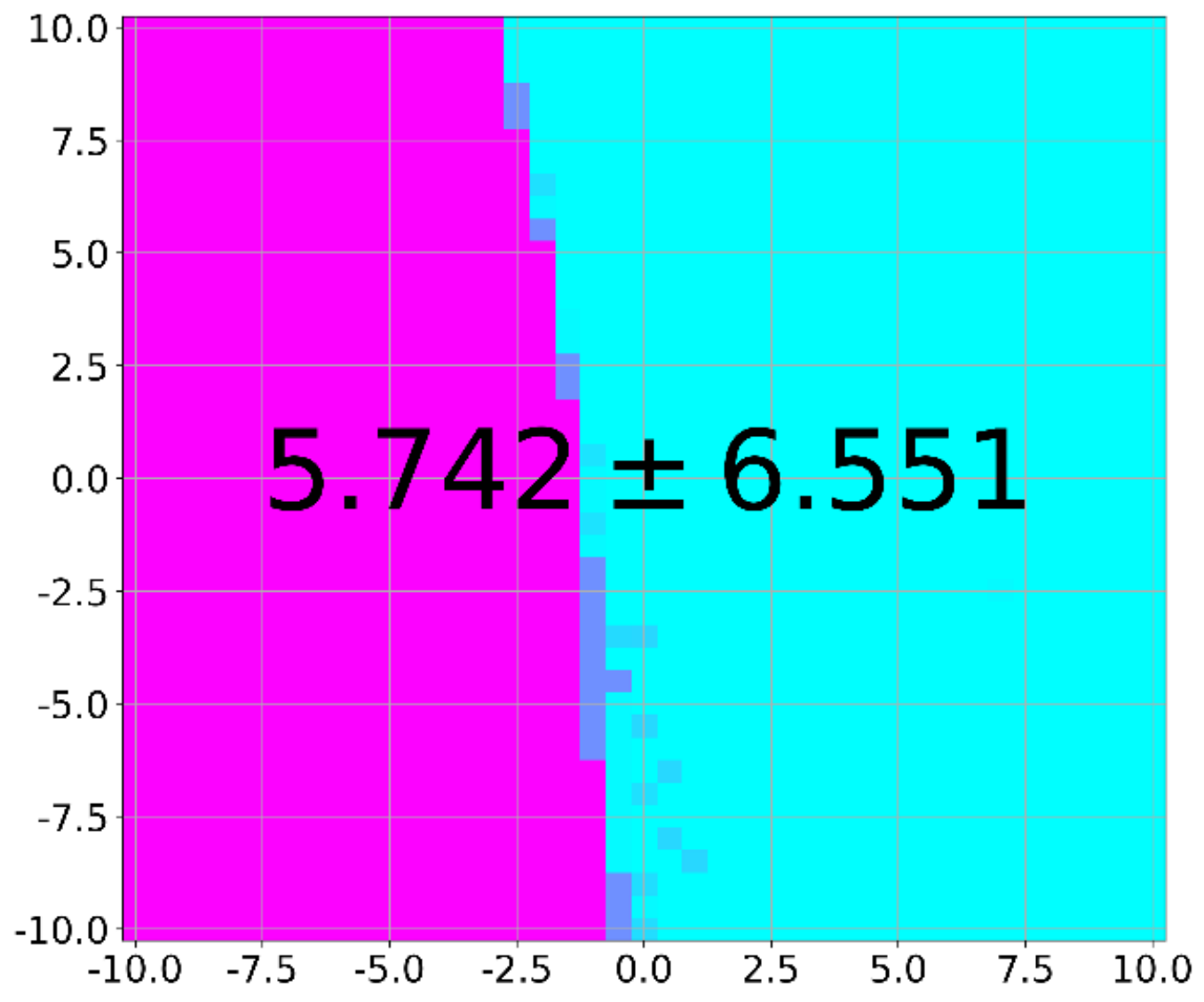}
\subcaption*{$\beta_0=0.01$}
\end{subfigure}
\begin{subfigure}{0.170\textwidth}
\includegraphics[width=\textwidth]{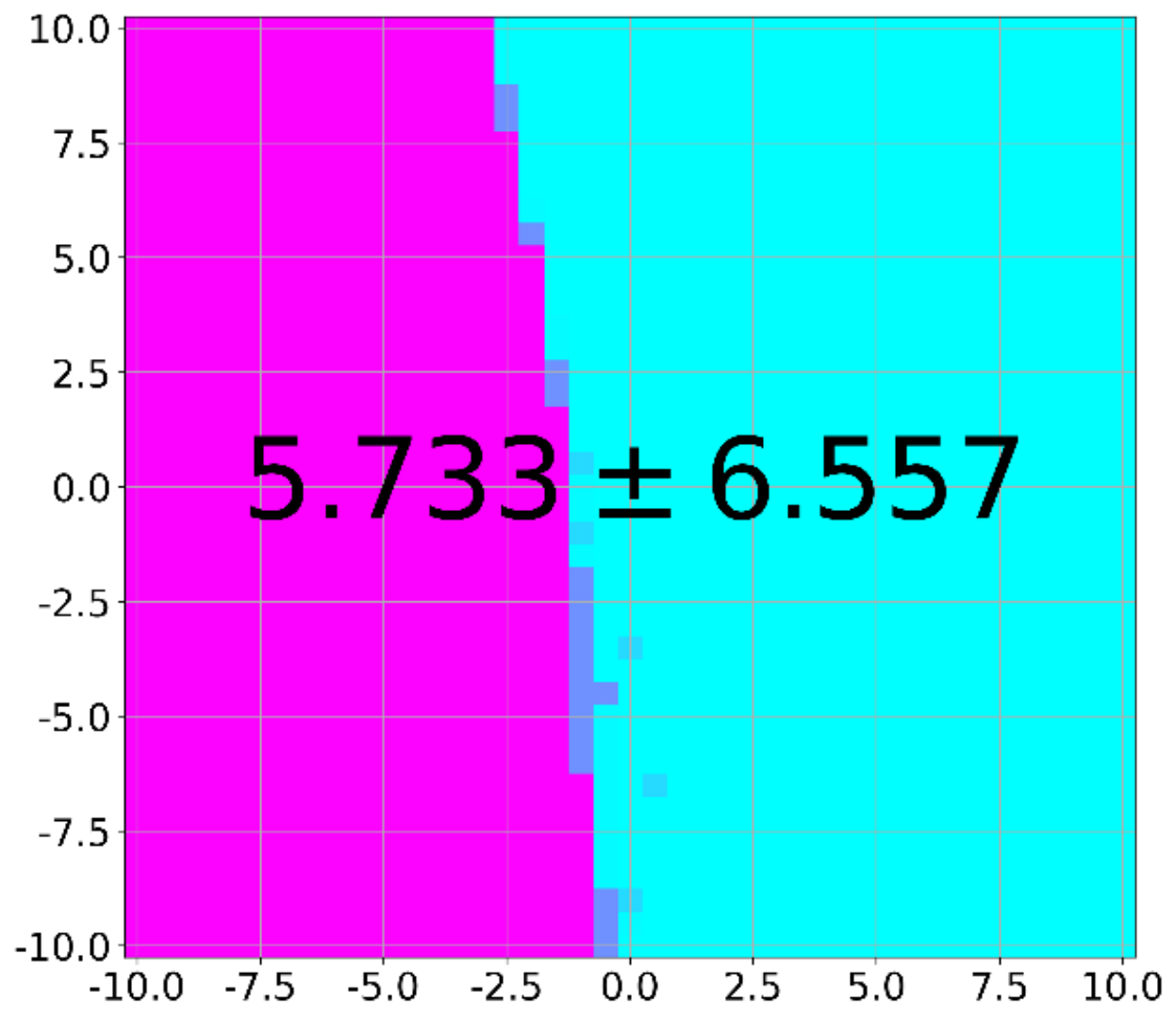}
\subcaption*{$\beta_0=0.1$}
\end{subfigure}
\begin{subfigure}{0.170\textwidth}
\includegraphics[width=\textwidth]{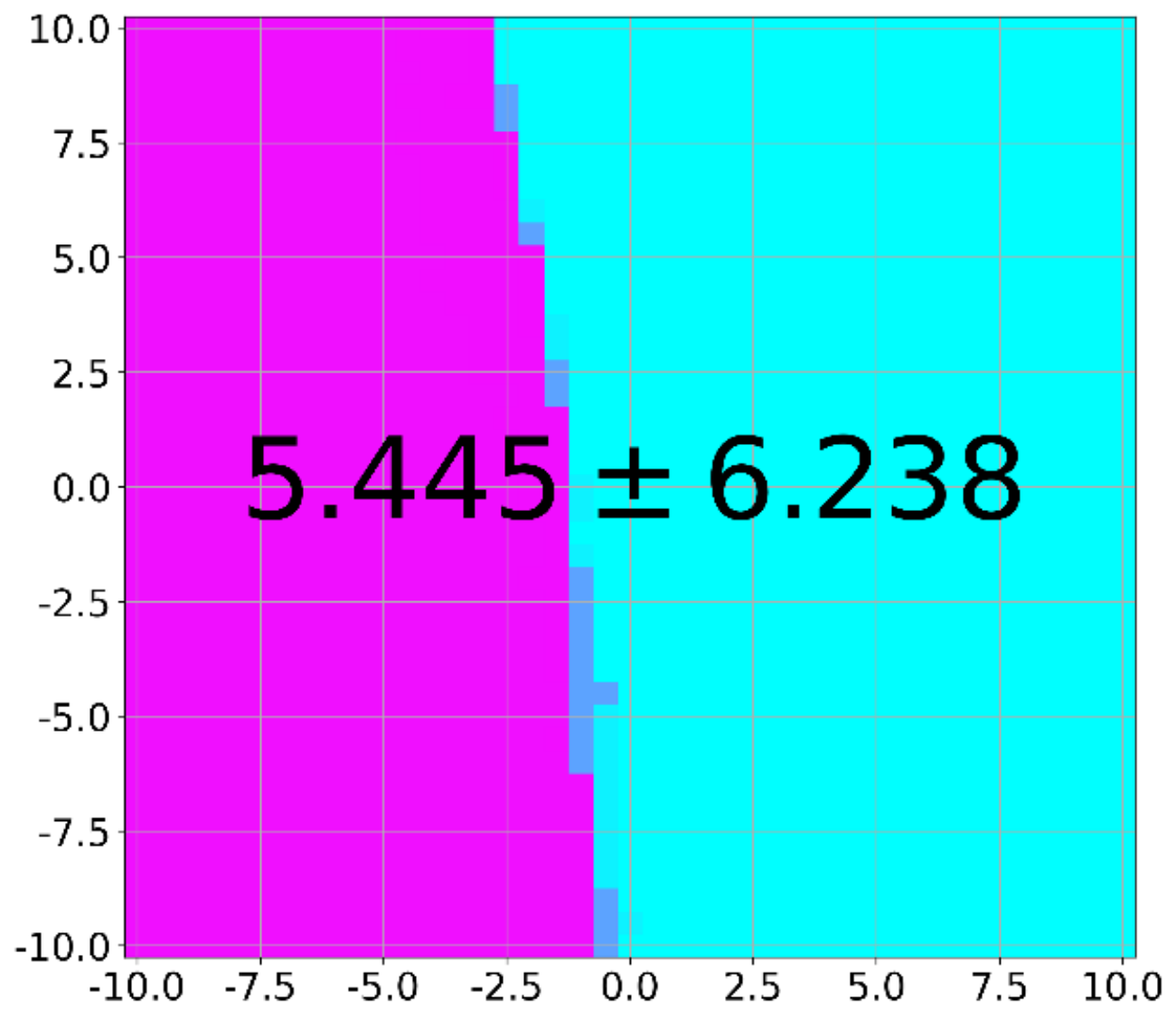}
\subcaption*{$\beta_0=1.0$}
\end{subfigure}
\begin{subfigure}{0.170\textwidth}
\includegraphics[width=\textwidth]{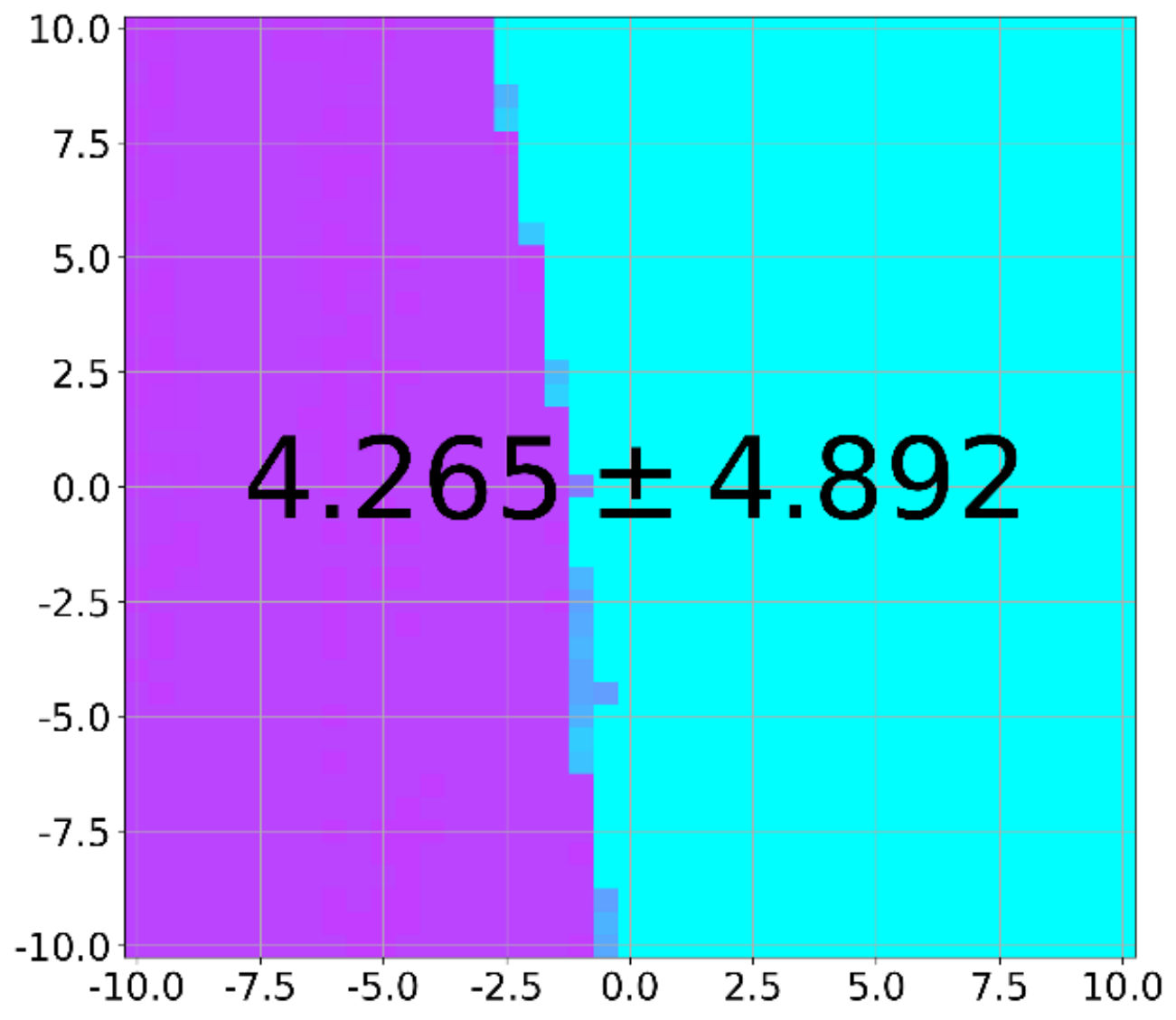}
\subcaption*{$\beta_0=10.0$}
\end{subfigure}
\begin{subfigure}{0.170\textwidth}
\includegraphics[width=\textwidth]{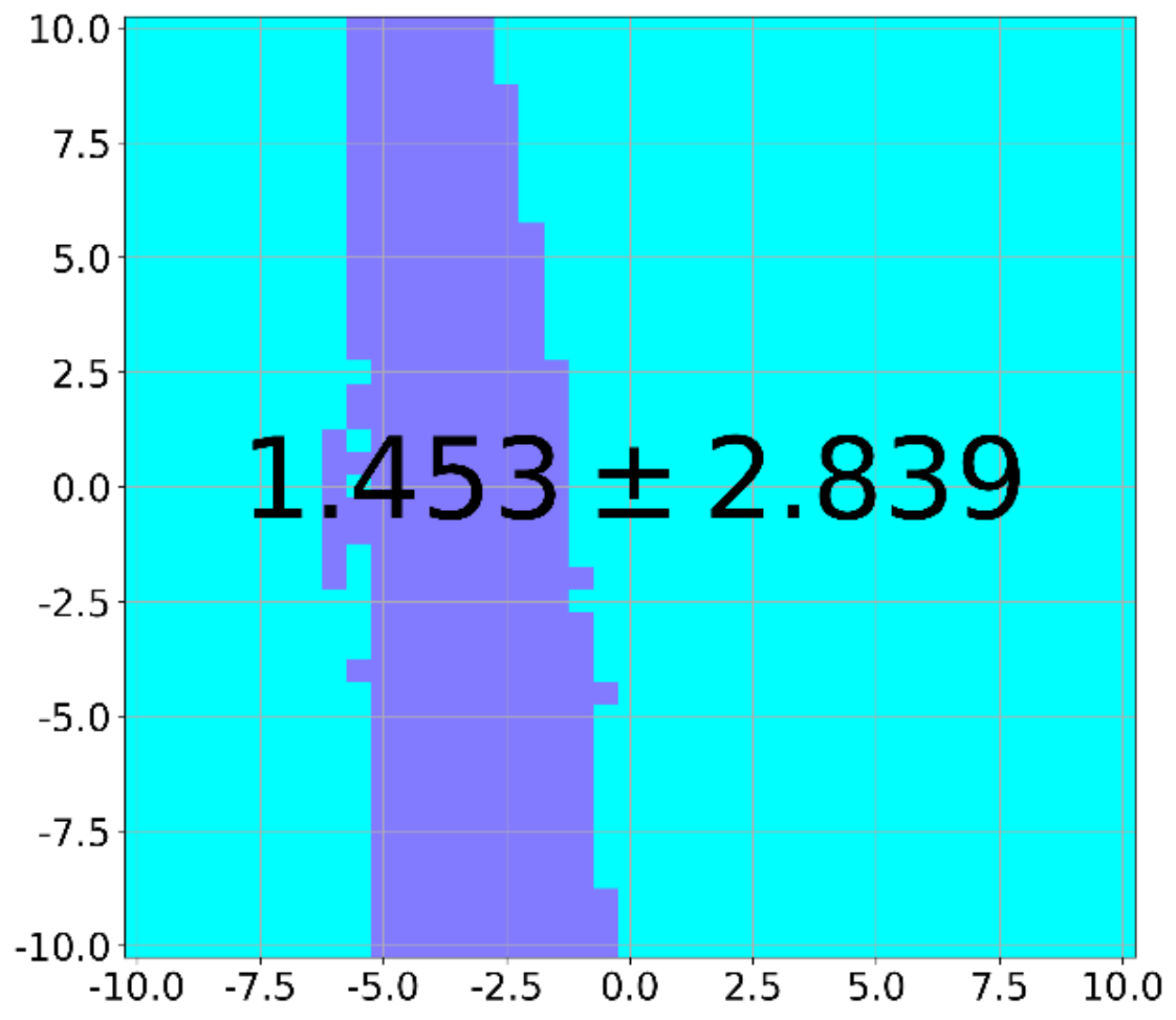}
\subcaption*{$\beta_0=100.0$}
\end{subfigure}
\begin{subfigure}{0.170\textwidth}
\includegraphics[width=\textwidth]{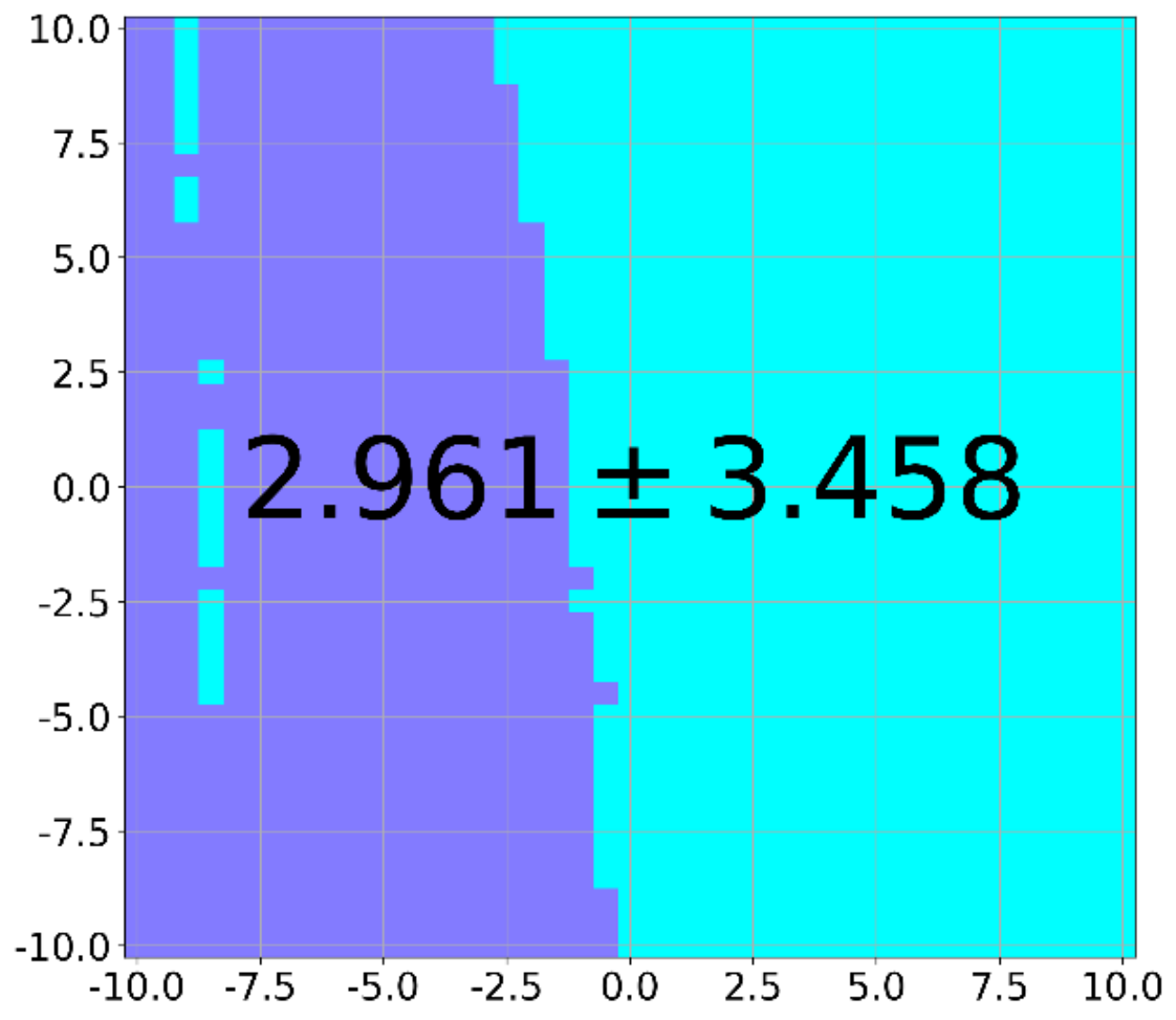}
\subcaption*{$\beta_0=1000.0$}
\end{subfigure}
\begin{subfigure}{0.170\textwidth}
\includegraphics[width=\textwidth]{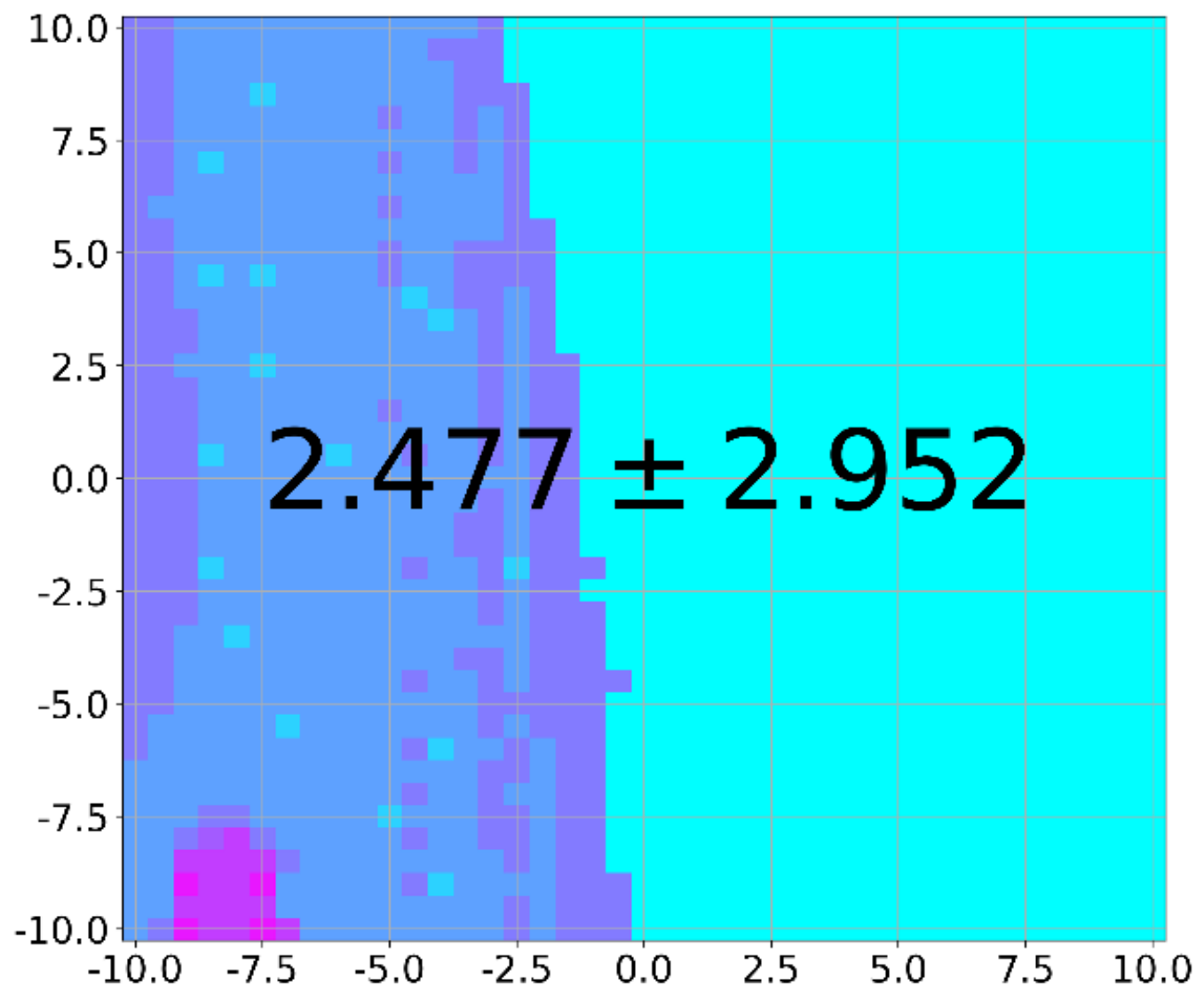}
\subcaption*{$\beta_0=10000.0$}
\end{subfigure}

\caption{Tuning the initial barrier parameter $\beta_0$ on
Env-A1, A2, A3, B1, B2, and B3 (top to bottom rows)
using the exact variant of \algref{alg:gainbiasbarrier_optim}.
Like in \figref{fig:optim_cmp}, each subplot shows the absolute difference (low:blue to high:magenta)
between the bias of the nBw-optimal deterministic policy and of the resulting randomized policy,
where the 2D-coordinate indicates the initial policy parameter value.
In each subplot, the numbers in the middle indicate the mean and standard deviation
across $1681 (=41^2)$ grid values.
Here, the lowest absolute bias difference (blue) is desirable.
For experimental setup, see \secref{sec:xprmt_setup_nbwpg}.
}
\label{fig:barrierparam_env_a_b}
\end{figure}
\end{landscape}

\section{Related works} \label{sec:rwork_nbwpg}

\begin{table}[t]
\centering
\caption{Model-free RL works that target nBw-optimal policies in unichain MDPs.
This table lists two criteria: $\gamma$-discounted and $n$-discount optimality,
as well as two algorithmic styles: value- and policy-iteration methods.
Here, $\gammabw$ denotes the so-called Blackwell's discount factor
as in \eqref{equ:bwoptim_v}.}
\label{tbl:rwork}
\begin{tabular}{lll}
\toprule
\multicolumn{1}{c}{\textbf{Styles \textbackslash\ Criteria}}
& \multicolumn{1}{c}{$\gamma$-discounted with $\gamma \in [\gammabw, 1)$}
& \multicolumn{1}{c}{$n$-discount with $n \ge 0$} \\
\midrule
Value iteration &
\cite{schneckenreither_2020_avgrew} &
\cite{mahadevan_1996_sensitivedo} \\
Policy iteration &
None &
Our work \\
\bottomrule
\end{tabular}
\end{table}

In order to classify the existing literature, we need to review two approaches
to obtaining nBw-optimal policies in unichain MDPs.
\textbf{First} is through $n$-discount optimality criterion with $n \ge 0$.
Here, a~policy $\pi_n^*$ for $n = -1, 0, 1, \ldots$, is $n$-discount optimal if
\begin{equation*}
\lim_{\gamma \uparrow 1} \frac{1}{(1 - \gamma)^n}
    \Big( v_\gamma(\pi_n^*, s) - v_\gamma(\pi, s) \Big) \ge 0,
\qquad \forall \pi \in \piset{S}, \forall s \in \setname{S},\
\text{and $v_\gamma(\pi, s)$ defined in \eqref{equ:disrew_v}}.
\end{equation*}
A policy is said to be ($n=\infty$)-discount optimal if
it is $n$-discount optimal for all $n \ge -1$.
\textbf{Second} is through $\gamma$-discounted optimality criterion with
a discount factor $\gamma$ that lies in the so-called Blackwell's interval.
That is, $\gamma \in [\gammabw, 1)$, where
$\gammabw \eqdef \sup_{s \in \setname{S}} \gammabw(s)$.
Here, $\gammabw(s)$ comes from the definition of Blackwell optimality:
a policy $\pistarbw$ is Blackwell optimal if for each $s \in \setname{S}$,
there exists a $\gammabw(s)$ such that
\begin{equation}
v_\gamma(\pistarbw, s) \ge v_\gamma(\pi, s),
\qquad \text{for}\ \gammabw(s) \le \gamma < 1, \text{and for all}\ \pi \in \piset{S}.
\label{equ:bwoptim_v}
\end{equation}
Intuitively, the Blackwell optimality claims that upon considering sufficiently
far into the future via $\gamma \ge \gammabw$, there is no policy better
than the Blackwell optimal policies \citep{denis_2019_bwopt}.

The above two approaches are based on the fact
that ($n=y$)-discount optimality implies ($n=x$)-discount optimality for all $y > x$,
\emph{and} on the equivalency between the nearly-Blackwell (nBw, bias)
and $(n=0)$-discount optimality criteria, as well as
between the Blackwell and $(n=\infty)$-discount optimality criteria
\citep[\thm{10.1.5}]{puterman_1994_mdp}.

Since $\gammabw$ is generally unknown in RL, the main challenge in obtaining
nBw-optimal policies through $\gamma$-discounted optimality is
to specify the discount factor~$\gamma$.
It should not be too far from~1, otherwise it is likely that $\gamma < \gammabw$.
Simultaneously, $\gamma$ should not be too close to~1 because
a higher $\gamma$ typically leads to slower convergence \citep{even-dar_2004_qlr}.
Moreover, as $\gamma$ approaches 1, the induced $\gamma$-discounted optimality yields
optimal policies similar to those of gain optimality
(which is underselective in unichain MDPs).
This implies that determining $\gamma$ is more difficult in unichain than
recurrent MDPs, for which the gain optimality (equivalently, ($n=-1$)-discount optimality)
is already the most selective.
We therefore do not discuss works that manipulate $\gamma$
towards achieving Blackwell optimality but are limited to recurrent MDPs.

\tblref{tbl:rwork} shows model-free RL works that target nBw-optimal policies
in unichain MDPs.
It classifies the works based on
whether they optimize $\gamma$-discounted or $n$-discount optimality criteria,
as well as whether they follow value- or policy-iteration algorithmic styles.
We found a relatively few works on this area.
The closest to our work is that of \cite{mahadevan_1996_sensitivedo},
then the work of \cite{schneckenreither_2020_avgrew}.
Both use tabular policy representation, which can be interpreted as
a special case of policy parameterization used in this work.
For discussion on \citep{mahadevan_1996_sensitivedo}, refer to \secref{sec:intro}.

\cite{schneckenreither_2020_avgrew} introduced a novel notion of
discounted state values that are adjusted by the gain.
The action selection utilizes an $\epsilon$-sensitive lexicographic order that operates
in two gain-adjusted discounted state-values with two different discount factors,
\ie $0.5 \le \gamma_0 <  \gamma_1 \le 1$.
Here, those two discount factors $\gamma_0, \gamma_1$ and
such a comparison measure $\epsilon$ are hyperparameters.

We are not aware of any work based on policy-iteration that
deliberately aims for Blackwell optimality via the $\gamma$-discounted approach.
There are (many) works that fine-tune $\gamma$,
for instance, those by \cite{zahavy_2020_stdrl, paul_2019_hoof, xu_2018_mgrad}.
Their tuning metrics however are not clearly motivated by Blackwell optimality
such that no concept of Blackwell's interval is involved.
Hence, we argue that they are beyond our current scope.

\section{Conclusions, limitations and future works} \label{sec:conclude_nbwpg}

We propose a policy gradient method that can obtain
nearly-Blackwell (bias) optimal policies.
It extends the existing average-reward methods in that
it performs further selection among
the resulting (approximately) average-reward (gain) optimal policies.
This is carried out through optimizing a function of
the bias and an RL-specific logarithmic barrier.
The corresponding stochastic optimization utilizes a novel expression that
enables sampling for natural gradients of the bias.
Experiment results validate the viability of our proposed method,
which can be regarded as a first step towards attaining $n$-discount optimality
through RL for $n \ge 0$.

Next, we identify some limitations of our proposed method.
They are described along with possible follow-up works that address them
but are left for future works.

Obtaining \emph{unbiased} estimates of the gradient \eqref{equ:grad_bias}
and the Fisher matrix \eqref{equ:fisher_b_sampling} of the bias requires
some knowledge (or a satisfied assumption) about the mixing time.
That is, the process is mixing before or at the last timestep of an experiment-episode (trial).
It is likely useful therefore to have a policy gradient agent that can control
the magnitude of the mixing time, such as the one proposed by \citet{morimura_2014_pgrl}.
The mixing-time control can be implemented for example, in the form of regularization in
the gain and bias-barrier objective functions such that the policy iterate $\pi$
induces one-step transition matrix $\ppimat$ whose mixing time is low.
Note that the issue of sampling states from a stationary state distribution
(which appears after mixing) is also present in the average-reward case
in order to obtain unbiased gradient estimates.
In the discounted-reward case, this corresponds to sampling states
from the discounted state distribution \eqref{equ:disrew_v_vector}.

The proposed \algref{alg:gainbiasbarrier_optim} has not been analysed
in terms of its sample complexity.
The challenge comes from the fact that it is stochastic barrier optimization
with potentially \emph{biased} samples, for instance
due to Markovian (not independent) state samples and/or due to
sampling from a distribution that is not identical to the stationary state distribution.
Moreover, there is a finite number of visitations to transient states
in a single experiment-episode.\footnote{
    Having an infinite number of transient state samples (in theory)
    implies running an infinite number of experiment-episodes (trials).
    This is because a single experiment-episode contains a finite number of
    visitations to transient states (by the definition of transience and recurrence).
}
It is therefore imperative to be able to derive a finite sample complexity
(non-asymptotic convergence) of the proposed method.
Apart from the analysis, the sample complexity may be reduced in several ways,
including
a) iterative techniques for computing the inverse of the Fisher matrix,
b) action-independent baseline substraction for the gradients, and
c) action-value function approximators for the gradients (as in the following passage).

Our proposed \algref{alg:gainbiasbarrier_optim} also has not been
incorporated with any function approximator for estimating the action values,
which are used in the bias gradients \eqref{equ:grad_bias}.
Particularly here, the function approximator is required to
estimate the policy values from both transient and recurrent states
(note that typical approximators are designed to estimate policy values
only from recurrent states).

\afterpage{
\clearpage%
\bibliography{main}
\bibliographystyle{apalike}

\clearpage%
}

\end{document}